\documentclass[doublespacing]{utdthesis}

\usepackage{microtype}
\usepackage{amsmath,amssymb,amsthm}
\usepackage{graphicx}
\usepackage{url}
\usepackage[authoryear]{natbib}
\let\cite=\citep
\usepackage{rotating}
\usepackage{ifpdf}
\ifpdf
  \usepackage{hyperref}
\fi
\usepackage{amsmath}
\usepackage{amsmath,bm}
\usepackage{amsthm}
\usepackage{booktabs}
\usepackage[ruled, vlined, linesnumbered]{algorithm2e}%
\usepackage{algpseudocode}
\usepackage{caption}
\usepackage{subcaption}
\usepackage{xcolor}
\usepackage{amsfonts}
\usepackage{tabularx}
\usepackage{bbm}
\usepackage{enumitem}
\usepackage{multirow}
\usepackage{mathtools}
\usepackage[normalem]{ulem}
\usepackage{threeparttable}

\newtheorem{theorem}{Theorem}

\newtheorem{lemma}{Lemma}
\newtheorem{proposition}{Proposition}
\newtheorem{case}{Case}

\newcommand\blfootnote[1]{%
  \begingroup
  \renewcommand\thefootnote{}\footnote{#1}%
  \addtocounter{footnote}{-1}%
  \endgroup
}

\makeatletter
\newcommand{\multiline}[1]{%
  \begin{tabularx}{\dimexpr\linewidth-\ALG@thistlm}[t]{@{}X@{}}
    #1
  \end{tabularx}
}
\makeatother
\makeatletter
\def\@opargbegintheorem#1#2#3{\trivlist
   \item[]{\bfseries #1\ #2\ (#3)} \itshape}
\makeatother

\def\BibTeX{{\rm B\kern-.05em{\sc i\kern-.025em b}\kern-.08em
    T\kern-.1667em\lower.7ex\hbox{E}\kern-.125emX}}

\newcommand{\x}{\textbf{x}}
\newcommand{\y}{\textbf{y}}
\newcommand{\p}{\textbf{p}}

\newcommand{\rr}{\textbf{r}}
\newcommand{\rrr}{\textbf{R}}

\newcommand{\w}{\textbf{w}}

\newcommand{\uu}{\textbf{u}}
\newcommand{\cc}{\textbf{c}}
\newcommand{\feng}{\textcolor{black}}
\newcommand{\zxj}{\textcolor{black}}
\newcommand{\blue}{\textcolor{blue}}

\DeclareMathOperator*{\argmin}{arg\,min}

\author{Xujiang Zhao}
\title{Multidimensional Uncertainty Quantification For \\ Deep Neural Networks}
\thesistype{Dissertation}  
\degreefull{Doctor of Philosophy}
\degreeabbr{PhD}
\subject{Computer Science}
\graduationmonth{August}
\graduationyear{2022}
\prevdegrees{} 

\committeemember*{Feng Chen}
\committeemember{Gopal Gupta}
\committeemember{Rishabh Iyer}
\committeemember{Chung Hwan Kim}
\begin{document}

\frontmatter

\signaturepage

\copyrightpage{2022} 


\maketitle

\begin{acks}{June 2022} 
    I am tremendously grateful to my advisor, Professor Feng Chen, for his excellent advice, encouragement, and support over the years. He brought my attention to machine learning, data mining, and other exciting research topics at the early stage of my PhD program. During my studies at UAlbany and UT Dallas, he was always willing to help me with suggestions for my research. His excellent research personality profoundly influences me. 

I want to thank my committee members, Professor Gopal Gupta, Professor Rishabh Iyer, and Professor Chung Hwan Kim, for their willingness to serve as my committee members. I would also like to thank my Examining Committee Chair, Professor Justin Ruths, for his valuable time. I got invaluable help and great feedback from them. Their many critically important comments helped me to improve this dissertation significantly.

I am grateful to Professor Jin-Hee Cho, Professor Rishabh Iyer, and Professor Qi Yu for their care and generous help, especially in the early stage of my PhD career. They gave me detailed instructions and support.

I received tremendous help from friends and colleagues: Professor Zhiqiang Tao, Dr. Chen Zhao, Dr. Baojian Zhou, Dr. Fei Jie, Dr. Adil Alim, Chunpai Wang, Changbin Li, Haoliang Wang, Yuzhe Ou, Linlin Yu, Yibo Hu, Zhuoyi Wang, and Junfeng Guo. I also had two fantastic summer internships at \textit{Alibaba Damo Academy} and \textit{NEC Laboratories America} in the summer of 2019 and 2021, advised by great team members Dr. Hongxia Yang and Dr. Xuhcao Zhang. 

I would also like to express my sincere gratitude to my parents, sister, and brother for their unconditional love.

Finally, this dissertation is dedicated to my wife, Chongchong, and my son, Shuhuan. I want to thank my wife and my son for their love and support. 
\end{acks}

\begin{abstract}
    Deep neural networks (DNNs) have received tremendous attention and achieved great success in various applications, such as image
and video analysis, natural language processing, recommendation systems, and drug discovery. However, inherent uncertainties derived from different root causes have been realized as serious hurdles for DNNs to find robust and trustworthy solutions for real-world problems. A lack of consideration of such uncertainties may lead to unnecessary risk. For example, a self-driving autonomous car can misdetect a human on the road. A deep learning-based medical assistant may misdiagnose cancer as a benign tumor. 

In this work, we study how to measure different uncertainty causes for DNNs and use them to solve diverse decision-making problems more effectively. In the first part of this thesis, we develop a general learning framework to quantify multiple types of uncertainties caused by different root causes, such as vacuity (i.e., uncertainty due to a lack of evidence) and dissonance (i.e., uncertainty due to conflicting evidence), for graph neural networks. We provide a theoretical analysis of the relationships between different uncertainty types. We further demonstrate that dissonance is most effective for misclassification detection and vacuity is most effective for Out-of-Distribution (OOD) detection. In the second part of the thesis, we study the significant impact of OOD objects on  semi-supervised learning (SSL) for DNNs and develop a novel framework to improve the robustness of existing SSL algorithms against OODs. In the last part of the thesis, we create a general learning framework to quantity multiple uncertainty types for multi-label temporal neural networks. We further develop novel uncertainty fusion operators to quantify the fused uncertainty of a subsequence for early event detection.

\end{abstract}

\tableofcontents
\listoffigures 
\listoftables 

\mainmatter

\chapter{Introduction}
\label{chapter:intro}
    \section{Motivations}
Deep neural networks have reached almost every field of science in the last decade and have become an essential component of a wide range of real-world applications~\cite{szegedy2015going,graves2013generating,wang2021clear,dong2022neural}. Because of the increasing spread, confidence in neural network predictions became increasingly important. However, basic neural networks do not provide certainty estimates and suffer from over-confidence or under-confidence, indicating poor calibration. For instance, given several cat breeds images as training data to train a neural network, the model should return a prediction with relatively high confidence when a test sample is a similar cat breed image. However, when the test sample is a dog image, is the neural network able to recognize the test sample coming from a different data distribution (\textit{out of distribution} test data)~\cite{gal2016dropout, wang2022layerada}? Ideally, we hope the neural networks can say "I do not know" when the test sample is an out-of-distribution sample. Unfortunately, the answer is "NO" because the model was trained with images of different kinds of cats and has hopefully learned to recognize them. But the model has never seen a dog before, and an image of a dog would be outside of the data distribution the model was trained on. Even more critically, in high-risk fields, such as medical image analysis with a diagnostics system that has never been observed before or scenes that a self-driving system has never been trained to handle~\cite{kendall2017uncertainties}. 

A possible desired behavior of a model in such cases would be to return a prediction with the additional information that the test sample lies outside the data distribution. In other words, we hope our model can quantity a high level of uncertainty (alternatively, low confidence) with such out-of-distribution inputs. In addition to the out-of-distribution situation, other scenarios may lead to uncertainty. One scenario is noisy data, such as the label noise due to the measurement imprecision~\cite{kendall2017uncertainties} causes aleatoric uncertainty and conflicting evidence due to dissonance uncertainty~\cite{josang2018uncertainty}. Another scenario is the uncertainty in model parameters that best explain the observed (training) data~\cite{gal2015bayesian}. 

Many researchers have been working on understanding and quantifying uncertainty in a neural network’s prediction to overcome over or under-confidence prediction issues. As a result, different types and sources of uncertainty have been identified, and various approaches to measure and quantify uncertainty in neural networks have been proposed~\cite{guo2022survey}. 
Predictive uncertainty quantification in the deep learning research domain has been mainly explored based on two uncertainty types: aleatoric and epistemic uncertainty. Aleatoric uncertainty refers to inherent ambiguity in outputs for a given input and cannot be reduced due to randomness in the data. Aleatoric uncertainty can be estimated by probabilistic neural networks, such as mixture density networks~\cite{mackay1999density}. Epistemic uncertainty indicates uncertainty about the model parameters estimated based on the training data. This uncertainty measures how well the model is matched to the data and can be reduced through the collection of additional data. Epistemic uncertainty can be estimated based on Bayesian neural networks (BNNs) that can learn a posterior distribution over parameters. The accuracy of uncertainty estimation depends on the choice of the prior distribution and the accuracy of the approximate posterior distribution, as the exact posterior is often infeasible to compute. Recent developments of approximate Bayesian approaches include $L$ Laplace approximation~\cite{mackay1992bayesian}, variational inference~\cite{graves2011practical}, dropout-based variational inference~\cite{gal2016dropout}, expectation propagation~\cite{hernandez2015probabilistic} and stochastic gradient Markov chain Monte Carlo (MCMC)~\cite{welling2011bayesian}. 

In the belief (or evidence) theory domain, uncertainty reasoning has been substantially explored, such as Fuzzy Logic~\citep{de2018intelligent}, Dempster-Shafer Theory (DST)~\citep{sentz2002combination}, or Subjective Logic (SL)~\citep{josang2016subjective}.  Belief theory focuses on reasoning inherent uncertainty in information caused by unreliable, incomplete, deceptive, or conflicting evidence.  SL considered predictive uncertainty in subjective opinions in terms of {\em vacuity} (i.e., a lack of evidence)~\cite{zhao2018deep,xu2021boosting,alim2019uncertainty} and {\em vagueness} (i.e., failing in discriminating a belief state)~\citep{josang2016subjective}. Recently, other uncertainty types have been studied, such as {\em dissonance} caused by conflicting evidence\citep{josang2018uncertainty,zhao2019uncertainty,shi2020multifaceted}.

However, inherent uncertainties derived from different root causes have been realized as serious hurdles for DNNs to find robust and trustworthy solutions for real-world problems. A lack of consideration of such uncertainties may lead to unnecessary risk. For example, a self-driving autonomous car can misdetect a human on the road. A deep learning-based medical assistant may misdiagnose cancer as a benign tumor. In this work, we study how to measure different uncertainty causes for DNNs and use them to solve diverse decision-making problems more effectively. 

(1) \textit{Uncertainty-aware semi-supervised learning on graph data}.
We study the multidimensional uncertainty quantification for graph neural networks(GNNs)~\citep{kipf2017semi, velickovic2018graph}, which have received tremendous attention in the data science community. Despite their superior performance in semi-supervised node classification and regression, they did not consider various uncertainties in their decision process. Although many methods~\citep{zhao2018deep2, liuuncertainty,zhang2018bayesian,zhao2018uncertainty} have been proposed to estimate uncertainty for GNNs, no prior work has considered uncertainty decomposition in GNNs. To address the challenge of uncertainty decomposition in GNNs, we propose the research question: \textit{can we quantify multidimensional uncertainty types in both deep learning and belief and evidence theory domains for node-level classification, misclassification detection, and out-of-distribution detection tasks on graph data?}

(2) \textit{Uncertainty-aware robust semi-supervised learning}. Recent Semi-supervised learning works show significant improvement in semi-supervised learning algorithms' performance using better-unlabeled data representations. However, recent work~\cite{oliver2018realistic} shows that the SSL algorithm's performance could degrade when the unlabeled set has out-of-distribution examples. To address the challenge of out-of-distribution issues in semi-supervised learning, we first propose the research questions: \textit{How out-of-distribution data hurt semi-supervised learning performance?} and \textit{Can we develop an efficient and effective uncertainty-based approach for robust semi-supervised learning with out of distribution data?}

(3) \textit{Uncertainty-aware early event detection with multi-labels}. 
Early event detection aims to detect event even before the event is complete.
To achieve the earliness of event detection, existing approaches can be broadly divided into several major categories. Prefix-based techniques \cite{gupta2020fault,gupta2020early,gupta2020divide,gupta2019early,gupta2019game} aim to learn a minimum prefix length of the time series using the training instances and utilize it to classify a testing time series. Shapelet-based approaches \cite{yan2020extracting,zhao2019asynchronous,yao2019dtec} focus on obtaining a set of key shapelets from the training dataset and utilizing them as class discriminatory features. Model-based methods for event early detection \cite{mori2019early,lv2019effective} are proposed to obtain conditional probabilities by either fitting a discriminative classifier or using generative classifiers on training. Although these approaches address the importance of early detection, they primarily focus on an event with a single label but fail to be applied to cases with multiple labels. Another non-negligible issue for early event detection is a prediction with overconfidence \cite{zhao2020uncertainty,sensoy2018evidential} due to high vacuity uncertainty exists in the early time series, which refers to a lack of evidence. It results in an over-confidence estimation and hence unreliable predictions. To address the aforementioned issues, we propose the research questions: \textit{How to quantify the uncertainty for multi-label time series classification?} and \textit{How to make a reliable prediction for early event detection?}

In this dissertation, we aim to solve the challenges mentioned above. Therefore, a general question we want to ask is: How can we design effective and efficient uncertainty methods for deep neural networks under different problem settings?

\section{Summary of Main Contributions}

To answer the above question, we propose a novel approach for each challenge by studying an effective and efficient algorithm that is followed by a detailed theoretic analysis. The main contributions are list out as follows.

(1) \textit{Semi-supervised learning on graph data}. We proposed a multi-source uncertainty framework for GNNs. The proposed framework first provides the estimation of various types of uncertainty from both deep learning and evidence/belief theory domains, such as dissonance (derived from conflicting evidence) and vacuity (derived from lack of evidence). In addition, we designed a Graph-based Kernel Dirichlet distribution Estimation (GKDE) method to reduce errors in quantifying predictive uncertainties. Furthermore, we first provided a theoretical analysis of the relationships between different types of uncertainties considered in this work. We demonstrate via a theoretical analysis that an OOD node may have a high predictive uncertainty under GKDE. Based on the six real graph datasets, we compared the performance of our proposed framework with that of other competitive counterparts. We found that dissonance-based detection yielded the best results in misclassification detection, while vacuity-based detection performed best in OOD detection. The main results are from~\cite{zhao2020uncertainty}

(2) \textit{Semi-supervised learning with OODs setting}. 
To answer the question of "\textit{How out-of-distribution data hurt semi-supervised learning performance?}", we first study the critical causes of OOD's negative impact on SSL algorithms. In particular, we found that 1) certain kinds of OOD data instances close to the decision boundary have a more significant impact on performance than those far away, and 2) Batch Normalization (BN), a popular module, could degrade the performance instead of improving the performance when the unlabeled set contains OODs. To address the above causes, we proposed a novel unified weighted robust SSL framework that can be easily extended to many existing SSL algorithms and improve their robustness against OODs. To address the limitation of low-order approximations in bi-level optimization, we developed an efficient hyper-parameter optimization algorithm that considers high-order approximations of the objective and is scalable to a higher number of inner optimization steps to learn a massive amount of weight parameters. In addition, we conduct a theoretical analysis of the impact of faraway OODs in the BN step and propose weighted batch normalization (WBN) to carry the weights over in the BN step. We also discuss the connection between our approach and low-order approximation approaches. 
Finally, we address a critical issue of the existing bi-level optimization-based reweighting schemes, which is that they are much slower (close to $3\times$) compared to the original learning (SSL) algorithms --  we show that several simple tricks like just considering the last layer in the inner loop and doing the weight updates every few epochs enable us to have a run-time comparable to the base SSL while maintaining the accuracy gains of reweighting.
Extensive experiments on synthetic and real-world datasets prove that our proposed approach significantly improves the robustness of four representative SSL algorithms against OODs compared with four state-of-the-art robust SSL approaches. The main content is from~\cite{zhao2020robust}.

(3) \textit{Early event detection with multi-label setting}. 
We first introduce a novel problem, namely \textit{early event detection with multiple labels}. A temporal event with multiple labels that occurs sequentially along the timeline is considered in this problem setting. This work aims to accurately detect all classes at the ongoing stage of an event within the least amount of time. To this end, technically, we propose a novel framework, 
Multi-Label Temporal Evidential Neural Network (MTENN), for early event detection in temporal data. MTENN is able to quality predictive uncertainty due to the lack of evidence for multi-label classifications at each time stamp based on belief/evidence theory. In addition, we introduce two novel uncertainty estimation heads (weighted binomial comultiplication (WBC) and uncertainty mena scan statistics(UMSS)) to quantify the fused uncertainty of a sub-sequence for early event detection. We demonstrate that WBC is effective for detection accuracy and UMSS is effective for detection delay. We validate the performance of our approach with state-of-the-art techniques on real-world audio and video datasets. Theoretic analysis and empirical studies demonstrate the effectiveness and efficiency of the proposed framework in both detection delay and accuracy. The results have been accepted in~\cite{zhao2022seed}.

\section{Outline}
The rest of the dissertation is organized as follows. In Chapter \ref{chapter:2}, We present our first research on multidimensional uncertainty quantification for GNNs. In Chapter \ref{chapter:3}, we study an uncertainty-based robust semi-supervised learning framework via bi-level optimization. In Chapter \ref{chapter:4}, we further quantify the multi-label uncertainty and sequential uncertainty for early event detection. We finally conclude the dissertation in Chapter \ref{chapter:conclusion}.

\chapter{Uncertainty Aware Semi-Supervised Learning on \\ Graph Data}
\label{chapter:2}
    \section{Introduction}\blfootnote{\copyright 2020 Advances in Neural Information Processing Systems. Reprinted, with permission, from Xujiang Zhao, Feng Chen, Shu Hu, Jin-Hee Cho, ``Uncertainty Aware Semi-Supervised Learning on Graph Data," Proceedings of the 34th International Conference on Neural Information Processing Systems, pp.12827-12836.} 
Inherent uncertainties derived from different root causes have realized as serious hurdles to find effective solutions for real-world problems. Critical safety concerns have been brought due to lack of considering diverse causes of uncertainties, resulting in high risk due to misinterpretation of uncertainties (e.g., misdetection or misclassification of an object by an autonomous vehicle).  Graph neural networks (GNNs)~\citep{kipf2017semi, velickovic2018graph} have received tremendous attention in the data science community. Despite their superior performance in semi-supervised node classification and regression, they didn't consider various types of uncertainties in the their decision process.  Predictive uncertainty estimation~\citep{kendall2017uncertainties} using Bayesian NNs (BNNs) has been explored for classification prediction and regression in the computer vision applications, based on aleatoric uncertainty (AU) and epistemic uncertainty (EU). AU refers to data uncertainty from statistical randomness (e.g., inherent noises in observations) while EU indicates model uncertainty due to limited knowledge (e.g., ignorance) in collected data.  In the belief or evidence theory domain, Subjective Logic (SL)~\citep{josang2018uncertainty} considered vacuity (or a lack of evidence or ignorance) as uncertainty in a subjective opinion. Recently other uncertainty types, such as dissonance, consonance, vagueness, and monosonance~\citep{josang2018uncertainty}, have been discussed based on SL to measure them based on their different root causes. 

We first considered multidimensional uncertainty types in both deep learning (DL) and belief and evidence theory domains for node-level classification, misclassification detection, and out-of-distribution (OOD) detection tasks.  By leveraging the learning capability of GNNs and considering multidimensional uncertainties, we propose an uncertainty-aware estimation framework by quantifying different uncertainty types associated with the predicted class probabilities.  
In this work, we made the following {\bf key contributions}:
\begin{itemize}[leftmargin=*, noitemsep]
\item \textbf{A multi-source uncertainty framework for GNNs}. Our proposed framework first provides the estimation of various types of uncertainty from both DL and evidence/belief theory domains, such as dissonance (derived from conflicting evidence) and vacuity (derived from lack of evidence).  In addition, we designed a Graph-based Kernel Dirichlet distribution Estimation (GKDE) method to reduce errors in quantifying predictive uncertainties.
\item \textbf{Theoretical analysis}:  Our work is the first that provides a theoretical analysis about the relationships between different types of uncertainties considered in this work.  We demonstrate via a theoretical analysis that an OOD node may have a high predictive uncertainty under GKDE.
\item \textbf{Comprehensive experiments for validating the performance of our proposed framework}: Based on the six real graph datasets, we compared the performance of our proposed framework with that of other competitive counterparts. We found that dissonance-based detection yielded the best results in misclassification detection while vacuity-based detection best performed in OOD detection. 
\end{itemize}
Note that we use the term `predictive uncertainty' in order to mean uncertainty estimated to solve prediction problems.

\section{Related Work} \label{sec:related-work}
DL research has mainly considered {\it aleatoric} uncertainty (AU) and {\it epistemic} uncertainty (EU) using BNNs for computer vision applications.  AU consists of homoscedastic uncertainty (i.e., constant errors for different inputs) and heteroscedastic uncertainty (i.e., different errors for different inputs)~\citep{gal2016uncertainty}.  A Bayesian DL framework was presented to simultaneously estimate both AU and EU in regression (e.g., depth regression) and classification (e.g., semantic segmentation) tasks~\citep{kendall2017uncertainties}.  Later, {\em distributional uncertainty} was defined based on distributional mismatch between testing and training data distributions~\citep{malinin2018predictive}.  {\em Dropout variational inference}~\citep{gal2016dropout} was used for approximate inference in BNNs using epistemic uncertainty, similar to \textit{DropEdge}~\citep{rong2019dropedge}.  Other algorithms have considered overall uncertainty in node classification~\citep{eswaran2017power, liuuncertainty, zhang2019bayesian}. However, no prior work has considered uncertainty decomposition in GNNs. 

In the belief (or evidence) theory domain, uncertainty reasoning has been substantially explored, such as Fuzzy Logic~\citep{de2018intelligent}, Dempster-Shafer Theory (DST)~\citep{sentz2002combination}, or Subjective Logic (SL)~\citep{josang2016subjective}.  Belief theory focuses on reasoning inherent uncertainty in information caused by unreliable, incomplete, deceptive, or conflicting evidence.  SL considered predictive uncertainty in subjective opinions in terms of {\em vacuity} (i.e., a lack of evidence) and {\em vagueness} (i.e., failing in discriminating a belief state)~\citep{josang2016subjective}. Recently, other uncertainty types have been studied, such as {\em dissonance} caused by conflicting evidence\citep{josang2018uncertainty}.
In the deep NNs, \citep{sensoy2018evidential} proposed evidential deep learning (EDL) model, using SL to train a deterministic NN for supervised classification in computer vision based on the sum of squared loss.  However, EDL didn't consider a general method of estimating multidimensional uncertainty or graph structure.

\section{Multidimensional Uncertainty and Subjective Logic}

This section provides an overview of SL and discusses multiple types of uncertainties estimated based on SL, called {\em evidential uncertainty}, with the measures of \textit{vacuity} and \textit{dissonance}.  In addition, we give a brief overview of {\em probabilistic uncertainty}, discussing the measures of \textit{aleatoric} uncertainty and \textit{epistemic} uncertainty.

\subsection{Notations}
Vectors are denoted by lower case bold face letters, \textit{e.g.}, belief vector $\bm{b} \in [0, 1]^K$ and class probability ${\bm p} \in [0, 1]^K$ where their \textit{i}-th entries are $b_i, p_i$. Scalars are denoted by lowercase italic letters, \textit{e.g.} $u\in[0, 1]$. Matrices are denoted by capital italic letters. $\omega$ denotes the subjective opinion. Some important notations are listed in Table~\ref{table:notation_2}

\begin{table*}[h]
\caption{Important notations and corresponding descriptions.}
\centering
  \begin{tabular}{l|l}
    \toprule
      \textbf{Notations} & \textbf{Descriptions}  \\
    \midrule
    $\mathcal{G}$   & Graph dataset  \\
    $\mathbb{V}$   &  A ground set of nodes   \\
    $\mathbb{L}$   &  Training nodes   \\
    $\mathbb{E}$   &  A ground set of edges   \\
    $\rrr$   &  Node-level feature matrix   \\
    $y_i$ & Class label of node $i$ \\
    $\p_i$ & Class probability of node $i$ \\
    ${\bm \theta}$ & model parameters \\
    $\omega$ & Subjective opinion \\
    $\bm{b}$ & Belief mass distribution \\
    $\bm{\alpha}$ & Dirichlet distribution parameters \\
    $u$ & Vacuity uncertainty \\
    $K$ & Number of classes \\
    $S$ & Dirichlet strength (summation of ${\bm \alpha}$) \\
    $\bm{e}$ & Evidence vector \\
    $diss(\omega)$ & Dissonance uncertainty based on opinion $\omega$ \\
    $I(\cdot)$ & Mutual information \\
    $H(\cdot)$ & Entropy function \\
    $f(\cdot)$ & GNNs model function \\
    \bottomrule
  \end{tabular}
   \label{table:notation_2}
\end{table*}

\subsection{Subjective Logic}\label{SL}

A multinomial opinion of a random variable $y$ is represented by $\omega = (\bm{b}, u, \bm{a})$ where a domain is $\mathbb{Y} \equiv \{1, \cdots, K\}$ and the additivity requirement of $\omega$ is given as $\sum_{k \in \mathbb{Y}} b_k + u = 1$.  To be specific, each parameter indicates,
\begin{itemize}
\item $\bm{b}$: {\em belief mass distribution} over $\mathbb{Y}$ and $\bm{b} = [b_1, \ldots, b_K]^T$;
\item $u$: {\em uncertainty mass} representing {\em vacuity of evidence};
\item $\bm{a}$: {\em base rate distribution} over $\mathbb{Y}$ and $\bm{a} = [a_1, \ldots, a_K]^T$.
\end{itemize}
The projected probability distribution of a multinomial opinion can be calculated as:
\begin{equation} \label{eq:multinomial-projected}
P(y=k) = b_k + a_k u,\;\;\; \forall k \in \mathbb{Y}. 
\end{equation}  

A multinomial opinion $\omega$ defined above can be equivalently represented by a $K$-dimensional Dirichlet probability density function (PDF), where the special case with $K=2$ is the Beta PDF as a binomial opinion. 
Let ${\bm \alpha}$ be a strength vector over the singletons (or classes) in $\mathbb{Y}$ and ${\bf p} = [p_1, \cdots, p_K]^T$ be a probability distribution over $\mathbb{Y}$. The Dirichlet PDF 
with ${\bf p}$ as a random vector $K$-dimensional variables is defined by:
\begin{eqnarray} \label{eq:multinomial-dir}
\mathrm{Dir}(\bm{p}| {\bm \alpha}) = \frac{1}{B({\bm \alpha})} \prod\nolimits_{k\in \mathbb{Y}} p_k ^{(\alpha_k-1)},
\end{eqnarray} 
where $\frac{1}{B({\bm \alpha})} = \frac{\Gamma (\sum_{k \in \mathbb{Y}} \alpha_k)}{\prod_{k \in \mathbb{Y}} (\alpha_k)}$, $\alpha_k \geq 0$, and $p_k \neq 0$, if $\alpha_k < 1$.

The term \textit{evidence} is introduced as a measure of the amount of supporting observations collected from data that a sample should be classified 
into a certain class. Let $e_k$ be the evidence derived for the class $k\in \mathbb{Y}$.  The total strength $\alpha_k$ for the  belief of each class $k \in \mathbb{Y}$ can be calculated as: 
$\alpha_k = e_k + a_k W$, 
where $e_k \geq 0, \forall k \in \mathbb{Y}$, and $W$ refers to a non-informative weight representing the amount of uncertain evidence.  Given the Dirichlet PDF as defined above, the expected probability distribution over $\mathbb{Y}$ can be calculated as:
\begin{equation} \label{eq:multinomial-expected}
\mathbb{E}[p_k] = \frac{\alpha_k}{\sum_{k=1}^K \alpha_k} = \frac{e_k+a_k W}{W+\sum_{k=1}^K e_k}. 
\end{equation}
The observed evidence in a Dirichlet PDF can be mapped to a multinomial opinion as follows:
\begin{equation} \label{eq:multinomial-belief}
b_k = \frac{e_k}{S}, \;
u = \frac{W}{S},  
\end{equation}
where $S = \sum_{k=1}^K \alpha_k$ refers to the Dirichlet strength.
Without loss of generality, we set  $a_k = \frac{1}{K}$ and the non-informative prior weight (i.e., $W = K$), which indicates that  $a_k \cdot  W = 1$ for each $k \in \mathbb{Y}$.

\subsection{Evidential Uncertainty}

In~\cite{josang2018uncertainty}, we discussed a number of multidimensional uncertainty dimensions of a subjective opinion based on the formalism of SL, such as singularity, vagueness, vacuity, dissonance, consonance, and monosonance.  These uncertainty dimensions can be observed from binomial, multinomial, or hyper opinions depending on their characteristics (e.g., the vagueness uncertainty is only observed in hyper opinions to deal with composite beliefs). In this work, we discuss two main uncertainty types that can be estimated in a multinomial opinion, which are {\em vacuity} and {\em dissonance}. 

The main cause of vacuity is derived from a lack of evidence or knowledge, which corresponds to the uncertainty mass, $u$, of a multinomial opinion in SL as:
$vac(\omega) \equiv u = K/S,$ as estimated in Eq.~(\ref{eq:multinomial-belief}).
This uncertainty exists because the analyst may have  insufficient information or knowledge to analyze the uncertainty. The {\em dissonance} of a multinomial opinion can be derived from the same amount of conflicting evidence and can be estimated based on the difference between singleton belief masses (e.g., class labels), which leads to `inconclusiveness' in decision-making applications. For example, a four-state multinomial opinion is given as $(b_1, b_2, b_3, b_4, u, a) = (0.25, 0.25, 0.25, 0.25, 0.0, a)$ based on Eq.~\eqref{eq:multinomial-belief}, although the vacuity $u$ is zero, a decision can not be made if there are the same amounts of beliefs supporting respective beliefs.  Given a multinomial opinion with non-zero belief masses, the measure of dissonance can be calculated as:
\begin{align}
\label{eq:dis}
    diss(\omega)=\sum_{i=1}^{K}\Big(\frac{b_i \sum_{j\neq i}b_j \text{Bal}(b_j,b_i)}{\sum_{j\neq i}b_j}\Big), 
\end{align}
where the relative mass balance between a pair of belief masses $b_j$ and $b_i$ is defined as $\mbox{Bal}(b_j, b_i) = 1 - |b_j - b_i|/(b_j + b_i)$. We note that  the dissonance is measured only when the belief mass is non-zero. If all belief masses equal to zero with vacuity being 1 (i.e., $u=1$), the dissonance will be set to zero.

\subsection{Probabilistic Uncertainty}   
\label{sect:multi-dim uncertainty}
For classification, the estimation of the probabilistic uncertainty relies on the design of an appropriate Bayesian DL model with parameters $\bm{\theta}$. Given input $x$ and dataset $\mathcal{G}$, we estimate a class probability by $P(y|x) = \int P(y|x;\bm{\theta}) P(\bm{\theta}|\mathcal{G}) d\bm{\theta}$, and obtain \textbf{\textit{epistemic uncertainty}} estimated by mutual information~\cite{depeweg2018decomposition, malinin2018predictive}:
\begin{eqnarray}
\underbrace{I(y, \bm{\theta}|x, \mathcal{G})}_{\text{\textbf{\textit{Epistemic}}}} =\underbrace{\mathcal{H}\big[ \mathbb{E}_{P(\bm{\theta}|\mathcal{G})}[P(y|x;\bm{\theta})] \big]}_{\text{\textbf{\textit{Entropy}}}} -  \underbrace{\mathbb{E}_{P(\bm{\theta}|\mathcal{G})}\big[\mathcal{H}[P(y|x;\bm{\theta})] \big]}_{\text{\textbf{\textit{Aleatoric}}}}, 
 \label{eq:epistemic}
\end{eqnarray}
where $\mathcal{H}(\cdot)$ is Shannon's entropy of a probability distribution. The first term indicates {\bf \textit{entropy}} that represents the total uncertainty while the second term is {\bf \textit{aleatoric}} that indicates data uncertainty.  By computing the difference between entropy and aleatoric uncertainties, we obtain epistemic uncertainty, which refers to uncertainty from model parameters. 

\section{Relationships Between Multiple Uncertainties}
\begin{figure*}[t!]
  \centering
  \includegraphics[width=0.5\textwidth]{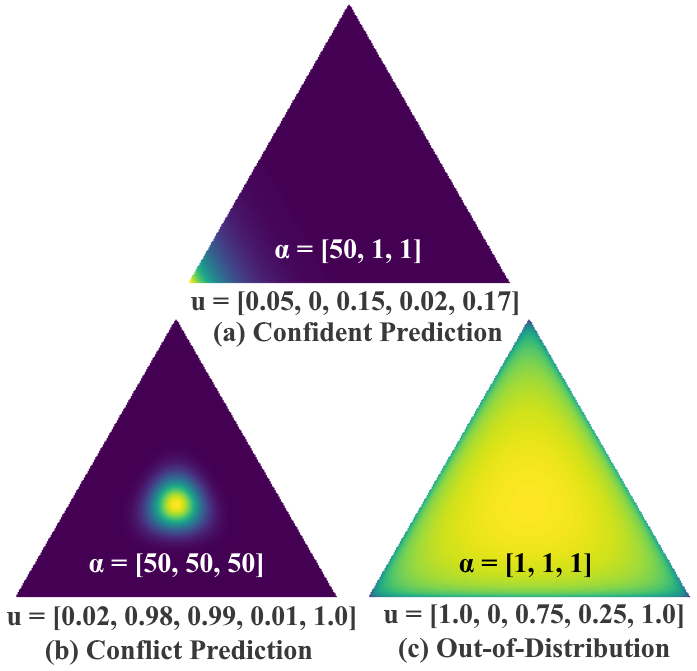}
  \caption{Multiple uncertainties of different predictions. Let ${\bf u}=[u_v, u_{diss}, u_{alea}, u_{epis}, u_{en}]$.}\label{fig:example}
\end{figure*}  

We use the shorthand notations $u_{v}$, $u_{diss}$, $u_{alea}$, $u_{epis}$, and $u_{en}$ to represent vacuity, dissonance, aleatoric, epistemic, and entropy, respectively.  

To interpret multiple types of uncertainty, we show three prediction scenarios of 3-class classification in Figure~\ref{fig:example}, in each of which the strength parameters $\alpha = [\alpha_1, \alpha_2, \alpha_3]$ are known.  To make a prediction with high confidence, the subjective multinomial opinion, following a Dirichlet distribution, will yield a sharp distribution on one corner of the simplex (see Figure~\ref{fig:example} (a)). For a prediction with conflicting evidence, called a conflicting prediction (CP), the multinomial opinion should yield a central distribution, representing confidence to predict a flat categorical distribution over class labels (see Figure~\ref{fig:example} (b)).  For an OOD scenario with $\alpha=[1, 1, 1]$, the multinomial opinion would yield a flat distribution over the simplex (Figure~\ref{fig:example} (c)), indicating high uncertainty due to the lack of evidence.  The first technical contribution of this work is as follows.

\begin{theorem}
We consider a simplified scenario, where a multinomial random variable $y$ follows a K-class categorical distribution: $y \sim \text{Cal}(\p)$, the class probabilities $\p$ follow a Dirichlet distribution: $\p\sim \text{Dir}({\bm \alpha})$, and ${\bm \alpha}$ refer to the Dirichlet parameters. Given a total Dirichlet strength $S=\sum_{i=1}^K \alpha_i$, 
for any opinion $\omega$ on a multinomial random variable $y$, we have

\begin{enumerate}
\item General relations on all prediction scenarios. 
    
(a) $u_v+ u_{diss} \le 1$; (b) $u_v > u_{epis}$.
    
    \item Special relations on the OOD and the CP.
    \begin{enumerate}
        \item For an OOD sample with a uniform prediction (i.e., $\alpha=[1, \ldots, 1]$), we have 
    \begin{eqnarray}
    1= u_v = u_{en}>  u_{alea} > u_{epis} > u_{diss} = 0 \nonumber
    \end{eqnarray}
        \item For an in-distribution sample with a conflicting prediction (i.e., $\alpha=[\alpha_1, \ldots, \alpha_K]$ with $\alpha_1 = \alpha_2 =\cdots = \alpha_K$, if $S \rightarrow \infty$),  we have 
    \begin{eqnarray}
    u_{en} = 1, \lim_{S\rightarrow \infty} u_{diss} =\lim_{S\rightarrow \infty} u_{alea} =1 , \lim_{S\rightarrow \infty} u_{v} =\lim_{S\rightarrow \infty} u_{epis} =0 \nonumber
    \end{eqnarray}
    \text{with} $u_{en} > u_{alea}> u_{diss}> u_{v}>u_{epis} $.
    \end{enumerate}
\end{enumerate}
\label{theorem_un}
\end{theorem}

The proof of Theorem~\ref{theorem_un} can be found in Appendix~\ref{app:A1}. As demonstrated in Theorem~\ref{theorem_un} and Figure~\ref{fig:example}, entropy cannot distinguish OOD (see Figure~\ref{fig:example} (c)) and conflicting predictions (see Figure~\ref{fig:example} (b)) because entropy is high for both cases. Similarly, neither aleatoric uncertainty nor epistemic uncertainty can distinguish OOD from conflicting predictions.  In both cases, aleatoric uncertainty is high while epistemic uncertainty is low.  On the other hand, vacuity and dissonance can clearly distinguish OOD from a conflicting prediction.  For example, OOD objects typically show high vacuity with low dissonance while conflicting predictions exhibit low vacuity with high dissonance.  This observation is confirmed through the empirical validation via our extensive experiments in terms of misclassification and OOD detection tasks.

\section{Uncertainty-Aware Semi-Supervised Learning}
In this section, we describe our proposed uncertainty framework based on semi-supervised node classification problem. It is designed to predict the subjective opinions about the classification of testing nodes, such that a variety of uncertainty types, such as vacuity, dissonance, aleatoric uncertainty, and epistemic uncertainty, can be quantified based on the estimated subjective opinions and posterior of model parameters. As a subjective opinion can be equivalently represented by a Dirichlet distribution about the class probabilities, we proposed a way to predict the node-level subjective opinions in the form of node-level Dirichlet distributions. The overall description of the framework is shown in Figure~\ref{fig:framework2}.

\subsection{Problem Definition} \label{subsec:problem-definition}
Given an input graph $\mathcal{G} = (\mathbb{V}, \mathbb{E}, {\bf r}, {\bf y}_\mathbb{L})$, where $\mathbb{V} = \{1, \ldots, N \}$ is a ground set of nodes, $\mathbb{E} \subseteq \mathbb{V}\times \mathbb{V}$ is a ground set of edges, $\rr = [\rr_1, \cdots, \rr_N]^T \in \mathbb{R}^{N\times d}$ is a node-level feature matrix, $\rr_i\in \mathbb{R}^d$ is the feature vector of node $i$, $\y_{\mathbb{L}}=\{y_i \mid i \in \mathbb{L}\}$ are the labels of the training nodes $\mathbb{L} \subset \mathbb{V}$, and $y_i \in \{1, \ldots, K\}$ is the class label of node $i$. {\bf We aim to predict}: (1) the \textbf{class probabilities} of the testing nodes: $\p_{\mathbb{V} \setminus \mathbb{L}} = \{\p_i \in [0, 1]^K \mid i \in \mathbb{V} \setminus \mathbb{L}\}$; and (2) the \textbf{associated multidimensional uncertainty estimates} introduced by different root causes: $\mathbf{u}_{\mathbb{V} \setminus \mathbb{L}} = \{\mathbf{u}_i \in [0, 1]^m \mid i \in \mathbb{V} \setminus \mathbb{L}\}$, where $p_{i, k}$ is the probability that the class label $y_i = k$ and $m$ is the total number of
uncertainty types. 

\begin{figure*}[t!]
  \centering
  \includegraphics[width=0.95\linewidth]{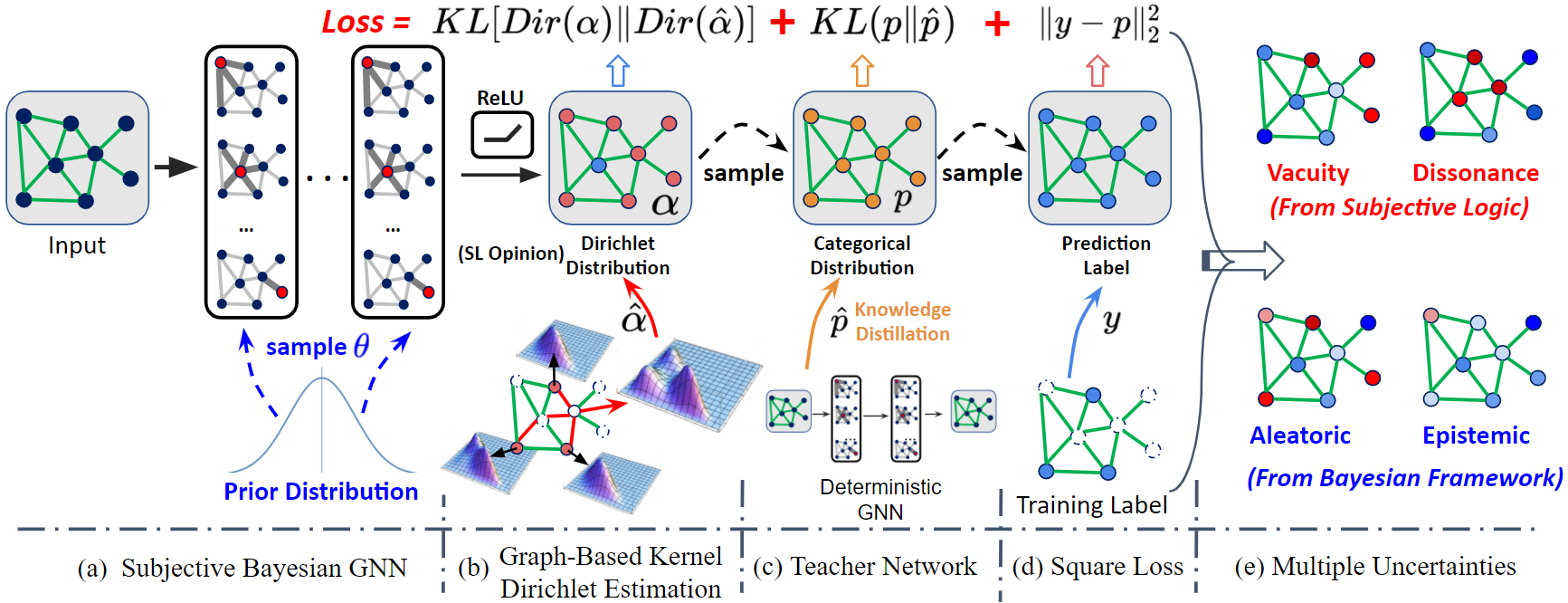}
 \caption{Uncertainty Framework Overview. Subjective Bayesian GNN (a) is designed for estimating the different types of uncertainties. The loss function includes a square error (d) to reduce bias, GKDE (b) to reduce errors in uncertainty estimation, and teacher network (c) to refine class probability.}
  \label{fig:framework2}
\end{figure*}  

\subsection{Proposed Uncertainty Framework} \label{subsec:bay-dl}
\textbf{Learning evidential uncertainty.}
As discussed in Section~\ref{SL}, evidential uncertainty can be derived from multinomial opinions or equivalently Dirichlet distributions to model a probability distribution for the class probabilities. Therefore, we design a Subjective GNN (S-GNN) $f$ to form their multinomial opinions for the node-level Dirichlet distribution $\text{Dir}(\p_i | {\bm \alpha}_i)$ of a given node $i$. Then, the conditional probability $P(\p |A, \rr; \bm{\theta})$ can be obtained by:
\begin{eqnarray} 
P(\p |A, \rr; \bm{\theta} )=\prod\nolimits_{i=1}^N \text{Dir}(\p_i|\bm{\alpha}_i), \ \bm{\alpha}_i=f_i(A,\rr;\bm{\theta}), \label{GCN_1}
\end{eqnarray}
where $f_i$ is the output of S-GNN for node $i$, $\bm{\theta}$ is the model parameters, and $A$ is an adjacency matrix. The Dirichlet probability function $\text{Dir}(\p_i | \bm{\alpha}_i)$ is defined by Eq.~\eqref{eq:multinomial-dir}.

Note that S-GNN is similar to classical GNN, except that we use an activation layer (e.g., \textit{ReLU}) instead of the \textit{softmax} layer (only outputs class probabilities).  This ensures that S-GNN would output non-negative values, which are taken as the parameters for the predicted Dirichlet distribution. 

\textbf{Learning probabilistic uncertainty.}
Since probabilistic uncertainty relies on a Bayesian framework, we proposed a Subjective Bayesian GNN (S-BGNN) that adapts S-GNN to a Bayesian framework, with the model parameters $\bm{\theta}$ following a prior distribution. The joint class probability of $\y$ can be estimated by: 
\begin{eqnarray}
P(\y |A, \rr; \mathcal{G}) &=& \int \int P(\y | \p) P(\p |A, \rr; \bm{\theta} ) P(\bm{\theta} | \mathcal{G}) d \p d\bm{\theta} \nonumber \\
&\approx& \frac{1}{M}\sum_{m=1}^M  \sum_{i=1}^N \int P(\y_i | \p_i) P(\p_i | A, \rr;\bm{\theta}^{(m)} ) d \p_i, \quad \bm{\theta}^{(m)} \sim q( \bm{\theta}) 
\label{Baye_model}
\end{eqnarray}
where $P(\bm{\theta} | \mathcal{G})$ is the posterior, estimated via dropout inference, that provides an approximate solution of posterior $q(\bm{\theta})$ and taking samples from the posterior distribution of models~\citep{gal2016dropout}.  Thanks to the benefit of dropout inference, training a DL model directly by minimizing the cross entropy (or square error) loss function can effectively minimize the KL-divergence between the approximated distribution and the full posterior (i.e., KL[$q(\bm{\theta})\|P(\theta|\mathcal{G})$]) in variational inference~\citep{gal2016dropout, kendall2015bayesian}. For interested readers, please refer to more detail in Appendix B.8.

Therefore, training S-GNN with stochastic gradient descent enables learning of an approximated distribution of weights, which can provide good explainability of data and prevent overfitting.  We use a {\em loss function} to compute its Bayes risk with respect to the sum of squares loss $\|\y-\p\|^2_2$ by:
\begin{eqnarray} 
\mathcal{L}(\bm{\theta}) &=&  \sum\nolimits_{i\in \mathbb{L}} \int \|\y_i-\p_i\|^2_2 \cdot P(\p_i |A, \rr; \bm{\theta}) d \p_i \nonumber \\
&=&  \sum\nolimits_{i\in \mathbb{L}} \sum\nolimits_{k=1}^K \big(y_{ik}-\mathbb{E}[p_{ik}]\big)^2 + \text{Var}(p_{ik}),
\label{loss}
\end{eqnarray}
where $\y_i$ is an one-hot vector encoding the ground-truth class with $y_{ij} = 1$ and $y_{ik} \neq $ for all $k \neq j$ and $j$ is a class label. Eq.~\eqref{loss} aims to minimize the prediction error and variance, leading to maximizing the classification accuracy of each training node by removing excessive misleading evidence.
\subsection{Graph-based Kernel Dirichlet distribution Estimation (GKDE)}
\begin{figure*}[t!]
  \centering
   \includegraphics[width=0.5\textwidth]{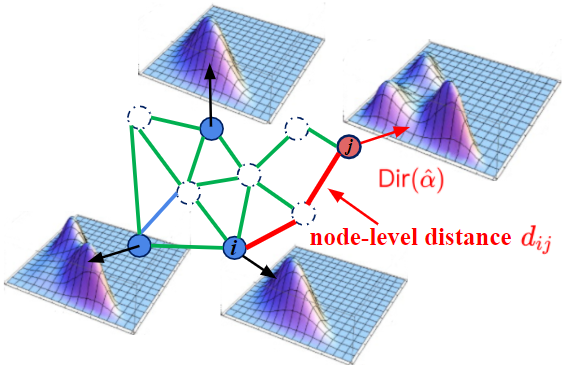}
 \caption{Illustration of GKDE. Estimate prior Dirichlet distribution $\text{Dir}(\hat{\alpha})$ for node $j$ (red) based on training nodes (blue) and graph structure information.}
  \label{fig:gkde}
\end{figure*}  

The loss function in Eq.~\eqref{loss} is designed to measure the sum of squared loss based on class labels of training nodes.  However, it does not directly measure the quality of the predicted node-level Dirichlet distributions~\cite{zhao2019quantifying,hu2021multidimensional}.  To address this limitation, we proposed \textit{Graph-based Kernel Dirichlet distribution Estimation} (GKDE) to better estimate node-level Dirichlet distributions by using graph structure information.  The key idea of the GKDE is to estimate prior Dirichlet distribution parameters for each node based on the class labels of training nodes (see Figure~\ref{fig:gkde}). Then, we use the estimated prior Dirichlet distribution in the training process to learn the following patterns: (i) nodes with a high vacuity will be shown far from training nodes; and (ii) nodes with a high dissonance will be shown near the boundaries of classes.

Based on SL, let each training node represent one evidence for its class label. Denote the contribution of evidence estimation for node $j$ from training node $i$ by $\mathbf{h}(y_i,d_{ij}) =[h_1, \ldots, h_k, \ldots, h_K] \in[0, 1]^K$, where $h_k(y_i,d_{ij})$ is obtained by: 
\begin{eqnarray}
h_k(y_i,d_{ij}) = \begin{cases}0 & y_i \neq k \\ g(d_{ij}) & y_i = k,  \end{cases}
\end{eqnarray}
 $g(d_{ij}) = \frac{1}{\sigma \sqrt{2\pi}}\exp({-\frac{{d^2_{ij}}}{2\sigma^2}})$ is the Gaussian kernel function used to estimate the distribution effect between nodes $i$ and $j$, and $d_{ij}$ means the \textbf{node-level distance} (\textbf{the shortest path between nodes $i$ and $j$}), and $\sigma$ is the bandwidth parameter. The prior evidence is estimated based GKDE: $\hat{\bm{e}}_j = \sum_{i\in \mathbb{L}} \mathbf{h}(y_i,d_{ij})$, where $\mathbb{L}$ is a set of training nodes and the prior Dirichlet distribution $\hat{\bm{\alpha}}_j = \hat{\bm{e}}_j +\bf 1$.  
During the training process, we minimize the KL-divergence between model predictions of Dirichlet distribution and prior distribution: $\min \text{KL}[\text{Dir}(\bm{\alpha}) \| \text{Dir}(\hat{\bm{\alpha}})]$.
This process can prioritize the extent of data relevance based on the estimated evidential uncertainty, which is proven effective based on the proposition below.

\begin{proposition}
Given $L$ training nodes, for any testing nodes $i$ and $j$, let ${\bm d}_i = [d_{i1}, \ldots, d_{iL}]$ be the vector of graph distances from nodes $i$ to training nodes and ${\bm d}_j = [d_{j1}, \ldots, d_{jL}]$ be the graph distances from nodes $j$ to training nodes, where $d_{il}$ is the node-level distance between nodes $i$ and $l$. If for all $l\in \{1, \ldots, L\}$, $d_{il} \ge d_{jl}$, then we have
\begin{eqnarray}
\hat{u}_{v_i} \ge \hat{u}_{v_j}, \nonumber
\end{eqnarray}
where $ \hat{u}_{v_i}$ and $\hat{u}_{v_j}$ refer to vacuity uncertainties of nodes $i$ and $j$ estimated based on GKDE.
\label{proposition_GKDE}
\end{proposition}
\begin{proof}
Let ${\bm y} = [y_1, \ldots, y_L]$ \feng{be} the label vector for training nodes. \feng{B}ased on GKDE, the evidence contribution for the node $i$ and a training node $l\in \{1, \ldots, L|\}$ is ${\bm h} (y_l, d_{il})=[h_1(y_l, d_{il}), \ldots, h_K(y_l, d_{il})]$, where
\begin{eqnarray}
h_k(y_l, d_{il}) = \begin{cases}0 & y_l \neq k\\ g(d_{il}) = \frac{1}{\sigma \sqrt{2\pi}}\exp({-\frac{{d_{il}}^2}{2\sigma^2}})  & y_l = k \end{cases} \label{po3_hk},
\end{eqnarray}
and the prior evidence \feng{can be estimated} based GKDE: 
\begin{eqnarray}
\hat{{\bm e}}_i = \sum_{m=1}^{L} \sum_{k=1}^K h_k(y_l, d_{il}),
\end{eqnarray}
\feng{where} $\hat{{\bm e}}_i=[e_{i1},...,e_{iK}]$. Since each training node only contributes the same evidence \feng{based on} its label based on Eq.~\eqref{po3_hk}, the total evidence is estimated by all \feng{the} contributing evidence as
\begin{eqnarray}
\sum_{k=1}^K e_{ik} = \sum_{m=1}^{L} \frac{1}{\sigma \sqrt{2\pi}} \exp({-\frac{{d_{il}}^2}{2\sigma^2}}) , \quad \sum_{k=1}^K e_{jk} = \sum_{m=1}^{L} \frac{1}{\sigma \sqrt{2\pi}} \exp({-\frac{{d_{jl}}^2}{2\sigma^2}}),
\end{eqnarray}
where the vacuity values for node $i$ and node $j$ based on GKDE are,
\begin{eqnarray}
\label{eq:i-j-vacuity}
\hat{u}_{v_i} = \frac{K}{\sum_{k=1}^K e_{ik} + K}, \quad \hat{u}_{v_j} = \frac{K}{\sum_{k=1}^K e_{jk} + K}.
\end{eqnarray}

Now, we prove Eq.~\eqref{eq:i-j-vacuity} above.  If $d_{il}\ge d_{jl}$ for $\forall l\in \{1, \ldots, L\}$, we have
\begin{eqnarray}
\sum_{k=1}^K e_{ik} &=& \sum_{m=1}^{L} \frac{1}{\sigma \sqrt{2\pi}} \exp({-\frac{{d_{il}}^2}{2\sigma^2}}) \\
&\le& \sum_{m=1}^{L} \frac{1}{\sigma \sqrt{2\pi}} \exp({-\frac{{d_{jl}}^2}{2\sigma^2}}) \nonumber\\
&=& \sum_{k=1}^K e_{jk}, \nonumber
\end{eqnarray}
such that
\begin{eqnarray}
\hat{u}_{v_i} = \frac{K}{\sum_{k=1}^K e_{ik} + K} \ge \frac{K}{\sum_{k=1}^K e_{jk} + K} = \hat{u}_{v_j}. 
\end{eqnarray}
\end{proof}
The above proposition shows that if a testing node is too far from training nodes, the vacuity will increase, implying that an OOD node is expected to have a high vacuity.

In addition, we designed a simple iterative knowledge distillation method~\citep{44873} (i.e., Teacher Network) to refine the node-level classification probabilities. The key idea is to train our proposed model (Student) to imitate the outputs of a pre-train vanilla GNN (Teacher) by adding a regularization term of KL-divergence. This leads to solving the following optimization problem: 
\begin{eqnarray} 
\min\nolimits_{\bm{\theta}}  \mathcal{L}(\bm{\theta}) + \lambda_1 \text{KL}[\text{Dir}({\bm \alpha}) \| \text{Dir}(\hat{\bm \alpha})] + \lambda_2  \text{KL}[P(\y \mid A,\rr;\mathcal{G}) \parallel P(\y|\hat{\p})],
\label{joint loss}
\end{eqnarray}
where $\hat{\p}$ is the vanilla GNN's (Teacher) output and $\lambda_1$ and $\lambda_2$ are trade-off parameters. Algorithm~\ref{algorithm_S-BGCN-T-K} shows the details of our proposed framework.

\begin{algorithm}[H]
\small{
\DontPrintSemicolon
\KwIn{$\mathbb{G} = (\mathbb{V}, \mathbb{E}, \rr)$ and $\y_{\mathbb{L}}$}
\KwOut{$\p_{\mathbb{V} \setminus \mathbb{L}}$, $\uu_{\mathbb{V} \setminus \mathbb{L}}$} 
\SetKwBlock{Begin}{function}{end function}
{
  $\ell=0$; \;
  Set hyper-parameters $\eta, \lambda_1, \lambda_2$; \;
  
  Initialize the parameters $\gamma, \beta$; \;
  
  Calculate the prior Dirichlet distribution $\text{Dir}(\hat{\alpha})$; \;
  Pretrain the teacher network to get $\text{Prob}(\y |\hat{\p})$; \;
 \Repeat{convergence }
 {
 
 Forward pass to compute $\bm{\alpha}$, $\text{Prob}(\p_i |A, \rr; \mathcal{G})$ for $i\in\mathbb{V}$;\;
 
 Compute joint probability $\text{Prob}(\y |A, \rr; \mathcal{G})$;\;
 
 Backward pass via the chain-rule the calculate the sub-gradient gradient: $g^{(\ell)} = \nabla_\Theta \mathcal{L}(\Theta)$ \;
 
 Update parameters using step size $\eta$ via 
  $\Theta^{(\ell+1)} = \Theta^{(\ell)} - \eta \cdot g^{(\ell)}$\; 

  $\ell = \ell + 1$;\;
 }
 Calculate $\p_{\mathbb{V} \setminus \mathbb{L}}$, $\uu_{\mathbb{V} \setminus \mathbb{L}}$ 

\Return{$\p_{\mathbb{V} \setminus \mathbb{L}}$, $\uu_{\mathbb{V} \setminus \mathbb{L}}$}
}
\caption{S-BGCN-T-K}\label{algorithm_S-BGCN-T-K}
}
\end{algorithm}

\section{Experiments} \label{sec:exp-results-analysis}
In this section, we conduct experiments on the tasks of misclassification and OOD detections to answer the following questions for semi-supervised node classification:

\noindent {\bf Q1. Misclassification Detection:} What type of uncertainty is the most promising indicator of high confidence in node classification predictions? 

\noindent {\bf Q2. OOD Detection:}  What type of uncertainty is a key indicator of accurate detection of OOD nodes? 

\noindent {\bf Q3. GKDE with Uncertainty Estimates:}  How can GKDE help enhance prediction tasks with what types of uncertainty estimates?

Through extensive experiments, we found the following answers to the above questions:

\noindent {\bf A1.} Dissonance (i.e., uncertainty due to conflicting evidence) is more effective than other uncertainty estimates in misclassification detection. 

\noindent {\bf A2.}  Vacuity (i.e., uncertainty due to lack of confidence) is more effective than other uncertainty estimates in OOD detection.

\noindent {\bf A3.}  GKDE can indeed help improve the estimation quality of node-level Dirichlet distributions, resulting in a higher OOD detection.

\subsection{Experiment Setup} 
\textbf{Datasets}: We used six datasets, including three citation network datasets~\citep{sen2008collective} (i.e., Cora, Citeseer, Pubmed) and three new datasets~\citep{shchur2018pitfalls} (i.e., Coauthor Physics, Amazon Computer, and Amazon Photo). 
We summarized the description and experimental setup of the used datasets as follows.  

\begin{table*}[h]
\footnotesize
\caption{Description of datasets and their experimental setup for the node classification prediction.}
\centering
\vspace{-0mm}
  \begin{tabular}{lcccccc}
    \toprule
     & \textbf{Cora} & \textbf{Citeseer}  & \textbf{Pubmed} & \textbf{Co. Physics} & \textbf{Ama.Computer} & \textbf{Ama.Photo} \\
    \midrule
    \textbf{\#Nodes}   & 2,708 & 3,327 &  19,717 &  34, 493 & 13, 381 & 7, 487 \\
    \textbf{\#Edges}   &  5,429 &   4,732 &  44,338 &  282, 455 & 259, 159 & 126, 530   \\
    \textbf{\#Classes}  & 7 & 6 &  3  & 5 & 10 & 8  \\
    \textbf{\#Features}  & 1,433 & 3,703 &  500 &  8,415 & 767 & 745   \\
    \textbf{\#Training nodes}  & 140 & 120 &  60 &  100 & 200 & 160   \\
    \textbf{\#Validation nodes}  & 500 & 500 &  500 &  500  &  500 &  500  \\
    \textbf{\#Test nodes}  & 1,000 & 1,000 &  1,000 &  1000  &  1,000 &  1000 \\
    \bottomrule
  \end{tabular}
   \label{table:datasets-description}
\end{table*}

\noindent {\bf Cora, Citeseer, and Pubmed}~\citep{sen2008collective}: These are citation network datasets, where \feng{each} network is a directed network \feng{in which} a node represents a document and an edge is a citation link, meaning that there exists an edge when $A$ document cites $B$ document, or vice-versa with a direction. Each node's feature vector contains a bag-of-words representation of a document. For simplicity, we don't discriminate the direction of links and treat citation links as undirected edges and construct a binary, symmetric adjacency matrix $\mathbf{A}$. Each node is labeled with the class to which it belongs. 

\noindent {\bf Coauthor Physics, Amazon Computers, and Amazon Photo }~\citep{shchur2018pitfalls}:
Coauthor Physics is the dataset for co-authorship graphs based on the Microsoft Academic Graph from the KDD Cup 2016 Challenge\footnote{KDD Cup 2016 Dataset: Online Available at \url{https://kddcup2016.azurewebsites.net/}}. In the graphs, a node is an author and an edge exists when two authors co-author a paper. A node's features represent the keywords of its papers and the node's class label indicates its most active field of study.  Amazon Computers and Amazon Photo are the segments of an Amazon co-purchase graph~\citep{mcauley2015image}, where a node is a good (i.e., product), and an edge exists when two goods are frequently bought together. A node's features are the bag-of-words representation of product reviews and the node's class label is the product category.

For all the used datasets, we deal with undirected graphs with 20 training nodes for each category. We chose the same dataset splits as in~\citep{yang2016revisiting} with an additional validation node set of 500 labeled examples for the hyperparameter obtained from the citation datasets, and followed the same dataset splits in~\citep{shchur2018pitfalls} for Coauthor Physics, Amazon Computer, and Amazon Photo datasets, for the fair comparison\footnote{The source code and datasets are accessible at \href{https://github.com/zxj32/uncertainty-GNN}{\color{magenta}{https://github.com/zxj32/uncertainty-GNN}}}. 

\noindent {\bf Metric}: We use\feng{d} the following metrics for our experiments:
\begin{itemize}[leftmargin=*]
\item {\em Area Under Receiver Operating Characteristics (AUROC)}: AUROC shows the area under the curve where FPR (false positive rate) is in $x$-axis and TPR (true positive rate) is in $y$-axis. It can be interpreted as the probability that a positive example is assigned a higher detection score than a negative example\citep{fawcett2006introduction}.
A perfect detector corresponds to an AUROC score of 100\%.
\item {\em Area Under Precision-Prediction Curve (AUPR)}: The PR curve is a graph showing the precision=TP/(TP+FP) and recall=TP/(TP+FN) against each other,and AUPR denote\feng{s} the area under the precision-recall curve. The ideal case is when Precision is 1 and Recall is 1. 
\end{itemize}

\noindent {\bf Comparing Schemes}:
We conducted an extensive comparative performance analysis based on our proposed models and several state-of-the-art competitive counterparts. We implemented all models based on the most popular GNN model, GCN~\citep{kipf2017semi}.  We compared our model (S-BGCN-T-K) against (1) Softmax-based GCN~\citep{kipf2017semi} with uncertainty measured based on entropy; and (2) Drop-GCN that adapts the Monte-Carlo Dropout~\citep{gal2016dropout, ryu2019uncertainty} into the GCN model to learn probabilistic uncertainty; (3) EDL-GCN that adapts the EDL model~\citep{sensoy2018evidential} with GCN to estimate evidential uncertainty; (4) DPN-GCN that adapts the DPN~\citep{malinin2018predictive} method with GCN to estimate probabilistic uncertainty. We evaluated the performance of all models considered using the area under the ROC (AUROC) curve and area under the Precision-Recall (AUPR) curve in both experiments~\citep{hendrycks17baseline}.

\noindent {\bf Model Setups for semi-supervised node classification}.
Our models \feng{were} initialized using Glorot initialization~\citep{glorot2010understanding} and trained to minimize loss using the Adam SGD optimizer~\citep{kingma2014adam}. For the S-BGCN-T-K model, we use{\feng d} the {\em early stopping strategy}~\citep{shchur2018pitfalls} on Coauthor Physics, Amazon Computer, and Amazon Photo datasets while {\em non-early stopping strategy} \feng{was} used in citation datasets (i.e., Cora, Citeseer and Pubmed). We set bandwidth $\sigma=1$ for all datasets in GKDE, and set trade-off parameters $\lambda_1=0.001$ for misclassification detection, $\lambda_1=0.1$ for OOD detection and $\lambda_2=\min(1,t/200)$ (where $t$ is the index of a current training epoch) for both task; other hyperparameter configurations are summarized in Table~\ref{table:BGCN-T}. 

For semi-supervised node classification, we use{\feng d} 50 random weight initialization for our models on Citation network datasets. For Coauthor Physics, Amazon Computer, and Amazon Photo datasets, we report{\feng ed} the result based on 10 random train/validation/test splits.  In both effects of uncertainty on misclassification and the OOD detection, we report{\feng ed} the AUPR and AUROC results in percent averaged over 50 times of randomly chosen 1000 test nodes in all of the test sets (except training or validation set) for all models tested on the citation datasets. For S-BGCN-T-K model in these tasks, we use{\feng d} the same hyperparameter configurations as in Table~\ref{table:BGCN-T}, except S-BGCN-T-K Epistemic using 10,000 epochs to obtain the best result.

\begin{table}[h]
  \caption{Hyperparameter configurations of S-BGCN-T-K model}
\centering
\scriptsize
  \begin{tabular}{ l c c c c c c } 
    \toprule
 & \textbf{Cora} & \textbf{Citeseer}  & \textbf{Pubmed}   & \textbf{Co.Physics} & \textbf{Ama.Computer}  & \textbf{Ama.Photo}\\
    \midrule
 \textbf{Hidden units}   & 16 & 16 &  16 &  64 & 64 & 64   \\
 \textbf{Learning rate}  &  0.01 & 0.01  & 0.01  &  0.01 & 0.01  & 0.01    \\
 \textbf{Dropout}  & 0.5 & 0.5 &  0.5  & 0.1 & 0.2 & 0.2 \\
 \textbf{$L_2$ reg.strength}  & 0.0005 & 0.0005 & 0.0005  &  0.001 &  0.0001 & 0.0001   \\
 \textbf{Monte-Carlo samples}  & 100 & 100 &  100  &  100 & 100 & 100  \\
 \textbf{Max epoch}  & 200 & 200 &  200 &  100000 & 100000 & 100000  \\
    \bottomrule
\end{tabular}
  \label{table:BGCN-T}
 
\end{table}

\noindent {\bf Baseline Setting}.
In the experiment part, we consider\feng{ed} 4 baselines. For GCN, we use{\feng d} the same hyper-parameters as~\citep{kipf2016semi}. For EDL-GCN, we use{\feng d} the same hyper-parameters as GCN, and replace{\feng d} softmax layer to activation layer (Relu) with squares loss~\citep{sensoy2018evidential}. For DPN-GCN, we use{\feng d} the same hyper-parameters as GCN, and change{\feng d} the softmax layer to activation layer (exponential)\feng{.} \feng{N}ote that \feng{as} we can not generate OOD node, we only use{\feng d} in\feng{-}distribution loss of (see Eq.12 in~\citep{malinin2018predictive}) and ignore{\feng d} the OOD part loss. For Drop-GCN,  we use{\feng d} the same hyper-parameters as GCN, and set Monte Carlo sampling times $M=100$, dropout rate equal to 0.5.

\noindent {\bf Experimental Setup for Out-of-Distribution (OOD) Detection}
For OOD detection on semi-supervised node classification, we randomly selected 1-4 categories as OOD categories and trained the models only based on training nodes of the other categories. In this setting, we still train\feng{ed} a model for the semi-supervised node classification task, but only part of node categories \feng{were} not \feng{used} for training. Hence, we suppose that our model only outputs partial categories (as we don't know the OOD category), see Table~\ref{tab:ood_data}.  For example, Cora dataset, we train\feng{ed} the model with 80 nodes (20 nodes for each category) with the predictions of 4 categories. The positive ratio is the ratio of out-of-distribution nodes among on all test nodes.
\vspace{-2mm}
\begin{table*}[h]
\caption{Description of datasets and their experimental setup for the OOD detection.}
\centering
\scriptsize
  \begin{tabular}{lcccccc}
    \toprule
    Dataset & \textbf{Cora} & \textbf{Citeseer}  & \textbf{Pubmed} & \textbf{Co.Physics} & \textbf{Ama.Computer}  & \textbf{Ama.Photo} \\
    \midrule
    \textbf{Number of training categories}   &  4 &  3 &  2 & 3  &5  &  4 \\
    \textbf{Training nodes}   &  80 &  60 &  40 & 60  &100   &80   \\
    \textbf{Test nodes} & 1000& 1000 & 1000 & 1000& 1000  &  1000  \\
    \textbf{Positive ratio}  & 38\% & 55\% & 40.4\% & 45.1\%  &  48.1\% &51.1\% \\
    \bottomrule
  \end{tabular}
    \label{tab:ood_data}
\vspace{-3mm}
\end{table*}

\subsection{Results}
\noindent {\bf Misclassification Detection.} The misclassification detection experiment involves detecting whether a given prediction is incorrect using an uncertainty estimate.  Table~\ref{AUPR:uncertainty} shows that S-BGCN-T-K outperforms all baseline models under the AUROC and AUPR for misclassification detection. The outperformance of dissonance-based detection is fairly impressive. This confirms that low dissonance (a small amount of conflicting evidence) is the key to maximizing the accuracy of node classification prediction.  We observe the following performance order: ${\tt Dissonance} > {\tt Entropy} \approx {\tt Aleatoric} > {\tt Vacuity} \approx {\tt Epistemic}$, which is aligned with our conjecture: higher dissonance with conflicting prediction leads to higher misclassification detection. We also conducted experiments on additional three datasets and observed similar trends of the results, as demonstrated in Appendix C.
\begin{table*}[th!]
\footnotesize
\caption{AUROC and AUPR for the Misclassification Detection.}
\centering
\begin{tabular}{c||c|ccccc|ccccc|c}
\hline
\multirow{2}{*}{Data} & \multirow{2}{*}{Model} & \multicolumn{5}{c|}{AUROC} & \multicolumn{5}{c|}{AUPR} & \multirow{2}{*}{Acc} \\  
                  &                   & Va.\tnote{*}& Dis. & Al. & Ep. & En. & Va. & Dis. &  Al. & Ep. &En. & \\ \hline
\multirow{5}{*}{Cora} & S-BGCN-T-K & 70.6 & \textbf{82.4} & 75.3 & 68.8 & 77.7&  90.3  & \textbf{95.4} & 92.4 & 87.8 &93.4 & \textbf{82.0} \\ 
                  &  EDL-GCN &  70.2  &  81.5  & -   &  - &  76.9  &90.0 &  94.6  &  -  & - &  93.6 & 81.5 \\    
                  &    DPN-GCN &  -  &  -  & 78.3   &  75.5 &  77.3  & -  &  -  &  92.4  & 92.0 &92.4 & 80.8 \\  
                  & Drop-GCN &  -  &  -  & 73.9   &  66.7  & 76.9  & -  &  -  &  92.7  & 90.0 &93.6 & 81.3 \\ 
                  &GCN & -  &  -  &  -  &  -   & 79.6&  -  &  -  &  -  &  -  & 94.1 & 81.5 \\ 
                  \hline
\multirow{5}{*}{Citeseer} &  S-BGCN-T-K & 65.4& \textbf{74.0}& 67.2& 60.7&  70.0& 79.8 & \textbf{85.6} & 82.2 & 75.2 &83.5& \textbf{71.0}  \\ 
                   &  EDL-GCN &  64.9  &  73.6  & -   &  - &  69.6  &79.2 &  84.6 &  -  & - &  82.9 & 70.2 \\    
                  &    DPN-GCN &  -  &  -  & 66.0   &  64.9 &  65.5  & -  &  -  & 78.7  & 77.6 &78.1 & 68.1 \\  
                  &                  Drop-GCN &   - &  -  & 66.4 & 60.8 &   69.8 &   -  & -& 82.3 & 77.8  &83.7 & 70.9  \\ 
                  &                  GCN &  -  &  -  &  -  &  -  & 71.4 &  -  &  -  &  -  & -   &83.2 & 70.3 \\ \hline
\multirow{5}{*}{Pubmed} &  S-BGCN-T-K & 64.1 & \textbf{73.3} & 69.3 & 64.2 &  70.7& 85.6 & \textbf{90.8} & 88.8& 86.1   &89.2 & \textbf{79.3} \\
                  &  EDL-GCN &  62.6  &  69.0  & -   &  - &  67.2  &84.6 &  88.9 &  -  & - &  81.7 & 79.0 \\    
                  &    DPN-GCN &  -  &  -  & 72.7   &  69.2 &  72.5  & -  &  -  &  87.8  & 86.8 &87.7 & 77.1 \\  
                  &                  Drop-GCN &  -  &  -  &  67.3 & 66.1 &   67.2&  -  &  -& 88.6 & 85.6   &89.0 & 79.0  \\ 
                  &                  GCN &  -  &  -  &   - & -  & 68.5&  -  &   - & -   &  - &89.2 & 79.0 \\ \hline
                 &  S-BGCN-T-K & 66.0 & \textbf{89.3}& 83.0& 83.4 &  83.2&   95.4 & \textbf{98.9}& 98.4& 98.1  & 98.4 & \textbf{92.0} \\
       Amazon    &  EDL-GCN &  65.1  &  88.5  & -   &  - &  82.2  &94.6 &  98.1 &  -  & - &  98.0 & 91.2 \\    
        Photo    &    DPN-GCN &  -  &  -  & 84.6   &  84.2 &  84.3  & -  &  -  &  98.4  & 98.3 &98.3 & 92.0 \\  
                  &                  Drop-GCN &  -  &  -  &  84.5 & 84.4 &   84.6&  -  &  -& 98.2 & 98.1 &98.2 & 91.3 \\ 
                  &                  GCN &  -  &  -  &   - & -   &   86.8&  -  &   - & -   &  -  &98.5 & 91.2 \\ \hline
 &  S-BGCN-T-K & 65.0 & \textbf{87.8}& 83.3& 79.6  & 83.6& 89.4 & \textbf{96.3} & 95.0& 94.2 & 95.0 & \textbf{84.0} \\
      Amazon      &  EDL-GCN &  64.1  &  86.5  & -   &  - &  82.2  &93.6 &  97.1 &  -  & - &  97.0 & 79.7 \\    
     Computer     &    DPN-GCN &  -  &  -  & 82.0   &  81.7 &  81.8  & -  &  -  &  95.8  & 95.7 &95.8 & 85.2 \\  
                  &                  Drop-GCN &  -  &  -  &  79.1 & 75.9 &   79.2&  -  &  -& 95.1 & 94.5   &  95.1 & 79.6 \\ 
                  &                  GCN &  -  &  -  &   - & -   &   81.7&  -  &   - & -   &  -  &95.4  & 79.7\\ \hline
 &  S-BGCN-T-K & 80.2 & \textbf{91.4}& 87.5& 81.7  & 87.6& 98.3 &\textbf{99.4} & 99.0& 98.4 & 98.9 & \textbf{93.0}  \\
      Coauthor     &  EDL-GCN &  78.8  &  89.5  & -   &  - &  86.2  &96.6 &  97.2 &  -  & - &  97.0 & 92.7 \\    
     Physics   &    DPN-GCN &  -  &  -  & 90.0   &  89.9 &  90.0  & -  &  -  &  99.3  & 99.3 &99.3 & 92.5 \\  
                  &                  Drop-GCN &  -  &  -  &  87.6 & 84.1 &   87.7&  -  &  -& 98.9 & 98.6   &  98.9 & 93.0 \\ 
                  &                  GCN &  -  &  -  &   - & -   &   88.7&  -  &   - & -   &  - &99.0 & 92.8 \\ \hline
\end{tabular}
\begin{tablenotes}\footnotesize
\centering
\item[*] Va.: Vacuity, Dis.: Dissonance, Al.: Aleatoric, Ep.: Epistemic, En.: Entropy 
\end{tablenotes}
\label{AUPR:uncertainty}
\end{table*}

\begin{table*}[th!]
\footnotesize
\caption{AUROC and AUPR for the OOD Detection.}
\centering
\begin{tabular}{c||c|ccccc|ccccc}
\hline
\multirow{2}{*}{Data} & \multirow{2}{*}{Model} & \multicolumn{5}{c|}{AUROC} & \multicolumn{5}{c}{AUPR} \\  
                  &                   & Va.\tnote{*}& Dis. & Al. & Ep.  &En. & Va. & Dis. &  Al. & Ep. & En. \\ \hline
\multirow{5}{*}{Cora} & S-BGCN-T-K & \textbf{87.6} & 75.5 & 85.5 & 70.8  & 84.8&  \textbf{78.4} & 49.0 & 75.3 & 44.5 &  73.1  \\ 
                  &    EDL-GCN &            84.5 & 81.0 & & -& 83.3&    74.2 & 53.2 & - & -&71.4  \\ 
                  &       DPN-GCN &  -  &  -  & 77.3   &  78.9 &  78.3  & -  &  -  &  58.5  & 62.8 &  63.0  \\ 
                  &       Drop-GCN &  -  &  -  & 81.9   &  70.5 &  80.9  & -  &  -  &  69.7  & 44.2 &  67.2  \\  
                  &                  GCN & -  &  -  &  -  &  -  &  80.7&  -  &  -  &  -  &  -  & 66.9 \\ \hline
\multirow{5}{*}{Citeseer} &  S-BGCN-T-K & \textbf{84.8} &55.2&78.4 & 55.1 &  74.0& \textbf{86.8} & 54.1 & 80.8 & 55.8 & 74.0  \\ 
                  &  EDL-GCN &                 78.4 &59.4&- & - &  69.1&         79.8 & 57.3 & - & - & 69.0  \\ 
                  & DPN-GCN &   - &  -  & 68.3 & 72.2 &  69.5 &   -  & -& 68.5 & 72.1   & 70.3  \\ 
                  &                  Drop-GCN &   - &  -  & 72.3 & 61.4 &  70.6 &   -  & -& 73.5 & 60.8   & 70.0  \\ 
                  &                  GCN &  -  &  -  &  -  &  -  &  70.8 &  -  &  -  &  -  & -   & 70.2  \\ \hline
\multirow{5}{*}{Pubmed} &  S-BGCN-T-K & \textbf{74.6} &67.9& 71.8 & 59.2 & 72.2& \textbf{69.6} & 52.9 & 63.6& 44.0   &56.5  \\
                  &    EDL-GCN & 71.5 &68.2& - & - &  70.5& 65.3 & 53.1 & -& -   & 55.0  \\
                   & DPN-GCN &  -  &  -  &  63.5 & 63.7 &   63.5&  -  &  -& 50.7 & 53.9   & 51.1  \\ 
                  &                  Drop-GCN &  -  &  -  &  68.7 & 60.8 &   66.7&  -  &  -& 59.7 & 46.7   & 54.8  \\ 
                  &                  GCN &  -  &  -  &   - & -   &  68.3&  -  &   -   &  -  & -&55.3 \\ \hline
\multirow{5}{*}{Amazon Photo} &  S-BGCN-T-K & \textbf{93.4} & 76.4& 91.4& 32.2  & 91.4&         \textbf{ 94.8} & 68.0 & 92.3& 42.3   & 92.5  \\
                  &     EDL-GCN & 63.4 & 78.1& - & - & 79.2&          66.2 & 74.8 & -&-   & 81.2  \\
                  & DPN-GCN &  -  &  -  &  83.6 &  83.6 &   83.6&  -  &  -& 82.6 & 82.4 &   82.5 \\ 
                  &                  Drop-GCN &  -  &  -  &  84.5 & 58.7 &   84.3&  -  &  -& 87.0 & 57.7   &86.9 \\ 
                  &                  GCN &  -  &  -  &   - & -   &   84.4&  -  &   -   &  -  & -&87.0 \\ \hline
\multirow{5}{*}{Amazon Computer} &  S-BGCN-T-K & \textbf{82.3} & 76.6& 80.9& 55.4  & 80.9& \textbf{70.5} & 52.8 & 60.9& 35.9 & 60.6  \\
                  &    EDL-GCN & 53.2 & 70.1& - & - &  70.0&    33.2 & 43.9 & -& - & 45.7  \\
                   &                  DPN-GCN &  -  &  -  &  77.6 & 77.7 & 77.7 &  -  &  -& 50.8 & 51.2 & 51.0 \\ 
                  &                  Drop-GCN &  -  &  -  &  74.4 & 70.5 &  74.3&  -  &  -& 50.0 & 46.7   &  49.8 \\ 
                  &                  GCN &  -  &  -  &   - & -   &   74.0&    - & -   &  -  & -&48.7 \\ \hline
\multirow{5}{*}{Coauthor Physics} &  S-BGCN-T-K & \textbf{91.3} & 87.6& 89.7& 61.8 & 89.8& \textbf{72.2} & 56.6 & 68.1& 25.9 & 67.9  \\
                  &     EDL-GCN & 88.2 & 85.8& - & -  & 87.6&    67.1 & 51.2 & -&- & 62.1  \\
                   &                  DPN-GCN &  -  &  -  & 85.5 &85.6 &   85.5&  -  &  -& 59.8 & 60.2 &   59.8 \\ 
                  &                  Drop-GCN &  -  &  -  &  89.2 & 78.4 &   89.3&  -  &  -& 66.6 & 37.1   &66.5 \\ 
                  &                  GCN &  -  &  -  &   -    &  - & 89.1&  -  &  -   &  -  & -&64.0 \\ \hline
\end{tabular}
\begin{tablenotes}\footnotesize
\centering
\item[*] Va.: Vacuity, Dis.: Dissonance, Al.: Aleatoric, Ep.: Epistemic, En.: Entropy 
\end{tablenotes}
\label{Table: AUROC_AUPR:ood}
\end{table*}

\noindent {\bf OOD Detection.} This experiment involves detecting whether an input example is out-of-distribution (OOD) given an estimate of uncertainty. For semi-supervised node classification, we randomly selected one to four categories as OOD categories and trained the models based on training nodes of the other categories. Due to the space constraint, the experimental setup for the OOD detection is detailed in Appendix B.3. 

In Table~\ref{Table: AUROC_AUPR:ood}, across six network datasets, our vacuity-based detection significantly outperformed the other competitive methods, exceeding the performance of the epistemic uncertainty and other type of uncertainties. This demonstrates that the vacuity-based model is more effective than other uncertainty estimates-based counterparts in increasing OOD detection. We observed the following performance order: ${\tt Vacuity} > {\tt Entropy} \approx {\tt Aleatoric} > {\tt Epistemic} \approx {\tt Dissonance}$, which is consistent with the theoretical results as shown in Theorem~\ref{theorem_un}.

\noindent {\bf Ablation Study.}
We conducted additional experiments (see Table~\ref{Table: Ablation_M} and Table~\ref{Table: Ablation_O}) in order to demonstrate the contributions of the key technical components, including GKDE, Teacher Network, and subjective Bayesian framework.  The key findings obtained from this experiment are: (1) GKDE can enhance the OOD detection (i.e., 30\% increase with vacuity), which is consistent with our theoretical proof about the outperformance of GKDE in uncertainty estimation, i.e., OOD nodes have a higher vacuity than other nodes; and (2) the Teacher Network can further improve the node classification accuracy.

\begin{table*}[th!]
\footnotesize
\caption{Ablation experiment on AUROC and AUPR for the Misclassification Detection.}
\label{Table: Ablation_M}
\vspace{1mm}
\centering
\begin{tabular}{c||c|ccccc|ccccc|c}
\hline
\multirow{2}{*}{Data} & \multirow{2}{*}{Model} & \multicolumn{5}{c|}{AUROC} & \multicolumn{5}{c}{AUPR} \\  
                  &                   & Va.\tnote{*}& Dis. & Al. & Ep. &En. & Va. & Dis. &  Al. & Ep.  &En. & Acc \\ \hline
\multirow{4}{*}{Cora} & S-BGCN-T-K & 70.6 & 82.4 & 75.3 & 68.8 & 77.7&  90.3  & \textbf{95.4} & 92.4 & 87.8 &93.4 & 82.0 \\ 
 & S-BGCN-T &  70.8  &  \textbf{82.5}  &75.3  & 68.9 & 77.8  &90.4  &  \textbf{95.4}  &  92.6  & 88.0 &93.4 & \textbf{82.2} \\
                  & S-BGCN &  69.8  &  81.4  & 73.9   &  66.7  & 76.9  & 89.4  &  94.3  &  92.3  & 88.0 &93.1 & 81.2 \\  
                   & S-GCN &  70.2  &  81.5  & -   & -  & 76.9  & 90.0  &  94.6  &  -  & - &93.6 & 81.5 \\
                  \hline
\multirow{4}{*}{Citeseer} &  S-BGCN-T-K & 65.4& \textbf{74.0}& 67.2& 60.7&  70.0& 79.8 & \textbf{85.6} & 82.2 & 75.2 &83.5 & 71.0 \\ 
                  &  S-BGCN-T & 65.4& 73.9& 67.1& 60.7& 70.1& 79.6 & 85.5 & 82.1 & 75.2 & 83.5 &\textbf{71.3} \\ 
                  &     S-BGCN &   63.9 &  72.1  & 66.1 & 58.9  & 69.2& 78.4 & 83.8   &80.6 &75.6&82.3&70.6  \\ 
                  &                  S-GCN &  64.9  &  71.9  &  -  & - & 69.4 &  79.5  & 84.2  &  -  & -   & 82.5& 71.0  \\ \hline
\multirow{4}{*}{Pubmed} &  S-BGCN-T-K & 63.1 & \textbf{69.9} & 66.5 & 65.3& 68.1& 85.6 & 90.8 & 88.8& 86.1   &89.2 & \textbf{79.3} \\
                  &     S-BGCN-T & 63.2 & \textbf{69.9} & 66.6 & 65.3 & 64.8&  85.6 & \textbf{90.9} & 88.9& 86.0   &89.3 & 79.2 \\
                  &                  S-BGCN &  62.7  &   68.1 & 66.1 &  64.4 & 68.0&  85.4  &  90.5& 88.6 & 85.6   &  89.2&78.8  \\ 
                  &                  S-GCN &  62.9  &   69.5 & -   &  - & 68.0&  85.3  &  90.4 & -   &  -  & 89.2&79.1 \\ \hline
 &  S-BGCN-T-K & 66.0 & 89.3& 83.0& 83.4  & 83.2&   95.4 & \textbf{98.9}& 98.4& 98.1 & 98.4 & 92.0\\
   Amazon    &     S-BGCN-T &  66.1 & 89.3& 83.1& 83.5  & 83.3&   95.6 & 99.0& 98.4& 98.2 & 98.4&\textbf{92.3}\\
   Photo     &                  S-BGCN &   68.6  & \textbf{ 93.6} & 90.6 &  83.6 & 90.6&  90.4  &  98.1& 97.3 & 95.8 &  97.3&81.0 \\ 
                  &                  S-GCN &    -  &   - & -   &  - & 86.7&  -  &   - & -   &  -  & -&98.4 \\ \hline
 &  S-BGCN-T-K & 65.0 & 87.8& 83.3& 79.6  & 83.6& 89.4 & 96.3 & 95.0& 94.2 &  95.0 & 84.0\\
    Amazon      &     S-BGCN-T & 65.2 & 88.0& 83.4& 79.7  & 83.6& 89.4 & \textbf{96.5} & 95.0& 94.5 & 95.1 &\textbf{84.1} \\
    Computer     &                  S-BGCN &  63.7  &  \textbf{89.1} & 84.3 &  76.1 & 84.4&  84.9  &  95.7& 93.9 & 91.4   &  93.9&76.1 \\ 
                  &                  S-GCN &   -  &   - & -   &  - & 81.5&  -  &   - & -   &  -  & -&95.2 \\ \hline
      &  S-BGCN-T-K & 80.2 & 91.4& 87.5& 81.7 & 87.6& 98.3 &\textbf{99.4} & 99.0& 98.4 & 98.9 & 93.0\\
   Coauthor     &     S-BGCN-T & 80.4 & \textbf{91.5}& 87.6& 81.7  & 87.6& 98.3 & \textbf{99.4} & 99.0& 98.6 &  99.0 & \textbf{93.2} \\
    Physics     &                  S-BGCN &  79.6   &  90.5 & 86.3 &  81.2 & 86.4&  98.0  &  99.2& 98.8 & 98.3   &  98.8&92.9 \\ 
                  &                  S-GCN &  89.1   &  89.0 & -   &  - & 89.2&  99.0  & 99.0 & -   &  -  & 99.0&92.9 \\ \hline
\end{tabular}
\begin{tablenotes}\footnotesize
\centering
\item[*] Va.: Vacuity, Dis.: Dissonance, Al.: Aleatoric, Ep.: Epistemic, En.: Entropy 
\end{tablenotes}
\end{table*}

\noindent {\bf Time Complexity Analysis}.
S-BGCN has a similar time complexity with GCN while S-BGCN-T has the double complexity of GCN. For a given network where $|\mathbb{V}|$ is the number of nodes, $|\mathbb{E}|$ is the number of edges, $C$ is the number of dimensions of the input feature vector for every node, $F$ is the number of features for the output layer, and $M$ is Monte Carlo sampling times. 

\begin{table*}[h]
\footnotesize
\caption{Big-O time complexity of our method and baseline GCN.}
\centering
\small
  \begin{tabular}{lccccc}
    \toprule
    Dataset & GCN & S-GCN  & S-BGCN & S-BGCN-T & S-BGCN-T-K   \\
    \midrule
    Time Complexity (Train) & $O(|\mathbb{E}|CF)$ &  $O(|\mathbb{E}|CF)$ & $O(2|\mathbb{E}|CF)$ & $O(2|\mathbb{E}|CF)$  &$O(2|\mathbb{E}|CF)$   \\
    Time Complexity (Test)  & $O(|\mathbb{E}|CF)$ &  $O(|\mathbb{E}|CF)$ & $O(M|\mathbb{E}|CF)$ & $O(M|\mathbb{E}|CF)$ &$O(M|\mathbb{E}|CF)$  \\
    
    \bottomrule
  \end{tabular}
\label{tab:complexity}
\end{table*}

\begin{table*}[th!]
\footnotesize
\caption{Ablation experiment on  AUROC and AUPR for the OOD Detection.}
\centering
\begin{tabular}{c||c|ccccc|ccccc}
\hline
\multirow{2}{*}{Data} & \multirow{2}{*}{Model} & \multicolumn{5}{c|}{AUROC} & \multicolumn{5}{c}{AUPR} \\  
                  &                   & Va.\tnote{*}& Dis. & Al. & Ep. &En. & Va. & Dis. &  Al. & Ep.  &En. \\ \hline
\multirow{4}{*}{Cora} & S-BGCN-T-K & \textbf{87.6} & 75.5 & 85.5 & 70.8  & 84.8&  \textbf{78.4} & 49.0 & 75.3 & 44.5 &  73.1  \\ 
                  &    S-BGCN-T &            84.5 & 81.2 & 83.5 & 71.8 & 83.5&    74.4 & 53.4 & 75.8 & 46.8 & 71.7  \\ 
                  &                  S-BGCN &  76.3  &  79.3  & 81.5   &  70.5  & 80.6  & 61.3  & 55.8  &  68.9  & 44.2 &  65.3  \\  
                  &                  S-GCN & 75.0  &  78.2  &  -  &  -    & 79.4&  60.1  &  54.5  &  -  &  -  & 65.3 \\ \hline
\multirow{4}{*}{Citeseer} &  S-BGCN-T-K & \textbf{84.8} &55.2&78.4 & 55.1 & 74.0& \textbf{86.8} & 54.1 & 80.8 & 55.8 & 74.0  \\ 
                  &  S-BGCN-T &                 78.6 &59.6&73.9 & 56.1 & 69.3&         79.8 & 57.4 & 76.4 & 57.8 & 69.3  \\ 
                  &                  S-BGCN &   72.7 &  63.9  & 72.4 & 61.4 &  70.5 &   73.0  & 62.7& 74.5 & 60.8   & 71.6  \\ 
                  &                  SGCN &  72.0 &  62.8  &  -  &  -   & 70.0 &  71.4  &  61.3  &  -  & -   & 70.5  \\ \hline
\multirow{4}{*}{Pubmed} &  S-BGCN-T-K & \textbf{74.6} &67.9& 71.8 & 59.2& 72.2& \textbf{69.6} & 52.9 & 63.6& 44.0   &56.5  \\
                  &     S-BGCN-T & 71.8 &68.6& 70.0 & 60.1 & 70.8&          65.7 & 53.9 & 61.8& 46.0   & 55.1  \\
                  &                  S-BGCN &  70.8 &  68.2  &  70.3 & 60.8 & 68.0& 65.4  &  53.2& 62.8 & 46.7   &  55.4  \\ 
                  &                  S-GCN &    71.4  &   68.8 & -   &  - & 69.7&  66.3  &   54.9 & -   &  -  &57.5 \\ \hline
 &  S-BGCN-T-K & \textbf{93.4} & 76.4& 91.4& 32.2  & 91.4&         \textbf{ 94.8} & 68.0 & 92.3& 42.3   & 92.5   \\
    Amazon     &     S-BGCN-T & 64.0 & 77.5& 79.9 & 52.6 & 79.8&          67.0 & 75.3 & 82.0& 53.7   & 81.9  \\
    Photo     & S-BGCN & 63.0 & 76.6& 79.8 & 52.7 & 79.7&          66.5 & 75.1 & 82.1& 53.9   & 81.7 \\ 
                  & S-GCN & 64.0 & 77.1& - & - & 79.6&          67.0 & 74.9 & -& -   & 81.6 \\ \hline
 &  S-BGCN-T-K & \textbf{82.3} & 76.6& 80.9& 55.4 & 80.9& \textbf{70.5} & 52.8 & 60.9& 35.9 &  60.6  \\
   Amazon       &     S-BGCN-T & 53.7 & 70.5& 70.4 & 69.9  & 70.1&    33.6 & 43.9 & 46.0& 46.8 & 45.9  \\
  Computer       &                  S-BGCN &  56.9  &  75.3  &  74.1 & 73.7 &  74.1&33.7  &  46.2& 48.3 & 45.6   &48.3 \\ 
                  &                  S-GCN &  56.9  &  75.3  & -&  - & 74.2&  33.7  &   46.2 & -   &  - &48.3 \\ \hline
 &  S-BGCN-T-K & \textbf{91.3} & 87.6& 89.7& 61.8  & 89.8& \textbf{72.2} & 56.6 & 68.1& 25.9 & 67.9  \\
     Coauthor     &     S-BGCN-T & 88.7 & 86.0& 87.9 & 70.2  & 87.8&    67.4 & 51.9 & 64.6& 29.4 & 62.4  \\
    Physics      &                  S-BGCN &  89.1  &  87.1  &  89.5 & 78.3 &   89.5&  66.1  &  49.2&64.6 & 35.6   & 64.3 \\ 
                  &                  S-GCN &  89.1  &    87.0 & -   &  - & 89.4&  -66.2  &   49.2 & -   &  -  & 64.3 \\ \hline
\end{tabular}
\begin{tablenotes}\footnotesize
\centering
\item[*] Va.: Vacuity, Dis.: Dissonance, Al.: Aleatoric, Ep.: Epistemic, D.En.: Differential Entropy, En.: Entropy 
\end{tablenotes}
\label{Table: Ablation_O}
\end{table*}

\noindent {\bf Compare with Bayesian GCN baseline}.
Compare with a (Bayesian) GCN baseline, Dropout+DropEdge~\cite{rong2019dropedge}. As shown in the table~\ref{AUPR:dropedge} below, our proposed method performed better than Dropout+DropEdge on the Cora and Citeer datasets for misclassificaiton detection. A similar trend was observed for OOD detection.

\begin{table*}[h!]
\small
\caption{Compare with DropEdge on Misclassification Detection .}
\vspace{-2mm}
\centering
\begin{tabular}{c||c|ccccc|ccccc}
\hline
\multirow{2}{*}{Dataset} & \multirow{2}{*}{Model} & \multicolumn{5}{c|}{AUROC} & \multicolumn{5}{c}{AUPR}  \\  
                  &                   & Va.\tnote{*}& Dis. & Al. & Ep. & En. & Va. & Dis. &  Al. & Ep. &En. \\ \hline
\multirow{2}{*}{Cora} & S-BGCN-T-K & 70.6 & \textbf{82.4} & 75.3 & 68.8 & 77.7&  90.3  & \textbf{95.4} & 92.4 & 87.8 &93.4 \\ 
                  & DropEdge &  -  &  -  & 76.6   &  56.1  & 76.6  & -  &  -  &  93.2  & 85.4 &93.2 \\ 
                  \hline
\multirow{2}{*}{Citeseer} &  S-BGCN-T-K & 65.4& \textbf{74.0}& 67.2& 60.7&  70.0& 79.8 & \textbf{85.6} & 82.2 & 75.2 &83.5  \\ 
                  &                  DropEdge &   - &  -  & 71.1 & 51.2 &   71.1 &   -  & -& 84.0 & 70.3  &84.0   \\ 
                  \hline
\end{tabular}
\begin{tablenotes}\footnotesize
\centering
\item[*] Va.: Vacuity, Dis.: Dissonance, Al.: Aleatoric, Ep.: Epistemic, En.: Entropy 
\end{tablenotes}
\label{AUPR:dropedge}
\end{table*}

\subsection{Why is Epistemic Uncertainty Less Effective than Vacuity?}
Although epistemic uncertainty is known to be effective to improve OOD detection~\citep{gal2016dropout, kendall2017uncertainties} in computer vision applications, our results demonstrate it is less effective than our vacuity-based approach.  The first potential reason is that epistemic uncertainty is always smaller than vacuity (From Theorem~\ref{theorem_un}), which potentially indicates that epistemic may capture less information related to OOD. Another potential reason is that the previous success of epistemic uncertainty for OOD detection is limited to supervised learning in computer vision applications, but its  effectiveness for OOD detection was not sufficiently validated in semi-supervised learning tasks.  Recall that epistemic uncertainty (i.e., model uncertainty) is calculated based on mutual information (see Eq.~\eqref{eq:epistemic}).  In a semi-supervised setting, the features of unlabeled nodes are also fed to a model for the training process to provide the model with high confidence in its output.  For example, the model output $P(\y|A, \rr;\theta)$ would not change too much even with differently sampled parameters $\bm{\theta}$, i.e., $P(\y|A, \rr;\theta^{(i)})\approx P(\y|A, \rr;\theta^{(j)})$, which result in a low epistemic uncertainty.  

To back up our conclusion, design an image classification experiment 
based on MC-Drop\cite{gal2016dropout} method to do the following experiment: 1) supervised learning on the MNIST dataset with 50 labeled images; 2) semi-supervised learning (SSL) on the MNIST dataset with 50 labeled images and 49950 unlabeled images, while there are 50\% OOD images (24975 FashionMNIST images) in the unlabeled set. For both experiments, we test the epistemic uncertainty on 49950 unlabeled set (50\% In-distribution (ID) images and 50\% OOD images). We conduct the experiment based on three popular SSL methods, VAT~\cite{miyato2018virtual}, Mean Teacher~\cite{tarvainen2017mean}, and pseudo label~\cite{lee2013pseudo}.
\begin{table*}[h]
\caption{Epistemic uncertainty for semi-supervised image classification.}
\centering
\small
  \begin{tabular}{lcccc}
    \toprule
    Epistemic & \textbf{Supervised } & \textbf{VAT}  & \textbf{Mean Teacher} & \textbf{Pseudo Label}\\
    \midrule
    \textbf{In-Distribution} & 0.140 &  \textbf{0.116}& \textbf{0.105} & \textbf{0.041}  \\
    \textbf{Out-of-Distribution}  & \textbf{0.249} & 0.049 & 0.076 & 0.020 \\
    \bottomrule
  \end{tabular}
\label{tab:ssl}
\end{table*}
Table~\ref{tab:ssl} shows the average epistemic uncertainty value for in-distribution samples and OOD samples. The result shows the same pattern with~\cite{kendall2017uncertainties, kendall2015bayesian} in a supervised setting, but an opposite pattern in a semi-supervised setting that low epistemic of OOD samples, which is less effective top detect OOD.
Note that the SSL setting is similar to our semi-supervised node classification setting, which feed the unlabeled sample to train the model.

\subsection{Graph Embedding Representations of Different Uncertainty Types} \label{subsec:uncertain-dist}

\begin{figure*}[th!]
  \centering
  \includegraphics[width=0.9\linewidth]{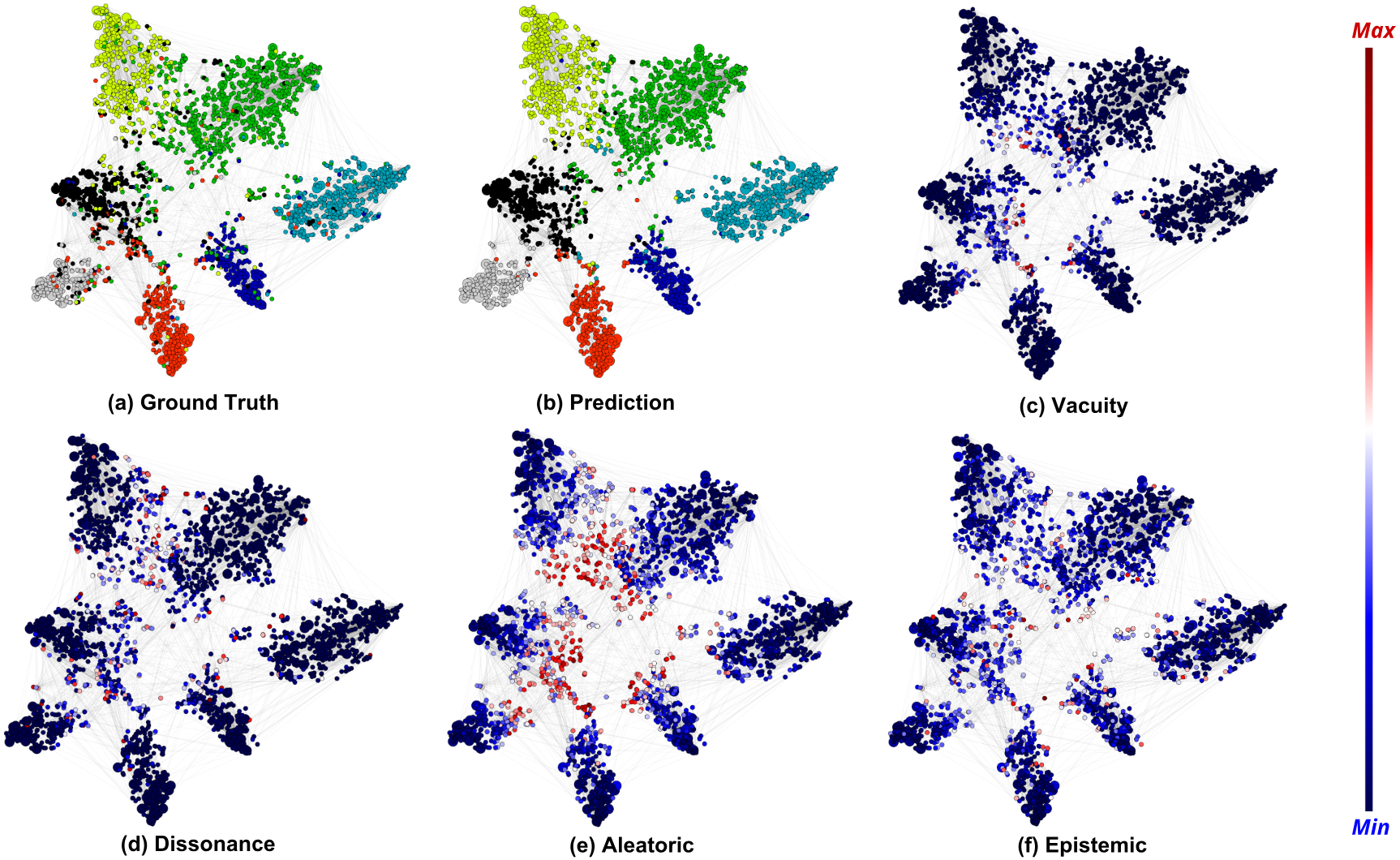}
  \small{
  \caption{Graph embedding representations of the Cora dataset for classes and the extent of uncertainty: (a) shows the representation of seven different classes; (b) shows our model prediction; and (c)-(f) present the extent of uncertainty for respective uncertainty types, including vacuity, dissonance, aleatoric, epistemic.}
  \label{fig:un_vis_cora}
\vspace{-3mm}
 }
\end{figure*}

\begin{figure*}[th!]
  \centering
  \includegraphics[width=0.9\linewidth]{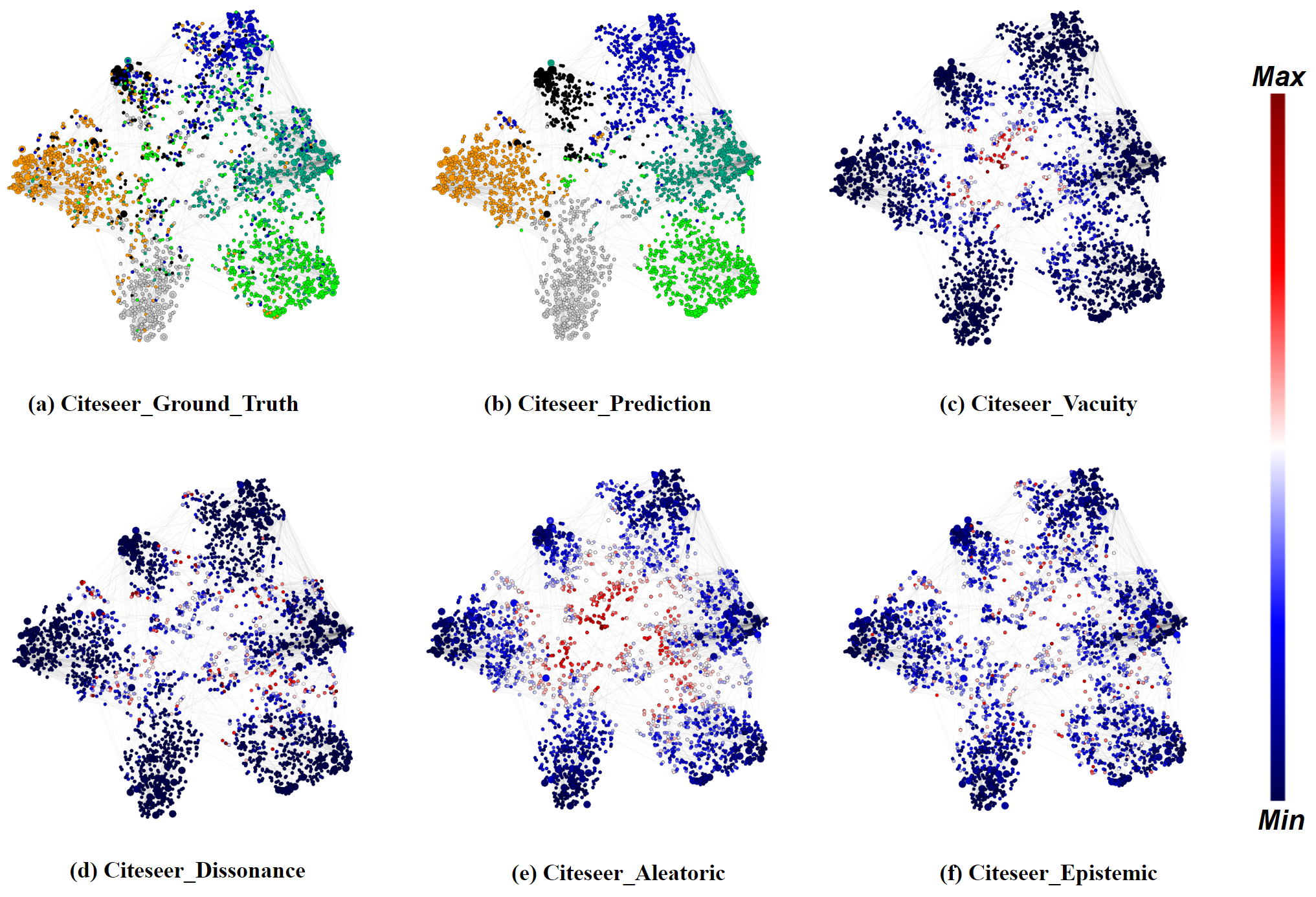}
  \small{
  \caption{Graph embedding representations of the Citeseer dataset for classes and the extent of uncertainty: (a) shows the representation of seven different classes, (b) shows our model prediction, and (c)-(f) present the extent of uncertainty for respective uncertainty types, including vacuity, dissonance, and aleatoric uncertainty, respectively.}
  \label{fig:un_vis}
 }
\end{figure*}

To better understand different uncertainty types, we used $t$-SNE ($t$-Distributed Stochastic Neighbor Embedding~\cite{maaten2008visualizing}) to represent the computed feature representations of a pre-trained BGCN-T model's first hidden layer on the Cora dataset and the Citeseer dataset. 

\noindent {\bf Seven Classes on Cora Dataset}: In Figure~\ref{fig:un_vis_cora}, (a) shows the representation of seven different classes, (b) shows our model prediction, and (c)-(f) present the extent of uncertainty for respective uncertainty types, including vacuity, dissonance, and aleatoric uncertainty, respectively.

\noindent {\bf Six Classes on Citeseer Dataset}: In Figure~\ref{fig:un_vis} (a), a node's color denotes a class on the Citeseer dataset where 6 different classes are shown in different colors. Figure~\ref{fig:un_vis} (b) is our prediction result.

\begin{figure*}[t!]
    \centering
    \begin{subfigure}[b]{0.44\textwidth}
        \centering
        \includegraphics[width=\linewidth]{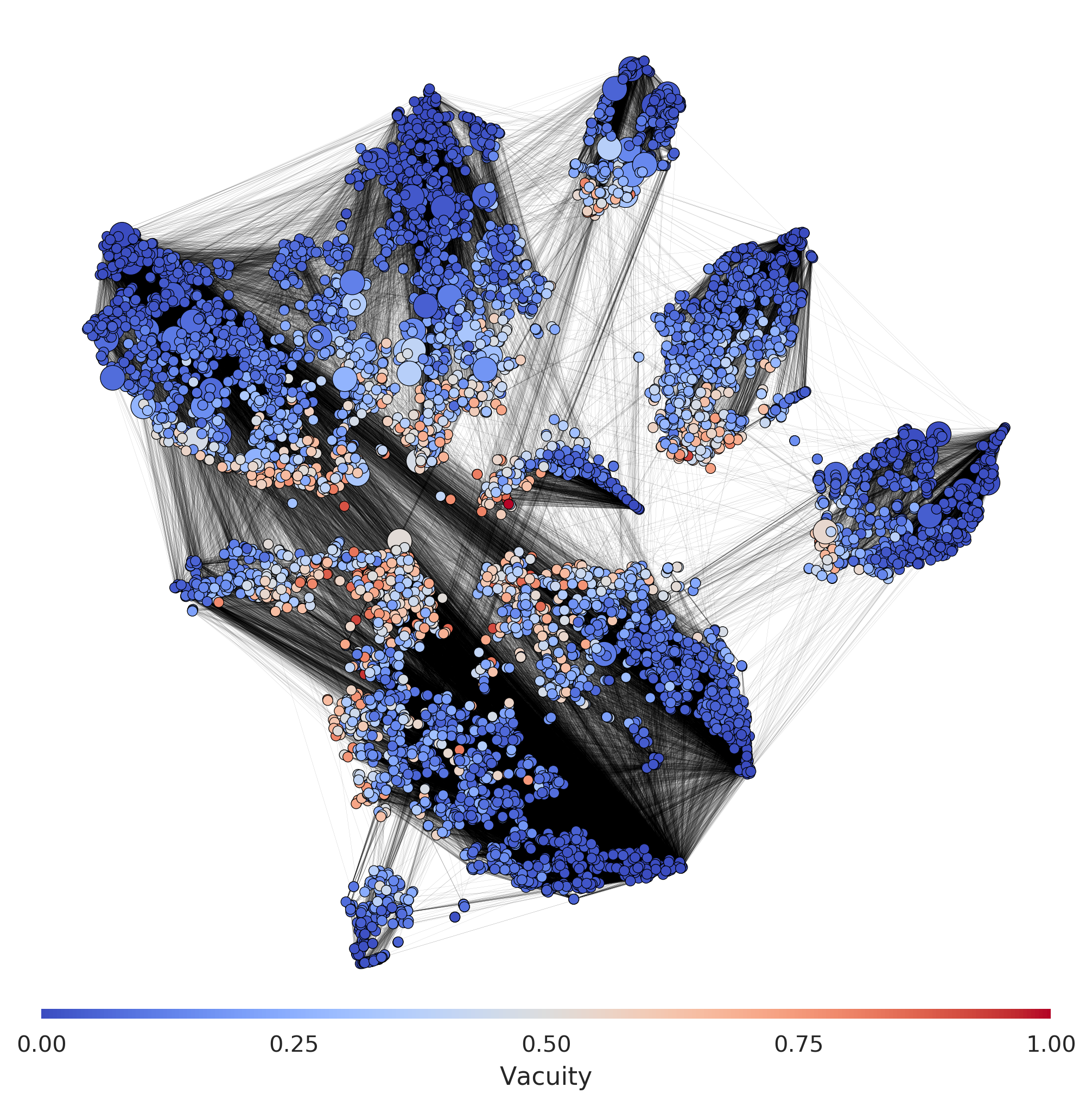}
        \caption{S-BGCN-T}
    \end{subfigure}
    \begin{subfigure}[b]{0.44\textwidth}
        \centering
        \includegraphics[width=\linewidth]{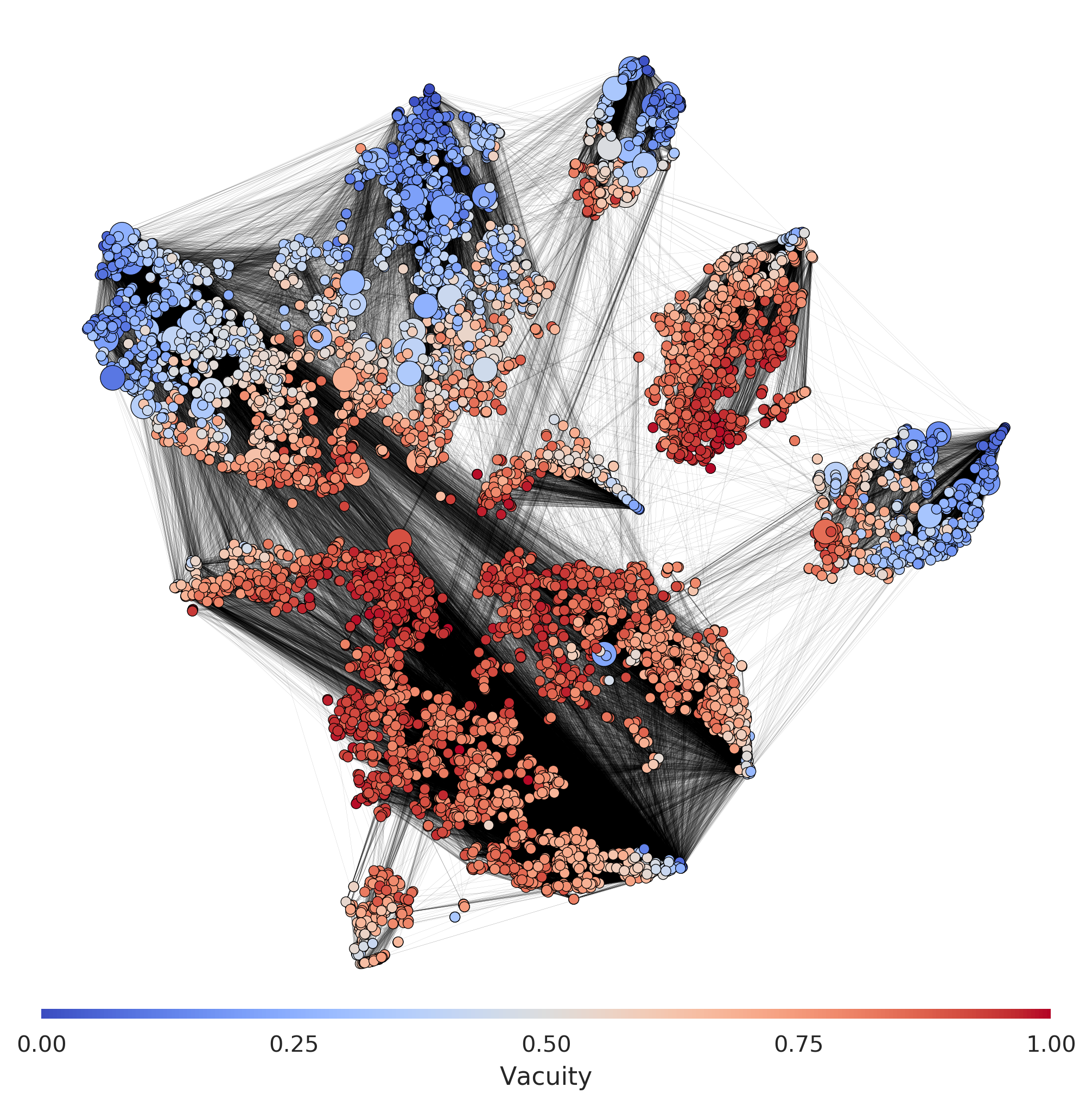}
        \caption{S-BGCN-T-K}
    \end{subfigure}
    \small{
    \caption{Graph embedding representations of the Amazon Photo dataset for the extent of vacuity uncertainty based on OOD detection experiment.}
    \label{fig:Ab_GKDE}
    }
\end{figure*}

\noindent {\bf Eight Classes on Amazon Photo Dataset}: In Figure~\ref{fig:Ab_GKDE}, a node's color denotes vacuity uncertainty value, and the big node represents the training node. These results are based on the OOD detection experiment. Comparing Figure~\ref{fig:Ab_GKDE} (a) and (b), we found that GKDE can indeed improve OOD detection.

For Figures~\ref{fig:un_vis} (c)-(f), the extent of uncertainty is presented where a blue color refers to the lowest uncertainty (i.e., minimum uncertainty) while a red color indicates the highest uncertainty (i.e., maximum uncertainty) based on the presented color bar. To examine the trends of the extent of uncertainty depending on either training nodes or test nodes, we draw training nodes as bigger circles than test nodes.  Overall we notice that most training nodes (shown as bigger circles) have low uncertainty (i.e., blue), which is reasonable because the training nodes are the ones that are already observed. Now we discuss the extent of uncertainty under each uncertainty type. 

\noindent {\bf Vacuity}: In Figure~\ref{fig:Ab_GKDE} (b), most training nodes show low uncertainty, we observe the majority of OOD nodes in the button cluster show high uncertainty as appeared in red.

\noindent {\bf Dissonance}: In Figure~\ref{fig:un_vis} (d), similar to vacuity, training nodes have low uncertainty. But unlike vacuity, test nodes are much less uncertain. Recall that dissonance represents the degree of conflicting evidence (i.e., a discrepancy between each class probability). However, in this dataset, we observe a fairly low level of dissonance and the obvious outperformance of Dissonance in node classification prediction.

\noindent {\bf Aleatoric uncertainty}: In Figure~\ref{fig:un_vis} (e), a lot of nodes show high uncertainty larger than 0.5 except for a small number of training nodes with low uncertainty. 

\noindent {\bf Epistemic uncertainty}: In Figure~\ref{fig:un_vis} (f), most nodes show very low epistemic uncertainty \feng{values} because uncertainty derived from model parameters can disappear as they are trained well.

\section{Conclusion} \label{sec:conclusion}
In this work, we proposed a multi-source uncertainty framework of GNNs for semi-supervised node classification. Our proposed framework provides an effective way of predicting node classification and out-of-distribution detection considering multiple types of uncertainty. We leveraged various types of uncertainty estimates from both DL and evidence/belief theory domains. Through our extensive experiments, we found that dissonance-based detection yielded the best performance on misclassification detection while vacuity-based detection performed the best for OOD detection, compared to other competitive counterparts.  In particular, it was noticeable that applying GKDE and the Teacher network further enhanced the accuracy of node classification and uncertainty estimates.

Although our method introduced in this Chapter achieves good performance in semi-supervised node classification, especially in an out-of-distribution setting where part of unlabeled nodes (\textit{e.g.}, OODs) does not belong to training (in-distribution) classes. However, we found that the OODs nodes may hurt the semi-supervised node classification performance for in-distribution nodes. A similar issue also occurs in transitional semi-supervised learning (SSL) problems. Semi-supervised learning algorithm's performance could degrade when the unlabeled set has out-of-distribution examples. To this end, in the following Chapter, we first study the impact of OOD data on SSL algorithms and propose a novel unified uncertainty-aware robust SSL framework to solve this issue.

\chapter{Uncertainty-Aware Robust Semi-Supervised Learning with Out of Distribution Data}
\label{chapter:3}
    \section{Introduction}\blfootnote{\copyright 2020 Arxiv. Reprinted, with permission, from Xujiang Zhao, Killamsetty Krishnateja, Rishabh Iyer, Feng Chen, ``Robust Semi-Supervised Learning with Out of Distribution Data", doi: 10.48550/arXiv.2010.03658.} 
Deep learning approaches have been shown to be successful on several supervised learning tasks, such as computer vision~\cite{szegedy2015going}, natural language processing~\cite{graves2013generating}, and speech recognition~\cite{graves2013speech}. However, these deep learning models are data-hungry and often require massive amounts of labeled examples to obtain good performance. Obtaining high-quality labeled examples can be very time-consuming and expensive, particularly where specialized skills are required in labeling (for example, in cancer detection on X-ray or CT-scan images). As a result, semi-supervised learning (SSL)~\cite{zhu2005semi} has emerged as a very promising direction, where the learning algorithms try to effectively utilize the large unlabeled set (in conjunction) with a relatively small labeled set. Several recent SSL algorithms have been proposed for deep learning and have shown great promise empirically. These include Entropy Minimization ~\cite{grandvalet2005semi}, pseudo-label based methods ~\cite{lee2013pseudo, arazo2019pseudo,berthelot2019mixmatch}
and consistency based methods ~\cite{sajjadi2016regularization,laine2016temporal,tarvainen2017mean,miyato2018virtual} to name a few.

\begin{figure}
 \centering
  \begin{subfigure}{0.48\textwidth}     
    \centering
    \includegraphics[width=\textwidth]{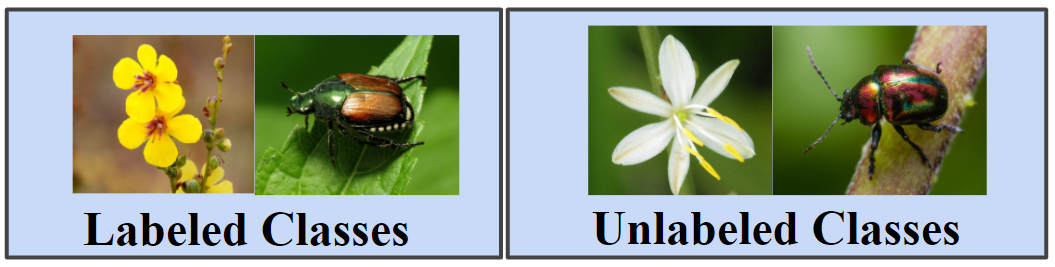}%
    \caption{Traditional semi-supervised learning }
    \label{fig:demo1}
  \end{subfigure}
  \begin{subfigure}{0.52\textwidth}        
    \centering
    \includegraphics[width=\textwidth]{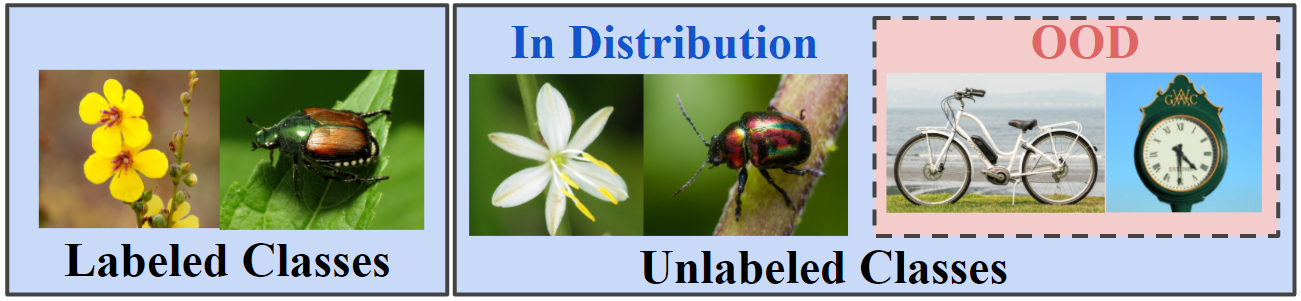}%
    \caption{Semi-supervised learning with OOD data}
    \label{fig:demo2}
  \end{subfigure}%
  \caption{(a) Traditional SSL. (b) SSL with OODs.}
  \label{fig:example_RSSL}
\end{figure}

Despite the positive results of the above SSL methods, they are designed with the assumption that both labeled and unlabeled sets have the same distribution. Fig~\ref{fig:example_RSSL} (a) shows an example of this. However, this assumption may not hold in many real-world applications, such as web classification~\cite{yang2011can} and medical diagnosis~\cite{yang2015maximum}, where some unlabeled examples are from novel classes unseen in the labeled data. For example, Fig~\ref{fig:example_RSSL} (b) illustrates an image classification scenario with out-of-distribution data, where the unlabeled dataset contains two novel classes (bicycle and clock) compared to the in-distribution classes (flower and beetle) in the labeled dataset. When the unlabeled set contains OOD examples (OODs), deep SSL performance can degrade substantially and is sometimes even worse than simple supervised learning (SL) approaches~\cite{oliver2018realistic}. Moreover, it is unreasonable to expect a human to go through and clean a large and massive unlabeled set in such cases.

A general approach to robust SSL against OODs is to assign a weight to each unlabeled example based on some criteria and minimize a weighted training or validation loss. In particular,~\cite{yan2016robust} applied a set of weak annotators to approximate the ground-truth labels as pseudo-labels to learn a robust SSL model.~\cite{chen2019distributionally} proposed a distributionally robust model that estimates a parametric weight function based on both the discrepancy and the consistency between the labeled data and the unlabeled data.
\cite{chen2020semi} (UASD) proposed to weigh the unlabeled examples based on an estimation of predictive uncertainty for each unlabeled example. The goal of UASD is to discard the potentially irrelevant samples having low confidence scores and estimate the parameters by optimizing a regularized training loss.
A state-of-the-art method~\cite{guo2020safe}, called DS3L, considers a shallow neural network to predict the weights of unlabeled examples and estimate the parameters of the neural network based on a clean labeled set via bi-level optimization. 
It is common to obtain a dataset composed of two parts, including a relatively small but accurately labeled set and a large but coarsely labeled set from inexpensive crowd-sourcing services or other noisy sources. 

There are three \textbf{main limitations} of DS3L and other methods as reviewed above. First, it lacks a study of the potential causes of the impact of OODs on SSL, and as a result, the interpretation of robust SSL methods becomes difficult. Second, existing robust SSL methods did not consider the negative impact of OODs on the utilization of BN in neural networks, and as a result, their robustness against OODs degrades significantly when a neural network includes BN layers. The utilization of BNs for deep SSL has an implicit assumption that the labeled and unlabeled examples follow a single or similar distribution, which is problematic when the unlabeled examples include OODs~\cite{ioffe2015batch}.
Third, the bi-level learning algorithm developed in DS3L relies on low-order approximations of the objective in the inner loop due to vanishing gradients or memory constraints. As a result of not using the high-order loss information, the learning performance of DS3L could be significantly degraded in some applications, as demonstrated in our experiments. Our main technical contributions over existing methods are summarized as follows:   

\noindent \textbf{The effect of OOD data points.} The first critical contribution of our work (Sec.~\ref{Impact_OOD}) is to analyze what kind of OOD unlabeled data points affect the performance of SSL algorithms. In particular, we observe that OOD samples lying close to the decision boundary have more influence on SSL performance than those far from the boundary. Furthermore, we observe that the OOD instances far from the decision-boundary (faraway  OODs) can degrade SSL performance substantially if the model contains a batch normalization (BN) layer. The last observation makes sense logically as well since the batch normalization heavily depends on the mean and variance of each batch's data points, which can be significantly different for OOD points that came from very different distributions. 
We find these observations about OOD points consistent across experiments on several synthetic and real-world datasets.

\noindent \textbf{Uncertainty-Aware Robust SSL Framework.} 
To address the above causes, we first proposed a simple modification to BN, called weighted batch normalization (WBN), to improve BN's robustness against OODs. We demonstrate that W-BN can learn a better representation with good statistics estimation via a theoretical analysis.
Finally, we proposed a unified, uncertainty-aware robust SSL approach to improve many existing SSL algorithms' robustness by learning to assign weights to unlabeled examples based on their meta-gradient directions derived from a loss function on a clean validation set. Different from safe SSL~\citep{guo2020safe}, we directly treat the weights as hyperparameters instead of the output of a parametric function. The effectiveness of this strategy has been well demonstrated in robust supervised learning problems against noise labels~\citep{ren2018learning}. Furthermore, treating the weights as hyperparameters enables the use of an implicit differentiation-based approach in addition to the meta-approximation algorithm considered in \citep{guo2020safe}. In addition, we proposed two efficient bi-level algorithms for our proposed robust SSL approach, including meta-approximation and implicit-differentiation based, that have different tradeoffs on computational efficiency and accuracy. We designed the first algorithm based on lower-order approximations of the objective in the meta optimization's inner loop. We designed the second algorithm based on higher-order approximations of the objective and is scalable to a higher number of inner optimization steps to learn a massive amount of weight parameters, but is less efficient than the first algorithm. 

\noindent \textbf{Speeding up the re-weighting algorithms.} A critical issue with the current approaches ~\cite{ren2018learning,shu2019meta,guo2020safe} for bi-level optimization is that they are around 3x slower (even after using a one-step gradient approximation) compared to the original training(in this case, the base SSL algorithm). Whereas more fancy and higher-order approaches like implicit differentiation~\cite{lorraine2020optimizing} are even slower. We show that by adopting some simple tricks like just considering the last layer of the network in the inner one-step optimization (while updating the hyper-parameters) and performing the weight update step, only every $L$ epochs can significantly speedup the reweighting while not significantly losing on its gains. In particular, we see that \emph{we can bring down the run-time from $3\times$ to $1.2\times$, thereby significantly improving the experimental turn around times and reducing cost, energy, and time requirements of the reweighting algorithms.} 

\noindent \textbf{Comprehensive experiments. } We conduct extensive experiments on synthetic and real-world datasets. The results demonstrate that our weighted robust SSL approach significantly outperforms existing robust approaches (L2RW, MWN, Safe-SSL, and UASD) on four representative SSL algorithms. We also perform an ablation study to demonstrate which components of our approach are most important for its success. 
Finally, we show the effect of using last layer gradients and infrequent weight updates on both accuracy and time speedups.

\section{Semi-Supervised Learning (SSL)}
\subsection{Notations}
Given a training set with a  labeled set of examples $\mathcal{D} = \{\x_i, y_i\}_{i=1}^n$ and an unlabeled set of examples $\mathcal{U} = \{\x_j\}_{j=1}^m$.
Vectors are denoted by lower case bold face letters, \textit{e.g.}, weight vector $\bm{w} \in [0, 1]^{|\mathcal{U}|}$ and uncertainty vector ${\bm u} \in [0, 1]^{|\mathcal{U}|}$ where their \textit{i}-th entries are $w_i, u_i$. Scalars are denoted by lowercase italic letters, \textit{e.g.}, trade-off parameter $\lambda \in \mathbb{R}$. Matrices are denoted by capital italic letters. Some important notations are listed in Table~\ref{table:notation_3}

\begin{table*}[h]
\caption{Important notations and corresponding descriptions.}
\centering
  \begin{tabular}{l|l}
    \toprule
      \textbf{Notations} & \textbf{Descriptions}  \\
    \midrule
    $\mathcal{D}$   & Labeled set  \\
    $\mathcal{U}$   &  Unlabeled set   \\
    $\mathcal{V}$   &  Validation set (labeled)   \\
    $\x_i$   &  Feature vector of sample $i$   \\
    $y_i$ & Class label of sample $i$ \\
    ${\bm \theta}$ & model parameters \\
    $f(\cdot)$ & semi-supervised learning model function\\
    $\mathcal{L}_V(\cdot)$ & Validation loss\\
    $\mathcal{L}_T(\cdot)$ & Training loss\\
    $l(\cdot)$ & Loss function for labeled data\\
    $r(\cdot)$ & Loss function for unlabeled data \\
    $\text{Un}(\cdot)$ & Uncertainty regularization term \\
    $J$ & Inner loop gradinet steps \\
    $P$ & Neumann series approximation parameters \\
    \bottomrule
  \end{tabular}
   \label{table:notation_3}
\end{table*}

Given a training set with a  labeled set of examples $\mathcal{D} = \{\x_i, y_i\}_{i=1}^n$ and an unlabeled set of examples $\mathcal{U} = \{\x_j\}_{j=1}^m$. For any classifier model $f(\x, \theta)$ used in SSL, where $\x\in \mathbb{R}^C$ is the input data, and $\theta$ refers to the parameters of the classifier model. The loss functions of many existing methods can be formulated in the following general form:
\vspace{-1mm}
\begin{eqnarray}
\sum\nolimits_{(\x_i, y_i)\in \mathcal{D}} l(f(\x_i, \theta), y_i) + \sum\nolimits_{x_j\in \mathcal{U}} r(f(\x_j, \theta)), 
\label{ssl_loss}
\end{eqnarray}
where $l(\cdot)$ is the loss function for labeled data (such as cross-entropy), and $r(\cdot)$ is the loss function (regularization function) on the unlabeled set.  The goal of SSL methods is to design an efficient regularization function to leverage the model performance information on the unlabeled dataset for effective training.
Pseudo-labeling~\cite{lee2013pseudo} uses a standard supervised loss function on an unlabeled dataset using “pseudo-labels” as a target label as a regularizer.
$\Pi$-Model~\cite{DBLP:conf/iclr/LaineA17, sajjadi2016regularization} designed a consistency-based regularization function that pushes the distance between
the prediction for an unlabeled sample and its stochastic perturbation (e.g., data augmentation or dropout~\cite{srivastava2014dropout}) to a small value. 
Mean Teacher~\cite{tarvainen2017mean} proposed to obtain a more stable target output $f(x, \theta)$ for unlabeled set by setting the target via an exponential moving average of parameters from previous training steps.
Instead of designing a stochastic $f(x, \theta)$, Virtual Adversarial Training (VAT)~\cite{miyato2018virtual} proposed to approximate a tiny perturbation to unlabeled samples that affect the output of the prediction function most.  MixMatch~\cite{berthelot2019mixmatch}, UDA~\cite{xie2019unsupervised}, and Fix-Match~\cite{sohn2020fixmatch} choose the pseudo-labels based on predictions of augmented samples, such as shifts, cropping, image flipping, weak and strong augmentation, and mix-up~\cite{zhang2017mixup} to design the regularization functions. However, the performance of most existing SSL can degrade substantially when the unlabeled dataset contains OOD examples~\cite{oliver2018realistic}.

\section{Impact of OOD on SSL Performance}\label{Impact_OOD}

In this section, we provide a systematic analysis of the impact of OODs for many popular SSL algorithms, such as Pseudo-Label(PL)~\cite{lee2013pseudo}, $\Pi$-Model~\cite{laine2016temporal}, 
Mean Teacher(MT)~\cite{tarvainen2017mean}, and Virtual Adversarial Training (VAT)~\cite{miyato2018virtual}. We illustrate the discoveries using the following synthetic and real-world datasets. While we mainly focus on VAT as the choice of the SSL algorithm, the observations extend to other SSL algorithms as well. 

\noindent {\bf Synthetic dataset.} We considered two moons dataset (red points are labeled data, gray circle points are in-distribution (ID) unlabeled data) with OOD (yellow triangle points)  points in three different scenarios that can exist in real-world, 1) Faraway OOD scenario where the OOD points exist far from decision boundary; 2) Boundary OOD scenario where the OOD points occur close to decision boundary; 3) Mixed OOD scenario where OOD points exist both far and close to the decision boundary, as shown in Fig~\ref{fig: impact_ood}.

\begin{figure*}[!t]
    \centering
    \begin{subfigure}[b]{0.32\textwidth}
        \centering
        \includegraphics[width=\linewidth]{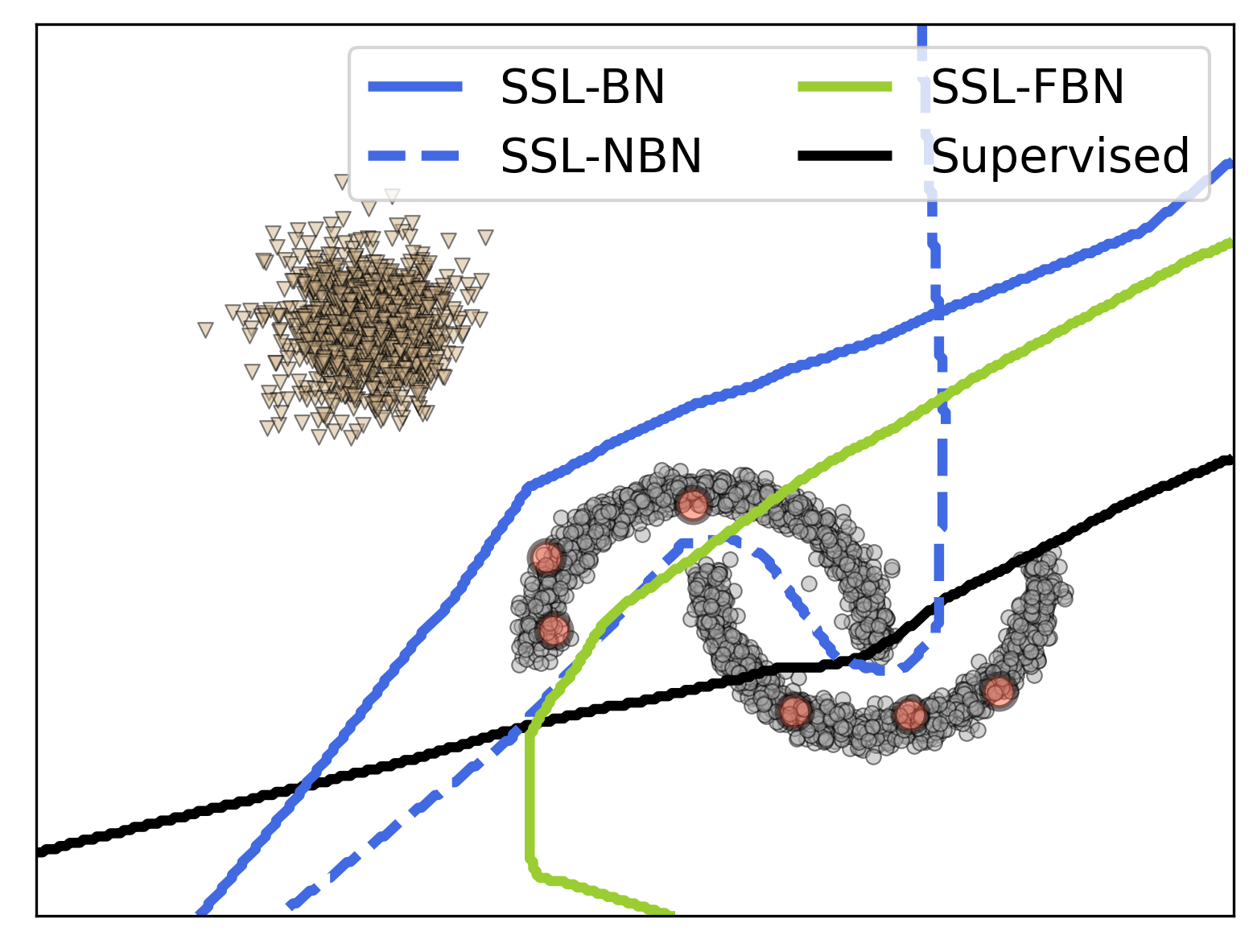}
        \caption{Faraway OODs}
    \end{subfigure}
    \begin{subfigure}[b]{0.32\textwidth}
        \centering
        \includegraphics[width=\linewidth]{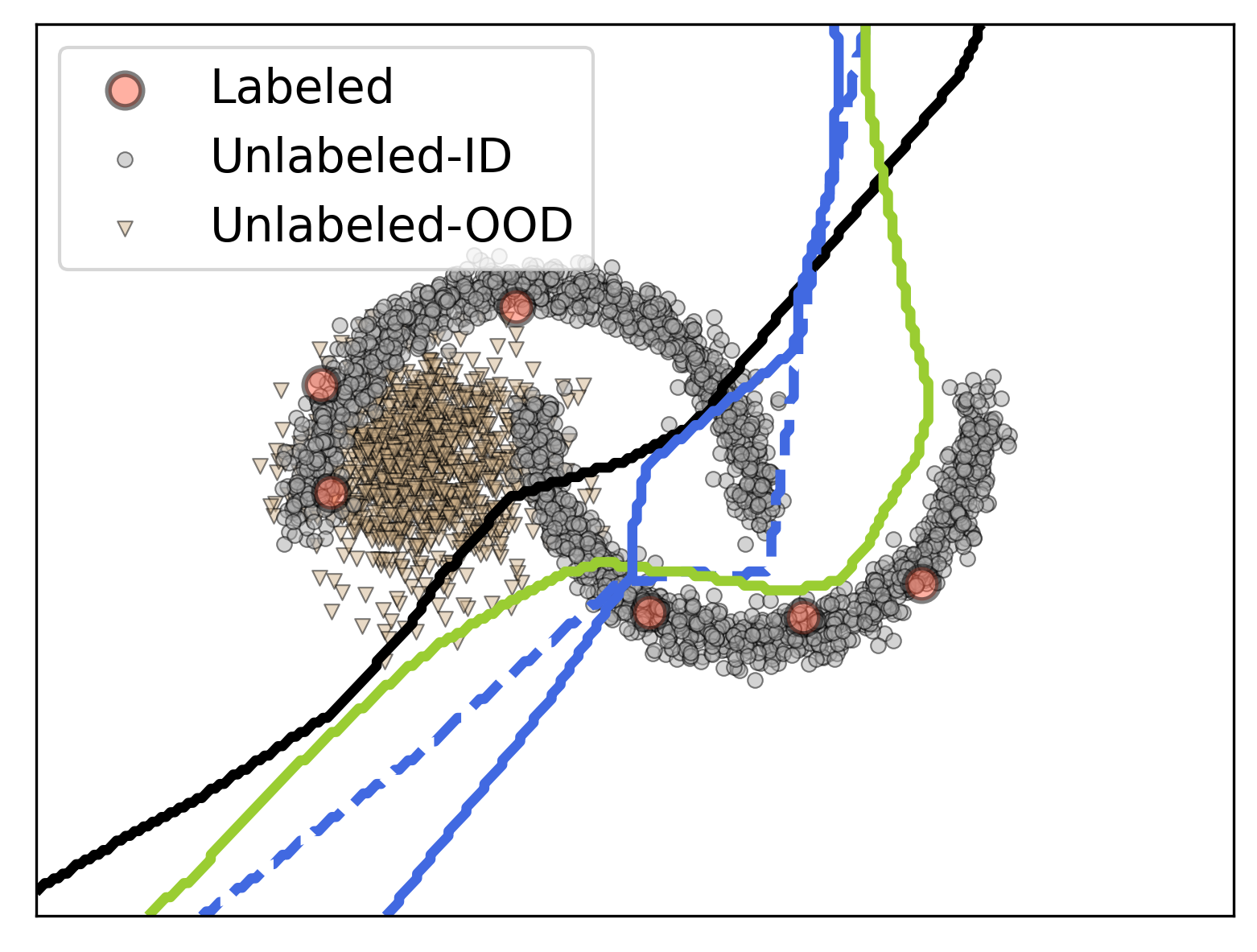}
        \caption{Boundary OODs}
    \end{subfigure}
    \begin{subfigure}[b]{0.32\textwidth}
        \centering
        \includegraphics[width=\linewidth]{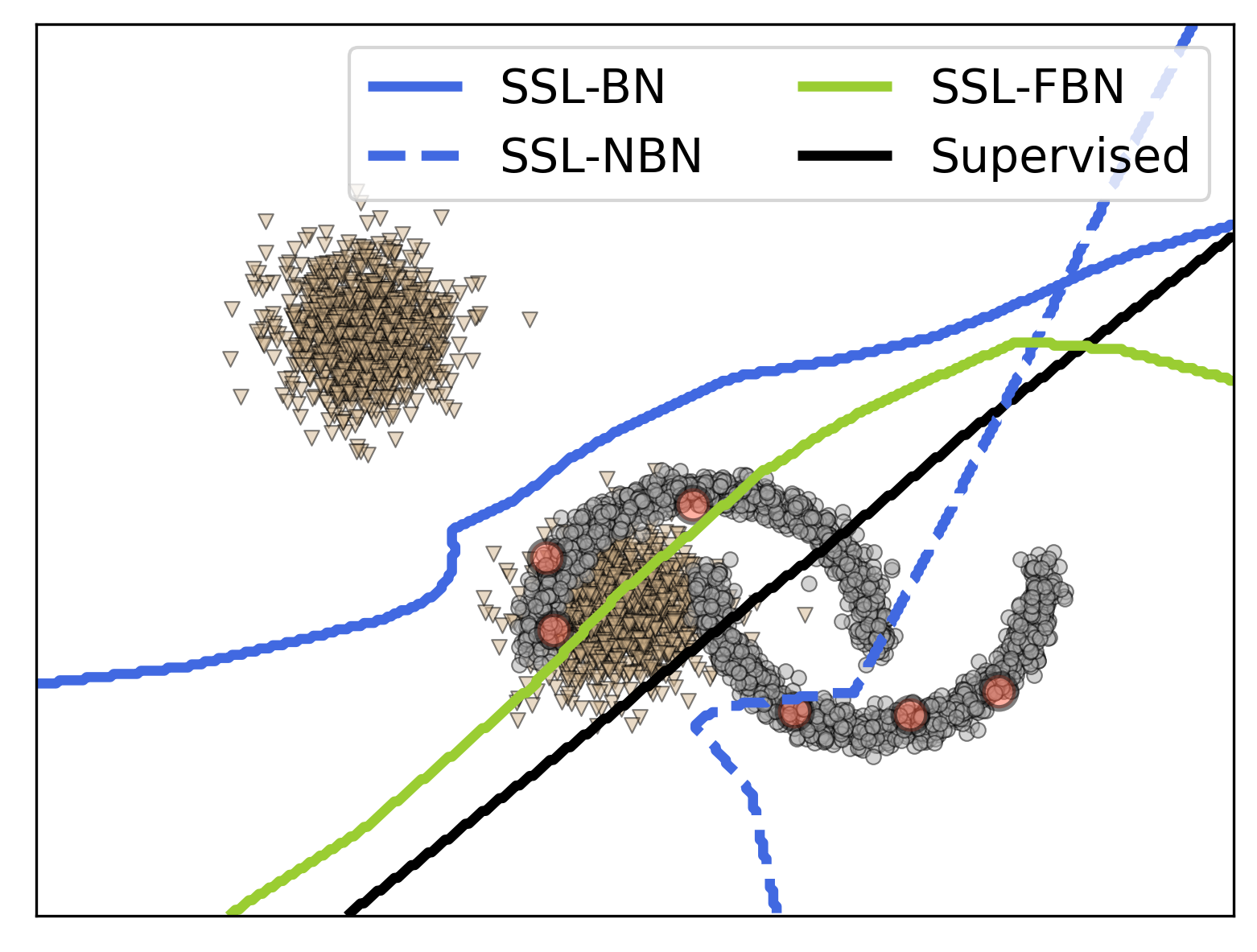}
        \caption{Mixed OODs}
    \end{subfigure}
    \caption{SSL performance with different type OODs in synthetic datasets. }
    \label{fig: impact_ood}
\end{figure*}

\noindent {\bf Real-world dataset.} We consider MNIST as ID data with three types of OODs to account for plausible real-world scenarios. 1)Faraway OOD: We used Fashion MNIST (F-MNIST) dataset, which contains fashion images as the Faraway OOD dataset as it inherently has different patterns compared to the MNIST dataset; 2) Boundary OOD: We used EMNIST  dataset, which contains handwritten character digits as Boundary OOD dataset as it has similar patterns compared to MNIST dataset; In addition to EMNIST, we also considered Mean MNIST (M-MNIST) as a boundary OOD dataset, which was generated by averaging MNIST images from two different classes (usage of M-MNIST as boundary OOD is also considered in ~\cite{guo2019mixup}); 4) Mixed OOD: For Mixed OOD dataset, we combined both Fashion MNIST and EMNIST together. To justify the usage of the above-mentioned datasets as Faraway, Boundary, and Mixed datasets, we use the vacuity uncertainty ~\cite{zhao2020uncertainty} values as an estimate of their distance to the decision boundary.
Table~\ref{epistemic} shows the average vacuity uncertainty for each type of OODs; the results indicate that Fashion MNIST has larger vacuity uncertainty compared to MNIST, whereas vacuity uncertainties of EMNIST and M-MNIST are closer to MNIST. For more details, refer to the supplementary material. We also conducted a similar experiment in Sec. 3 based on CIFAR10 dataset. We adapt CIFAR10 to a 6-class classification task, using 400 labels per class (from the 6 classes), the ID classes are: "bird", "cat", "deer", "dog", "frog", "horse", and OOD data are from classes: "airline", "automobile", "ship", "truck"), we regard this type of OODs as boundary OODs, named CIFAR-4. Besides, we use the sample from the SVHN dataset as faraway OODs and consider mixed OODs that combined both SVHN and CIFAR-4. Table~\ref{epistemic2} showed the vacuity uncertainty for each dataset and got a similar pattern. Note that we use WRN-28-2 as the backbone for the CIFAR10 experiment such that we can not remove the BN layer (the SSL performance would significantly decrease when removing BN from WRN-28-2), so we only show "SSL-BN" and "SSL-FBN" in this case.
The results are shown in Fig~\ref{fig: impact_ood_acc} (b), which shows a similar pattern made on the MNIST dataset.

\begin{figure*}[!t]
    \centering
    \begin{subfigure}[b]{0.42\textwidth}
        \centering
        \includegraphics[width=\linewidth]{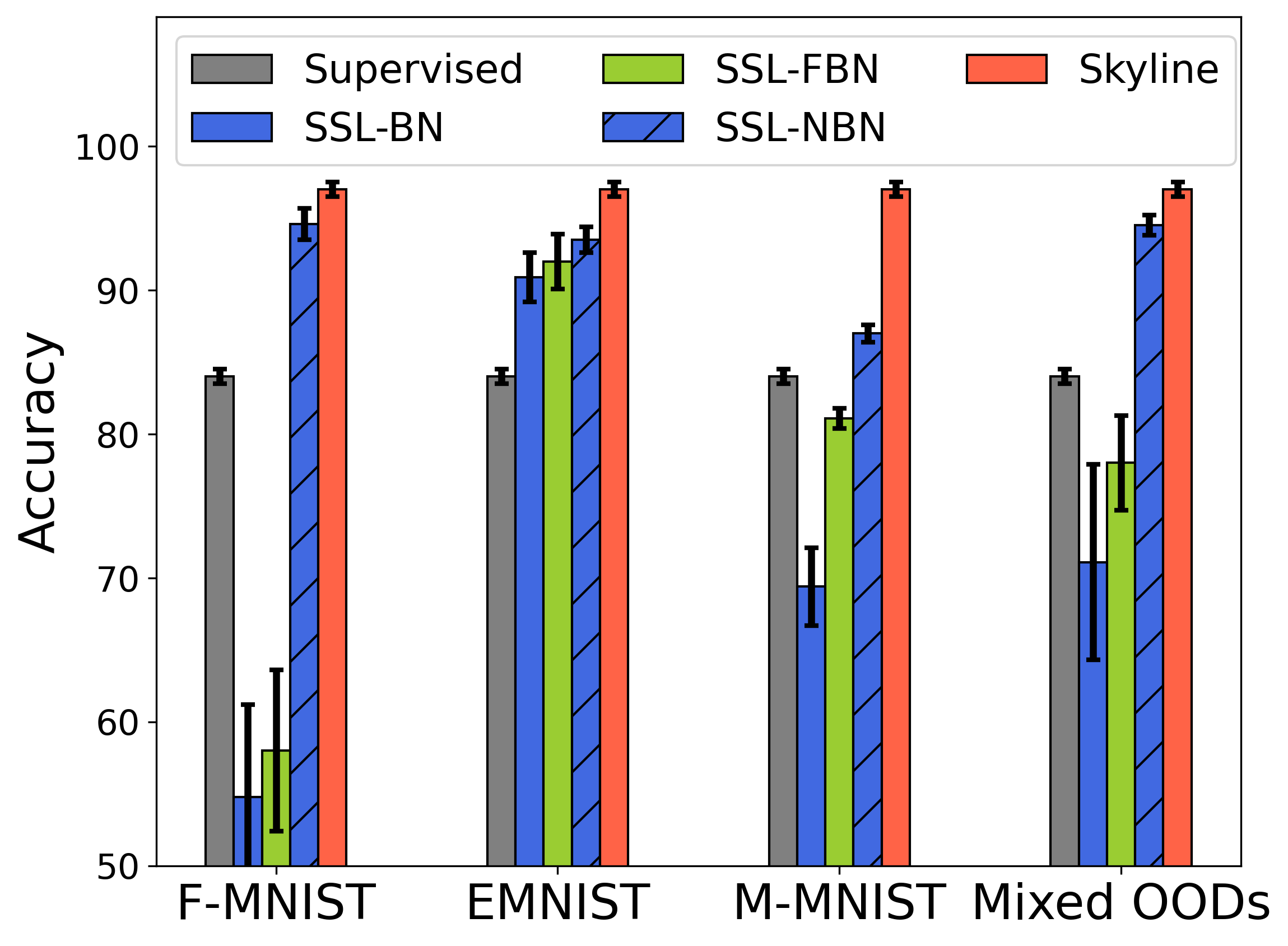}
        \caption{Real world OODs based on MNIST}
    \end{subfigure}
     \begin{subfigure}[b]{0.41\textwidth}
        \centering
        \includegraphics[width=\linewidth]{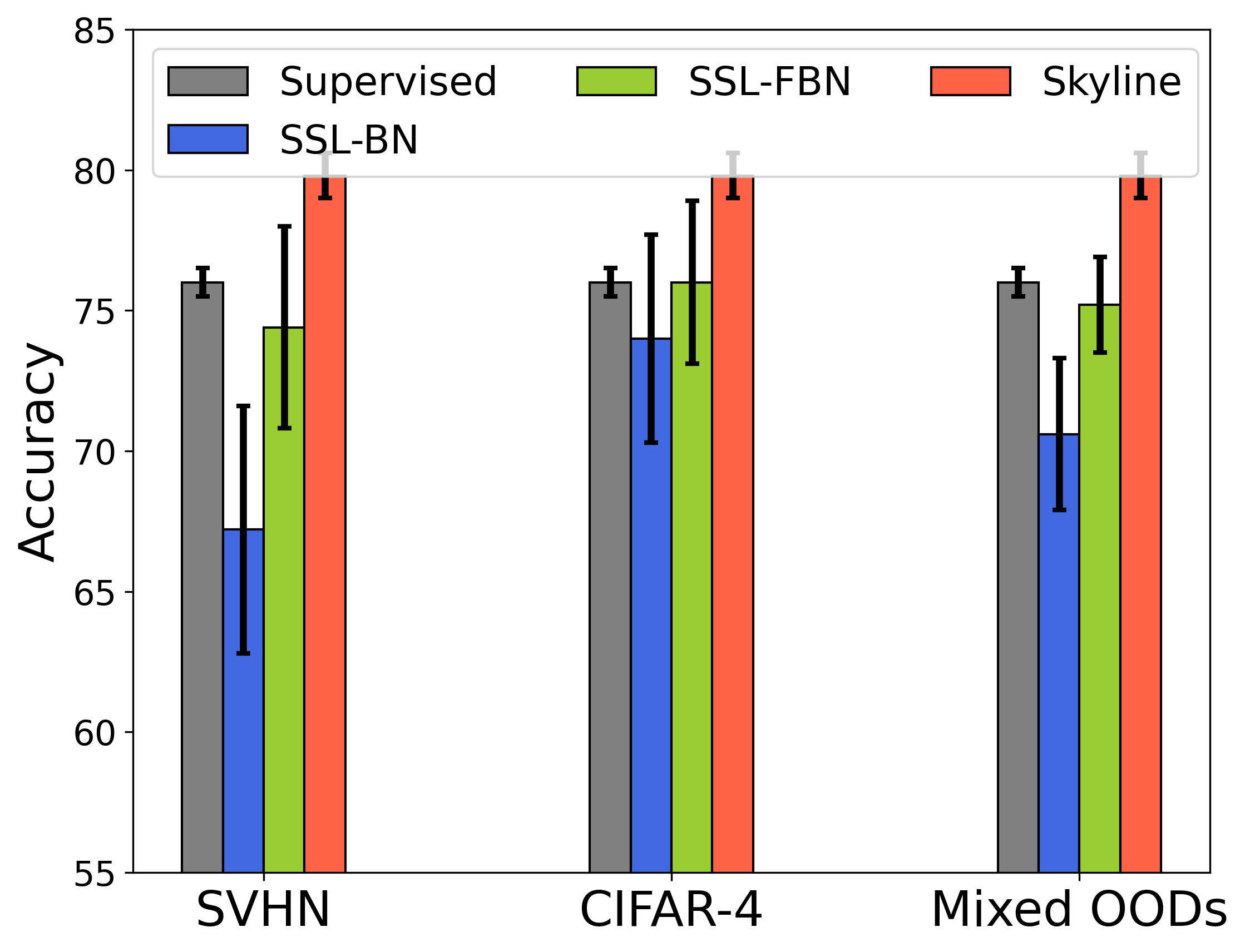}
        \caption{Real world OODs based on CIFAR10}
    \end{subfigure}
    \caption{SSL performance with different type OODs in real world datasets. }
    \label{fig: impact_ood_acc}
\end{figure*}

For all experiments in this section, we used a multilayer perceptron neural network (MLP) with three layers as a backbone architecture for the synthetic dataset and LeNet~\cite{lecun1989backpropagation} as a backbone for the real-world datasets. We consider the following models in the experiments: 1) \emph{SSL-NBN}: MLP or LeNet model without Batch Normalization; 2) \emph{SSL-BN}: MLP or LeNet model with Batch Normalization; 3) \emph{SSL-FBN}: MLP or LeNet model where we freeze the batch normalization layers for the unlabeled instances. Freezing BN (FBN) ~\cite{oliver2018realistic} is a common trick to improve the SSL model robustness where we freeze batch normalization layers by not updating \textit{running\_mean} and \textit{running\_variance} in the training phase.

\begin{table}[th!]
\caption{Uncertainty for different type OODs based on MNIST.} 
\centering
\begin{tabular}{ccccc}
\hline
\textit{F-MNIST} & \textit{EMNIST} & \textit{M-MNIST} &\textit{Mixed} & MNIST(ID) \\
                                    \hline
                  0.25&0.14& 0.15 & 0.20 & 0.09    \\
                 \hline 
\end{tabular}
\label{epistemic}
\end{table}

\begin{table}[th!]
\caption{Uncertainty for different type OODs based on CIFAR10.} 
\centering
\begin{tabular}{cccc}
\hline
\textit{SVHN} & \textit{CIFAR-4} &\textit{Mixed} &\textit{ CIFAR(ID)} \\
                                    \hline
                  0.27&0.18& 0.22 &0.15    \\
                 \hline 
\end{tabular}
\label{epistemic2}
\end{table}

The following are the main observations. First, from Fig \ref{fig: impact_ood}, we see that with BN (i.e., SSL-BN), there is a significant impact on model performance and learned decision boundaries in the presence of OOD. This performance degradation is even more pronounced in Faraway OOD since the BN statistics, like the running mean/variance can be significantly changed by faraway OOD points. Secondly, when we do not use BN (i.e., SSL-NBN), the impact of the Faraway OOD and mixed OOD data is reduced. However, in the case of boundary OOD (Fig \ref{fig: impact_ood} (b) and EMNIST/M-MNIST case of Fig \ref{fig: impact_ood_acc}), we still see significant performance degradation compared to the skyline. However, BN is a crucial component in more complicated models (Eg: ResNet family), and we expect OOD instances to play a significant role there. Finally, when freezing the BN layers for the unlabeled data (i.e., SSL-FBN), we see that the Faraway and Mixed OODs' effect is alleviated; but SSL-FBN still performs worse than the SSL-NBN in Faraway and Mixed OODs (and there is a big scope of improvement w.r.t the skyline). Finally, both SSL-NBN and SSL-FBN fail to efficiently mitigate the performance degradation caused by boundary OOD data points. In the supplementary material, we also show similar observations made on the CIFAR-10 dataset.

The main takeaways of the synthetic and real data experiments are as follows: 1) OOD instances close to the decision boundary (Boundary OODs) hurt SSL performance irrespective of the use of batch normalization; 2) OOD instances far from the decision boundary (Faraway OODs) hurt the SSL performance if the model involves BN. Freezing BN can reduce some impact of OOD to some extent but not entirely; 3) OOD instances far from the decision boundary will not hurt SSL performance if there is no BN in the model. In the next section, we propose a robust SSL reweighting framework to address above mentioned issues caused by OOD data points.

\section{Methodology}
In this section, we first proposed the uncertainty-aware robust SSL framework, and then introduce two efficient bi-level algorithms to train the robust SSL approach. More importantly, we proposed Weighted Batch Normalization (WBN) to improve the robustness of our robust SSL framework against OODs.

\subsection{Uncertainty-Aware Robust SSL Framework}~\label{sec.3.4.1}.
\noindent  {\bf Reweighting the unlabeled data.} Consider the semi-supervised classification problem with training data (labeled $\mathcal{D}$ and unlabeled $\mathcal{U}$) and classifier $f(x ;\theta)$. Generally, the optimal classifier parameter $\theta$ can be extracted by minimizing the SSL loss (Eq.~\eqref{ssl_loss}) calculated on the training set. In the presence of unlabeled OOD data, sample reweighting methods enhance the robustness of training by imposing weight $w_j$ on the $j$-th unlabeled sample loss,
\begin{eqnarray}
\mathcal{L}_{T}(\theta, \w) = \sum_{(\x_i, y_i)\in \mathcal{D}} l(f(\x_i; \theta), y_i) + \sum_{x_j\in \mathcal{U}} w_j r(f(\x_j; \theta)), \nonumber
\label{dss_ssl} 
\end{eqnarray}
where we denote $\mathcal{L}_U$ is the robust unlabeled loss, and we treat weight $\w$ as hyperparameter. Our goal is to learn a sample weight vector $\w$ such that $\w=0$ for OODs, $\w=1$ for In-distribution (ID) sample.

\begin{figure*}[!t]
    \centering
    \includegraphics[width=0.92\textwidth]{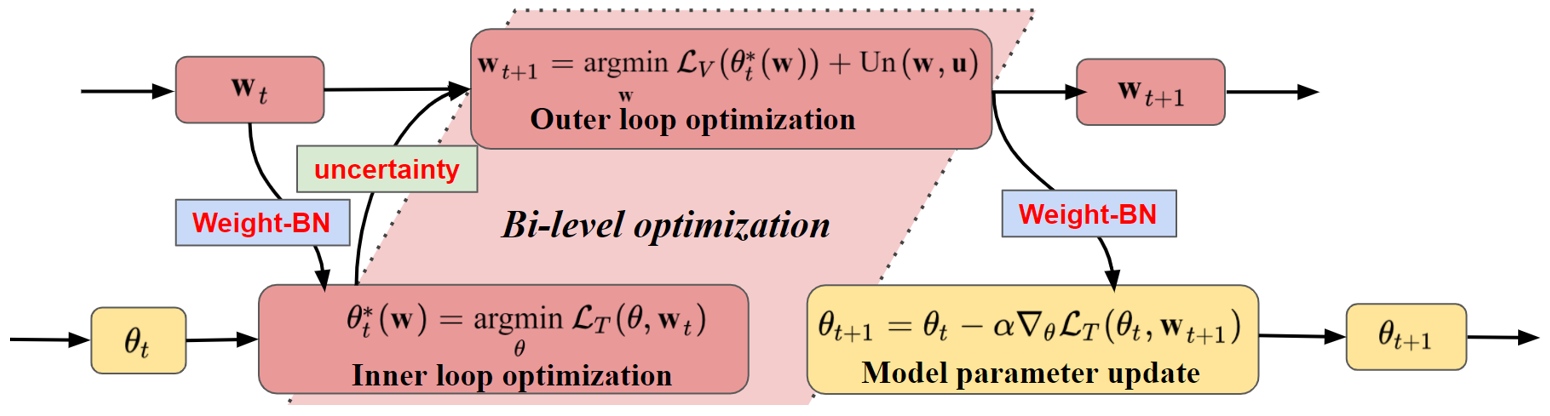}
    \vspace{-2mm}
    \caption{ Main flowchart of the proposed Weighted Robust SSL algorithm.}
    \label{fig:framework3}
\end{figure*}

\noindent Denote  $\mathcal{L}_{V}(\theta^*(\w), \w) = \mathcal{L}_{V}(\theta^*(\w))+\lambda \cdot \text{Reg}(\w)$ as the validation loss with a regularization term over the validation dataset, where $\text{Reg}(\w)$ is the regularization term, $\lambda$ is the regularization coefficient, and $\mathcal{L}_{V} (\theta^*(\w)) \triangleq \sum_{(\x_i, y_i)\in \mathcal{V}} l(f(\x_i, \theta^*(\w)), y_i)$. This labeled set could either be a held-out validation set, or the original labeled set $\mathcal D$. Intuitively, the problem given in Eq.~\eqref{ssl+dss_w} aims to choose weights of unlabeled samples $\w$ that minimize the supervised loss evaluated on the validation set when the model parameters $\theta^*(\w)$ are optimized by minimizing the weighted SSL loss $\mathcal{L}_T(\theta, \w)$.

\noindent  {\bf Uncertainty-aware bi-level optimization objective.} Since manual tuning and grid-search for each $w_i$ is intractable, we pose the weights optimization problem described above as a \emph{bi-level} optimization problem.
\begin{eqnarray}
&\mathop{\min}\limits_{\w}& \mathcal{L}_{V}(\theta^*(\w), \w) , \nonumber \\
&\text{s.t. }& \theta^*(\w) = \argmin_{\theta}\mathcal{L}_{T}(\theta, \w). 
\label{ssl+dss_w}
\end{eqnarray}
where $\mathcal{L}_{V}(\theta^*(\w), \w)$ is the objective to optimize $\w$. Based on the observation from Sec~\ref{Impact_OOD}, we find that OOD data usually holds high vacuity uncertainty. Therefore, we design the inner loop objective function with an uncertainty regularization term,
\begin{eqnarray}
    \mathcal{L}_{V}(\theta^*(\w), \w) = \mathcal{L}_{V}(\theta^*(\w))+\lambda \cdot \text{Un}(\w, \uu) \label{ssl+dss_u}
\end{eqnarray}
where $\text{Un}(\w, \uu)=\uu(1-\w)^2$ is the uncertainty regularization term, and $\lambda$ is the regularization coefficient. Based on this design, the objective function will push the OOD sample with a large weight and push the in-distribution sample with a small weight.

Calculating the optimal $\theta^*$ and $\w$ requires two nested loops of optimization, which is expensive and intractable to obtain the exact solution~\cite{franceschi2018bilevel}, especially when optimization involves deep learning model and large datasets. Since gradient-based methods like Stochastic Gradient Descent (SGD) have shown to be very effective for machine learning and deep learning problems~\cite{bengio2000gradient}, we adapt both high-order approximation and meta approximation strategies, as described in Sec~\ref{sec.3.4.2}.

\subsection{Bi-level Optimization Approximation}\label{sec.3.4.2}
In this section, We developed two efficient bi-level algorithms that have different tradeoffs in computational efficiency and accuracy.

\noindent  {\bf Implicit Differentiation.} Directly calculate the weight gradient $\frac{\partial \mathcal{L}_{V}(\theta^*(\w), \w)}{ \partial \w}$ by chain rule:
\begin{eqnarray}
    \frac{\partial \mathcal{L}_{V} (\theta^*(\w), \w)}{ \partial \w}
    =\underbrace{\frac{\partial \mathcal{L}_{V} }{ \partial \w}}_{(a)} + \underbrace{\frac{\partial \mathcal{L}_{V} }{ \partial \theta^*(\w)}}_{(b)} \times \underbrace{\frac{\partial \theta^*(\w)}{\partial \w}}_{(c)} \label{ID_gradient}
\end{eqnarray}
where (a) is the weight direct gradient (e.g., gradient from regularization term, $\text{Reg}(\w)$), (b) is the parameter direct gradient, which is easy to compute. The difficult part is the term (c) (best-response Jacobian). We approximate (c) by using the Implicit function theorem~\cite{lorraine2020optimizing},
\begin{eqnarray}
    \frac{\partial \theta^* (\w)}{\w} = - \underbrace{\Big[ \frac{\partial \mathcal{L}_{T}}{\partial \theta\partial \theta^T} \Big]^{-1}}_{(d)} \times  \underbrace{\frac{\partial \mathcal{L}_{T}}{\partial \w\partial \theta^T}}_{(e)} \label{IFT}
\end{eqnarray}

However, computing Eq.~\eqref{IFT} is challenging when using deep nets because it requires inverting a high dimensional Hessian (term (d)), which often requires $\mathcal{O}(m^3)$ operations. Therefore, we give the Neumann series approximations~\cite{lorraine2020optimizing} of the term (d) which we empirically found to be effective for SSL,
\begin{equation}
    \Big[ \frac{\partial \mathcal{L}_{T}}{\partial \theta\partial \theta^T} \Big]^{-1} \approx  \lim_{P \rightarrow{}\infty} \sum_{p=0}^P \Big[I- \frac{\partial \mathcal{L}_{T}}{\partial \theta\partial \theta^T} \Big]^p \label{neumann_serires}
\end{equation}
where $I$ is the identity matrix.

Since the algorithm mentioned in ~\cite{lorraine2020optimizing} utilizes the Neumann series approximation and efficient Hessian vector product to compute the Hessian inverse product, it can efficiently compute the Hessian-inverse product even when a larger number of weight hyperparameters are present. We should also note that the implicit function theorem's assumption $\frac{\partial \mathcal{L}_{T}}{\partial \theta} = 0$ needs to be satisfied to accurately calculate the Hessian inverse product. However in practice, we only approximate $\theta^{*}$, and simultaneously train both $\w$ and $\theta$ by alternatively optimizing $\theta$ using $\mathcal{L}_T$ and $\w$ using $\mathcal{L}_V$. 

\noindent  {\bf Meta approximation.} Here we proposed the meta-approximation method to jointly update both network parameters $\theta$ and hyperparameter $\w$ in an iterative manner. At iteration step $t$, we approximate $\theta^*_t \approx \theta^J_t$ on training set via low order approximation, where $J$ is inner loop gradient steps, Eq.~\eqref{gd} shows each gradient step update,
\begin{equation}
    \theta^j_{t} (\w_t) = \theta_{t}^{j-1} - \alpha \nabla_{\theta} \mathcal{L}_{T}(\theta_{t}^{j-1}, \w_{t}) \label{gd}
\end{equation}
then we update hyperparameter $\w_{t+1}$ on the basis of the net parameter $\theta^*_t$ and weight $\w_{t}$ obtained in the last iteration. To guarantee efficiency and general feasibility, the outer loop optimization to update weight is employed by one gradient step on the validation set $\mathcal{V}$, 
\begin{equation}
    \w_{t+1} = \w_{t} - \beta \nabla_{\w} \mathcal{L}_{V}(\theta^*_{t}, \w_{t})
\end{equation}
We analyze the convergence of the meta-approximation method and derive the following theorem.
\begin{theorem}\label{theorem_rssl_conv}
Suppose the validation function is Lipschitz smooth with constant $L$, and the supervised loss and unsupervised loss have $\gamma$-bounded gradients, let the step size $\alpha$ for $\theta$ satisfies $\alpha=\min \{1, \frac{k}{T}\}$ for some constant $k>0$ such that $\frac{k}{T}< 1$ and $\beta=\min \{\frac{1}{L}, \frac{C}{\sqrt{T}}\}$ for some constant $C>0$ such that $\frac{\sqrt{T}}{C}\le L$. Then, the meta approximation algorithm can achieve $\mathbb{E}[\|\nabla_\w \mathcal{L}_V(\theta_t)  \|^2_2] \le \epsilon$ in $\mathcal{O}(1/\epsilon^2)$ And more
specifically,
\begin{eqnarray}
    \min_{0\le t \le T} \mathbb{E}[\|\nabla_\w \mathcal{L}_V(\theta_t)  \|^2_2] \le \mathcal{O}(\frac{C}{\sqrt{T}}) 
\end{eqnarray}
where $C$ is some constant independent to the convergence
process.
\end{theorem}

\begin{proof}
See Appendix~\ref{app:A3}.
\end{proof}

\noindent  {\bf Complexity.} Compared with regular optimization on a
the single-level problem, our robust SSL requires $J$ extra forward and backward passes of the classifier network. To compute the weight gradient via bi-level optimization, Meta approximation requires extra forward and backward passes. Therefore, compared with the regular training procedures of SSL, our robust SSL with Meta approximation needs approximately $(2+J)\times$ training time. For Implicit Differentiation, the training time is complex to estimate, we show the running time in Fig~\ref{fig:addition_ana} (c) on the experiment part.

\vspace{0.5mm}

\noindent{\bf Connections between implicit-differentiation and meta-approximation.}

\begin{proposition}\label{IFT_connection}
Suppose that the Hessian inverse of training loss $\mathcal{L}_T$  with the model parameters $\theta$  is equal to the identity matrix 
$\frac{\partial^2 \mathcal{L}_T}{\partial \theta \partial \theta^T} = \mathbf{I}$ (i.e.,$P=0$ for implicit differentiation approach). Suppose the model parameters are optimized using single-step gradient descent (i.e., $J=1$ for the low-order approximation approach), and the model learning rate is equal to one. Then, the weight update step in both implicit differentiation and low-order approximation approach is equal.
\end{proposition}

\begin{proof}
See Appendix~\ref{app:A4}.
\end{proof}

Proposition~\ref{IFT_connection} shows that the weight gradient from implicit differentiation is the same as the weight gradient from the meta approximation method.
Hence, the meta-approximation is the same as the Implicit differentiation method when $P=0$. This also signifies that using larger $P$ values for inverse Hessian estimations will use higher-order information and is more accurate. We designed these two efficient hyper-parameter optimization algorithms that have different trade-offs on computational efficiency and accuracy. Meta approximation was designed based on lower-order approximations of the objective in the inner loop of meta optimization due to vanishing gradients or memory constraints. Implicit Differentiation was designed based on higher-order approximations of the objective is more accurate than Meta approximation while at the same time being computationally more expensive.

\subsection{Weighted Batch Normalization} \label{sec:WBN}
In practice, most deep SSL models would use deep CNN. While BN usually serves as an essential component for many deep CNN models~\cite{he2016deep,huang2017densely}. Specifically, BN normalizes input features by the mean and variance computed within each mini-batch. At the same time, OODs would indeed affect the SSL performance due to BN (discussed this issue in Sec.~\ref{Impact_OOD}). 
To address this issue, we proposed a \textit{Weighted Batch Normalization}  (WBN) that performs the normalization for each training mini-batch with sample weights $\w$. We present the WBN in Algorithm~\ref{algorithm_WBN}, where $\epsilon$ is a constant added to the mini-batch variance for numerical stability.

\begin{proposition}\label{proposition_WBN}
Give in-distribution mini-batch $\mathcal{I}=\{\x_i\}_{i=1}^m$, OOD mini-batch $\mathcal{O}=\{\hat{\x}_i\}_{i=1}^m$, and the mixed mini-batch $\mathcal{IO}=\mathcal{I}\cup \mathcal{O}$. Denote $\mu_{\mathcal{M}}$ is the mini-batch mean of $\mathcal{M}$ (either $\mathcal{O}, \mathcal{I}$ or $\mathcal{IO}$).
With faraway OOD: $\|\mu_{\mathcal{O}}- \mu_{\mathcal{I}}\|_2 > L $, where $L$ is large ($L\gg0$), we have:

\begin{enumerate}
    \item $ \|\mu_{\mathcal{IO}} - \mu_{\mathcal{I}}\|_2 > \frac{L}{2} $ and 
    $BN_{\mathcal{I}}(\x_i) \neq BN_{\mathcal{IO}}(\x_i)$;
    \item Given perfect weights $\w=\w_\mathcal{I} \cup \w_\mathcal{O}$, where $\w_\mathcal{I}={\bm 1}$ for mini-batch $\mathcal{I}$ and $\w_{\mathcal{O}}={\bm 0}$ for mini-batch $\mathcal{O}$, then $\mu_{\mathcal{I}} = \mu^{\w}_{\mathcal{IO}}$ and $BN_{\mathcal{I}}(\x_i) = WBN_{\mathcal{IO}}(\x_i, \w)$
\end{enumerate}
where $\mu^{\w}_{\mathcal{IO}}$ is the weighted mini-batch mean of $\mathcal{IO}$, $BN_{\mathcal{M}}(\x_i)$ is traditional batch normalizing transform based on mini-batch from $\mathcal{M}$ (either $\mathcal{I}$ or $\mathcal{IO}$) and $WBN_{\mathcal{IO}}(\x_i, \w)$ is weighted batch normalizing transform based on the set $\mathcal{IO}$ with weights $\w$.
\end{proposition}

\begin{proof}
See Appendix~\ref{app:A5}.
\end{proof}

Proposition~\ref{proposition_WBN} shows that when an unlabeled set contains OODs, the traditional BN behavior could be problematic, therefore resulting in incorrect statistics estimation. e.g., the output of BN for mixed mini-bath $BN_{\mathcal{IO}}(\x_i) \approx \gamma \frac{\x_i -\mu_{\mathcal{O}}}{\|\mu_{\mathcal{O}}-\mu_{\mathcal{I}} \|_2} + \beta$, which is not our expected result ($BN_{\mathcal{I}}(\x_i) = \gamma \frac{\x_i -\mu_{\mathcal{I}}}{\sqrt{\sigma^2_{\mathcal{I}}}+\epsilon} + \beta$).
While our proposed weighted batch normalization (WBN) can reduce the OOD effect and get the expected result. Therefore, our uncertainty-aware robust SSL framework uses WBN instead of BN if the model includes a BN layer.
Ablation studies in Sec.~\ref{experiment} demonstrate that such our approach with WBN can improve performance further. Finally, our uncertainty-aware robust SSL framework is detailed in Algorithm~\ref{algorithm_robustssl}.

\begin{algorithm}
\small{
\DontPrintSemicolon
\KwIn{A mini-batch $\mathcal{M}=\{\x_i\}_{i=1}^m$ and sample weight $\textcolor{blue}{\w}=\{w_i\}_{i=1}^m$; Parameters to be learned: $\gamma, \beta$}
\KwOut{$\{t_i= \text{\textit{WBN}}_{\mathcal{M}}(\x_i,\textcolor{blue}{\w}) \}$} 
\SetKwBlock{Begin}{function}{end function}
{
  Weighted mini-batch mean,
  \vspace{-2mm}
  \[\mu^{\textcolor{blue}{\w}}_{\mathcal{M}}\leftarrow \frac{1}{\sum_{i=1}^m \textcolor{blue}{w_i}} \sum_{i=1}^m \textcolor{blue}{w_i} \x_i \nonumber \vspace{-5mm} \] \;
  Weighted mini-batch variance,
  \vspace{-2mm}
  \[ {\sigma^{\textcolor{blue}{\w}}_{\mathcal{M}}}^2  \leftarrow \frac{1}{\sum_{i=1}^m \textcolor{blue}{w_i}} \sum_{i=1}^m \textcolor{blue}{w_i}(\x_i -\mu^{\textcolor{blue}{\w}}_{\mathcal{M}})^2 \nonumber \vspace{-5mm}\]\;
  Normalize,
  \[\vspace{-3mm}\hat{x}_i \leftarrow \frac{\x_i -\mu_{\mathcal{M}}^{\textcolor{blue}{\w} } }{\sqrt{{\sigma^{\textcolor{blue}{\w}}_{\mathcal{M}}}^2 +\epsilon}}\nonumber \vspace{-3mm}\] \;
  Scale and shift,
  \[\vspace{-3mm} t_i \leftarrow  \gamma \hat{x}_i + \beta \equiv \text{\textit{WBN}}_{\mathcal{M}}(\x_i, \textcolor{blue}{\w})\nonumber \] \;
}
\caption{\textcolor{blue}{Weighted} Batch Normalization}\label{algorithm_WBN}
}
\end{algorithm}

\begin{algorithm}
\small{
\DontPrintSemicolon
\KwIn{$\mathcal{D}, \mathcal{U}$}
\KwOut{$\theta, \w$} 
\SetKwBlock{Begin}{function}{end function}
{
  $t=0$; \;
  Set learning rate $\alpha$, $\beta$,and Hessian approximation $P$; \;
  Initialize model parameters $\theta$ and weight $\w$; \;
  Apply K-means do K-clusters for $\mathcal{U}$; \;
  \If {Model includes BN layer} {
    Apply \textit{Weight-Batch Normalization} instead of BN;}
 \Repeat{convergence}
 {
 \textcolor{gray}{**** Inner loop optimization, initial $\theta^0_t = \theta_t$ **** } \;
 \For{$j=1, \ldots, J$}{$ \theta^j_{t} (\w) = \theta^{j-1}_{t} - \alpha \nabla_{\w} \mathcal{L}_{T}(\theta^{j-1}_{t}, \w_{t})$} 
 \textcolor{gray}{**** Outer loop optimization, set $\theta^*_t = \theta^J_t$ **** }\;
 \uIf {Meta Approximation} {
    update weight via
 $\w_{t+1} = \w_{t} - \beta \nabla_{\w} \mathcal{L}_{V}(\theta^*_{t}(\w), \w_{t})$;}
 \uElseIf {Implicit Differentiation} {
    Approximate inverse Hessian via Eq.~\eqref{neumann_serires};\;
    Calculate best-response Jacobian by Eq.~\eqref{IFT};\;
    Calculate weight gradient $\nabla_{\w}\mathcal{L}_{V}$ via Eq.~\eqref{ID_gradient};\;
    update weight via $\w_{t+1} = \w_{t} - \beta \cdot \nabla_{\w}\mathcal{L}_{V}$;\;
    }
\textcolor{gray}{**** Update net parameters ****}
$\theta_{t+1} = \theta_{t} - \alpha \nabla_{\theta} \mathcal{L}_{T}(\theta_{t}, \w_{t+1})$\; 
$t=t+1$\;
 }
\Return{$\theta_{t+1}, \w_{t+1}$}
}
\caption{Uncertainty-Aware Robust SSL}\label{algorithm_robustssl}
}
\end{algorithm}

\subsection{ Additional Implementation Details:}
In this subsection, we discuss additional implementational and practical tricks to make our weighted robust SSL scalable and efficient.

\noindent {\bf Last-layer gradients.} Computing the gradients over deep models is time-consuming due to the enormous number of parameters in the model. To address this issue, we adopt a last-layer gradient approximation similar to ~\cite{ash2019deep,killamsetty2021retrieve,killamsetty2021grad} by
only considering the last classification layer gradients of the classifier model in inner loop optimization (step 10 in algorithm~\ref{algorithm_robustssl}). By simply using the last-layer gradients, we achieve significant speedups in weighted robust SSL.

\noindent {\bf Infrequent update $\w$.} We update the weight parameters every $L$ iterations ($L>2$). In our experiments, we see that we can set $L = 5$ without a significant loss in accuracy. For MNIST experiments, we can be even more aggressive and set $L = 20$.

\noindent  {\bf Weight Sharing and Regularization.} Considering the entire weight vector $\w$ (overall unlabeled points) is not practical for large datasets and easily overfits (see ablation study experiments), we propose two ways to fix this. The first is weight sharing via clustering, which we call Cluster Re-weight (CRW) method. Specifically, we use an unsupervised cluster algorithm (e.g., K-means algorithm) to embed unlabeled samples into $K$ clusters and assign a weight to each cluster such that we can reduce the dimensionality of $\w$ from $|M|$ to $|K|$, where $|K| \ll |M|$. In practice, for high dimensional data, we may use a pre-trained model to calculate embedding for each point before applying the cluster method. In cases where we do not have an effective pre-trained model for embedding, we consider another variant that applies weights to every unlabeled point but considers an L1 regularization in Eq~\eqref{ssl+dss_w} for sparsity in $\w$. We show that both these tricks effectively improve the performance of reweighting and prevent overfitting on the validation set.

\section{Experiment}\label{experiment}
To corroborate our algorithm, we conduct extensive experiments
comparing our approaches (R-Meta and R-IFT) with some popular baseline methods. We aim to answer the following questions:

\noindent {\bf Question 1:} Can our approach achieve better performance on both different types of OODs with varying OOD ratios compared with baseline methods?

\noindent {\bf Question 2:} How does our approach compare in terms of running times compared with baseline methods? 

\noindent {\bf Question 3:} What is the effect of each of the components of our approach (e.g., WBN, clustering/regularization, inverse Hessian approximation,  inner loop gradient steps)?  

\subsection{Evaluation on Synthetic Dataset}
We designed a synthetic experiment to illustrate how OODs affect SSL performance. The experimental setting is the same as the setting used in the section "Impact of OOD on SSL Performance". We used the Two Moons dataset with six labeled points and 2000 unlabeled (in-distribution) points and considered two OOD types, including faraway OODs and boundary OODs. We conducted the experiments with OOD ratio = $\{25\%, 50\%, 75\%\}$ and reported the averaged accuracy rate with mean and standard deviation over ten runs. Table~\ref{OOD_ex:syn} shows \feng{that} our robust SSL (named R-SSL-IFT (Implicit Differentiation) and R-SSL-Meta (Meta approximation)) is more effective than the four baselines on test accuracy. We also conducted experiments on additional synthetic datasets and observed similar trends of the results, \feng{as} shown in Fig~\ref{fig:syn_1}.

\begin{figure}[htb]
    \centering 
\begin{subfigure}{0.3\textwidth}
  \includegraphics[width=\linewidth]{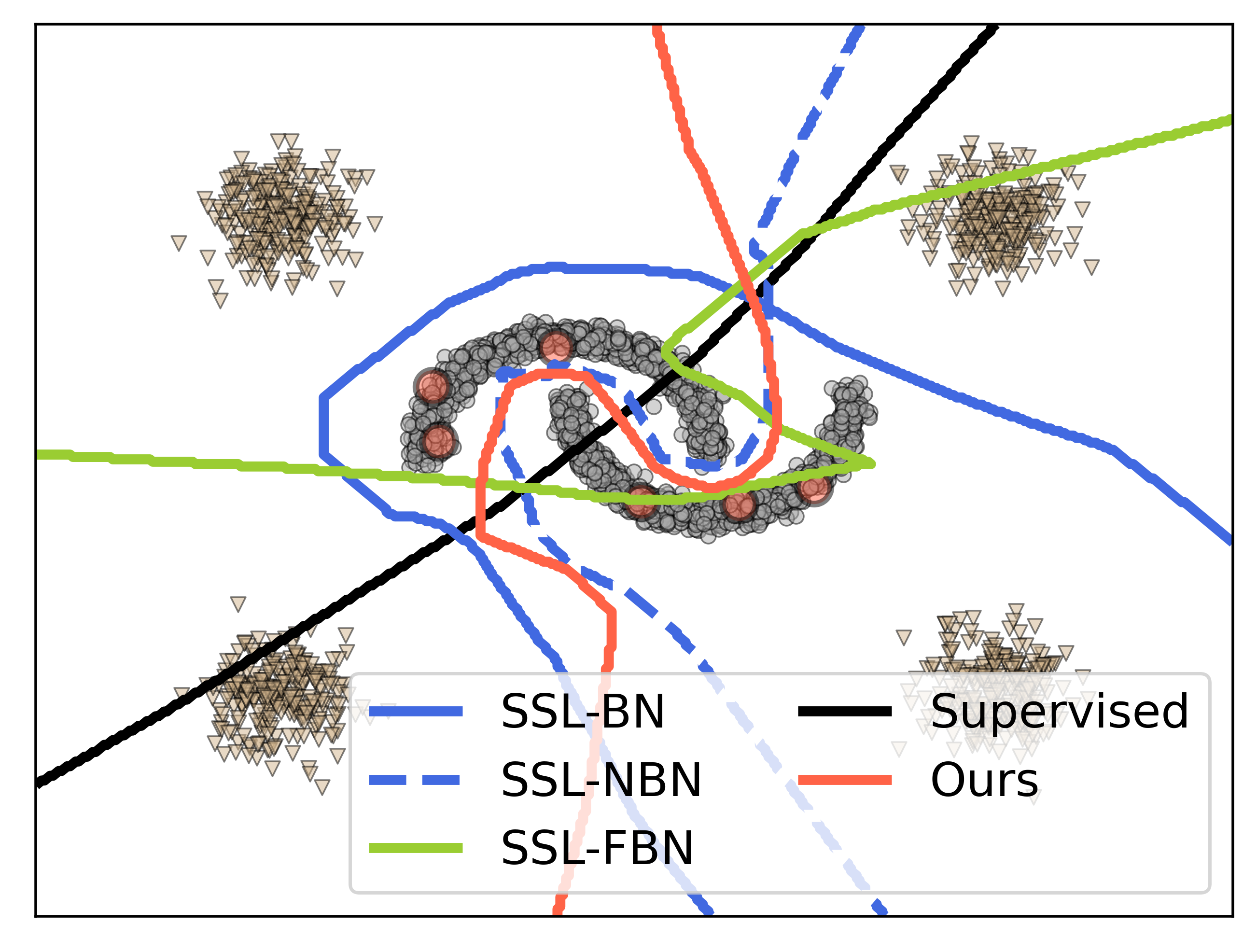}
  \caption{Faraway OODs}
  \label{fig:21}
\end{subfigure}\hfil 
\begin{subfigure}{0.3\textwidth}
  \includegraphics[width=\linewidth]{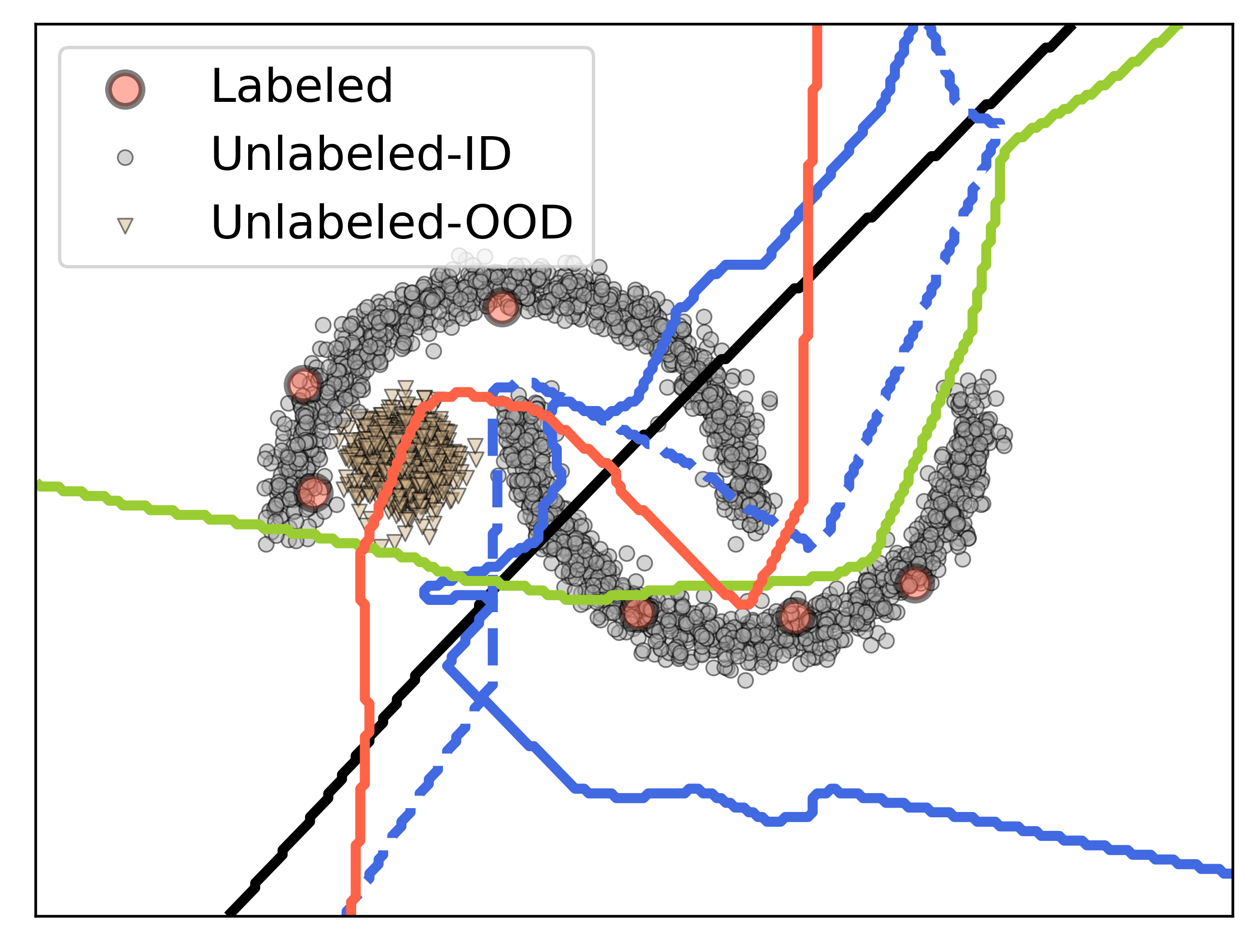}
  \caption{Boundary OODs}
  \label{fig:22}
\end{subfigure}\hfil 
\begin{subfigure}{0.3\textwidth}
  \includegraphics[width=\linewidth]{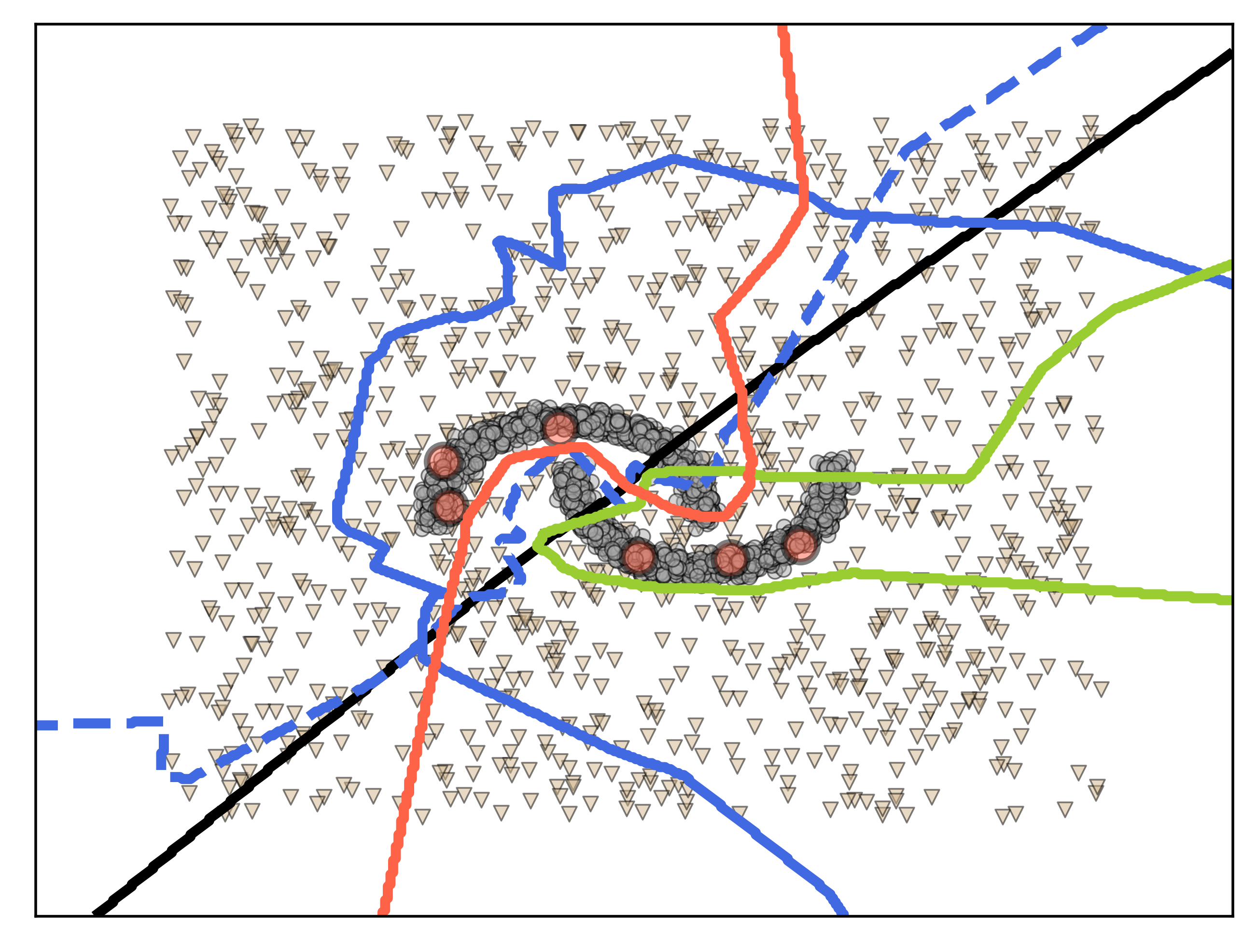}
  \caption{Random Noise}
  \label{fig:23}
\end{subfigure}

\medskip
\begin{subfigure}{0.3\textwidth}
  \includegraphics[width=\linewidth]{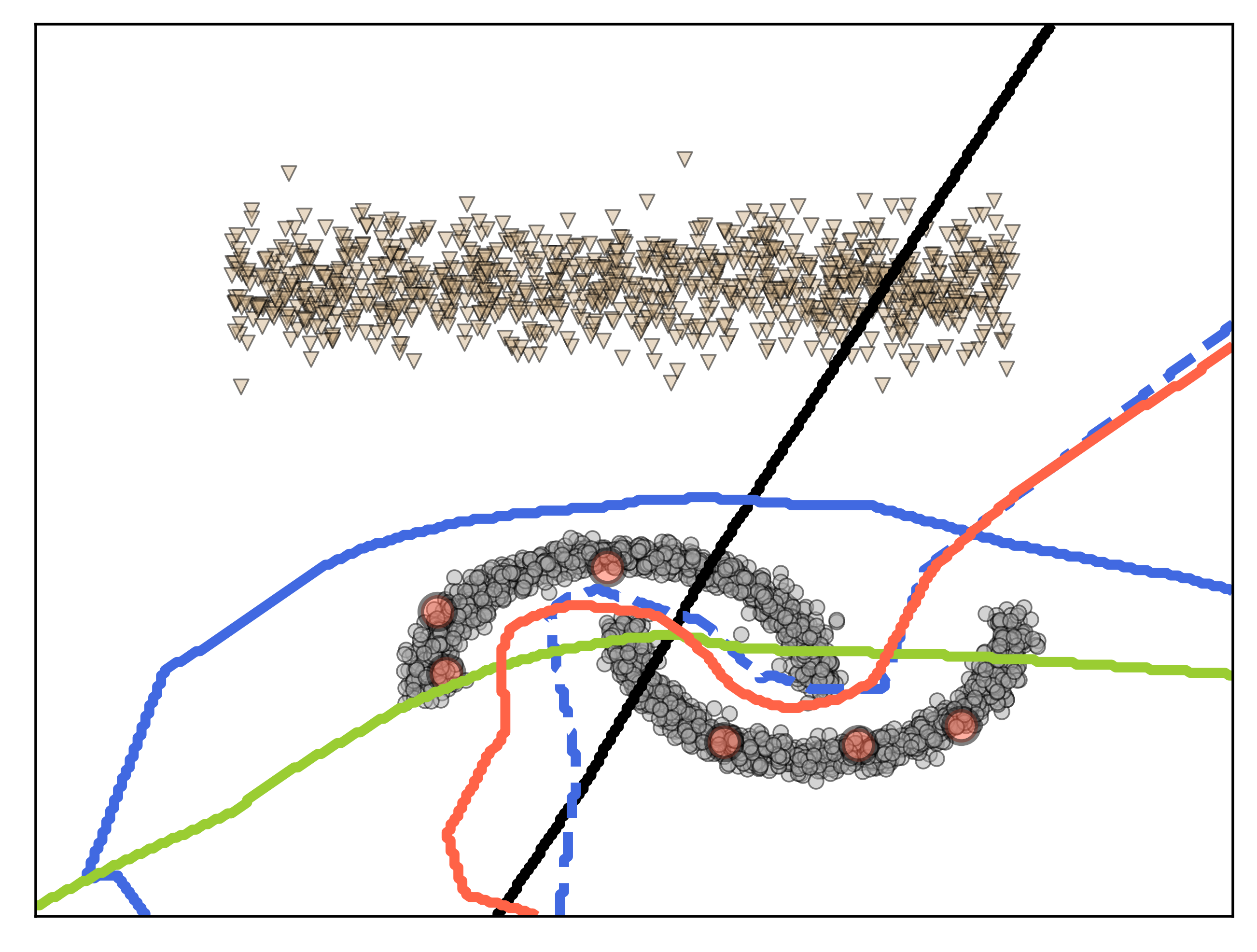}
  \caption{Faraway OODs}
  \label{fig:24}
\end{subfigure}\hfil 
\begin{subfigure}{0.3\textwidth}
  \includegraphics[width=\linewidth]{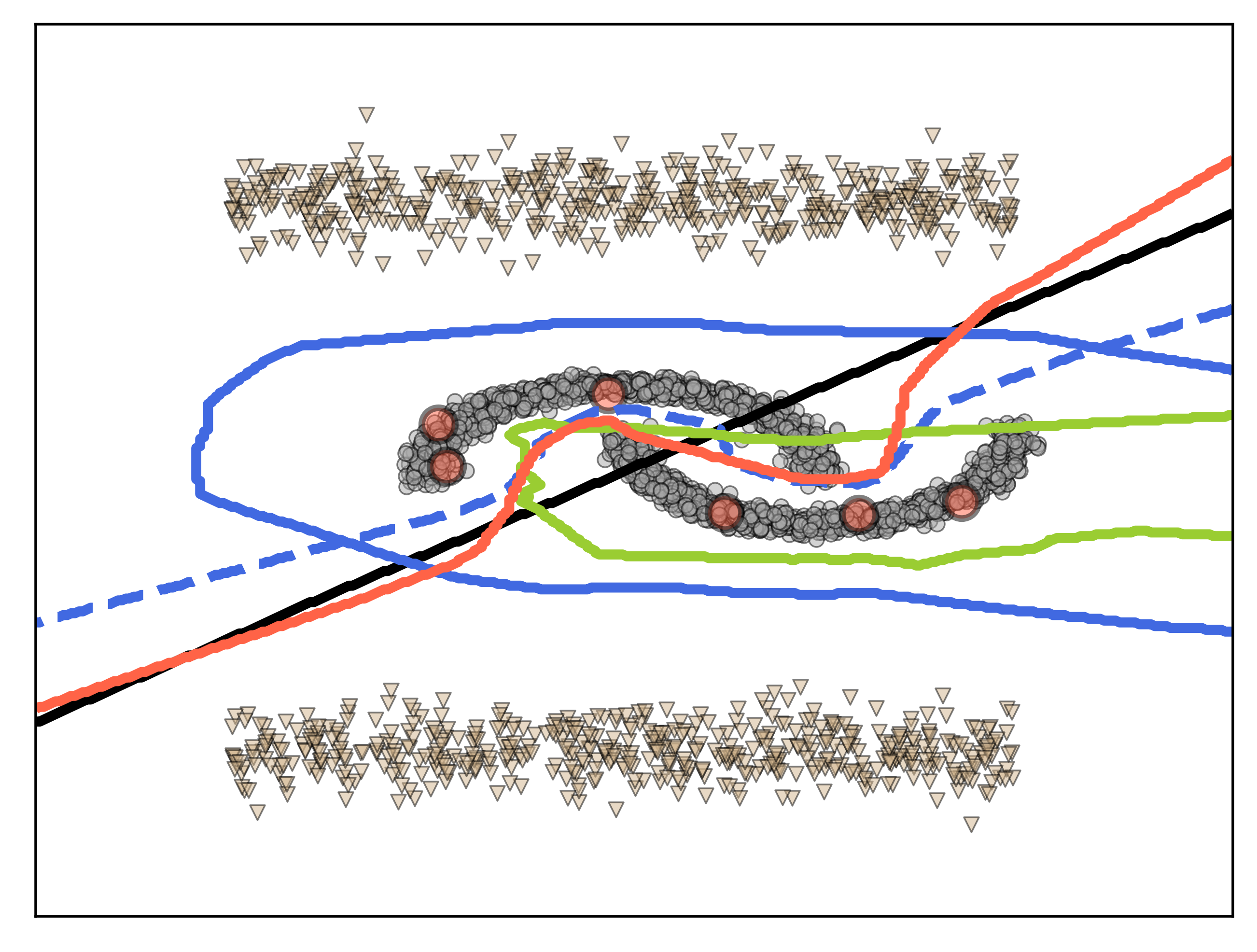}
  \caption{Faraway OODs}
  \label{fig:25}
\end{subfigure}\hfil 
\begin{subfigure}{0.3\textwidth}
  \includegraphics[width=\linewidth]{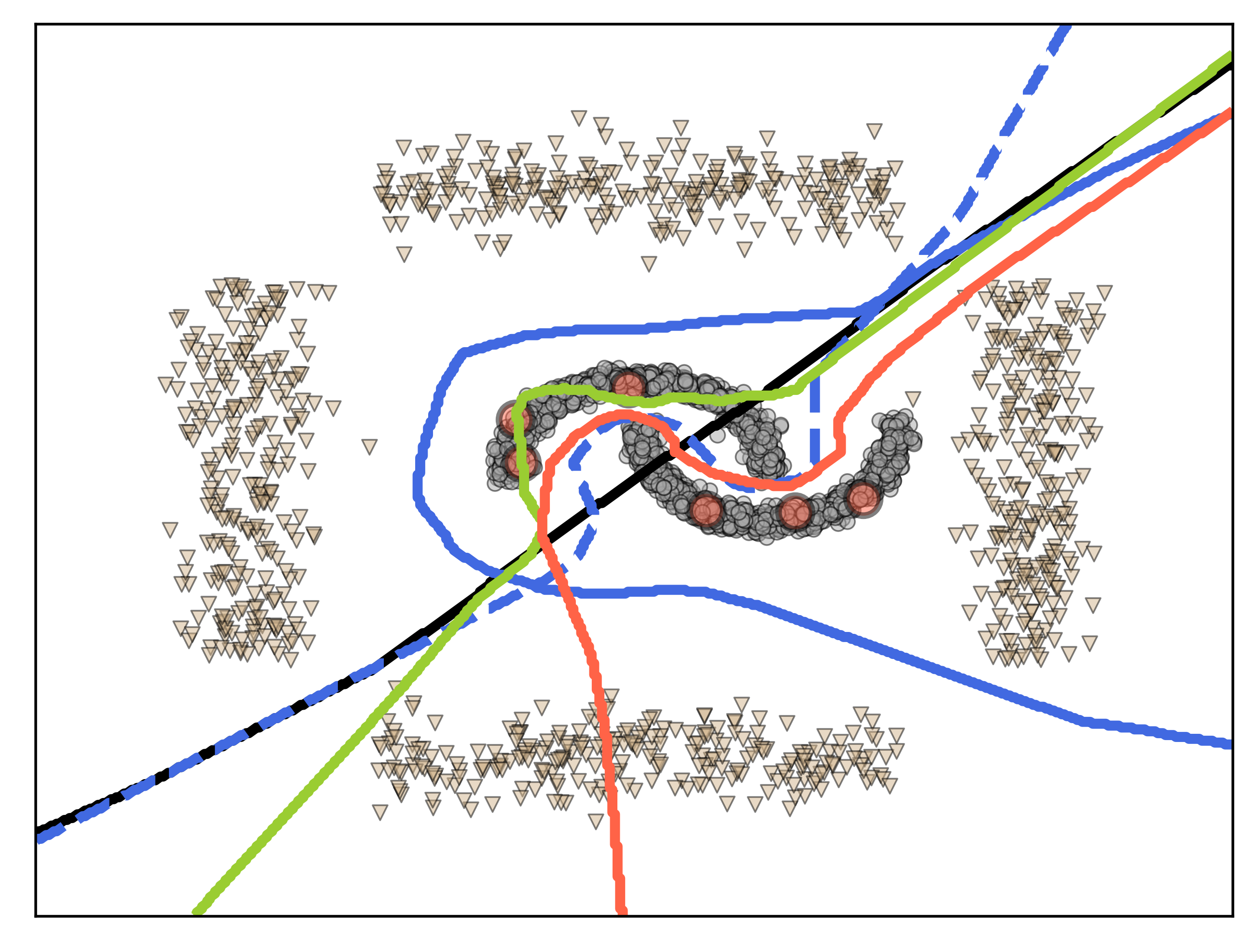}
  \caption{Faraway OODs}
  \label{fig:26}
\end{subfigure}
\caption{Additional experiment\feng{s} on \feng{the} synthetic dataset.}
\label{fig:syn_1}
\end{figure}

\begin{table}[th!]
\small
\caption{Test accuracies for \feng{the} two moons dataset \feng{(}$P=10, J=3$ for R-SSL-IFT and $J=1$ for R-SSL-Meta\feng{)}.} 
\vspace{1mm}
\centering
\begin{tabular}{c|ccc}
\hline
\multirow{2}{*}{Model} & \multicolumn{3}{c}{\textit{Faraway OODs}}  \\  
                                      & 25\%& 50\% &75\% \\ \hline
                  Supervised  &84.4$\pm$ 0.3 & 84.4$\pm$ 0.3 &  84.4$\pm$ 0.3 \\
                  SSL-NBN  &100.0$\pm$ 0.0 & 100.0$\pm$ 0.0&  100.0$\pm$ 0.0 \\
                  SSL-BN  &60.7$\pm$ 1.5 &50.0$\pm$ 0.0  & 50.0$\pm$ 0.0  \\
                  SSL-FBN  &89.7$\pm$ 0.7 &87.0$\pm$ 1.1  & 81.0$\pm$ 1.7  \\
                  \textbf{R-SSL-IFT}  &\textbf{100.0$\pm$ 0.0} &\textbf{100.0$\pm$ 0.0}  &\textbf{100.0$\pm$ 0.0}  \\
                  \textbf{R-SSL-Meta}  &\textbf{100.0$\pm$ 0.0} &\textbf{100.0$\pm$ 0.0}  &\textbf{100.0$\pm$ 0.0}  \\
                 \hline
& \multicolumn{3}{c}{\textit{Boundary OODs}}  \\  
                                      & 25\%& 50\% &75\% \\ \hline
                  Supervised  &84.4$\pm$ 0.3 & 84.4$\pm$ 0.3 &  84.4$\pm$ 0.3 \\
                  SSL-NBN  &87.3$\pm$ 0.3 & 83.4$\pm$ 0.4 &  82.4$\pm$ 0.5 \\
                  SSL-BN  &84.7$\pm$ 0.4 & 82.3$\pm$ 0.3 &  80.4$\pm$ 0.5  \\
                  SSL-FBN  &89.7$\pm$ 0.7 &87.0$\pm$ 1.1  & 81.0$\pm$ 1.7  \\
                  \textbf{R-SSL-IFT}  &\textbf{100.0$\pm$ 0.0} &\textbf{100.0$\pm$ 0.0}  &\textbf{100.0$\pm$ 0.0}  \\
                 \textbf{ R-SSL-Meta}  &\textbf{100.0$\pm$ 0.0} &\textbf{100.0$\pm$ 0.0}  &\textbf{100.0$\pm$ 0.0} \\
                 \hline 
\end{tabular}
\vspace{2mm}
\label{OOD_ex:syn}
\end{table}

\begin{table}[th!]
\small
\caption{Test accuracies for different $P$ at OOD ratio = 50\% on the synthetic dataset.} 
\vspace{1mm}
\centering
\begin{tabular}{c|cc}
\hline
Model &    \textit{ Faraway OODs} & \textit{Boundary OODs} \\
                                    \hline
                  R-SSL-IFT P=1  &55.0$\pm$ 2.1 &83.9$\pm$ 0.7  \\
                 R-SSL-IFT P=5  &91.0$\pm$ 3.1 &93.9$\pm$ 0.9 \\
                 R-SSL-IFT P=10  &100.0$\pm$ 0.0 &\textbf{100.0$\pm$ 0.0} \\
                 \hline 

\end{tabular}
\label{OOD_ex:syn_P}
\end{table}

\subsection{Real-world Dataset Details}
\noindent {\bf Datasets.}
We consider four image classification benchmark datasets. (1)~\textbf{MNIST}: a handwritten digit classification dataset, with
50,000/ 10,000/ 10,000 training/validation/test samples, with training data, split into two groups -- labeled and unlabeled in-distribution (ID) images (labeled data has ten images per class), and with two types of OODs: a) Fashion MNIST, b) Mean MNIST (where the OOD instances are mean images of two classes, more details refer to see Appendix D);
(2)~\textbf{CIFAR10}: a natural image dataset with 45,000/ 5,000/ 10,000 training/validation/test samples from 10 object classes, and following~\cite{oliver2018realistic}, we adapt CIFAR10 to a 6-class classification task, using 400 labels per class (from the 6 classes) and rest of the classes OOD (ID classes are "bird", "cat", "deer", "dog", "frog", "horse", and OOD data are from classes: "airline", "automobile", "ship", "truck"); 
(3)~\textbf{CIFAR100}: another natural image dataset with 45,000/5000/10,000 training/validation/test images, similar to CIFAR10, we adapt CIFAR100 to a 50-class classification task, with 40 labels per class -- the ID classes are the first 50 classes, and OOD data corresponds to the last 50 classes;
(4)~\textbf{SVHN-extra}: This is SVHN dataset with 531,131 additional digit images~\cite{tarvainen2017mean}, and we adapt SVHN-extra to a 5-class classification task, using 400 labels per class. The ID classes are the first five classes, and OOD data corresponds to the last five classes.


\noindent {\bf Comparing Methods.}
To evaluate the effectiveness of our proposed weighted robust SSL approaches, 
we compare with five state-of-the-art robust SSL approaches, including UASD~\cite{chen2020semi}, DS3L (DS3L)~\cite{guo2020safe}, L2RW~\cite{ren2018learning}, and MWN~\cite{shu2019meta}. The last two approaches L2RW and MWN, were originally designed for robust supervised learning (SL), and we adapted them to robust SSL by replacing the supervised learning loss function with an SSL loss function. 
We compare these robust approaches on four representative SSL methods, including Pseudo-Label (PL)~\cite{lee2013pseudo}, $\Pi$-Model (PI)~\cite{laine2016temporal,sajjadi2016regularization}, Mean Teacher (MT)~\cite{tarvainen2017mean}, and Virtual Adversarial Training (VAT)~\cite{miyato2018virtual}. One additional baseline is the supervised learning method, named "Sup," which ignored all the unlabeled examples during training. 
All the compared methods (except UASD~\cite{chen2020semi}) were built upon the open-source Pytorch implementation\footnote{\url{https://github.com/perrying/realistic-ssl-evaluation-pytorch}} by~\cite{oliver2018realistic}. As UASD has not released its implementation, we implemented UASD by ourselves.
For DS3L, we implemented it based on the released code \footnote{\url{https://github.com/guolz-ml/DS3L}}.
For L2RW~\cite{ren2018learning}, we used the open-source Pytorch implementation\footnote{\url{https://github.com/danieltan07/learning-to-reweight-examples}}and adapted it to the SSL settings. For MWN~\cite{shu2019meta}, we used the authors' implementation\footnote{\url{https://github.com/xjtushujun/meta-weight-net}} and adapt to SSL. 

\begin{table}[th!]
\caption{Hyperparameter settings used in MNIST experiments for four representative SSL. All robust SSL methods (e.g., ours (WR-SSL), DS3L, and UASD) are developed based on these representative SSL methods.} 
\centering
\begin{tabular}{lc}
\hline
 \multicolumn{2}{c}{\textbf{Shared}} \\
\hline
Learning decayed by a factor of & 0.2\\
at training iteration & 1,000 \\
coefficient = 1 (Do not use warmup) & \\
\hline 
 \multicolumn{2}{c}{\textbf{Supervised}} \\
\hline
Initial learning rate & 0.003 \\
\hline
 \multicolumn{2}{c}{$\Pi$-\textbf{Model}} \\
\hline
Initial learning rate &0.003\\
Max consistency coefficient & 20\\
\hline
 \multicolumn{2}{c}{\textbf{Mean Teacher}} \\
 \hline
Initial learning rate & 0.0004 \\
Max consistency coefficient & 8 \\ 
Exponential moving average decay & 0.95 \\
\hline
\multicolumn{2}{c}{\textbf{VAT}} \\
\hline
Initial learning rate &0.003 \\
Max consistency coefficient &0.3 \\
VAT $\epsilon$ & 3.0 \\
VAT $\xi$ & $10^{-6}$ \\
\hline
\multicolumn{2}{c}{\textbf{Pseudo-Label}} \\
\hline
Initial learning rate& 0.0003 \\
Max consistency coefficient &1.0 \\
Pseudo-label threshold &0.95 \\
\hline
\end{tabular}
\label{hyper_mnist}
\end{table}

\noindent {\bf Setup.}
In our experiments, we implement our approaches (Ours-SSL) for four representative SSL methods, including Pseudo-Label (PL), $\Pi$-Model (PI), Mean Teacher (MT), and Virtual Adversarial Training (VAT). The term "SSL" in Ours-SSL represents the SSL method, e.g., (Ours-VAT denotes our weighted robust SSL algorithm implemented based on VAT.
We used the standard LeNet model as the backbone for the MNIST experiment and used \textit{WRN-28-2}~\cite{zagoruyko2016wide} as the backbone for CIFAR10, CIFAR100, and SVHN experiments. For a comprehensive and fair comparison of the CIFAR10 experiment, we followed the same experiment setting of~\cite{oliver2018realistic}. All the compared methods were built upon the open-source Pytorch implementation by ~\cite{oliver2018realistic}. The code and datasets are temporarily available for reviewing purposes at here\footnote{\url{https://anonymous.4open.science/r/WR-SSL-406F/README.md}}.
Our code and datasets are also submitted as supplementary material.

\noindent {\bf Hyperparameter setting.} For our WR-SSL approach, we update the weights only using last layer for the inner optimization,  we set $J=3$ (for inner loop gradient steps), $P = 5$ (for inverse Hessian approximation), $K=20$ (for CRW), $\lambda = 10^{-7}$ (for L1), and $L = 5$ (for infrequent update) for all experiments. We trained all the networks for 2,000 updates with a batch size of 100 for MNIST experiments, and 500,000 updates with a batch size of 100 for CIFAR10, CIFAR100, and SVHN experiments.
We did not use any form of early stopping but instead continuously monitored the validation set performance and reported test errors at the point of the lowest validation error.
We show the specific hyperparameters used with four representative SSL methods on MNIST experiments in Table~\ref{hyper_mnist}. For CIFAR10, we used the same hyperparameters as~\cite{oliver2018realistic}. For CIFAR100 and SVHN datasets, we used the same hyperparameters as CIFAR10.

\subsection{Performance with different real-world OOD datasets}
In all experiments, we report the performance over five runs. Denote OOD ratio= $ \mathcal{U}_{ood}/(\mathcal{U}_{ood}+ \mathcal{U}_{in})$ where $\mathcal{U}_{in}$ is ID unlabeled set,  $\mathcal{U}_{ood}$ is OOD unlabeled set, and $\mathcal{U}=\mathcal{U}_{in} + \mathcal{U}_{ood}$.

\begin{figure*}[!t]
    \centering
    \begin{subfigure}[b]{0.24\textwidth}
        \centering
        \includegraphics[width=\linewidth]{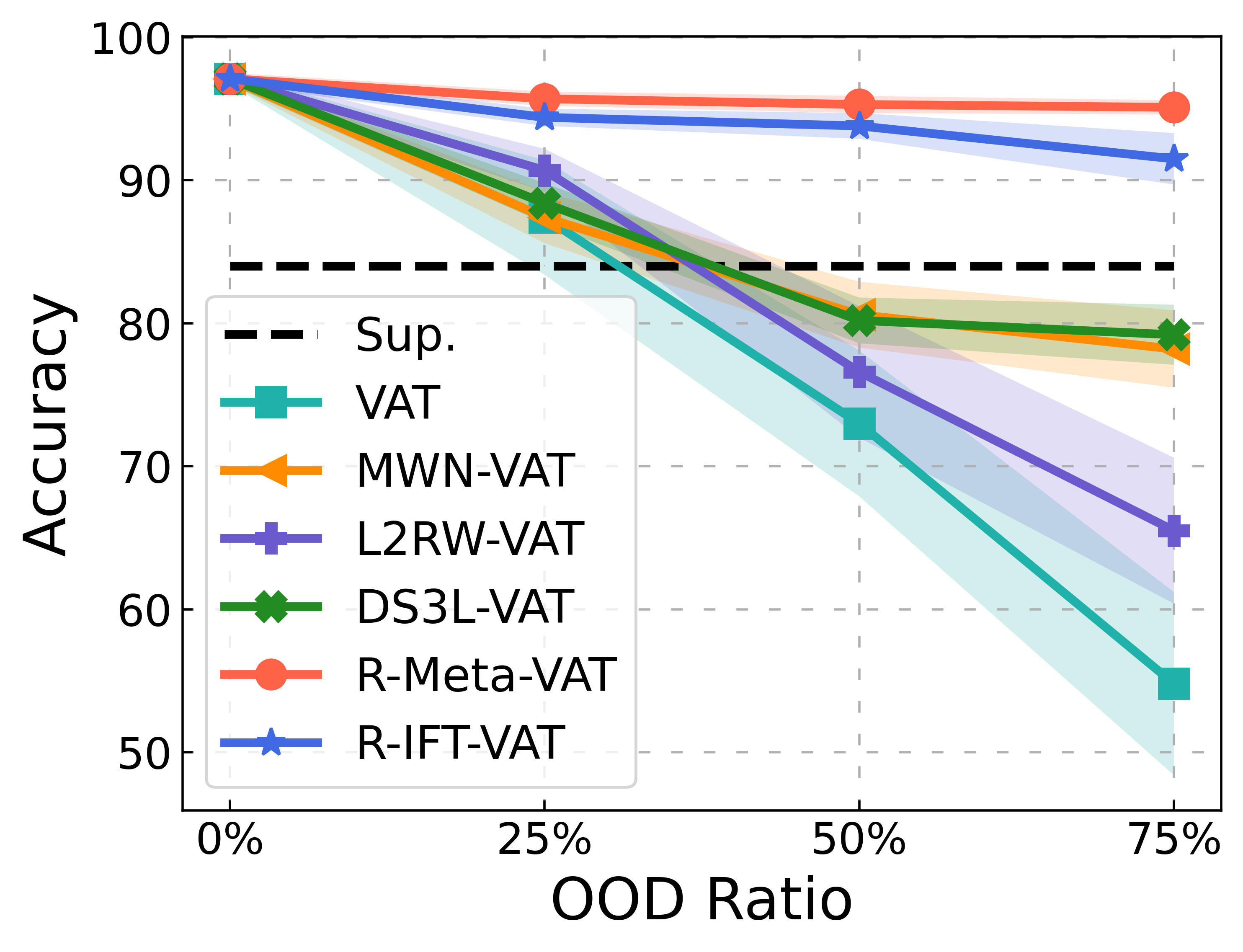}
        \caption{FashionMNIST.}
    \end{subfigure}
    \hfill
    \begin{subfigure}[b]{0.24\textwidth}
        \centering
        \includegraphics[width=\linewidth]{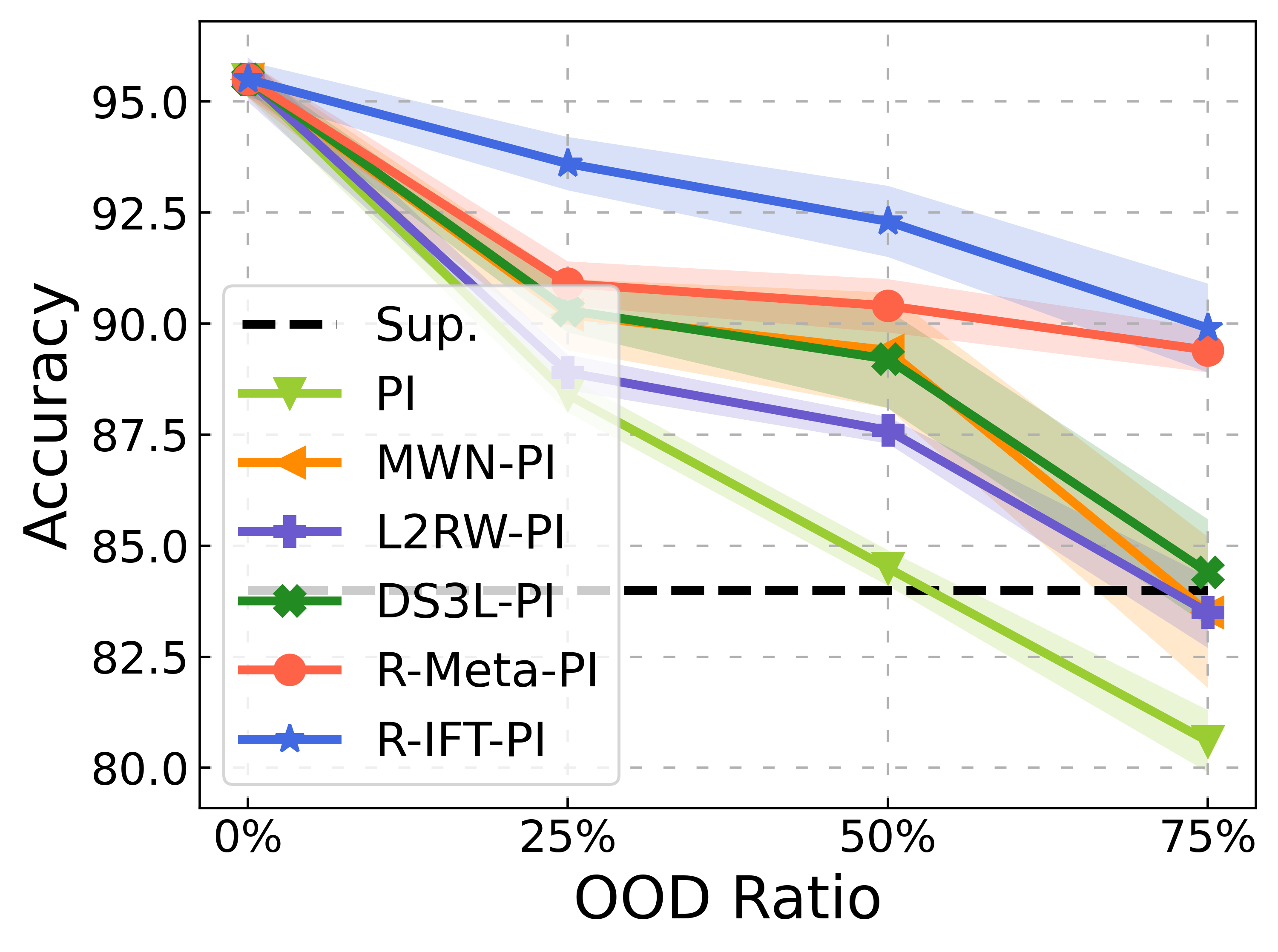}
        \caption{FashionMNIST.}
    \end{subfigure}
     \hfill
    \begin{subfigure}[b]{0.24\textwidth}
        \centering
        \includegraphics[width=\linewidth]{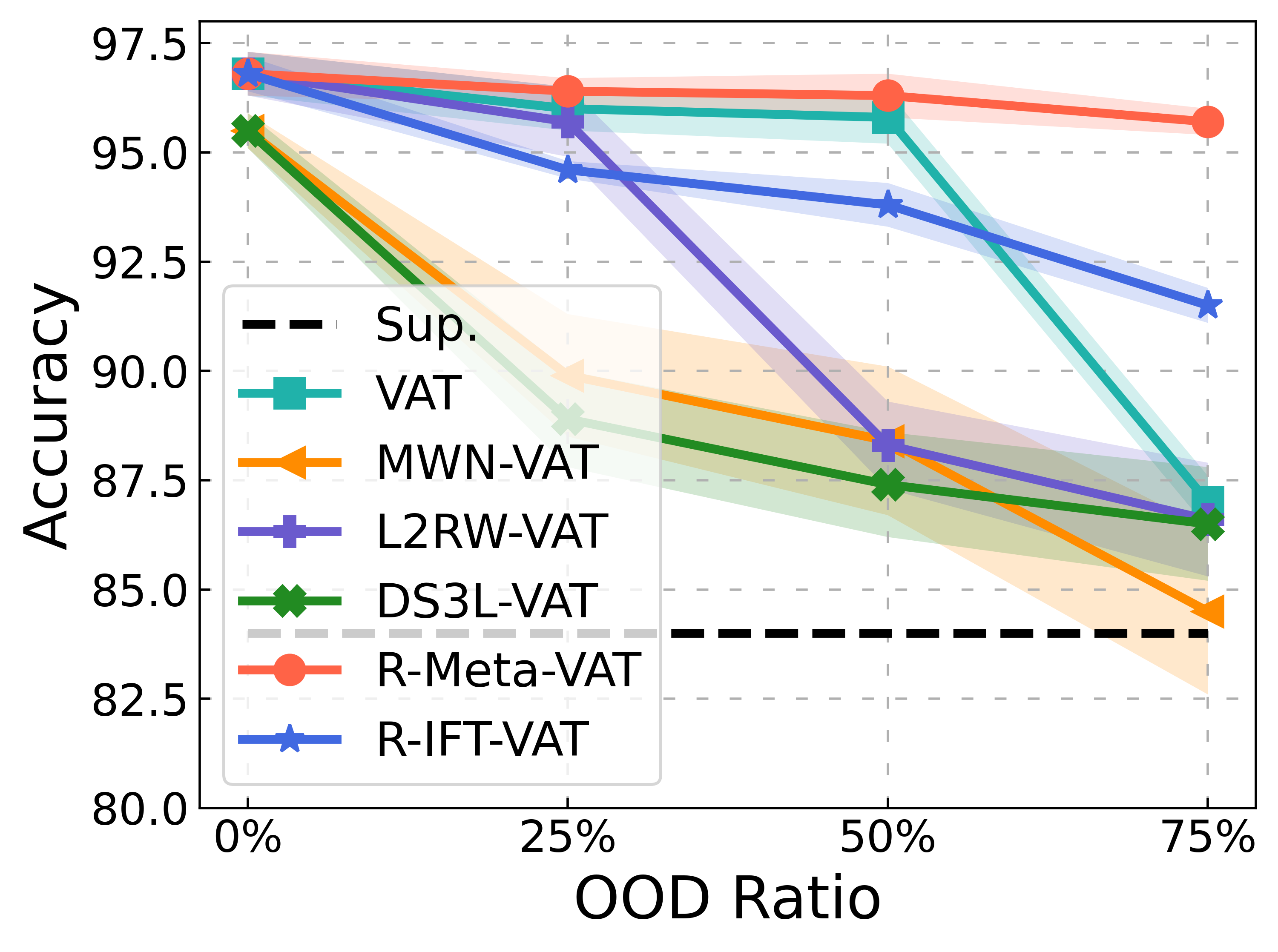}
        \caption{Boundary OOD.}
    \end{subfigure}
     \hfill
    \begin{subfigure}[b]{0.24\textwidth}
        \centering
        \includegraphics[width=\linewidth]{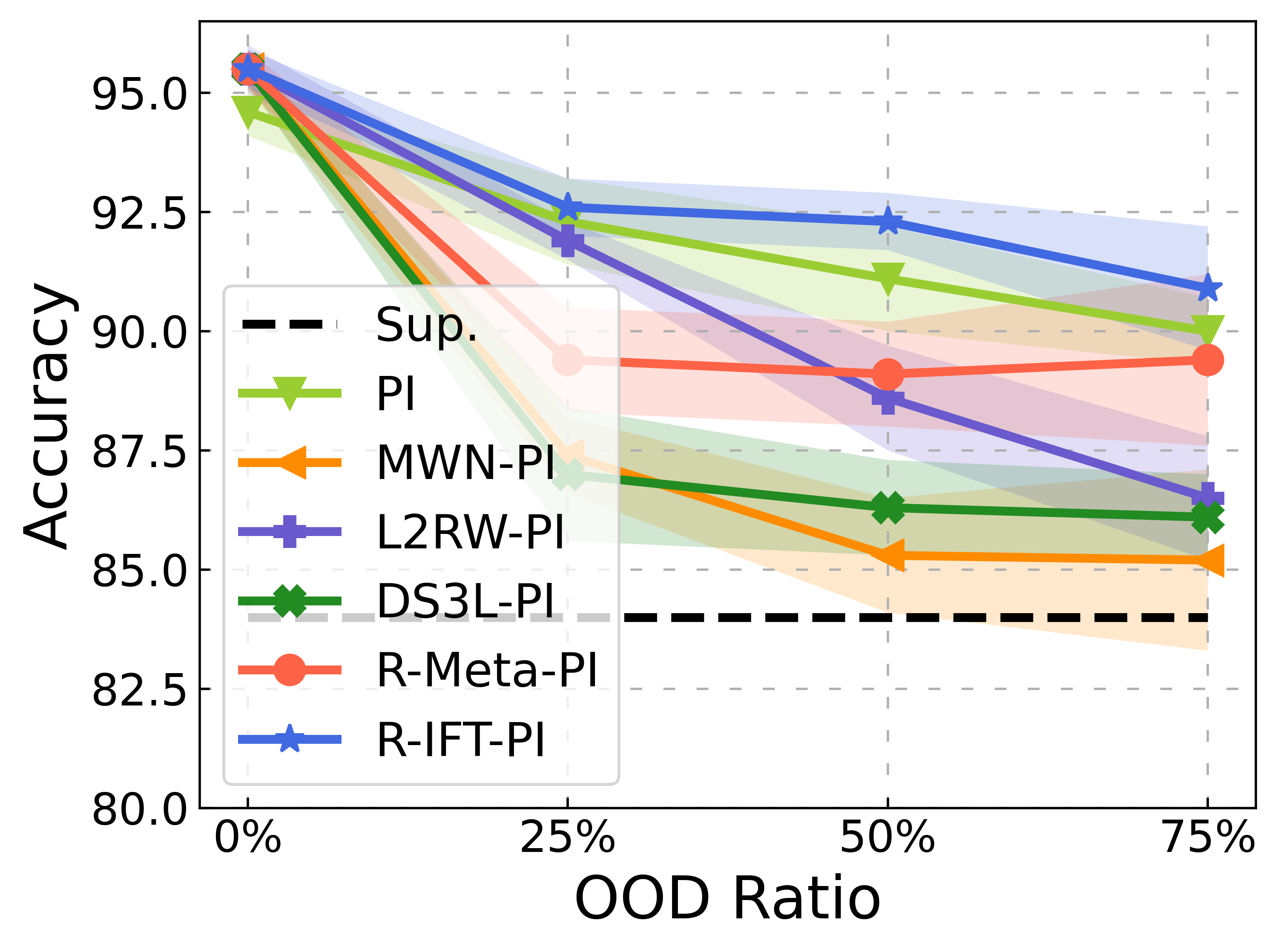}
        \caption{Boundary OOD.}
    \end{subfigure}
    \caption{Classification accuracy with varying OOD ratio on MNIST. (a)-(b) consider faraway OODs with batch normalization; (c)-(d) consider boundary OODs without batch normalization. Shaded regions indicate standard deviation.}
    \label{fig:rssl_mnist}
\end{figure*}

\noindent {\bf Impact of faraway OODs on SSL performance (with batch normalization).}  
We used FashionMNIST~\citep{xiao2017fashion} dataset to construct a set of OODs $\mathcal{U}_{ood}$ faraway from the decision boundary of the MNIST dataset, as FashionMNIST and MNIST have been shown very different and considered as cross-domain benchmark datasets~\citep{meinke2019towards}. 
As shown in Fig~\ref{fig:rssl_mnist} (a)-(b),  with the OOD ratio increase, the performance of the existing SSL method decreases rapidly, whereas our approach can still maintain clear performance improvement, i.e., the outperformance of our methods are fairly impressive(e.g., 10\% increase with R-IFT-PI over PI model when OOD ratio = 75\%). Compare with other robust SSL methods, our methods surprisingly improves the accuracy and suffers much less degradation under a high OOD ratio. 

\begin{figure*}[!t]
    \centering
    \begin{subfigure}[b]{0.4\textwidth}
        \centering
        \includegraphics[width=\linewidth]{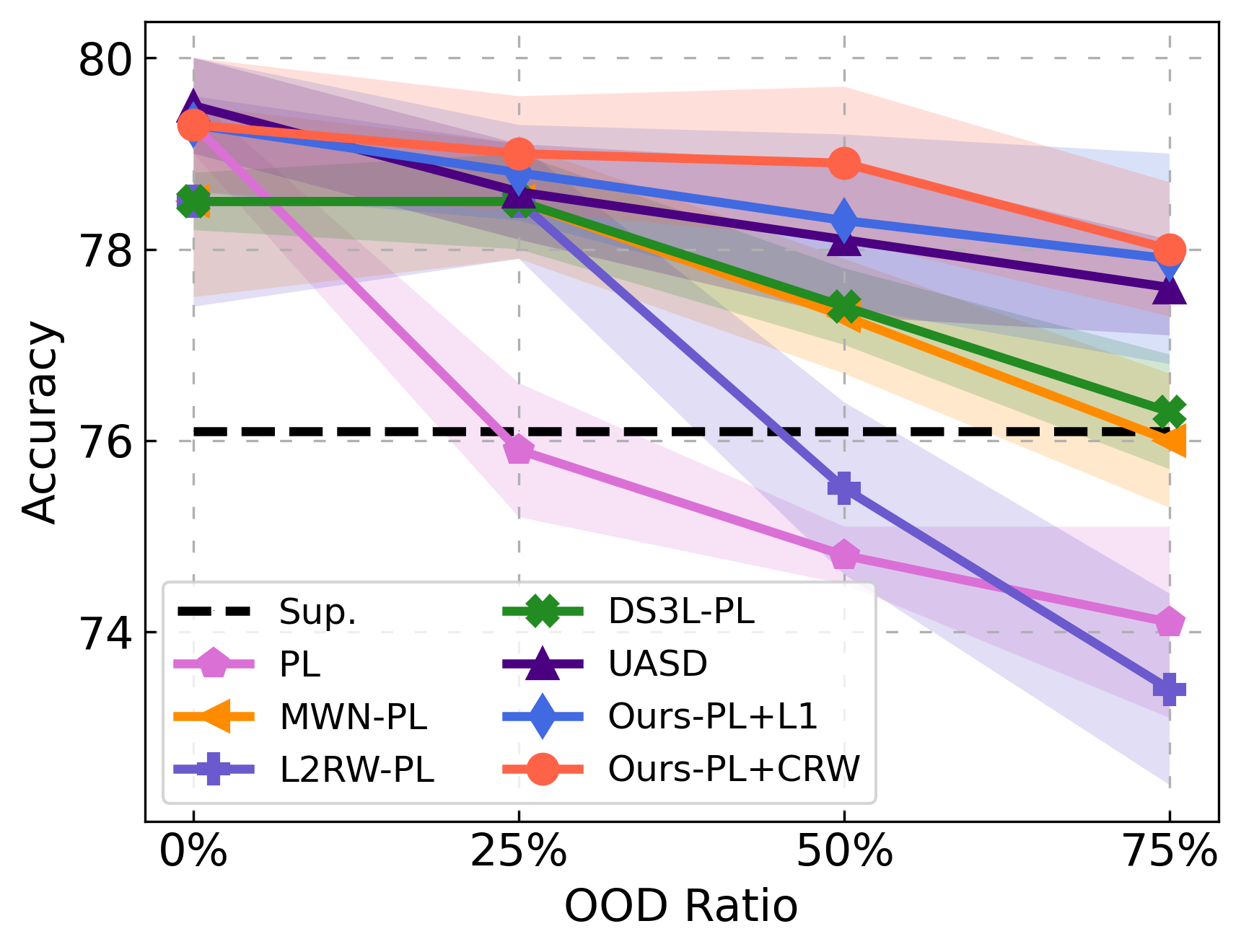}
        \caption{CIFAR10.}
    \end{subfigure}
    \begin{subfigure}[b]{0.4\textwidth}
        \centering
        \includegraphics[width=\linewidth]{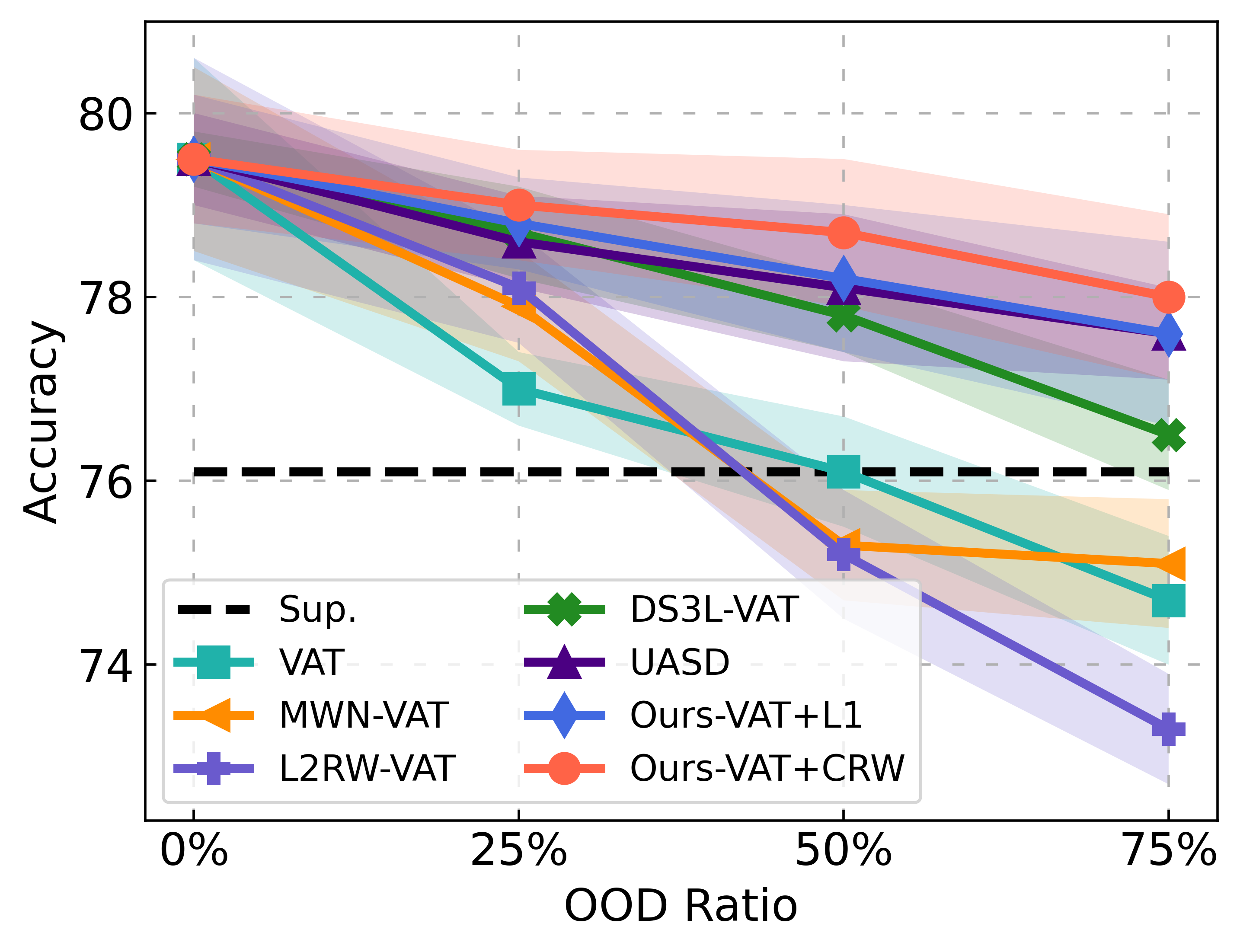}
        \caption{CIFAR10.}
    \end{subfigure}
    \caption{Classification accuracy with varying OOD ratio on CIFAR10. We use \textit{WRN-28-2} (contains BN module) as the backbone. Shaded regions indicate standard deviation.}
  \label{fig:cifar6_acc}
\end{figure*}

\noindent {\bf Impact of boundary OODs on SSL performance (without batch normalization).} We generate boundary OOD by mix up (fusion) existing ID unlabeled samples, i.e., $\hat{\x}_{ood} = 0.5(\x_i+ \x_j)$, where $\x_i$, $\x_j$ is ID images from a different class, and $\hat{\x}_{ood}$ can be regarded as boundary OOD~\citep{guo2019mixup}. As shown in Fig~\ref{fig:rssl_mnist} (c)-(d), we get a similar result pattern that the accuracy of existing SSL methods decreases when the OOD ratio increases. Across different OOD ratios, particularly our method significantly outperformed among all (e.g., 10\% increase with Meta-VAT over VAT when OOD ratio = 75\%).
We also conduct experiments based on other SSL algorithms, results are shown in the supplementary material.

\noindent {\bf Impact of mixed OODs on SSL performance (with batch normalization).} 
We followed~\citep{oliver2018realistic} to adapt CIFAR10 for a 6-class classification task, using 400 labels per class. The ID classes are: "bird", "cat", "deer", "dog", "frog", "horse", and OOD classes are: "airline", "automobile", "ship", "truck". 
As our implementation follows ~\citep{oliver2018realistic} that freeze BN layers for the WRN model, we freeze BN layers for all the methods for CIFAR10 for fair comparisons. In this dataset, the examples of the OOD classes were considered as OODs. As these OODs are from the same dataset, it may have OODs close to or far away from the decision boundary of the ID classes. We hence called these OODs as mixed-type OODs. The averaged accuracy of all compared methods v.s. OOD ratio is plotted in Fig~\ref{fig:cifar6_acc}. Across different OOD ratios, particularly our method significantly outperformed among all, strikingly exceeding the performance when the OOD ratio is large (i.e., 4.5\% increase with Ours-Meta over L2RW when OOD ratio = 75\%). Unlike most SSL methods that degrade drastically when increasing the OOD ratio, ours achieves a stable performance even in a 75\% OOD ratio. 

\begin{figure*}[!t]
    \centering
    \begin{subfigure}[b]{0.2\textwidth}
        \centering
        \includegraphics[width=\linewidth]{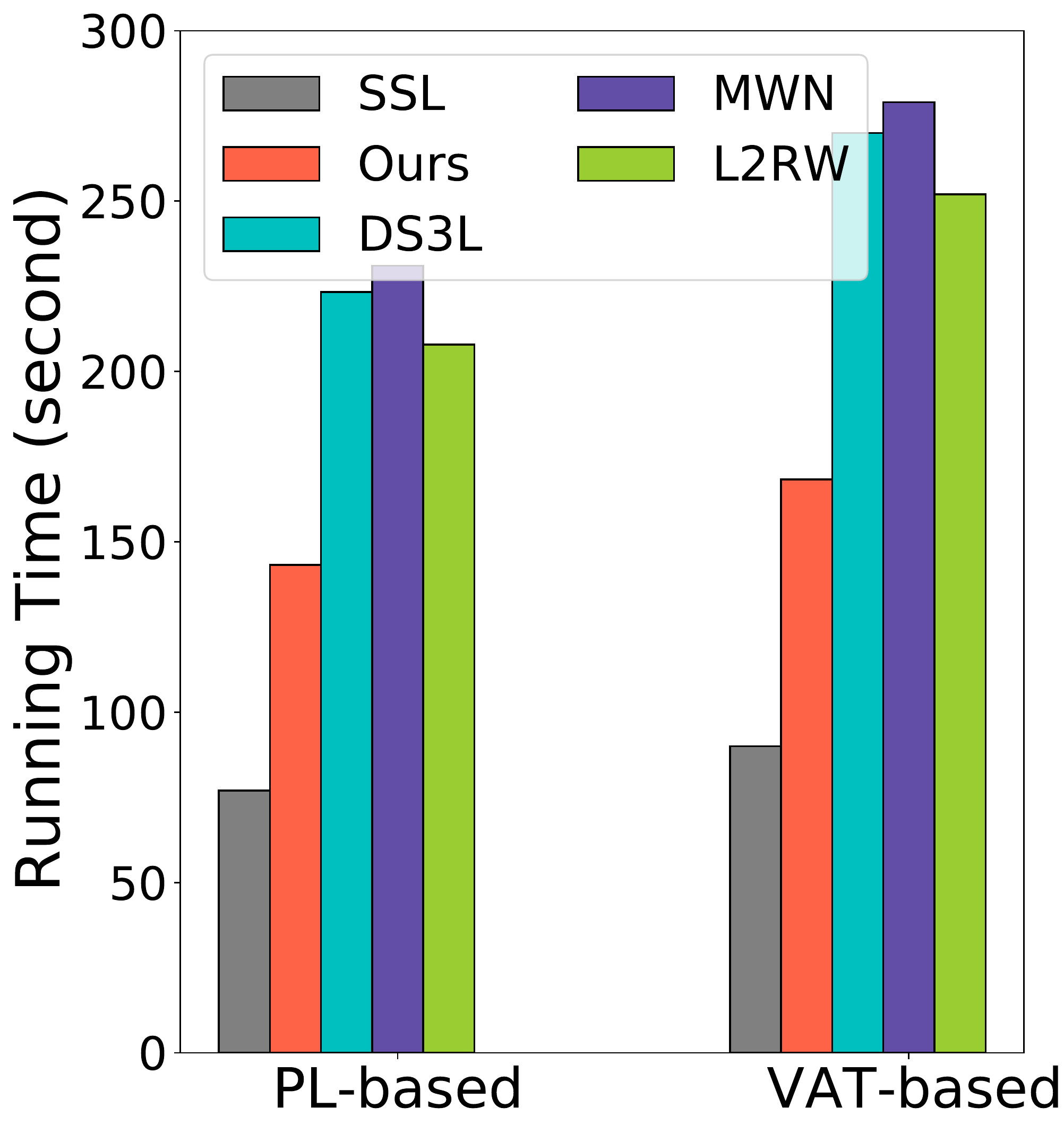}
        \caption{MNIST.}
    \end{subfigure}
    \begin{subfigure}[b]{0.2\textwidth}
        \centering
        \includegraphics[width=\linewidth]{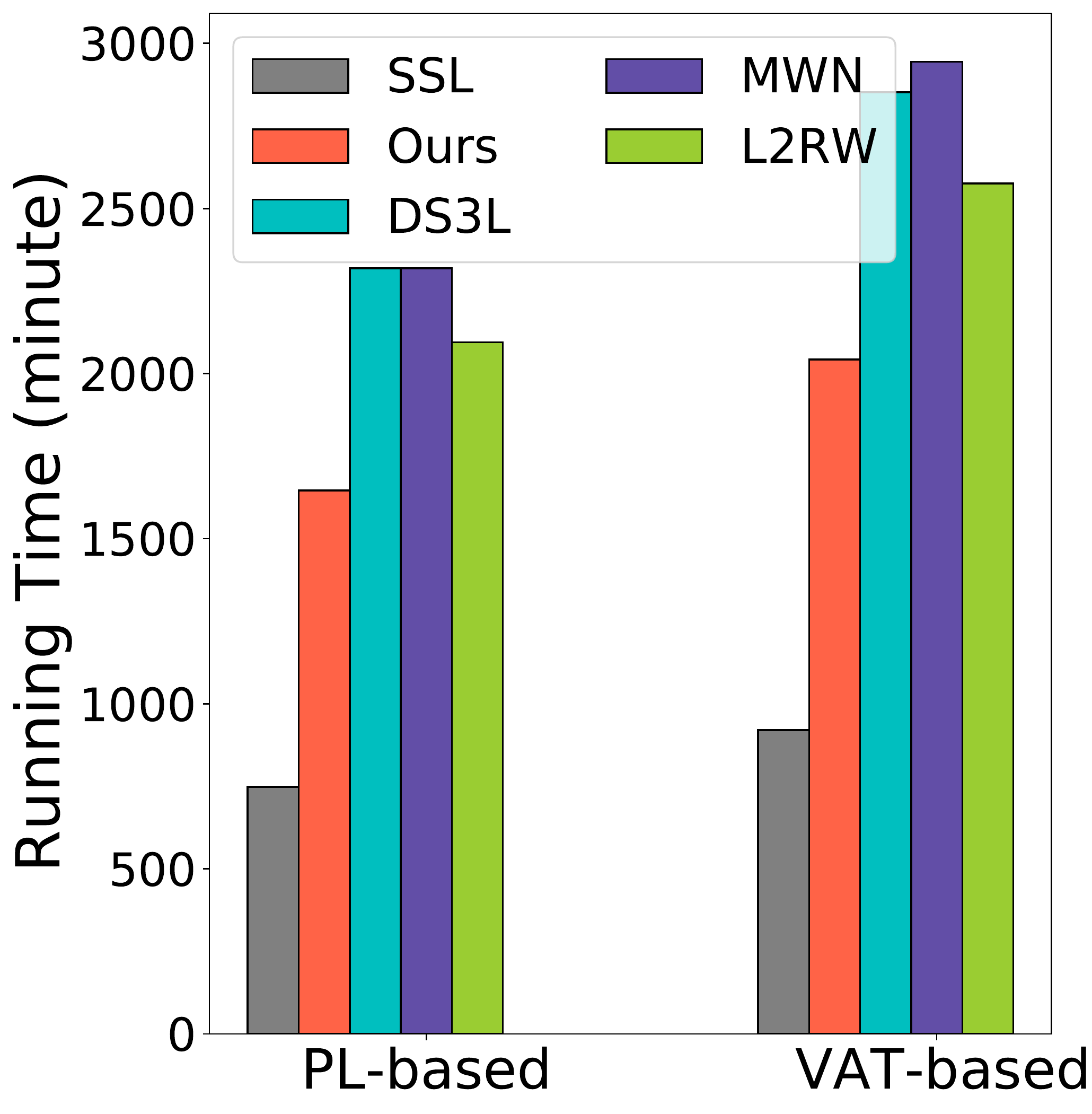}
        \caption{CIFAR10.}
    \end{subfigure}
    \begin{subfigure}[b]{0.18\textwidth}
        \centering
        \includegraphics[width=\linewidth]{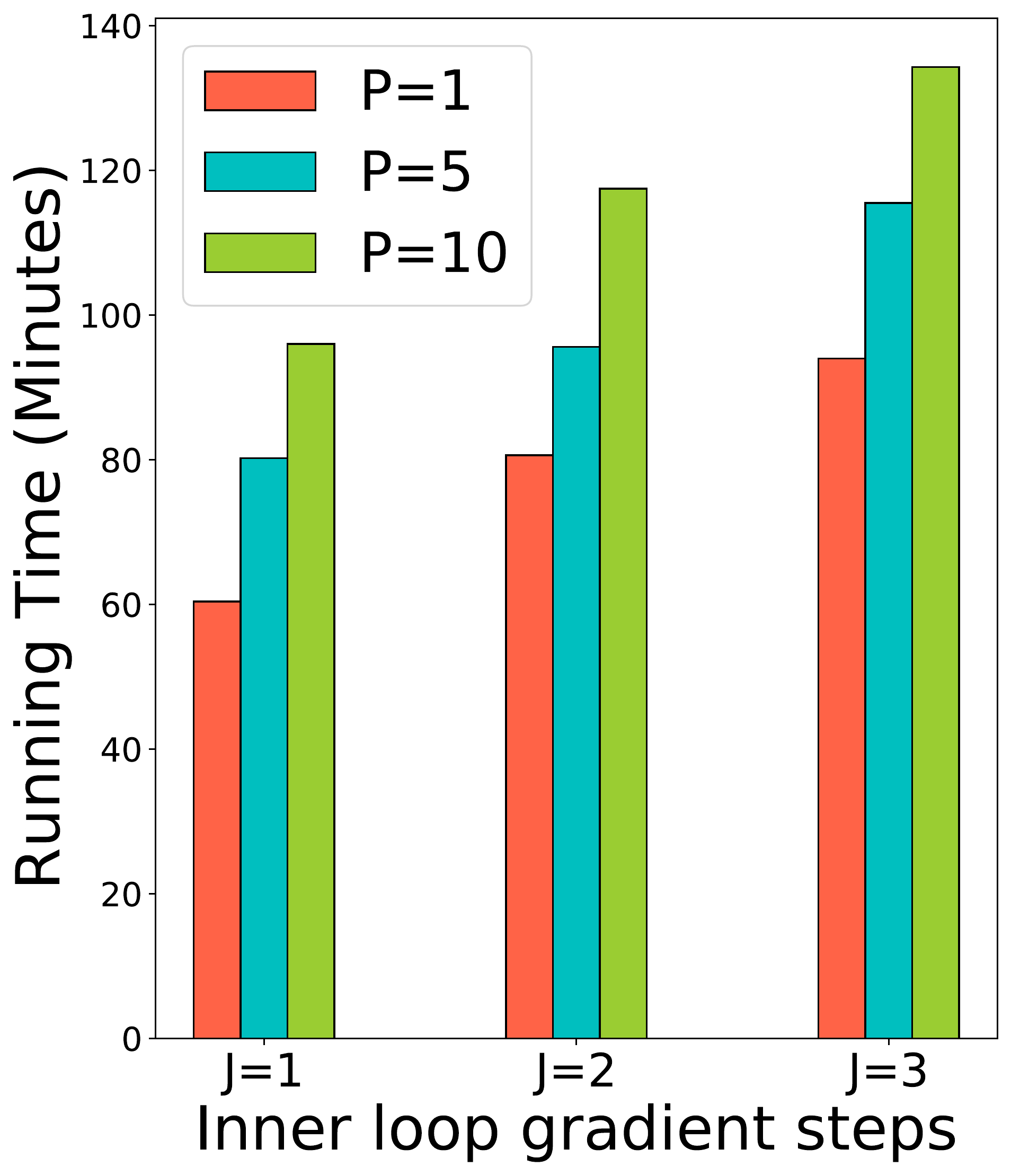}
        \caption{Running Time}
    \end{subfigure}
    \begin{subfigure}[b]{0.175\textwidth}
        \centering
        \includegraphics[width=\linewidth]{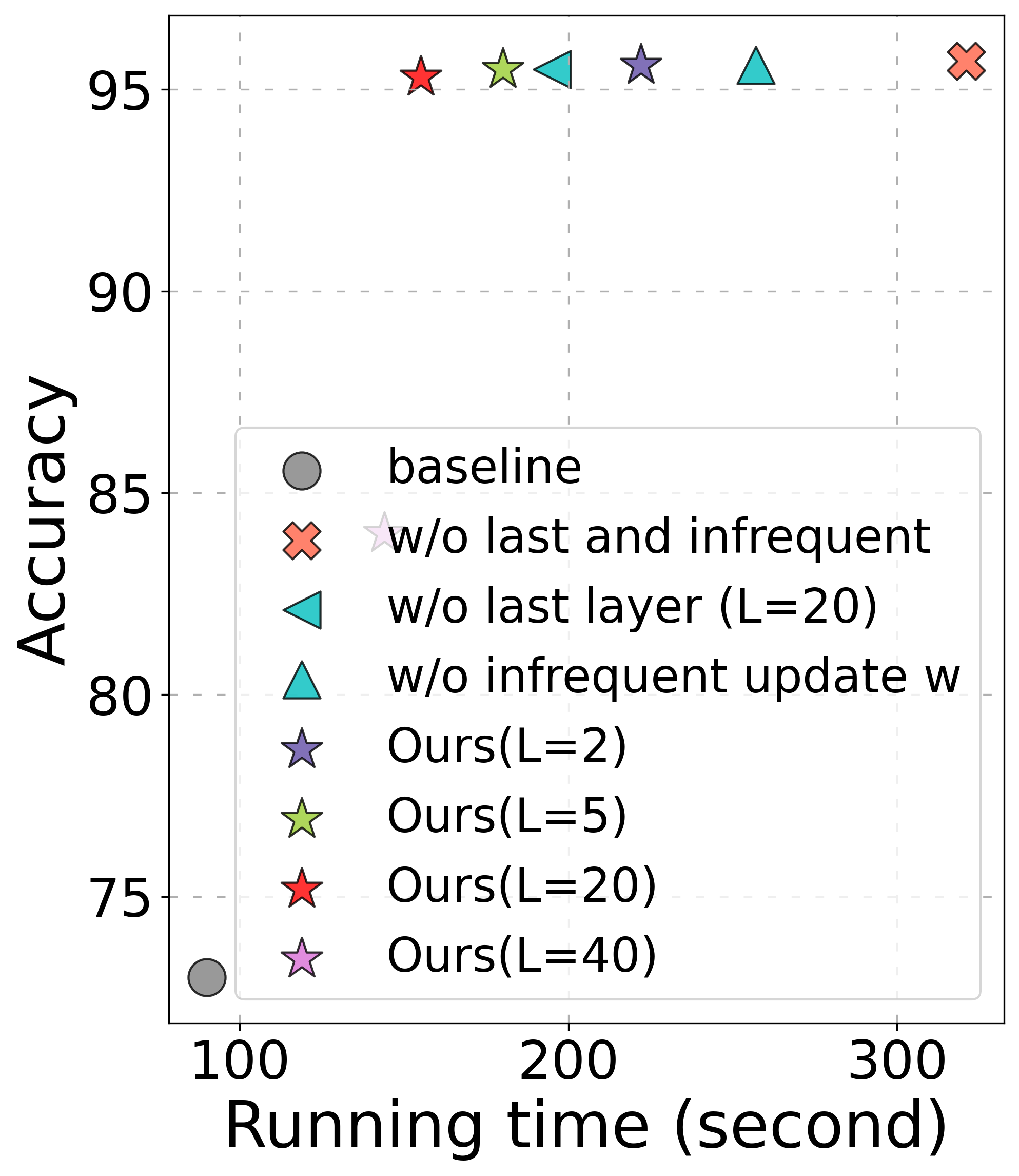}
        \caption{MNIST.}
    \end{subfigure}
    \begin{subfigure}[b]{0.175\textwidth}
        \centering
        \includegraphics[width=\linewidth]{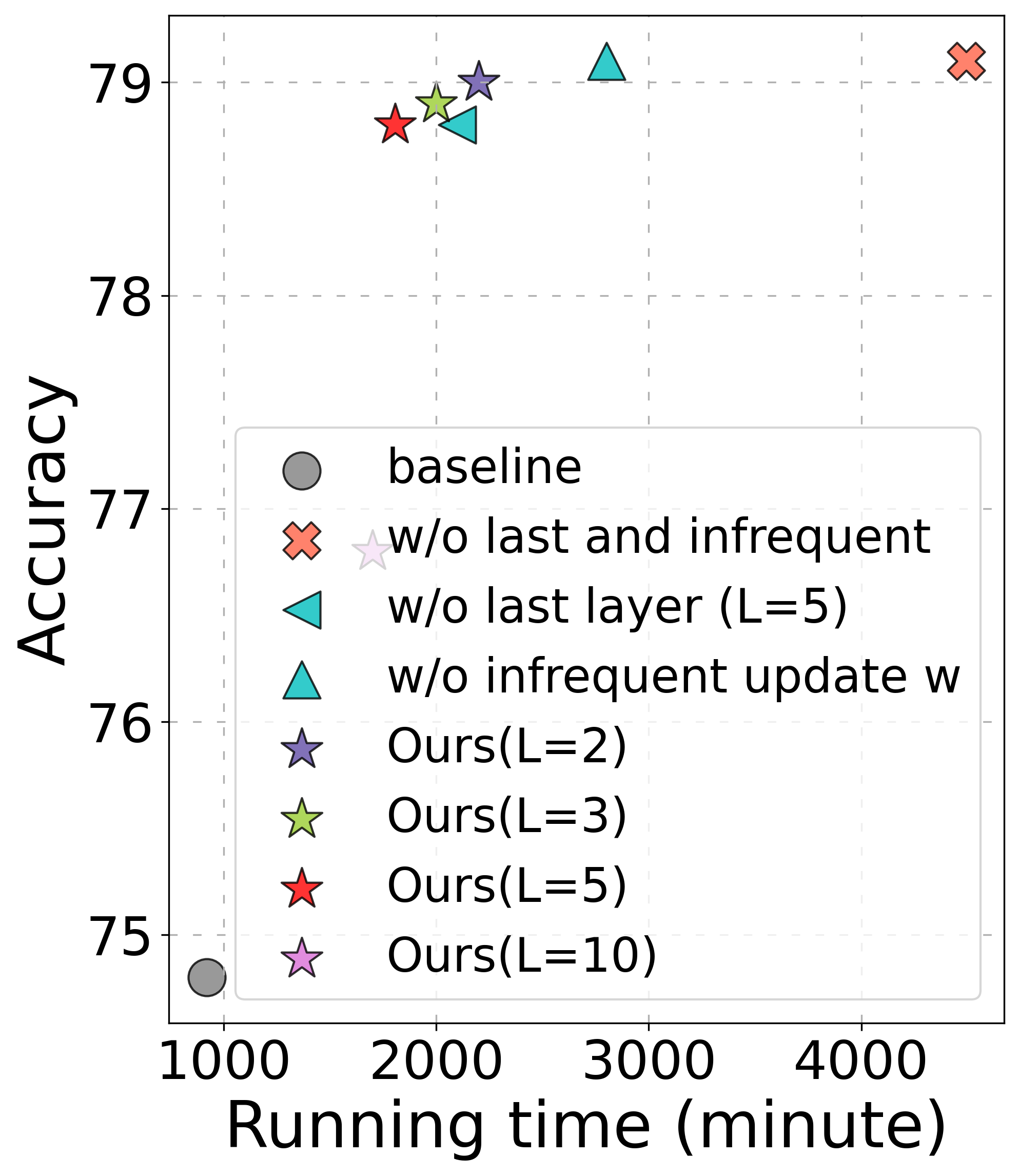}
        \caption{CIFAR10.}
    \end{subfigure}
    \caption{Running time results. (a)-(b) show our proposed approaches are only $1.7\times$ to $1.8\times$ slower compared to base SSL algorithms, while other robust SSL methods are $3\times$ slower. (c) shows that the running time of our method would increase with $J$ (inner loop gradients steps) and $P$ (inverse Hessian approximation) increase.
    (d)-(e) show the running time of our strategies with different combinations of tricks viz; last layer updates and updating weights every $L$ iterations. Note that by using only last layer updates, our strategies are around $2\times$ slower. With $L = 5$ and last layer updates, we are around $1.7 \times$ to $1.8 \times$ slower with comparable test accuracy.}
    \label{fig:run time}
\end{figure*}

\subsection{Efficiency Analysis}
To evaluate the efficiency of our proposed approach, we first compare the running time among all methods. Fig~\ref{fig:run time} (a)-(b) shows the running time (relative to the original SSL algorithm) for MNIST and CIFAR-10. We see that our proposed approaches are only $1.7\times$ to $1.8\times$ slower than the original SSL algorithm, while other robust SSL methods (L2RW and DS3L) are almost $3\times$ slower. We note that our implementation tricks can also be applied to these other techniques (DS3L, L2RW, MWN), but this would possibly degrade performance since these approaches' performance is worse than ours even without these tricks. To further analyze our proposed speedup strategies, we plot the running time v.s. accuracy in Fig~\ref{fig:run time} (d)-(e) for different settings (with/without last and with/without infrequent updates). As expected, the results show that, without the only last layer updates and the infrequent updates (i.e. if $L = 1$), Algorithm~\ref{algorithm_robustssl} is $3 \times$ slower than SSL baseline. Whereas with the last layer updates, it is around $2\times$ slower. We get the best trade-off between speed and accuracy considering both $L = 5$ and the last layer updates. In addition, we analyze the efficiency of our approach with varying inner loop gradient steps ($J$) and inverse Hessian approximation ($P$). The result shows that the running time of our method would increase with $J$ and $P$ increase, we choose the best trade-off between speed and accuracy considering $J = 3$ and $P = 5$.

\begin{figure*}[!t]
    \centering
    \begin{subfigure}[b]{0.4\textwidth}
        \centering
        \includegraphics[width=\linewidth]{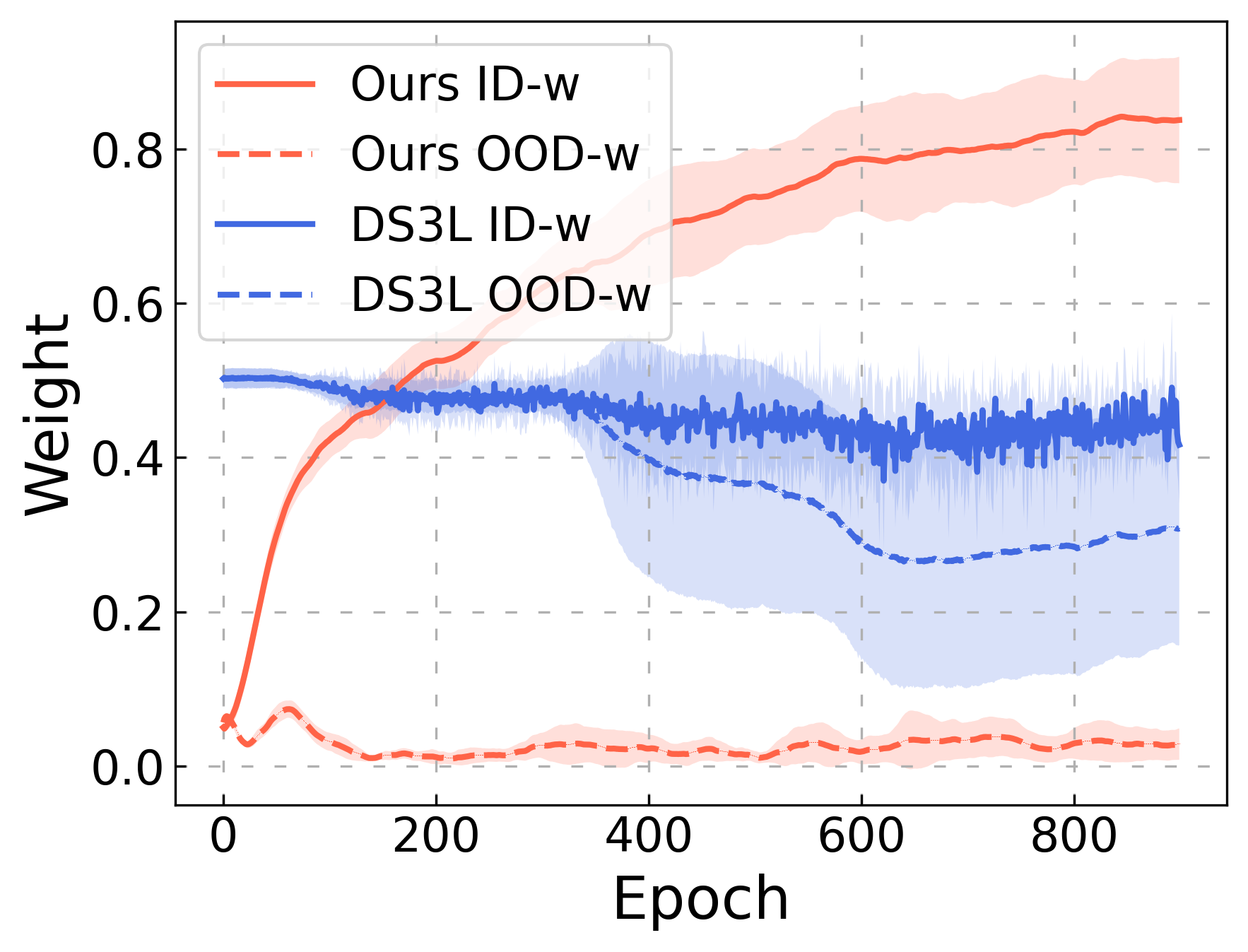}
        \caption{Weight variation curves.}
        \label{fig6a}
    \end{subfigure}
    \begin{subfigure}[b]{0.4\textwidth}
        \centering
        \includegraphics[width=\linewidth]{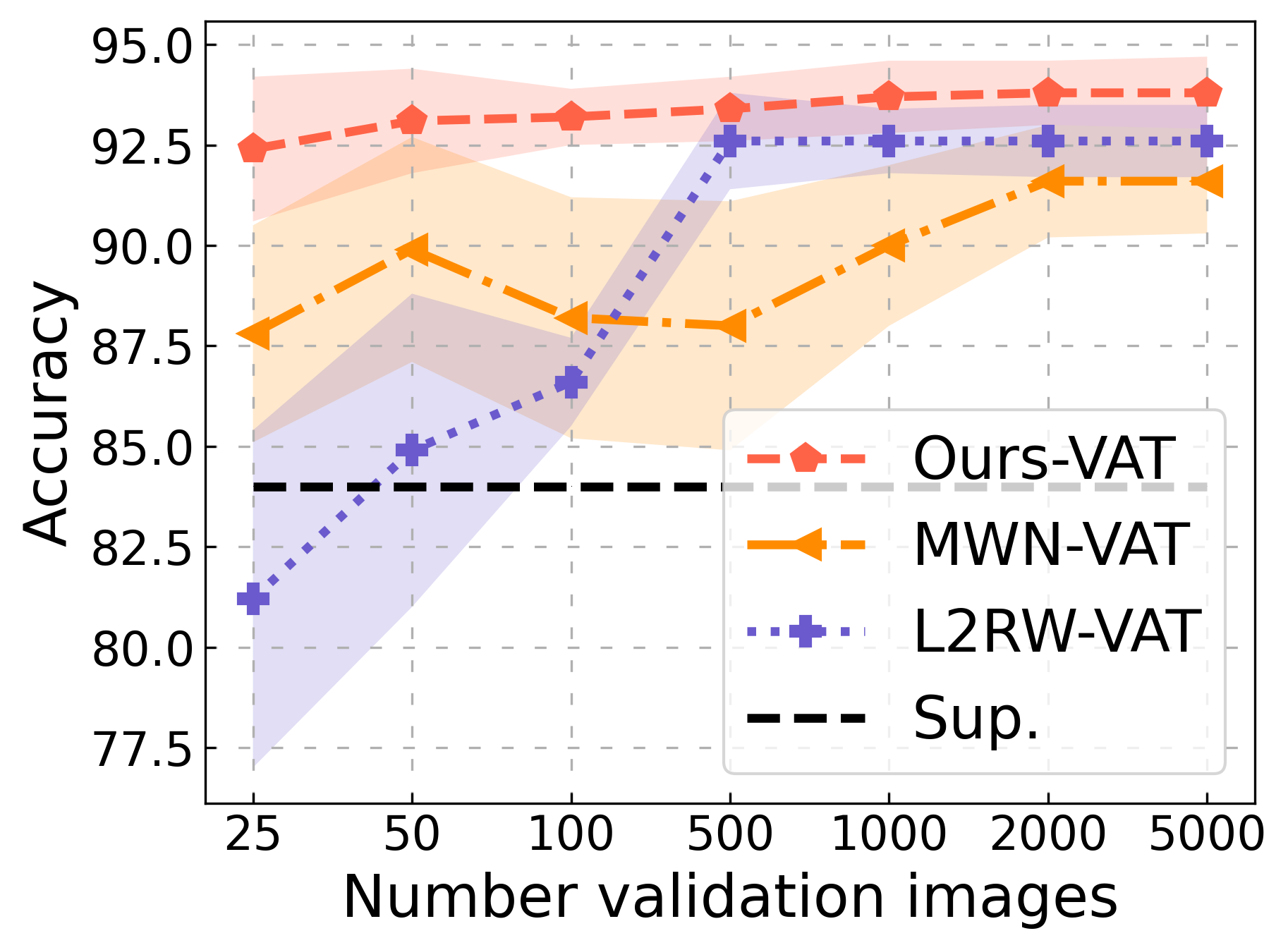}
        \caption{Effect of validation set.}
        \label{fig6b}
    \end{subfigure}
    \caption{(a) shows that our method learns optimal weights for ID and OOD samples; (b) shows that our method is stable even for a small validation set containing 25 images.}
    \label{fig:addition_ana} 
\end{figure*}

\begin{figure*}[!t]
    \centering
    \begin{subfigure}[b]{0.27\textwidth}
        \centering
        \includegraphics[width=\linewidth]{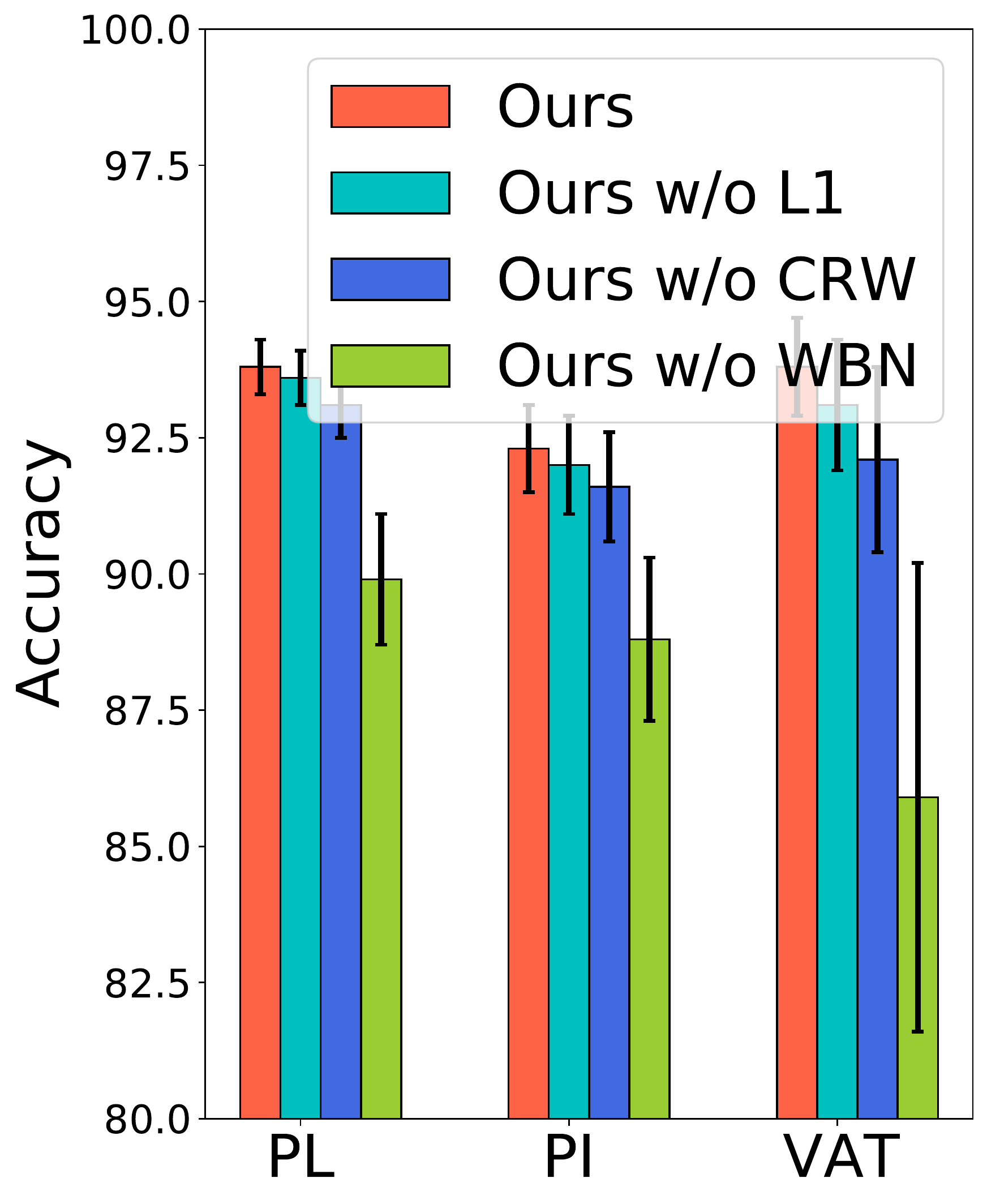}
        \caption{Ablation Study.}
        \label{fig6c}
    \end{subfigure}
    \begin{subfigure}[b]{0.25\textwidth}
        \centering
        \includegraphics[width=\linewidth]{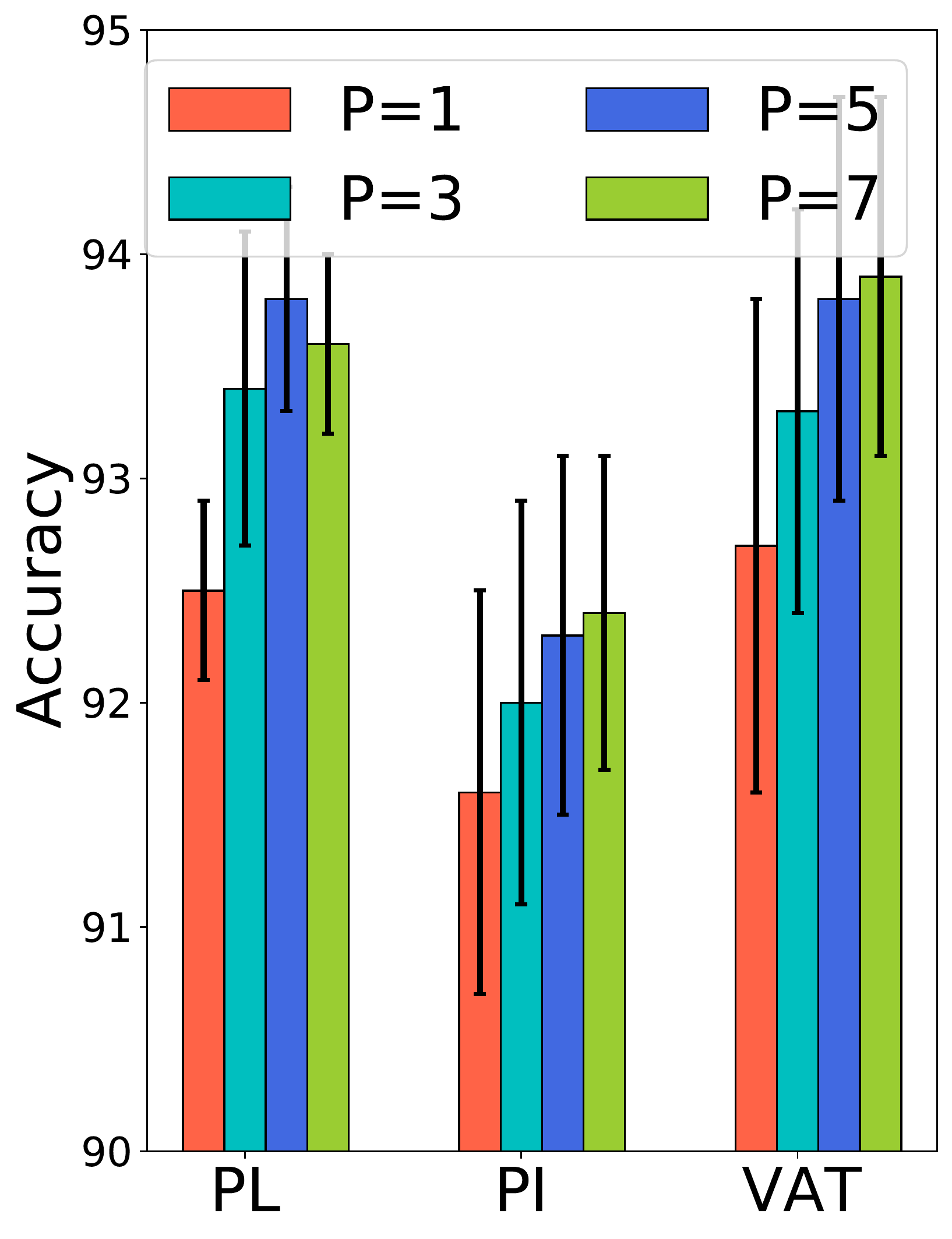}
        \caption{Ablation Study}
        \label{fig6d}
    \end{subfigure}
    \begin{subfigure}[b]{0.25\textwidth}
        \centering
        \includegraphics[width=\linewidth]{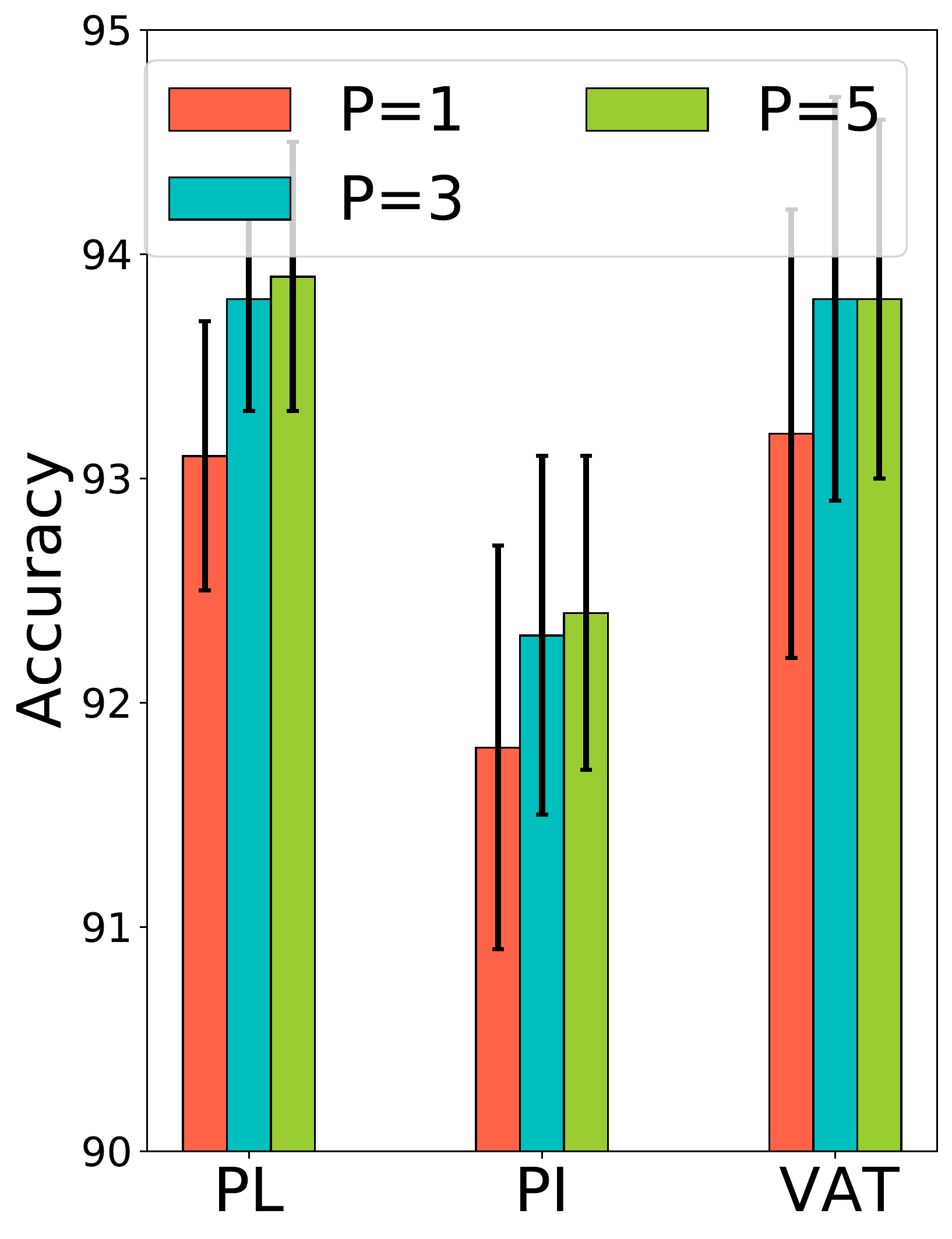}
        \caption{Ablation Study}
        \label{fig6e}
    \end{subfigure}
    \caption{ (a) shows that WBN and CRW (or L1 regularization) are critical in retaining the performance gains of reweighting; (b)-(c) demonstrate that the performance of our approach would increase with (inner loop gradients steps) and inverse Hessian approximation) increase due to high-order approximation.}
    \label{fig:addition_ana2} 
\end{figure*}

\subsection{Additional Analysis}

\begin{table}[h!]
\caption{SVHN-Extra (VAT) with different OOD ratios.} 
\centering
\begin{tabular}{c|ccc}

\textbf{OOD ratio} &   \textit{25\%} & \textit{50\%} & \textit{75\%} \\
\hline
VAT & 94.1$\pm$ 0.5 &93.6$\pm$ 0.7 &92.8$\pm$ 0.9   \\
L2RW-VAT & 96.0$\pm$ 0.6 &93.5$\pm$ 0.8 &92.7$\pm$ 0.8   \\
MWN-VAT & 96.2$\pm$ 0.5 &93.8$\pm$ 1.1 &93.0$\pm$ 1.3   \\
DS3L-VAT & 96.4$\pm$ 0.7 &93.9$\pm$ 1.0 &92.9$\pm$ 1.2   \\
Ours-VAT+L1 & 96.6$\pm$ 0.6 &94.4$\pm$ 0.8 &93.2$\pm$ 1.1   \\
Ours-VAT+CRW &\textbf{96.8$\pm$ 0.7} &\textbf{95.2$\pm$ 0.9} &\textbf{94.9$\pm$ 1.3} \\
\end{tabular}
\label{SVHN}
\end{table}

\begin{table}[h!]
\caption{CIFAR100 (MT) with different OOD ratios.} 
\centering
\begin{tabular}{c|ccc}
 
\textbf{OOD ratio} &   \textit{25\%} & \textit{50\%} & \textit{75\%} \\
                                    \hline
                  DS3L-MT   &60.8$\pm$ 0.5 &60.1$\pm$ 1.1 &57.2$\pm$ 1.2   \\
                 Ours-MT+L1 &61.5$\pm$ 0.4 &60.7$\pm$ 0.6 &59.0$\pm$ 0.8 \\
                 Ours-MT+CRW &\textbf{62.1$\pm$ 0.5} &\textbf{61.0$\pm$ 0.5} &\textbf{59.7$\pm$ 0.9} \\
                 \hline
                  DS3L-PI   &60.5$\pm$ 0.6 &60.1$\pm$ 1.0 &57.4$\pm$ 1.3   \\
                 Ours-PI+L1 &61.2$\pm$ 0.4 &60.4$\pm$ 0.4 &58.9$\pm$ 0.6 \\
                 Ours-PI+CRW &\textbf{61.6$\pm$ 0.4} &\textbf{60.7$\pm$ 0.5} &\textbf{59.5$\pm$ 0.7} \\
\end{tabular}
\label{CIFAR100}
\end{table}

\noindent {\bf Analysis of weight variation.} Fig~\ref{fig:addition_ana} (a) shows the weight learning curve of our proposed approach and DS3L on Fashion-MNIST OOD. The results show that our method learns better weights for unlabeled samples compared to DS3L. The weight distribution learned in other cases are also similar. 

\noindent {\bf Size of the clean validation set.}  We explore the sensitivity of the clean validation set used in robust SSL approaches on
Mean MNIST OOD. Fig~\ref{fig:addition_ana} (b) plots the classification performance with varying the size of the clean validation set. Surprisingly, our methods are stable even when using only 25 validation images, and the overall classification performance does not grow after having more than 1000 validation images. 

\noindent {\bf Ablation Studies.} We conducted additional experiments on
Fashion-MNIST OOD (see Fig~\ref{fig:addition_ana2} (a)-(c)) in order to demonstrate the contributions of the key technical components, including Cluster Re-weight (CRW) and weighted Batch Normalization (WBN). The key findings obtained from this experiment are \textit{1)} WBN plays a vital role in our uncertainty-aware robust SSL framework to improve the robustness of BN against OOD data; \textit{2)} Removing CRW (or L1 regularization) results in performance decrease, especially for VAT based approach, which demonstrates that CRW (and L1 regularization) can further improve performance for our robust-SSL approach;
\textit{3)} Fig~\ref{fig:addition_ana2} (b)-(c) demonstrate that the performance of our approach would increase with (inner loop gradients steps) and inverse Hessian approximation) increase due to high-order approximation. Based on this, we choose the best trade-off between speed and accuracy considering $J=3$ and $P=5$ when considering running time analysis in Fig~\ref{fig:run time} (c).

\begin{table}[th!]
\caption{Test accuracies for different numbers of clusters $K$ on the MNIST dataset with 50\%  Mean MNIST as OODs .} 
\centering
\begin{tabular}{c|cccc}
\# Clusters & K=5 & K=10& K=20 & K =30 \\
\hline
Ours-VAT &94.7$\pm$ 0.7 &95.3$\pm$ 0.4 &96.3$\pm$ 0.5  &95.6$\pm$ 0.5 \\
\hline
Ours-PL &95.2$\pm$ 0.5 &95.3$\pm$ 0.5 &96.2$\pm$ 0.4  &95.9$\pm$ 0.5 
\end{tabular}
\label{OOD_ex:K}
\end{table}

\noindent {\bf L1 vs CRW tricks. } Next, we discuss the trade-offs between L1 and CRW. We first analyzed the sensitivity of our proposed CRW methods to the number of clusters used. Table~\ref{OOD_ex:K} demonstrates the test accuracies of our approach with varying numbers of clusters. The results indicate a low sensitivity of our proposed methods to the number of clusters. We also find that CRW generally can further improve performance. Part of this success can be attributed to the good pretrained features (ImageNet). In addition, we compared our approach with the DS3L (safe SSL, SOTA robust SSL method) on CIFAR100 and SVHN dataset and get a similar pattern. The results are shown in Table~\ref{SVHN} and \ref{CIFAR100}.

\section{Conclusion}
In this work, we first propose the research question: \textit{How out-of-distribution data hurt semi-supervised learning performance?} To answer this question, 
we study the impact of OOD data on SSL algorithms and demonstrate empirically on synthetic data that the SSL algorithms' performance depends on how close the OOD instances are to the decision boundary (and the ID data instances). To address the above causes, we proposed a novel unified uncertainty-aware robust SSL framework that treats the weights directly as hyper-parameters in conjunction with weighted batch normalization, which is designed to improve the robustness of BN against OODs. To address the limitation of low-order approximations in bi-level optimization (DS3L), we designed an implicit-differentiation-based algorithm that considered high-order approximations of the objective and is scalable to a higher number of inner optimization steps to learn a massive amount of weight parameters. Next, we made our reweighting algorithms significantly more efficient by considering only last layer updates and infrequent weight updates, enabling us to have the same run time as simple SSL (naive reweighting algorithms are generally $3\times$ more expensive). 
In addition, we conduct a theoretical analysis of the impact of faraway OODs in the BN step and discuss the connection between our approach (high-order approximation based on implicit differentiation) and low-order approximation approaches. We show that our weighted robust SSL approach significantly outperforms existing robust approaches (L2RW, MWN, Safe-SSL, and UASD)  on several real-world datasets.

In Chapter~\ref{chapter:2} and Chapter~\ref{chapter:3}, we studied the uncertainty quantification on graph and image data with a semi-supervised learning setting. However, the uncertainty framework used in Chapter~\ref{chapter:2} and Chapter~\ref{chapter:3} only focuses on multi-class classification and statics problems. Therefore, it can not extend to multi-label classification or time series problems. To this end, in the following Chapter, we will study a time series multi-label classification setting for early event detection. Specifically,  we first propose a novel framework, Multi-Label Temporal Evidential Neural Network, to estimate the evidential uncertainty of multi-label time series classification. Then, we propose two novel uncertainty estimation heads to quantify the fused uncertainty of a sub-sequence for early event detection.

\chapter{Multi-Label Temporal Evidential Neural Networks for Early Event Detection}
\label{chapter:4}
    \section{Introduction}\blfootnote{\copyright 2022 IEEE. Reprinted, with permission, from Xujiang Zhao, Xuchao Zhang, Wei Cheng, Wenchao Yu, Yuncong Chen, Haifeng Chen, and Feng Chen, ``SEED: Sound Event Early Detection via Evidential Uncertainty,"  In ICASSP 2022-2022 IEEE International Conference on Acoustics, Speech and Signal Processing (ICASSP) (pp. 3618-3622), doi: 10.1109/ICASSP43922.2022.9746756.} 
In recent decades, early detection of temporal events has aroused a lot of attention and has applications in a variety of industries, including security \cite{sai2017application}, quality monitoring \cite{he2020confidence}, medical diagnostic \cite{zhao2019asynchronous}, transportation \cite{gupta2020early}, \textit{etc}. According to the time series, an event can be viewed with three components, pre-event, ongoing event, and post-event. Early event detection in machine learning identifies an event during its initial ongoing phase after it has begun but before it concludes \cite{hoai2014max,phan2018enabling}. As illustrated in Figure \ref{fig:motivation}, given a video clip with multiple frames, the goal is to accurately and rapidly detect human action(s) in the box (\textit{i.e.} smoke and watch a person) in the observation of incomplete video segments so that timely responses can be provided. This demands the detection of events prior to their completion.

To achieve the earliness of event detection, existing approaches can be broadly divided into several major categories. Prefix-based techniques
\cite{gupta2020fault,gupta2020early} 
aim to learn a minimum prefix length of the time series from the training instances and utilize it to classify a testing time series. Shapelet-based approaches
\cite{yan2020extracting,zhao2019asynchronous} 
focus on obtaining a set of key shapelets from the training dataset and utilizing them as class discriminatory features. Model-based methods for event early detection \cite{mori2019early,lv2019effective} are proposed to obtain conditional probabilities by either fitting a discriminative classifier or using generative classifiers on training. Although these approaches address the importance of early detection, they primarily focus on an event with a single label but fail to be applied to situations with multiple labels. 

\begin{figure}[!t]
    \centering
    \includegraphics[width=0.96\textwidth]{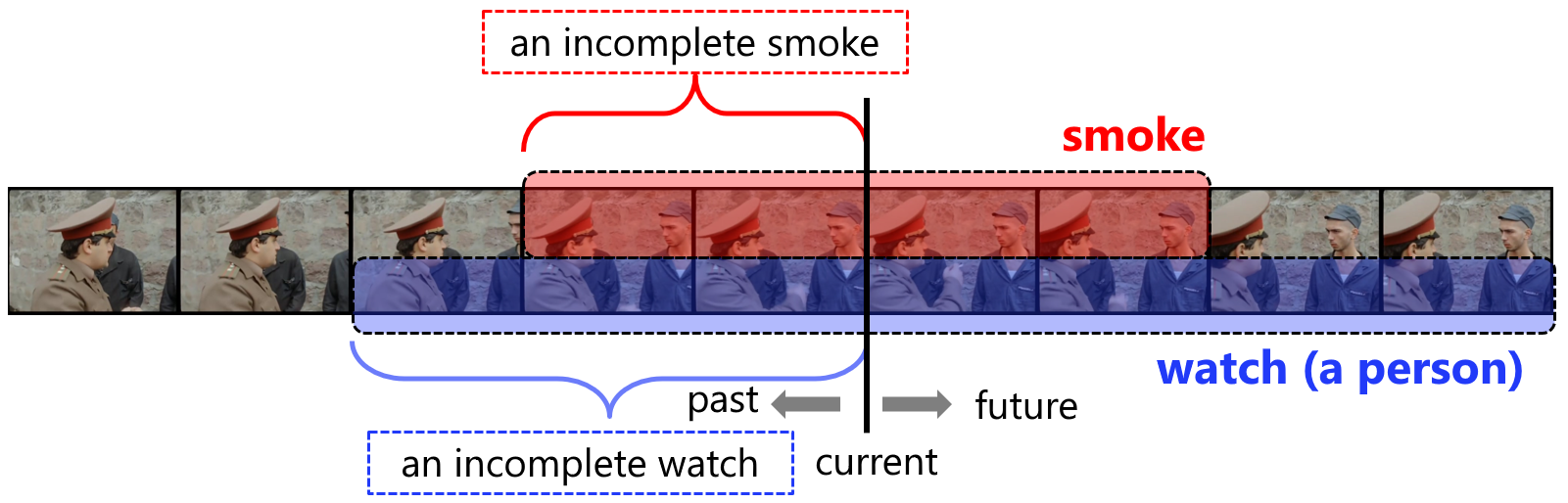}
    \caption{How many frames do we need to detect smoke and watch actions reliably?
    Can we even detect these actions before they finish? Existing event detectors are trained to recognize complete events only; they require
    seeing the entire event for a reliable decision, preventing early detection. We propose a learning formulation to recognize partial
    events, enabling early detection.}
    \label{fig:motivation}
\end{figure}

Another non-negligible issue for early event detection is the prediction with overconfidence \cite{zhao2020uncertainty,sensoy2018evidential}. In general, the occurrence of an event is determined by its predicted probability. An event with a high probability is considered as an occurrence. This, however, may not be reliable. Figure \ref{fig:overconfidence} shows an example that the prediction of the occurrence of an event (\textit{i.e.} an action) in a video clip with a binary class (occurs or not) based on its predicted probability is overconfident at the pre-event stage. In this case, the groundtruth (red line) demonstrates that the ongoing stage starts at the $20$th frame. 
Nevertheless, the event is falsely detected, prior to it actually occurring (Figure \ref{fig:overconfidence}, left), because a greater probability (\textit{i.e.} 0.9 indicated on the green line) is given by positive evidence. Here, the evidence indicates data samples (\textit{i.e.} actions) that are closest to the predicted one in the feature space and used to support the decision-making. Positive (negative) evidence is the observed samples that have the same (opposite) class labels. The event prediction with overconfidence at its early stage is due to high vacuity uncertainty~\cite{josang2016subjective} which is a terminology representing a lack of evidence. Therefore, it makes event detection based on probability unreliable. To overcome this flaw, methods developed on uncertainty estimation using evidence are desirable for early event detection.

In this work, to address the aforementioned issues, we first introduce a novel problem, namely \textit{early event detection with multiple labels}. In the problem setting, a temporal data with multiple events occurs sequentially over time. The goal of this work is to accurately detect the occurrence of all events within the least amount of time. 
Inspired by \cite{sensoy2018evidential} and subjective logic (SL) \cite{yager2008classic}, a proposed framework is proposed, which is composed of two phases: in phase one, a time series data is viewed as a sequence of segments with equal temporal length, where each segment comes one after another. Instead of predicting occurrence probabilities for all events, their positive and negative evidence is estimated through the proposed Multi-Label Temporal Evidential Neural Network (MTENN). The positive and negative evidence is seen as parameters of a Beta distribution which is a conjugate prior to Binomial likelihood. In the second phase, a sliding window spanning the most recent collected segments is designed to validate whether an event is successfully detected through two novel uncertainty estimation heads: (1) Weighted Binomial Comultiplication (WBC), where the belief of the occurrence of an event is successively updated through a binomial comultiplication operation \cite{josang2016subjective} from SL, and (2) Uncertainty Mean Scan Statistics (UMSS) for early event detection aims to detect the distribution change of the vacuity uncertainty through hypothesis testing from statistics. 
\textbf{Key contributions} of this work are summarized:
\begin{itemize}[leftmargin=*]
    \item We introduce a novel framework consisting of two phases
    for early event detection with multiple labels. At each timestamp, the framework estimates positive and negative evidence through the proposed Multi-Label Temporal Evidential Neural Network (MTENN). Inspired by subjective logic and belief theory,
    the occurrence uncertainty of an event is sequentially estimated over a subset of temporal segments.
    
    \item We introduce two novel uncertainty fusion operators (weighted binomial comultiplication (WBC) and uncertainty mean scan statistics (UMSS)) based on MTENN  to quantify the fused uncertainty of a sub-sequence for early event detection. We demonstrate the effectiveness of WBC and UMSS on detection accuracy and detection delay, respectively.
    
    \item We validate the performance of our approach with state-of-the-art techniques on real-world audio and video datasets. Theoretic analysis and empirical studies demonstrate the effectiveness and efficiency of the proposed framework.

\end{itemize}

\begin{figure}
    \centering
    \includegraphics[width=0.9\textwidth]{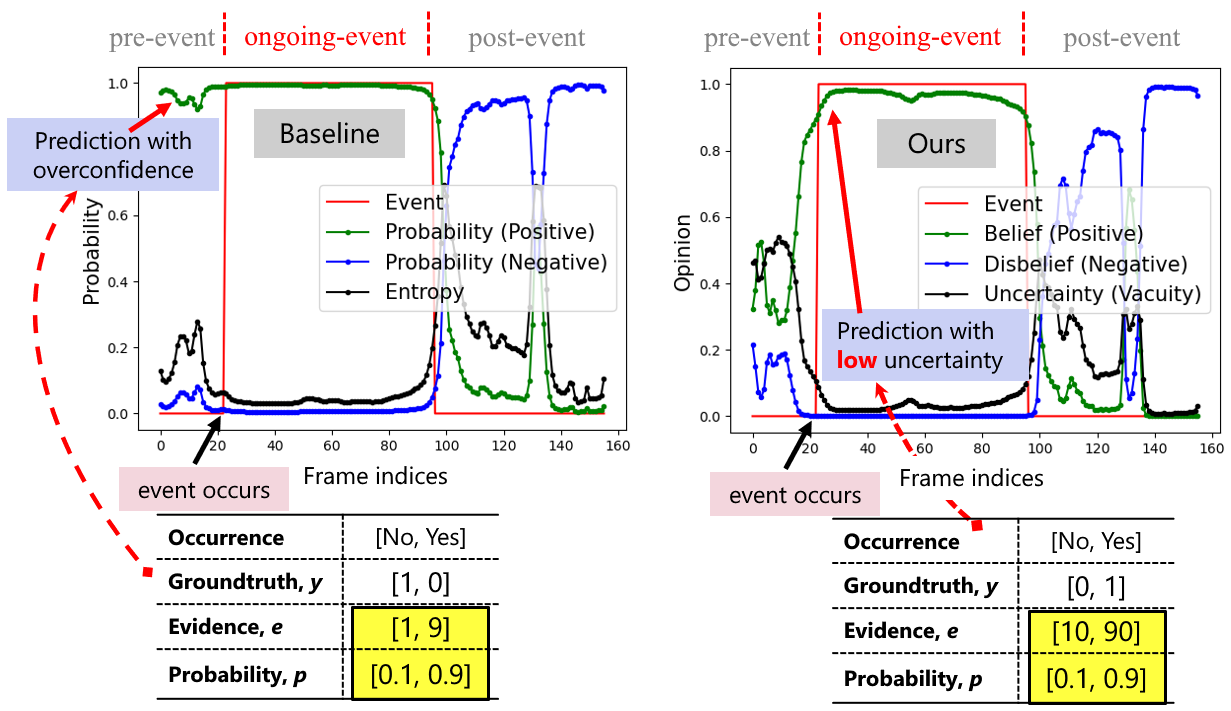}
    \caption{Illustration of overconfidence prediction. (Left) The occurrence of the event is falsely detected at the pre-event stage prior to its starting. This indicates that predicted probabilities are not reliable due to insufficient evidence. (Right) Instead of probabilities, subjective opinions (\textit{e.g.,} belief, disbelief, uncertainty) are used in the proposed method for early event detection.}
    \label{fig:overconfidence}
\end{figure}

\section{Related Work}
{\bf Early Event Detection} has been studied extensively in the literature of the time series domain. The primary task of it is to classify an incomplete time series event as soon as possible with some
desired level of accuracy. \cite{gupta2020fault} attempts to classify various complex human activities such as sitting on a sofa, sitting on the floor, standing while talking, walking upstairs, and eating, using only partial time series. A maximum-margin framework is proposed in ~\cite{hoai2014max} for training temporal event detectors to recognize partial events, enabling early detection. 
The generative adversarial network introduced in ~\cite{wang2019early} improves the early recognition accuracy of partially observed videos though narrowing the feature difference of partially observed videos from complete ones. Dual-DNN~\cite{phan2018enabling} is proposed for sound event early detection via a monotonous function design. \cite{mcloughlin2018early} identifies seed regions from spectrogram features to detect events at the early stage. Other algorithms have considered epistemic uncertainty for reliable event prediction~\citep{soleimani2017scalable}.
However, little attention has been paid to early event detection with multi-label settings, where multiple events may occur at the same time. Although there are several existing works designed for multi-label event detection~\cite{pan2021actor,tang2020asynchronous}, most of them cannot be applied to early detection event problems.

{\bf Uncertainty Estimation.} In machine learning and data mining, researchers have mainly focused on aleatoric uncertainty and epistemic uncertainty using Bayesian Neural Networks for computer vision applications. 
Bayesian Neural Networks frameworks are presented to simultaneously estimate both aleatoric and epistemic uncertainty in regression~\citep{gal2016dropout} and classification tasks~\citep{kendall2017uncertainties}. 
However, aleatoric or epistemic uncertainty is not able to estimate evidence uncertainty, which is essential for early event detection. As for evidential uncertainty, its origins from the belief (or evidence) theory domain, such as Dempster-Shafer Theory (DST)~\citep{sentz2002combination}, or Subjective Logic (SL)~\citep{josang2016subjective}. 
SL considered predictive uncertainty in subjective opinions in terms of {\em vacuity}
~\citep{josang2016subjective}. Evidential neural networks (ENN) was proposed in  \cite{sensoy2018evidential} to estimate evidential uncertainty for multi-class classification problem in the deep learning domain. However, ENN is designed for single-label classification due to the subjective logic assumption. 

To address the limitation of ENN and existing early event detection methods, we proposed a novel framework, namely Multi-Label Temporal Evidential Neural Network (MTENN), for early event detection in temporal data. MTENN is able to quality predictive uncertainty due to lack of evidence for multi-label classifications at each time stamp based on belief/evidence theory.

\section{Preliminaries}\label{sec:background}
The introduction of subjective logic is in Sec~\ref{SL}
In this section, we start with  Evidential neural networks (ENN) \cite{sensoy2018evidential}, a hybrid model of subjective logic and neural networks. Then we recap multi-label classification.

\subsection{Notations}
A time series data $\{(\x^t,\y^t)\}_{t=1}^T\in(\mathcal{X}\times\mathcal{Y})$ consists of $T$ segments where each $(\x^t, \y^t)$ is collected one after another over time. $\x^t$ represents the feature vector. $\y^t=[y^t_1, \ldots, y^t_K]^T$ denotes the multi-label formula with $y^t_k=\{0, 1\},\forall k\in\{1,\cdots K\}$ representing an event occurs or not and $K$ is the number of classes.
Vectors are denoted by lower case bold face letters, \textit{e.g.}, class probability ${\bm p} \in [0, 1]^T$ where their \textit{i}-th entries are $p_i$. Scalars are denoted by lowercase italic letters, \textit{e.g.} $u\in[0, 1]$. Matrices are denoted by capital italic letters. $\omega$ denotes the subjective opinion. We use subscripts to denote the index of the class, and we use superscripts to denote the index of the time stamp.
Some important notations are listed in Table~\ref{table:notation_4}

\begin{table*}[h]
\caption{Important notations and corresponding descriptions.}
\centering
  \begin{tabular}{l|l}
    \toprule
      \textbf{Notations} & \textbf{Descriptions}  \\
    \midrule
    $\mathcal{B}$   & Segment buffer  \\
    $\x^t$   &  feature vector for segment $t$   \\
    $\y^t$   &  multi-label for segment $t$   \\
    $\p^t$ & Class probability for segment $t$ \\
    ${\bm \theta}$ & model parameters \\
    $\omega$ & Subjective opinion \\
    $b$ & Belief mass  \\
    $d$ & Disbelief mass  \\
    $u$ & Vacuity uncertainty \\
    $a$ & Base rate \\
    $\alpha^t_k$ & Positive evidence at segment $t$ for class $k$ \\
    $\beta^t_k$ & Negative evidence at segment $t$ for class $k$  \\
    $K$ & Number of classes \\
    $\cc$ & Opinion weight in WBC \\
    $f(\cdot)$ & MTENN model function \\
    $\psi(\cdot)$ & Digamma function \\
    $\textbf{Beta}(p|\alpha, \beta)$ & PDF of Beta distribution \\
    $\textbf{BCE}(\cdot)$ & binary cross-entropy loss \\
    \bottomrule
  \end{tabular}
   \label{table:notation_4}
\end{table*}

\subsection{Evidential neural networks}
Evidential neural networks (ENNs)~\cite{sensoy2018evidential} is a hybrid framework of subjective belief models~\cite{josang2016subjective} with neural networks. ENNs is designed to estimate the evidential uncertainty of classification problem. They are similar to classic neural networks for classification. The main difference is that the softmax layer is replaced with an activation function in ENNs, \textit{e.g.,} ReLU, to ensure non-negative output (range of $[0, +\infty)$), which is taken as the evidence vector for the predicted Dirichlet (or Beta) distribution, or equivalently, multinomial (binomial) opinion. 

Given of an input sample ${\bf x}$, let $f({\bf x} | {\bm \theta})$ represent the evidence vector by the network for the classification, where ${\bm \theta}$ is network parameters. Then the corresponding Dirichlet distribution has parameters ${\bm \alpha} = f({\bf x}_i | {\bm \theta})$, where ${\bm \alpha}=[\alpha_1, \ldots, \alpha_K]$, $K$ is the number of classes. Let $\p=(p_1, \dots, p_K)^T$ be the class probabilities, which can be sampled from the Dirichlet distribution.
Therefore, the output of an ENNs can be applied to measure the evidential uncertainty about the predictive class variable $\y$, such as vacuity. Usually, we consider the following loss function~\cite{sensoy2018evidential} to train an ENNs model:
\begin{eqnarray}
\mathcal{L}_{ENN} = \int \Big[\sum_{i=1}^K -y_i \log(p_i)\Big] \textbf{Dir}(p|{\bm \alpha}) dp 
\end{eqnarray}

\subsection{Multi-label classification}
In machine learning and data mining, multi-class classification is a learning problem where each data sample is associated with a unique label from a set of disjoint classes. 
The single-label classification problem can be divided to binary classification or multi-class classification depending on the  number of classes. 
Unlike multi-class classification problems, multi-label classification allows each data sample belongs to more than one class. 
Similar to deep learning models for multi-label classification where they typically use a sigmoid layer on top of deep neural networks for each of class, for time series data, a temporal multi-label classifier is considered to handle the temporal dependency between segments collected at each timestamp.
Both traditional binary and multi-class problems can be posed as specific cases of multi-label problems. For most single-label classification methods, cross-entropy is used as the loss function. In contrast, for multi-label classification problems, traditional multi-label temporal neural networks take binary cross entropy as the objective function,
\begin{eqnarray}
\mathcal{L}_{ML} &=& \sum_{t=1}^T \sum_{k=1}^K  \textbf{BCE}(y_k^t, p_k^t) \nonumber \\
&=& \sum_{t=1}^T \sum_{k=1}^K -y_{k}^t \log(p_{k}^t) - (1-y_k^t) \log(1-p_{k}^t) 
\label{eq:bce}
\end{eqnarray} 
where $\textbf{BCE}(\cdot)$ is the binary cross-entropy loss.

\section{Problem Formulation}
Given a time series data with multiple labels (\textit{e.g.} a person in a video clip has multiple actions) where each class label is viewed as an event, 
let $\mathcal{X}\times\mathcal{Y}$ be the data space, where $\mathcal{X}$ is an input space and $\mathcal{Y}=\{0,1\}^K$ is an output space. A time series data $\{(\x^t,\y^t)\}_{t=1}^T\in(\mathcal{X}\times\mathcal{Y})$ consists of $T$ segments where each $(\x^t, \y^t)$ is collected one after another over time. $\x^t$ represents the feature vector. $\y^t=[y^t_1, \ldots, y^t_K]^T$ denotes the multi-label formula with $y^t_k=\{0, 1\},\forall k\in\{1,\cdots K\}$ representing an event occurs or not and $K$ is the number of classes. A segment buffer $\mathcal{B}$ is initialized as empty. It is maintained by adding each segment one at a time. That is, at timestamp $t$, the buffer includes all segments from previous $\mathcal{B}=\{(\x^i,\y^i)\}_{i=1}^t$ and $|\mathcal{B}|=t$. At each time, a predictive model $f:\mathcal{X}\rightarrow\mathcal{Y}$ parameterized by $\boldsymbol{\theta}$ takes segments in $\mathcal{B}$ as the input and outputs a event prediction vector $\hat{\y}^t=[\hat{y}^t_1, \ldots, \hat{y}^t_K]^T$ where $\hat{y}^t_k\in\{0,1\}$ represents the predicted result of the $k$-th event at time $t$ (1 represents occurrence 0 otherwise). Therefore, for some events which are predicted as occurrences, one may conclude that they can be detected at time $t$.

In the following sections, as demonstrated in Figure \ref{fig:framework}, we propose a novel framework for early event detection, that is Multi-label Temporal Evidential Neural Networks (MTENN) followed by sequential uncertainty estimation. This framework is composed of two phases. In the first phase, the data is viewed as a sequence of segments with equal temporal length, where each segment comes one after another over time. At the time $t$, instead of the prediction result $\hat{\y}^t$, a pair of vectors consisting of positive evidence $\boldsymbol{\alpha}^t$ and negative evidence $\boldsymbol{\beta}^t$ are estimated. They can be seen as parameters of a Beta distribution which is a conjugate prior to Binomial likelihood. In the second phase, a sliding window including $m$ most recent collected segments is used to validate whether an event is successfully detected through an early detection function. The function maps a sequence of $\{\boldsymbol{\alpha}^i,\boldsymbol{\beta}^i\}_{i=t-m}^t$ in the window and outputs a subjective opinion for all events by recursively combining a set of opinions of each segment. The integrated opinion is used to determine the occurrence of each event.

\section{Multi-Label Temporal Evidential Neural Networks}\label{sec:method}
In this section, we introduce the proposed Multi-Label Temporal  Evidential Neural Networks (MTENN) Framework for reliable early event detection. The overall description of the framework is shown in Figure~\ref{fig:framework}. 
\begin{figure*}[!t]
    \centering
    \includegraphics[width=1\textwidth]{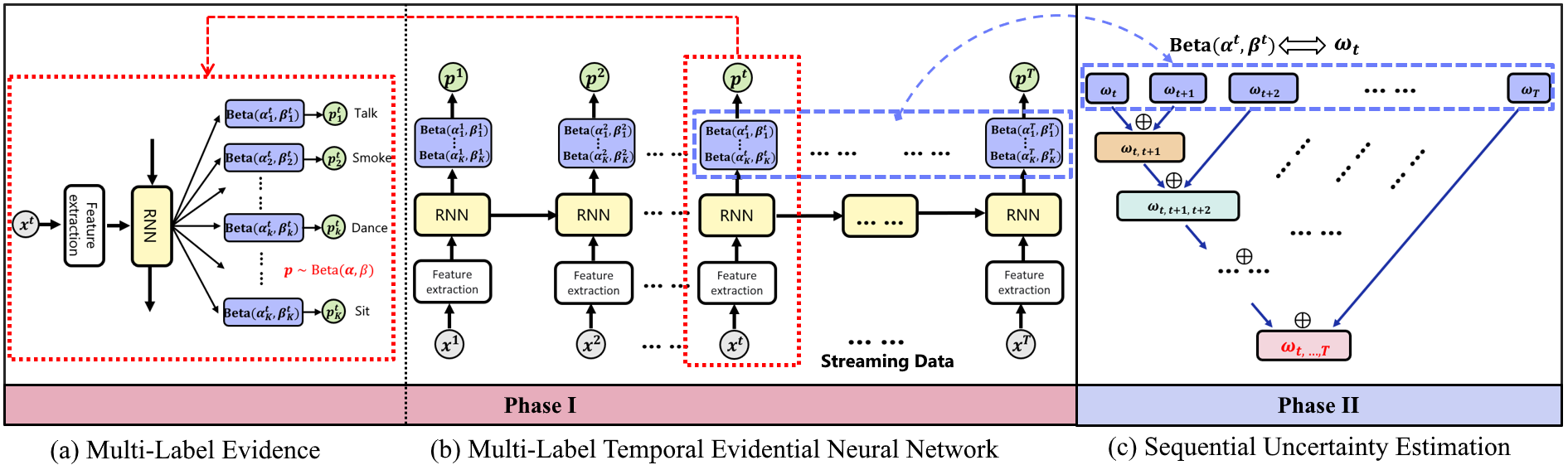}
    \caption{\textbf{Framework Overview.} Given the streaming data, (b) MTENN is able to quality predictive uncertainty due to a lack of evidence for multi-label classifications at each time stamp based on belief/evidence theory. Specifically, (a) at each time step with data segment $x^t$, MTENN is able to predict Beta distribution for each class, 
    which can be equivalent transfer to subjective opinion $\omega_t$; (c) based on a sliding window, two novel fusion operators (weighted binomial comultiplication and uncertainty mean scan statistics) are introduced to quantify the fused uncertainty of a sub-sequence for an early event .}
    \label{fig:framework}
\end{figure*}

\subsection{MTENN Framework}
For multi-label early event detection, most existing methods would consider a binary classification for each class, such as sigmoid output~\cite{Turpault2019_DCASE,hershey2021benefit}. 
As discussed in Section~\ref{sec:background}, evidential uncertainty can be derived from binomial opinions or Beta distributions to model an event distribution for each class. Therefore, we proposed a novel Multi-label Evidential Temporal Neural Network (MTENN) $f(\cdot)$ to form their binomial opinions for the class-level Beta distribution of a given time series segments $[\x^1, \ldots, \x^t]$. 
Then, the conditional probability $P(p^t_k|\x^1, \ldots, \x^t;{\bm \theta})$ of class $k$ at timestamp $t$ can be obtained by:
\begin{eqnarray}
[f^1, \ldots, f^t] &\leftarrow& f(\x^1, \ldots, \x^t;{\bm \theta}) \nonumber \\
(\alpha^t_1, \beta^t_1),\ldots, (\alpha^t_K, \beta^t_K),  &\leftarrow& f^t(\x^1, \ldots, \x^t;{\bm \theta}), \nonumber \\
p^t_k  &\sim&\textbf{Beta}(p_k^t|\alpha^t_k, \beta^t_k), \nonumber \\
y^t_k &\sim&\textbf{Bernoulli}(p^t_k),
\label{eq:mlenn}
\end{eqnarray} 
where $k\in \{1, \cdots, K\}$, $t \in \{1, \cdots, T\}$, $f^t$ is the output of MTENN at timestamp $t$, and ${\bm \theta}$ refers to model parameters. $\textbf{Beta}(p^t|\alpha^t_k, \beta^t_k)$ is the Beta probability function. 
Note that MTENN is similar to the classical multi-label time classification model (e.g., CRNN~\cite{Turpault2019_DCASE}), except that we use an activation layer (e.g., ReLU) instead of the sigmoid layer (only outputs class probabilities). This ensures that MTENN would output non-negative values taken as the evidence for the predicted Beta distribution. 
Therefore,  MTENN is able to quality predictive uncertainty (vacuity) due to a lack of evidence for multi-label classifications at each time stamp based on belief/evidence theory, and vacuity can be calculated from the estimated Beta distribution.

\subsection{Loss}
In this work, we design and train MTENN to form their binomial opinions for the classification of a given streaming segment as a Beta distribution. For the binary cross-entropy loss, we have the MTENN loss by computing its Bayes risk for the class predictor,
\begin{eqnarray}
\mathcal{L}_{MTENN} &=& \sum_{t=1}^T \sum_{k=1}^K \int \Big[\textbf{BCE}(y_k^t, p_k^t) \Big] \textbf{Beta}({p}_k^t;\alpha_k^t, \beta_k^t) d {p}_k^t \nonumber \\
&=& \sum_{t=1}^T \sum_{k=1}^K \Big[y_k^t\Big(\psi(\alpha_k^t+\beta_k^t) - \psi(\alpha_k^t)\Big) \nonumber \\
&& + (1-y_k^t)\Big(\psi(\alpha_k^t+\beta_k^t) - \psi(\beta_k^t)\Big) \Big],
\label{eq:loss}
\end{eqnarray} 
where $\textbf{BCE}(y_k^t, p_k^t)= -y_{k}^t \log(p_{k}^t) - (1-y_k^t) \log(1-p_{k}^t)$ is the binary cross-entropy loss, and $\psi(\cdot)$ is the \textit{digamma} function. The log expectation of Beta distribution derives the second equality. The details of the first phase of our proposed framework for MTENN are shown in Algorithm~\ref{algorithm_MTENN}.


\begin{algorithm}[t]
\small{
\DontPrintSemicolon
\KwIn{A time series data $\{(\x^t,\y^t)\}_{t=1}^T$}
\KwOut{Model params: ${\bm \theta}$} 
\SetKwBlock{Begin}{function}{end function}
{
  Set $t=0$; learning rate $\alpha$; \;
  Initialize model parameters ${\bm \theta}$; \;
 \Repeat{convergence}
 {
 Estimate the Beta distribution via Eq.~\eqref{eq:mlenn};\;
 Calculate the gradient $\nabla_{{\bm \theta}}\mathcal{L}_{MTENN}$ via Eq.~\eqref{eq:loss};\;
Update net parameters \;
${\bm \theta}_{t+1} = {\bm \theta}_{t} - \alpha \nabla_{{\bm \theta}} \mathcal{L}_{MTENN}({\bm \theta}_{t})$\; 
$t=t+1$\;
 }
\Return{${\bm \theta}_{t+1}$}
}
\caption{MTENN-Phase I}\label{algorithm_MTENN}
}
\end{algorithm}
\setlength{\textfloatsep}{1pt}

\subsection{Theoretical Analysis}
In this section, we demonstrate two main theoretical results. First, traditional evidential neural networks (ENN) can be posed as specific cases of multi-label temporal evidential neural networks. Second, traditional multi-label temporal neural networks (MTNN) can be posed as specific cases of multi-label Temporal Evidential neural networks (MTENN) when training data has large evidence for each class. 
Theorem~\ref{MTENN_ENN} shows that the loss function in MTENN is equivalent to it in MTNN when sufficient evidence is observed. 
Furthermore, given a feature vector, MTNN predicts probabilities for each class, while MTENN outputs positive and negative evidence estimations ($\alpha$ and $\beta$) that can be seen as parameters of a Beta distribution. Statistically, the predicted class probabilities are sampled from a Beta probability density function parameterized by the estimated $\alpha$ and $\beta$. When uncertainty mass approaches zero, the variance of sampled class probabilities is small and close to its expectation $\frac{\alpha}{\alpha + \beta}$. As a consequence, one may infer the probability of each class using the predicted evidence from MTENN, which is equivalent to the output of MTNN.
The results indicate that our proposed MTENN is a generalization of ENN and MTNN and has thus inherited the merits of both models for the task of early event detection. 
\begin{theorem}
Denote $\mathcal{L}_{ML}$ is the loss function of traditional multi-label temporal neural networks, $\mathcal{L}_{TMENN}$ is the loss function of multi-label Temporal Evidential neural networks, and $u$ is the vacuity uncertainty estimated form  MTENN. We have $\lim_{\alpha\rightarrow \infty, \beta\rightarrow \infty}|\mathcal{L}_{ML}- \mathcal{L}_{MTENN}| \rightarrow 0$.
\label{MTENN_MTNN}
\end{theorem}
\begin{proof}
When $\alpha \rightarrow \infty$, we have
\begin{align*}
\small
&\psi(\alpha_k^t+\beta_k^t) - \psi(\alpha_k^t) \nonumber \\
&\xrightarrow[]{\alpha\rightarrow \infty} \ln(\alpha_k^t+\beta_k^t) - \frac{1}{2(\alpha_k^t+\beta_k^t)} -  \ln(\alpha_k^t) + \frac{1}{2\alpha_k^t} \nonumber \\
&= -\ln (\frac{\alpha_k^t}{\alpha_k^t+\beta_k^t}) + \frac{1}{2\alpha_k^t} - \frac{1}{2(\alpha_k^t+\beta_k^t)} \nonumber \\
&\xrightarrow[]{\alpha\rightarrow \infty} -\ln (\frac{\alpha_k^t}{\alpha_k^t+\beta_k^t}) = -\ln p_k^t
\end{align*}
and when $\beta \rightarrow \infty$, we have
\begin{align*}
&\psi(\alpha_k^t+\beta_k^t) - \psi(\beta_k^t) \nonumber \\
&\xrightarrow[]{\beta\rightarrow \infty}\ln(\alpha_k^t+\beta_k^t) - \frac{1}{2(\alpha_k^t+\beta_k^t)} -  \ln(\beta_k^t) + \frac{1}{2\beta_k^t} \nonumber \\
&= -\ln (\frac{\beta_k^t}{\alpha_k^t+\beta_k^t}) + \frac{1}{2\beta_k^t} - \frac{1}{2(\alpha_k^t+\beta_k^t)} \nonumber \\
&\xrightarrow[]{\beta\rightarrow \infty} -\ln (\frac{\beta_k^t}{\alpha_k^t+\beta_k^t}) = -\ln (1-p_k^t)
\end{align*}
where the expected probability $p_k^t = \frac{\alpha_k^t}{\alpha_k^t+\beta_k^t}$. Then we have $|\mathcal{L}_{ML}- \mathcal{L}_{MTENN}| \rightarrow 0$.
\end{proof}

\begin{theorem}
Denote $\mathcal{L}_{ENN}$ is the loss function of traditional evidential neural networks, $\mathcal{L}_{Beta}$ is the loss function of multi-label Temporal Evidential neural networks, and $K$ is the total number of the class. We have $\mathcal{L}_{ENN} = \mathcal{L}_{MTENN}$ when $K=1, T=1$.
\label{MTENN_ENN}
\end{theorem}

\begin{proof}
When $K=1, T=1$, we have
\begin{align*}
\small
&\mathcal{L}_{TMENN} = \sum_{t=1}^1 \sum_{k=1}^1 \int \Big[\textbf{BCE}(y_k^t, p_k^t) \Big] \textbf{Beta}({p}_k^t;\alpha_k^t, \beta_k^t) d {p}_k^t \nonumber \\
&= \int \Big[\textbf{BCE}(y, p) \Big] \textbf{Beta}({p};\alpha, \beta) d {p} \nonumber \\
&= y\Big(\psi(\alpha+\beta) - \psi(\alpha)\Big) 
+ (1-y)\Big(\psi(\alpha+\beta) - \psi(\beta)\Big) \nonumber
\end{align*}
and for the ENN model, $k=1$ means the binary classification, and we denote the positive evidence $\alpha=\alpha_1$, negative evidence $\beta = \alpha_2$, positive label $y=y_1$, and negative label $(1-y)=y_2$ 
\begin{align*}
\tiny
&\mathcal{L}_{ENN} = \int \Big[\textbf{CrossEntopy}(y, p)\Big] \textbf{Dir}(p;\alpha) dp \nonumber \\
&= \sum_{j=1}^2 y_j \Big(\psi(\sum_{j=1}^2 \alpha_j) - \psi(\alpha_j)\Big)  \nonumber \\
&= y\Big(\psi(\alpha+\beta) - \psi(\alpha)\Big) 
+ (1-y)\Big(\psi(\alpha+\beta) - \psi(\beta)\Big) \nonumber
\end{align*}

Then we proof that $\mathcal{L}_{ENN} = \mathcal{L}_{TMENN}$.
\end{proof}

\section{Multi-label Sequential Uncertainty Quantitation}
\label{sec:uncertainQ}

In the second phase, for early event detection, at time $t$, a subset including $m$ most recent collected segments are considered to validate whether an event is successfully detected or not, as shown in Fig~\ref{fig:framework} (c). We name the subset as a sliding window, as it dynamically restructures a small sequence of segments from $t-m$ to $t$ and performs validation through an early detection function at each time. 
Based on the sliding window, we introduce two novel uncertainty fusion operators based on MTENN to quantify the fused uncertainty of a sub-sequence for early event detection.

\subsection{Weighted Binomial Comultiplication}

{\bf Binomial Comultiplication Operator.} After we get the sequential Beta distribution output, a sequential fusional opinion can be estimated via a subjective operator (e.g., union operator). As shown in Fig~\ref{fig:framework} (b), we can use subjective operator $\oplus$ to fuse the opinions. Here we consider to use comultiplication operator~\cite{Audun2006BeliefCalculus} to fusion two opinion $\omega_i$ and $\omega_j$ via Eq.~\eqref{coproduct},
\begin{equation}
\begin{aligned}
&b_{i \oplus j}=b_{i}+b_{j}-b_{i} b_{j} \\
&d_{i \oplus j}=d_{i} d_{j}+\frac{a_{i}\left(1-a_{j}\right) d_{i} u_{j}+\left(1-a_{i}\right) a_{j} u_{i} d_{j}}{a_{i}+a_{j}-a_{i} a_{j}} \\
&u_{i \oplus j}=u_{i} u_{j}+\frac{a_{j} d_{i} u_{j}+a_{i} u_{i} d_{j}}{a_{i}+a_{j}-a_{i} a_{j}} \\
&a_{i \oplus j}=a_{i}+a_{j}-a_{i} a_{j}
\end{aligned}
\label{coproduct}
\end{equation}

Based on $m$ sliding windows, the sequential fusional opinion can be calculated by
\begin{equation}
    \omega_{t-m, \ldots, t} = \omega_{t-m} \oplus \omega_{t-m+1} \oplus \ldots \oplus \omega_t \label{BC}
\end{equation}

The above operator ignores the order information, which means $\omega_x \oplus \omega_y$ has the same effect as $\omega_y \oplus \omega_x$. To consider the order information and emphasize the importance of current time step $t$, we propose a weighted comultiplication operator that assigns the weight $\cc$ for each opinion when executing the operator, then the weighted sequential opinion can be obtained,
\begin{equation}
   \hat{\omega}^t =c_{t-m} \cdot \omega^{t-m} \oplus c_{t-m+1} \cdot \omega^{t-m+1} \oplus \ldots \oplus c_t \cdot \omega^t \label{WBC}
\end{equation}
We consider the vacuity from $\hat{\omega}^t$ as sequential uncertainty for a sub-sequence. 


{\bf Uncertainty-based Inference.}
At the test stage, we consider a simple strategy to make a reliable prediction based on the sequential uncertainty estimated from WBC. For each class, we predict sound events happened only when the belief is larger than disbelief with a small vacuity,
\begin{eqnarray}
\hat{y}_k^t =\begin{cases}1,& \text{if } b^t_k > d^t_k \text{ and } u_k^t < V \\0 ,& \text{otherwise}\end{cases}
\label{eq:inference}
\end{eqnarray} 
where $\hat{y}_k^t\in \{0, 1\}$ is the model prediction for class $k$ in segment $t$, $V$ is the vacuity threshold.

\subsection{Uncertainty mean scan statistics.}

\begin{figure}[!t]
    \centering
    \includegraphics[width=0.6\textwidth]{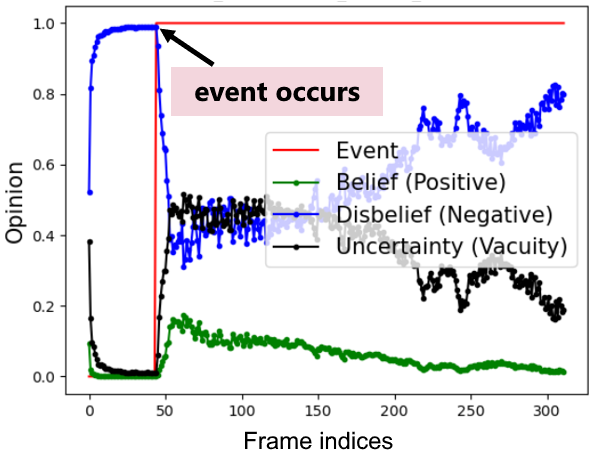}
    \caption{Uncertainty distribution changed at the ongoing event early stage .}
    \label{fig:UMSS}
\end{figure}

Intuition: uncertainty distribution changes at the early ongoing event stage. One example is shown in Fig~\ref{fig:UMSS}. Here we proposed a simple uncertainty fusion operator based on the mean scan statistics method for early event detection.
Let $u_t$ be the vacuity uncertainty at time $t$ and $T$ be the sliding window size. Let $S_t$ be the sliding window at time $t$: $S_t=\{u_{t-T+1}, u_{t-T+2}, \dots, u_t\}$. We consider the following hypothesis testing about the detection of an event at time $t$.
\begin{itemize}[leftmargin=*]
    \item {\bf Null hypothesis $H_0$:} The uncertainty at $S_t$ follow the distribution $u\in \mathcal{N}(0, 1), \forall  u\in S_t$
    \item {\bf Alternative hypothesis $H_1$:} The uncertainty at $S_t$ follow the distribution $u\in\mathcal{N}(\mu, 1), \forall  u\in S_t$ with $\mu>0$
\end{itemize}
Based on the above hypothesis testing, we can derive the log-likelihood ratio, called positive alleviated mean scan statistic (EMS),  that is used as the test statistic for event detection. The larger the test statistic, the higher the chance there is an event in this time window.  
\begin{equation}
    \hat{u}_t= \frac{\sum_{u\in S_t} u_i}{T} \label{UMSS}
\end{equation}
After we get the scan statistic score $\hat{u}_t$, we need to decide a threshold, i.e., $V$, such that we will reject $H_0$ and accept $H_1$ if $\hat{u}_t>V$, based on the confidence level, and such as 0.05. In order to decide the threshold $V$, we use Monte Carlo sampling as follows: We calculate the EMS statistics scores for all historical windows that have no event, we call them historical windows under $H_0$, and find the 5\% quantile value, i.e., the value, such that the number of historical windows under $H_0$ less than this value is 5\% of the total number of historic windows. The details of the second phase of our framework for early event detection are shown in Algorithm~\ref{algorithm_MTENN2}.

\begin{algorithm}
\small{
\DontPrintSemicolon
\KwIn{ A streaming time series data $\{\x^t\}_{t=1}^T$, 
 MTENN parameter ${\bm \theta}$, threshold $V$, and weight $\cc$}
\KwOut{Early event prediction: $[\hat{\y}^1, \ldots \hat{\y}^T]$} 
\SetKwBlock{Begin}{function}{end function}
{
 Set $t=0$; sliding window size $m$; \;
 Initialize Buffer $\mathcal{B}=\emptyset$; \;
 \For{$t=1, \ldots, T$}{
 collect current stream data $\x^t$,
 $\mathcal{B}=\mathcal{B}\cup \x^t$ \;
 estimate Beta distribution via trained MTENN  
 $(\boldsymbol{\alpha}^1, \boldsymbol{\beta}^1),\ldots, (\boldsymbol{\alpha}^t, \boldsymbol{\beta}^t)  \leftarrow f(\mathcal{B};{\bm \theta}),$ \;
 Obtain the sliding window $\mathcal{S}=\{\boldsymbol{\alpha}^i,\boldsymbol{\beta}^i\}_{i=t-m}^t$\;
 Transfer the Beta distribution in $\mathcal{S}$ into subjective opinion\;
 \If{WBC}{Estimate sequential opinion $\hat{\omega}^t$ by 
 $\hat{\omega}^t =c_{t-m} \omega^{t-m} \oplus c_{t-m+1} \omega^{t-m+1} \oplus \ldots \oplus c_t \omega^t$ \;
 calculate the early event prediction for each class,
 $\hat{y}_k^t =\begin{cases}1,& \text{if } \hat{b}^t_k > \hat{d}^t_k \text{ and } \hat{u}_k^t < V \\0 ,& \text{otherwise}\end{cases}$
 }
 \If{UMSS}{
 Estimate sequential uncertainty $\hat{u}^t$ via Eq.~\eqref{UMSS}\;
 calculate the early event prediction for each class, \\
 $\hat{y}_k^t =\begin{cases}1,& \text{if } \hat{u}^t > V \\ 0 ,& \text{otherwise}\end{cases}$
 }
 } 
\Return{$[\hat{\y}^1, \ldots \hat{\y}^T]$}
}
\caption{MTENN-Phase II}\label{algorithm_MTENN2}
}
\end{algorithm}

\section{Experiments}\label{sec:experiment}
Our experimental section aims to verify the effectiveness of our proposed method (MTENN) by evaluating MTENN through both \textit{early sound event detection} scenario and \textit{early human action detection} scenario on real-world datasets. Furthermore, our work’s experimental scenarios are very relevant in terms of research and real-world applications. Finally, we have implemented the MTENN framework using PyTorch. 
We repeat the same experiment for three runs with different initialization and report the mean detection delay and detection F1 score (ref to Evaluation Metrics section).

\subsection{Experiment Details}\label{sec:audio_data}

\begin{table*}[h]
\small
\centering
\caption{Description of datasets and their experimental setup for the early event detection.}
\begin{tabular}{l|c|c|c|c|c|c}
 \toprule
 & \textbf{DESED2021} & \textbf{Explosion} & \textbf{Alarm} & \textbf{Liquid}  &\textbf{Engine} & \textbf{AVA} \\
\midrule
\# Classes & 10  & 5 & 7 & 9 & 8 & 60\\
\midrule
\# Training samples & 10,000  & 5,518 & 8,085  & 6,517 & 18,741 & 43,232 \\
\# Validation samples & 1,168  & 788 & 1,355   & 931 & 2,677 & 20,000 \\
\# Test samples & 1,016  & 1,577 & 2,311  & 1,862 & 5,354 & 23,498 \\
\bottomrule
\end{tabular}
\label{tab:datasets}
\end{table*}

{\bf Dataset.} \textit{a) Early sound event detection task}, 
we conduct the experiments on DESED2021 dataset~\cite{Turpault2019_DCASE} and AudioSet-Strong-Labeled dataset~\cite{hershey2021benefit}.
DESED2021 dataset is composed of 10 second audio clips recorded in domestic environments or synthesized using Scaper to simulate a domestic environment. 
The original AudioSet-Strong-Labeled dataset consists of an expanding ontology of 632 audio event classes and a collection of 2,084,320 human-labeled 10-second sound clips drawn from YouTube videos. To simulate the application of early event detection in industries, we select some subsets from the AudioSet-Strong-Labeled dataset to form early event detection datasets. Specifically, we select four subsets from the AudioSet-Strong-Labeled dataset, including explosion, alarm, liquid, and engine subclasses. 
\textit{b)~Early human action detection task}, we consider AVA dataset~\cite{gu2018ava}. AVA is a video dataset for spatiotemporal localizing atomic visual actions with 60 classes. For AVA, box annotations and their corresponding action labels are provided on key frames
of 430 15-minute videos with a temporal stride of 1 second. We use version 2.2 of the AVA dataset by default.
The details of each dataset are shown in Table~\ref{tab:datasets}.

{\bf Features.} For the early sound event detection task,
the input features used in the experiments are log-mel spectrograms extracted from the audio signal resampled to 16000 Hz.
The log-mel spectrogram uses 2048 STFT windows with a hop size
of 256 and 128 Mel-scale filters. At the training stage, the input is the full observed 10-second sound clip. As a result, each 10-second sound clip is transformed into a 2D time-frequency representation with a size of (626×128). At the test stage, we collect an audio segment at each timestamp, which can be transformed into a 2D time-frequency representation with a size of (4×128). For the early human action detection task, each input video contains 91 image frames. To satisfy the streaming data at the inference stage, we cut the video into $T$ segments. In our experiment, we consider $T=\{4, 6, 8, 10, 12, 15\}$, the corresponding segment length $ST=\{0.75s, 0.5s, 0.375s, 0.3s, 0.25s, 0.2s\}$. 

{\bf Comparing Methods.}
To evaluate the effectiveness of our proposed approach with two variants, named MTENN-WBC and MTENN-UMSS, \textit{a) Early sound event detection task}, we compare it with two state-of-the-art early sound event detection methods: Dual DNN~\cite{phan2018enabling} and SEED~\cite{zhao2022seed}; two sound event detection methods: CRNN~\cite{Turpault2019_DCASE} and Conformer~\cite{miyazaki2020conformer}; 
\textit{b) Early human action detection task}, we compare it with three sound human action detection methods: ACAR~\cite{pan2021actor}, AIA~\cite{tang2020asynchronous}, and SlowFast~\cite{feichtenhofer2019slowfast};
For both early detection tasks, we consider three different uncertainty methods as the baselines, which include \textit{Entropy}, 
\textit{Epistemic} uncertainty~\cite{gal2016dropout} (represents the uncertainty of model parameters), and \textit{Aleatoric} uncertainty~\cite{depeweg2018decomposition} ( represents the uncertainty of data noise). 
We consider using uncertainty to filter the high uncertainty prediction for three uncertainty-based methods.
We use MC-drop~\cite{gal2016dropout} to estimate epistemic and aleatoric uncertainties in the experiments.

{\bf Evaluation Metrics.} We consider both early detection F1 score and detection delay as the evaluation metrics to evaluate early detection performance. 
We first define the true positive prediction for the event $k$, which only happens when the first prediction timestamp $d_p$ is located in the ongoing event region. 
\begin{eqnarray}
\text{TP}_k =\begin{cases}1, & \text{if }y^{d_p}_k==1 \text{ and } d_p-d_t \ge L \\ 0 ,& \text{otherwise}  \end{cases}
\label{eq:tp}
\end{eqnarray} 
where $d_t$ is the onset timestamp of the predicted event. In contrast, the false positive prediction happened when the first prediction timestamp $d_p$ is not located in the ongoing event region. Then we can calculate precision, recall, and F1 score based on true positive prediction and false positive prediction for each event.
For detection delay, it's only measured when we have a true positive prediction. 
Then the detection delay is defined as follows,
\begin{eqnarray}
\text{delay} =\begin{cases}d_p-d_t, & \text{if } d_p\ge d_t \\ 0 ,& \text{if } d_p< d_t  \end{cases}
\label{eq:delay}
\end{eqnarray} 

{\bf Settings.}  \textit{a) Early sound event detection task}, we use CRNN~\cite{turpault2020training} as the backbone except for Conformer. We use the Adam optimizer for all methods and follow the same training setting as~\cite{turpault2020training}. 
\textit{b) Early human action detection task}, we use ACAR as the backbone except for AIA and SlowFast. We use the SGD optimizer for all methods and follow the same training setting as~\cite{pan2021actor}.
For the uncertainty threshold, we set 0.5 for epistemic uncertainty and 0.9 for other uncertainties (entropy, vacuity, aleatoric). 

\subsection{Results and Analysis}\label{sec:result}

{\bf Early Sound Event Detection Performance.}
Table~\ref{tab:results1} shows that our proposed methods (MTENN-WBC and MTENN-UMSS) outperform all baseline models under the detection delay and early detection F1 score for the sound event early detection. The outperformance of MTENN-WBC is fairly impressive. 
This confirms that the belief comultiplication operator is the key to improving the sequential uncertainty estimation such that MTENN-WBC can significantly improve the early detection accuracy. In addition, MTENN-UMSS is a more aggressive method for early event detection, which significantly reduced the detection delay, demonstrating that vacuity uncertainty distribution changed at the event's ongoing early stage.
Furthermore, the test inference time of our approach is around 5ms, less than the streaming segment duration (60ms), which indicates that our method satisfies the real-time requirement.

\begin{table*}[ht!]
\footnotesize
\centering
\caption{Early sound event detection performance on Audio datasets.}
\begin{tabular}{l|c|c|c|c|c}
 \toprule
\multirow{2}{*}{\textbf{Datasets}} & \textbf{DESED2021} & \textbf{Explosion} & \textbf{Alarm} & \textbf{Liquid}  &\textbf{Engine} \\ &  \multicolumn{5}{c}{\textbf{Detection Delay} $\downarrow$ / \textbf{Detection F1 Score} $\uparrow$}
  \\ 
\midrule
Dual DNN & 0.386  / 0.682 & 0.325 / 0.295 &	0.257 /	0.221 &	0.467 /	0.162 & 0.324 / 0.323 \\ 
SEED & 0.252 / 0.691& 0.339 / 0.288 & 0.293 / 0.407 & 0.334 / 0.172 & 0.428 / 0.342 \\ 
Conformer & 0.372  / 0.639 & 0.444	/ 0.268 & 0.292 / 0.429 & 0.463 / 0.166 &	0.427 /	0.323 \\
CRNN & 0.284  / 0.677 & 0.415 / 0.278	&0.273 / 0.408 & 0.451 / 0.144 & 0.404 / 0.301 \\
CRNN + entropy & 0.312  / 0.669& 0.422 / 0.272 & 0.282 / 0.406 & 0.465 / 0.142 & 0.423 / 0.313 \\
CRNN + epistemic & 0.278  /  0.647 & 0.401 / 0.28 & 0.244 / 0.413 & 0.411 / 0.152 & 0.356 / 0.31 \\
CRNN + aleatoric & 0.281  /  0.643 &  0.404	/ 0.288 & 0.252 / 0.419 & 0.421 / 0.157 &	0.377 /	0.312  \\
\midrule

MTENN-WBC & \textbf{0.206}  /\textbf{ 0.727} &  \textbf{0.119} / 0.314 & 0.217 /  0.470 &	0.059 /	\textbf{0.200} &	0.294 /	\textbf{0.391}  \\
MTENN-UMSS & 0.267  / 0.575&  0.120 / \textbf{0.345} & \textbf{0.191} / \textbf{0.473} & \textbf{0.026} / 0.188 & \textbf{0.237} /	0.349 \\
\bottomrule
\end{tabular}
\label{tab:results1}
\end{table*}

{\bf Early human action detection performance.}
Table~\ref{tab:results2} shows the experiment results of early human action detection, which get a similar pattern to early sound event detection. Note that MTENN-WBC achieved the best detection accuracy (e.g., 50\% increase compared with ACAR baseline) with a large detection delay. While our MTENN-UMSS outperforms all baseline models under the detection delay and early detection F1 score.
In addition, as segment length decreases (more challenging to detect event), baseline methods' performance decreases significantly. In contrast, our proposed methods (MTENN-WBC and MTENN-UMSS) can still hold a robust performance.

\begin{table*}[h]
\scriptsize
\centering
\caption{Early human action detection performance on AVA datasets with different segment lengths (ST).}
\begin{tabular}{l|c|c|c|c|c|c}
 \toprule
  \multirow{2}{*}{\textbf{Segment Length}} & \textbf{0.75s} & \textbf{0.5s} & \textbf{0.375s} & \textbf{0.3s} & \textbf{0.25s} & \textbf{0.2s}  \\ 
&  \multicolumn{6}{c}{\textbf{Detection Delay} $\downarrow$ / \textbf{Detection F1 Score} $\uparrow$}
  \\ 
\midrule
SlowFast & 0.156 / 0.402 & 0.175 / 0.376 &	0.183 / 0.372 &	0.208 / 0.367 &	0.220 / 0.360 &	0.229 / 0.357  \\ 
AIA & 0.214 / 0.410 & 0.231 / 0.387 & 0.230 / 0.382  &	0.239 / 0.379 &	0.250 / 0.368 &	0.264 / 0.360 \\
ACAR & 0.227  / 0.446 & 0.245 / 0.430	& 0.245 / 0.424  &	0.240 / 0.406 &	0.235 / 0.399 &	0.253 / 0.389 \\
ACAR + entropy & 0.234  / 0.464 & 0.268 / 0.462 & 0.267 / 0.457  &	0.284 / 0.456 &	0.276 / 0.456 &	0.307 / 0.447\\
ACAR + epistemic & 0.232  / 0.451 & 0.265 / 0.448 & 0.264 / 0.442  &	0.441 / 0.413 &	0.274 / 0.438 &	0.301 / 0.430 \\
ACAR + aleatoric & 0.213 / 0.434 & 0.244 /0.428 & 0.243 / 0.423  &	0.259 / 0.421 &	0.254 / 0.416 &	0.275 / 0.407 \\
\midrule
MTENN-WBC & 0.308  /\textbf{ 0.736} &  0.343 / \textbf{0.734} & 0.341 / \textbf{0.706} &	0.361 / \textbf{0.708} &	0.354 / \textbf{0.706} &	0.377 / \textbf{0.698} \\
MTENN-UMSS & \textbf{0.089}  / 0.533&  \textbf{0.090} / 0.535 & \textbf{0.090} / 0.527&	\textbf{0.093} / 0.532 &	\textbf{0.093} / 0.525 &	\textbf{0.081} / 0.524 \\
\bottomrule
\end{tabular}
\label{tab:results2}
\end{table*}

\vspace{3mm}
{\bf Sensitive Analysis.} (1) Uncertainty threshold. We explore the sensitivity of the vacuity threshold used in the MTENN-WBC model. Figure~\ref{fig:experiment_effect} (a) shows the detection delay and early detection F1 score with varying vacuity threshold values. When the vacuity threshold increases, the detection delay decreases continuously, while the early detection accuracy (F1 score) decreases as well. There is a tradeoff between detection delay and detection accuracy. The higher the uncertainty threshold increase, the more overconfident predictions (predictions with high uncertainty) result in an aggressive early prediction (may predict event happen early but may cause a false positive prediction).

\begin{figure}[h]
    \centering
    \begin{subfigure}[b]{0.4\textwidth}
        \centering
        \includegraphics[width=\linewidth]{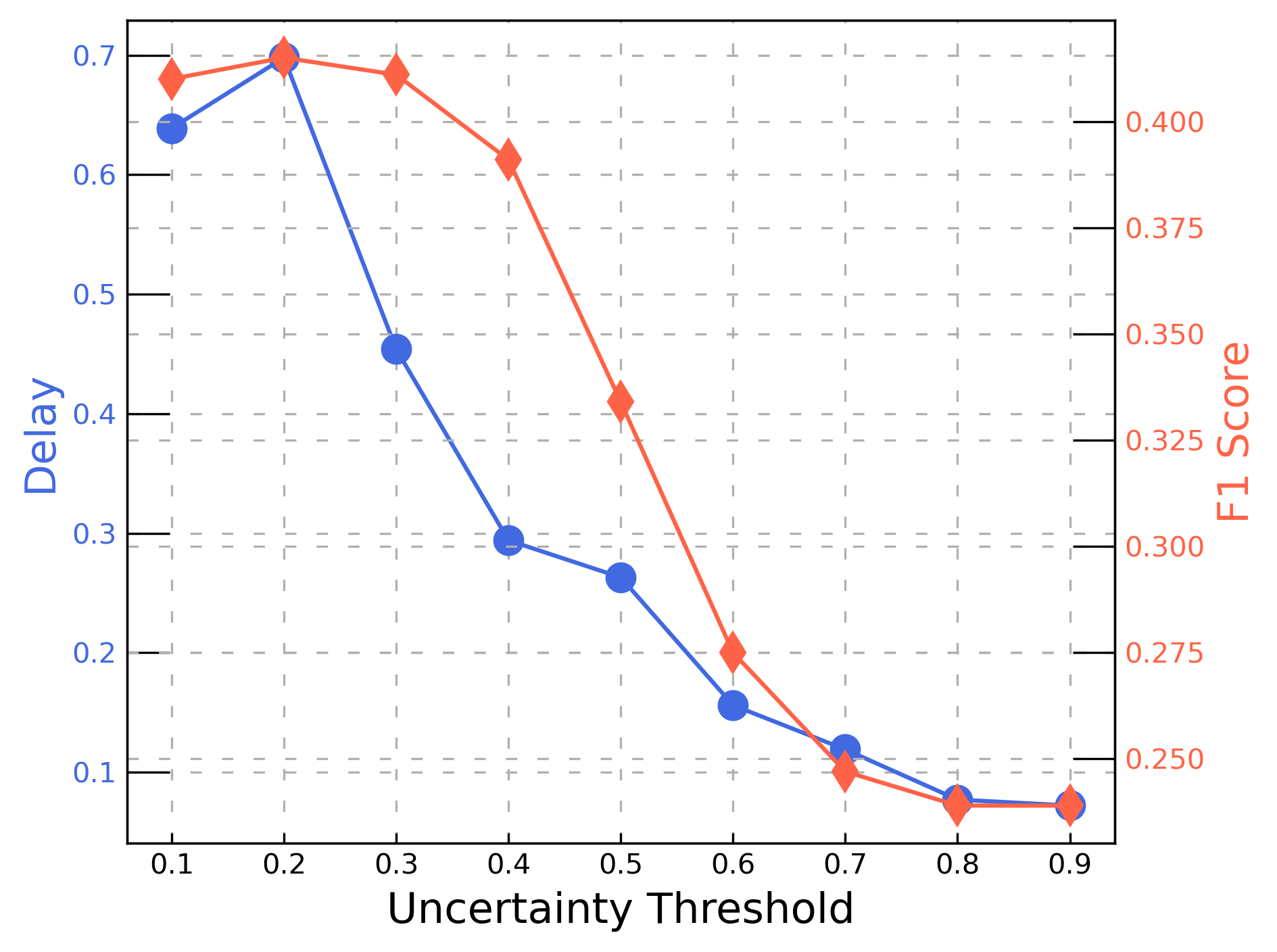}
    \caption{AudioSet (Engine)}
    \end{subfigure}
     \begin{subfigure}[b]{0.4\textwidth}
        \centering
        \includegraphics[width=\linewidth]{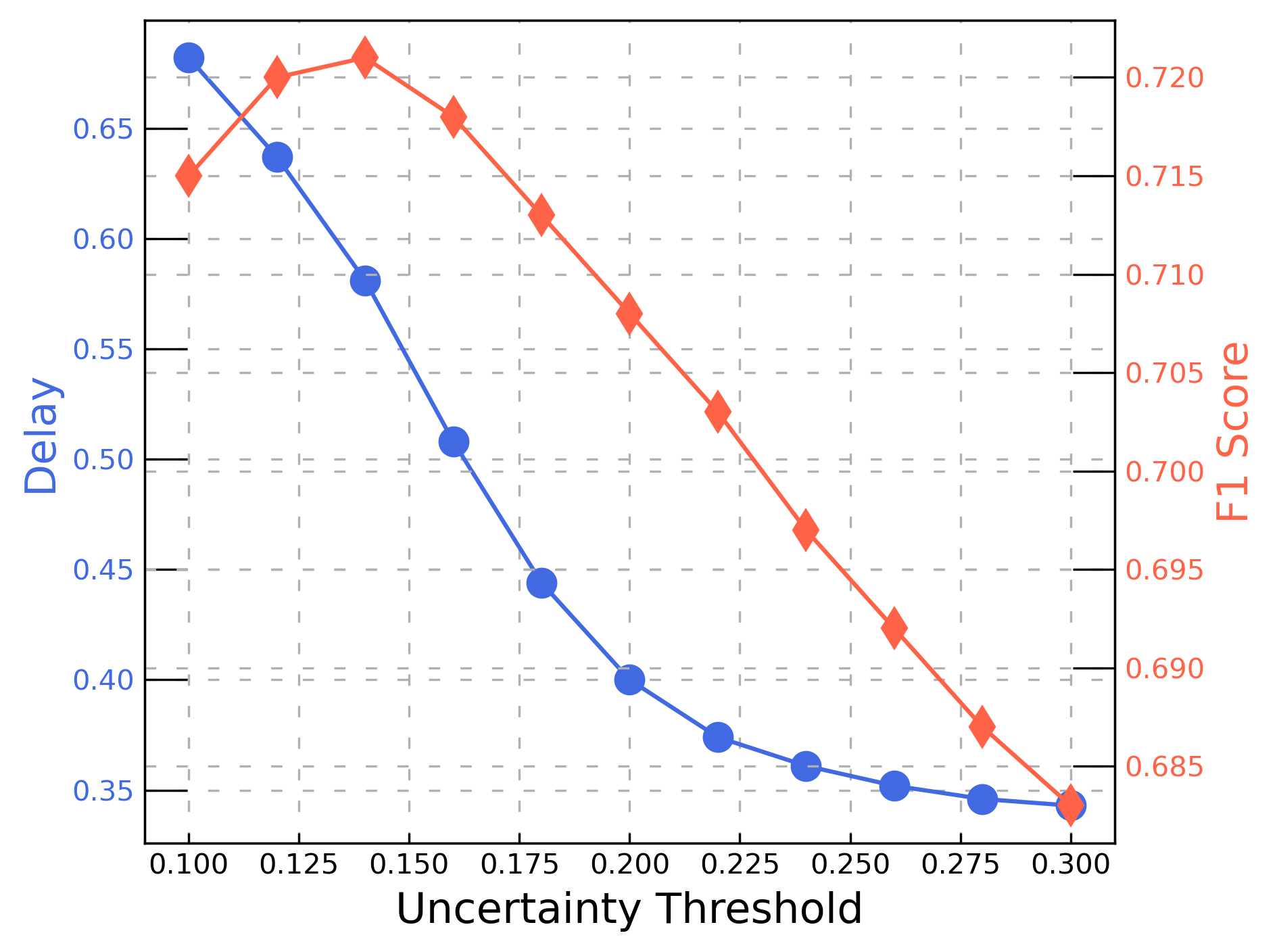}
    \caption{AVA (ST=0.3s)}
    \end{subfigure}
    \small{
    \caption{Sensitive Analysis of uncertainty threshold. There is a tradeoff between detection delay and detection accuracy. The higher uncertainty threshold increase, the more overconfidence predictions.}
    \label{fig:experiment_effect}
    }
\end{figure}

(2) Effect of sliding window size.
We analyzed the sensitivity of our proposed sequential fusional opinion to the size of sliding windows. Fig~\ref{fig:sliding_window} (b) shows the performance of detection delay and F1 score with the varying size of sliding windows. When the sliding window size increases, the detection delay continuously decreases, and detection F1 increases until the sliding window size is large enough.
The results demonstrate that sequential uncertainty estimation is critical to improving early event detection performance.

\begin{figure}[h]
    \centering
     \begin{subfigure}[b]{0.4\textwidth}
        \centering
        \includegraphics[width=\linewidth]{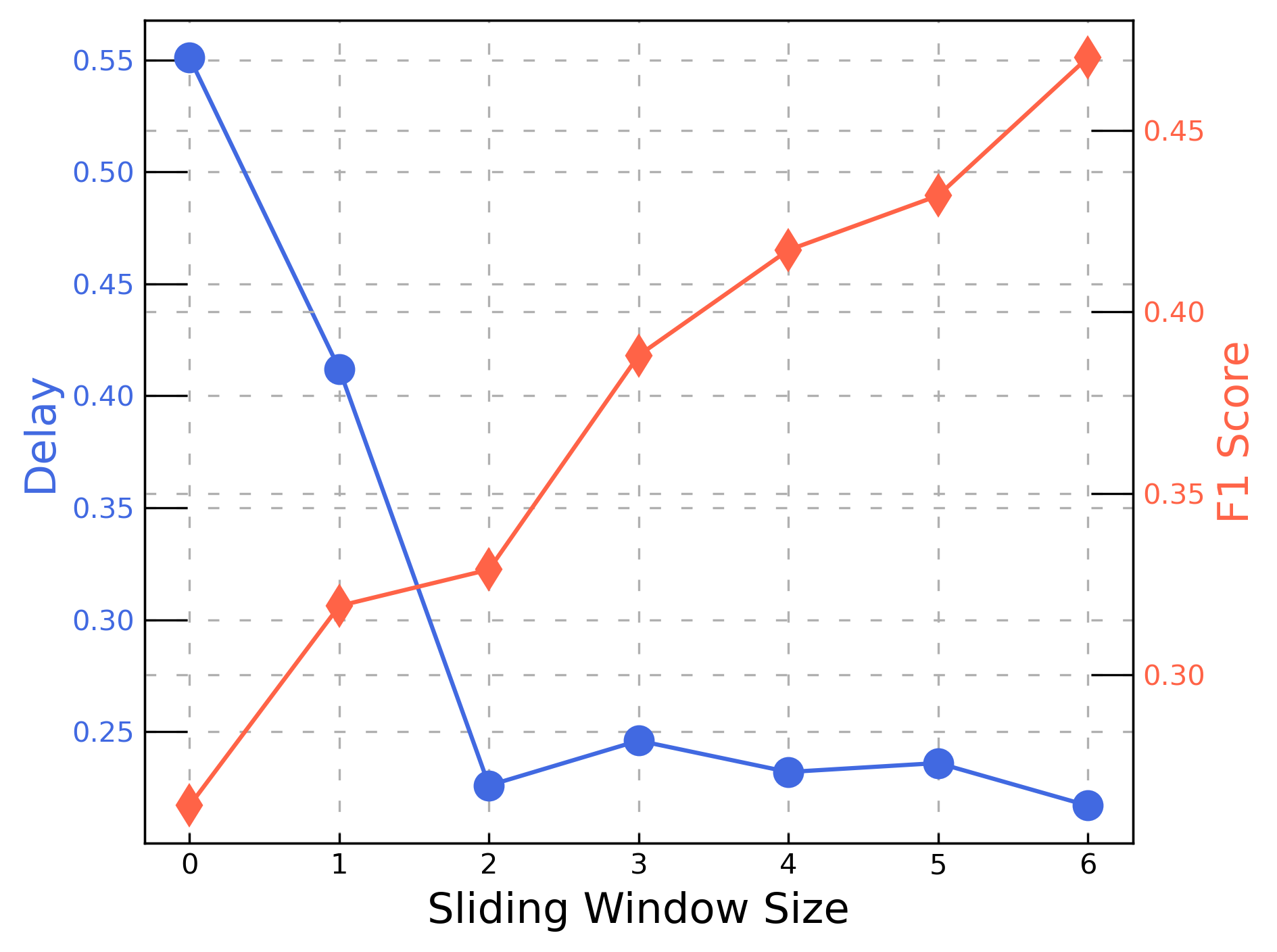}
    \caption{AudioSet (Alarm)}
    \end{subfigure}
    \begin{subfigure}[b]{0.4\textwidth}
        \centering
        \includegraphics[width=\linewidth]{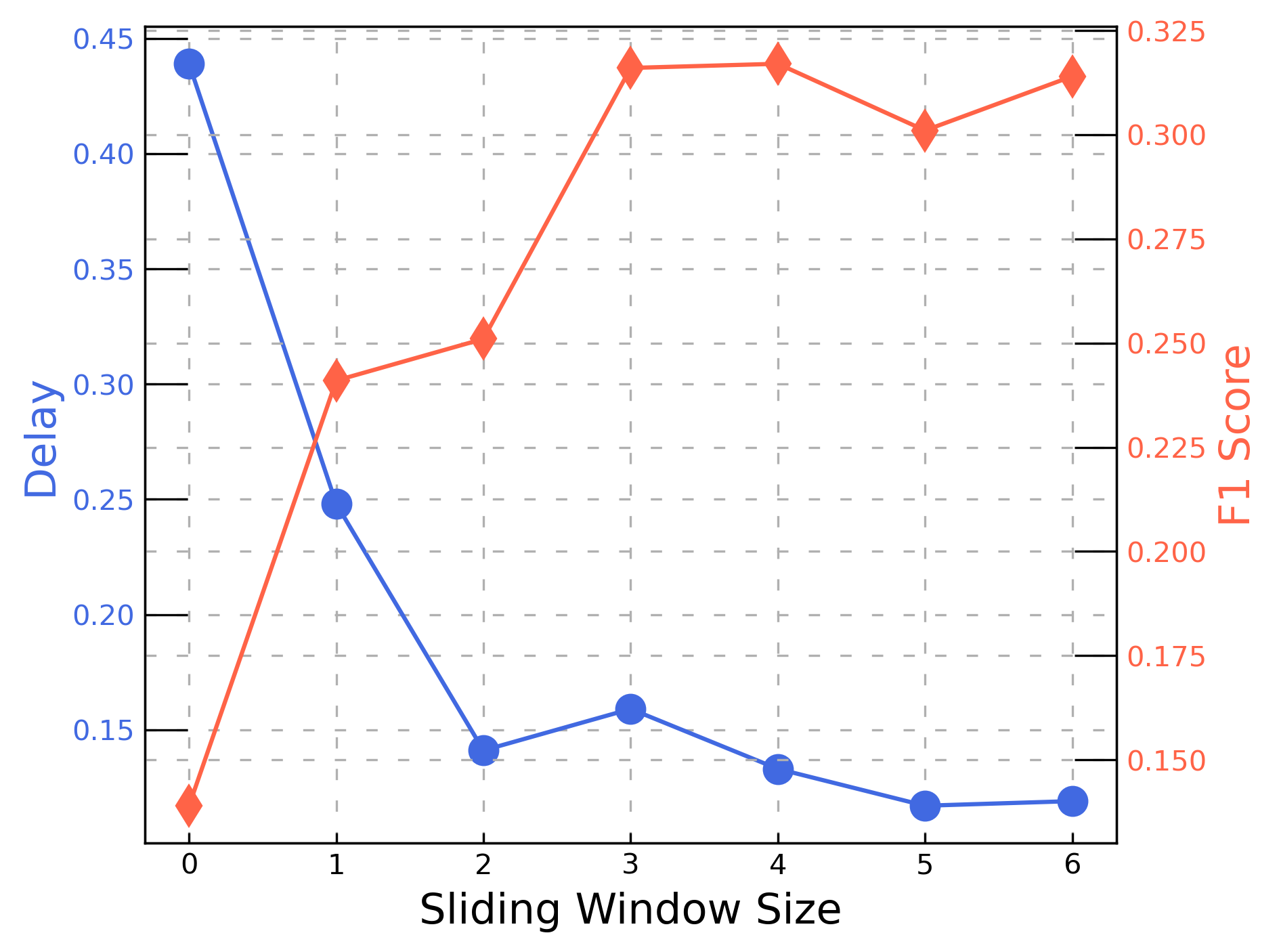}
    \caption{AudioSet (Explosion)}
    \end{subfigure}
    \small{
    \caption{Sensitive Analysis of sliding window size. When the sliding window size increases, the detection delay continuously decreases, and detection F1 increases until the sliding window size is large enough. }
    \label{fig:sliding_window}
    }
\end{figure}

{\bf Per-class performance.}
In addition to the overall comparison, we plot the per-class performance on the Audio (Engine) dataset compared with our method and SEED baselines. As shown in Figure~\ref{fig:per_class}, MTENN-WBC outperforms others under the detection accuracy among most classes. MTENN-UMSS outperforms others under the detection delay. Among most classes. Note that both SEED and MTENN cannot detect difficult events (`Heavy engine' or `Idling') due to class imbalance in the training set. But our MTENN-UMSS can detect these complex events with a small detection delay.

\begin{figure*}[h]
    \centering
    \begin{subfigure}[b]{0.48\textwidth}
        \centering
        \includegraphics[width=\linewidth]{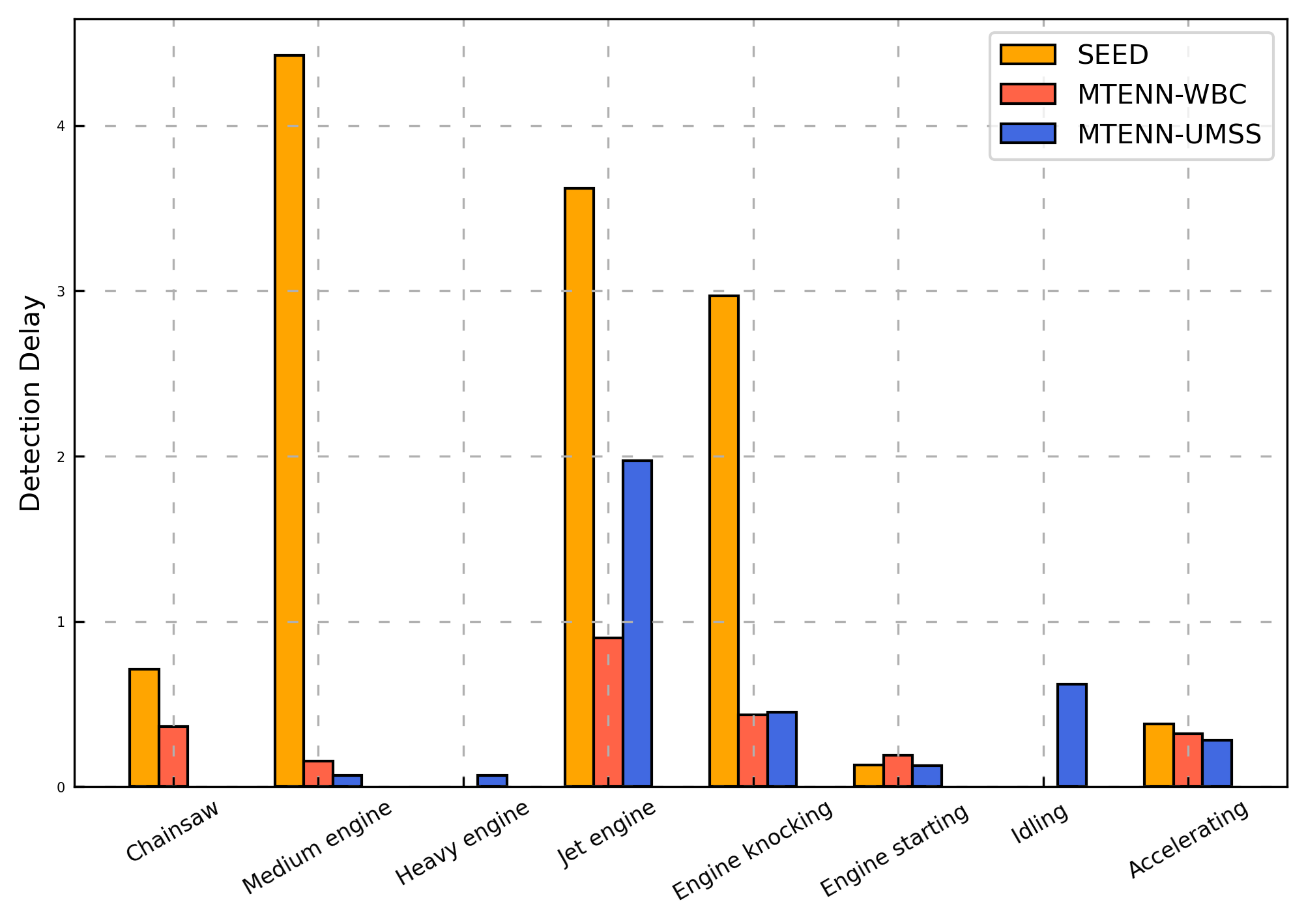}
    \caption{Detection Delay (The less the better)}
    \end{subfigure}
     \begin{subfigure}[b]{0.48\textwidth}
        \centering
        \includegraphics[width=\linewidth]{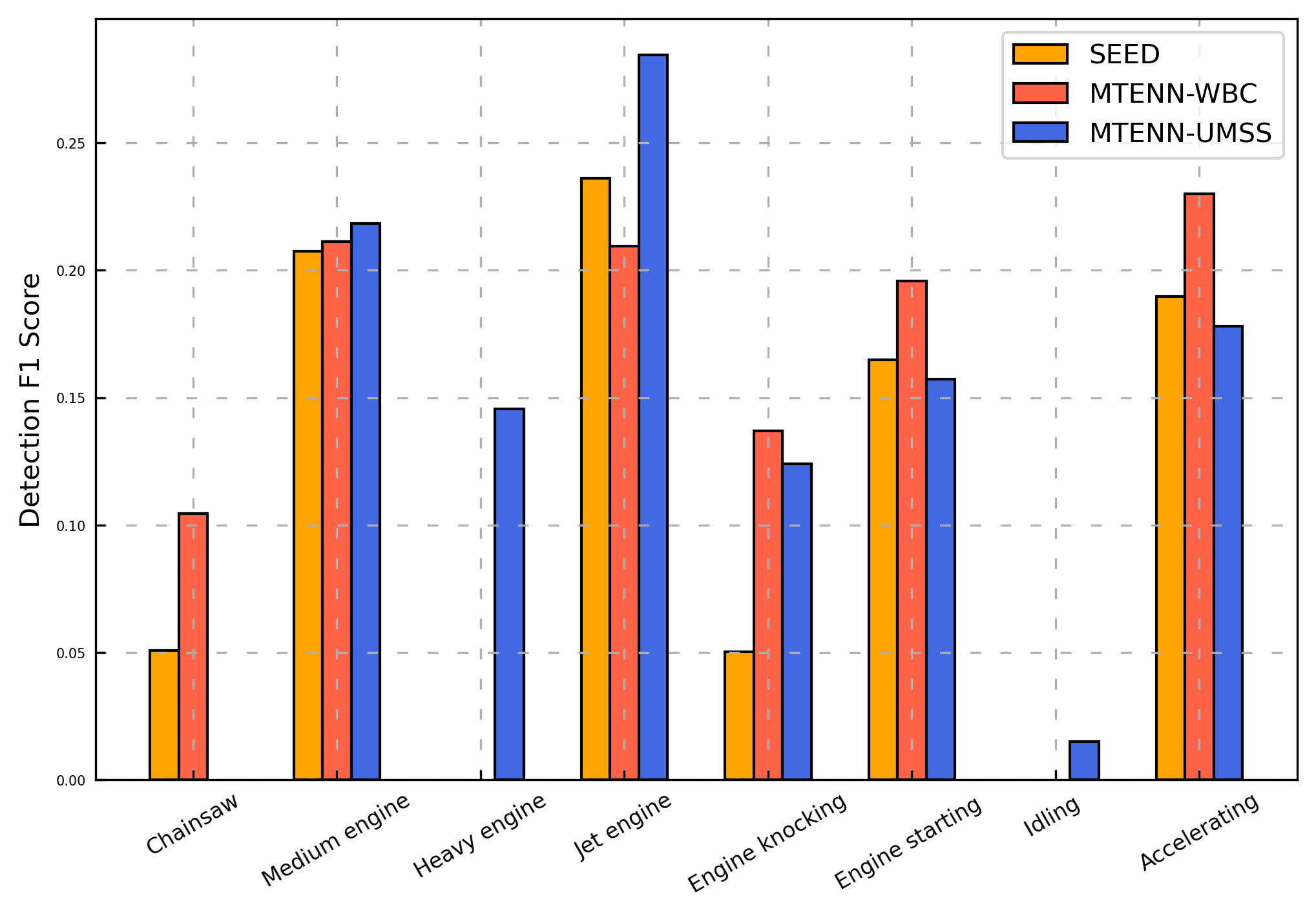}
    \caption{Detection F1 Score (The large the better)}
    \end{subfigure}
    \small{
    \caption{Per-Class Evaluation on Audio (Engine) dataset.}
    \label{fig:per_class}
    }
\end{figure*}

\begin{table}[h]
\centering
\caption{Ablation study. MTENN-BC: a variant of MTENN-WBC that uses binomial comultiplication instead of weighted binomial comultiplication; MTENN (Phase I): only consider phase I to predict event without any sequential uncertainty head; MTENN w/o MTENN loss: a variant of MTENN (Phase I) that consider BCE loss.}
\begin{tabular}{l|c|c}
 \toprule
\multirow{2}{*}{\textbf{Datasets}} & \textbf{AudioSet(Engine)} & \textbf{AudioSet(Liquid)}  \\ &  \multicolumn{2}{c}{\textbf{Detection Delay} $\downarrow$ / \textbf{Detection F1 Score} $\uparrow$}
  \\ 
\midrule
MTENN w/o MTENN loss & 0.463 / 0.307 & 0.326 / 0.083  \\ 
MTENN (Phase I) & 0.448 / 0.313 & 0.329 / 0.142  \\ 
MTENN-BC & 0.312 /0.383 &	0.074 /	0.207  \\
\midrule
MTENN-WBC & 0.294 / \textbf{0.391} &	0.057 /	\textbf{0.196}  \\
MTENN-UMSS & \textbf{0.237} / 0.349 & \textbf{0.026} / 0.193 \\
\bottomrule
\end{tabular}
\label{tab:ablation}
\end{table}

{\bf Ablation study.}
We conducted additional experiments (see Table~\ref{tab:ablation}) in order to demonstrate the contributions of the key technical components, including MTENN loss, WBC, and UMSS. Specifically, we consider three ablated models: (a) MTENN-BC, a variant of MTENN-WBC that uses binomial comultiplication via Eq.~\eqref{BC} instead of weighted binomial comultiplication; (b) MTENN (Phase I): only consider phase I to predict event via Eq.~\eqref{eq:inference} without any sequential uncertainty head; (c) MTENN w/o MTENN loss: a variant of MTENN (Phase I) that consider BCE loss (Eq.~\eqref{eq:bce}), where the probability can be calculated based on the expected probability of Beta distribution. 
The key findings obtained from this experiment are: (1) both MTENN loss and sequential uncertainty (WBC or UMSS)  can enhance the early event detection in detection delay and detection accuracy; (2) WBC is more effective for detection accuracy and UMSS is more effective for detection delay.

{\bf Inference time.} Table~\ref{tab:inference} shows the inference time for all methods used in our experiments. Note that the audio streaming segment duration is 60 milliseconds (ms), and the video streaming segment duration is 300 milliseconds when $ST=0.3s$, while our approaches only take around 5ms and 190 ms for audio and video streaming segments, respectively. This indicates that our proposed framework satisfies the real-time requirement for early event detection.

\begin{table}[h]
\centering
\caption{Compare inference time with different methods.}
\begin{tabular}{l|c|l|c}
\toprule
\multicolumn{2}{c|}{DESED2021} & \multicolumn{2}{c}{AVA (ST=0.3s)} \\
\midrule
Dual DNN & 5.1ms &   &  \\ 
SEED & 5.0ms & SlowFast & 175ms  \\ 
Conformer & 6.6ms & AIA & 181ms \\
CRNN & 5.0ms & ACAR & 187ms \\
CRNN + entropy &  5.0ms & ACAR + entropy &  188ms  \\
CRNN + epistemic & 27.0ms & ACAR + epistemic & 564ms \\
CRNN + aleatoric &  27.0ms & ACAR + aleatoric &  564ms  \\
\midrule
MTENN-WBC & 5.5ms & MTENN-WBC & 192ms  \\
MTENN-UMSS & 5.3ms & MTENN-UMSS & 190ms \\
\bottomrule
\end{tabular}
\label{tab:inference}
\end{table}

\section{Conclusion}
In this work, we propose a novel framework, Multi-Label Temporal Evidential Neural Network (MTENN), for early event detection in temporal data. MTENN is able to quality predictive uncertainty due to the lack of evidence for multi-label classifications at each time stamp based on belief/evidence theory. In addition, we introduce two novel uncertainty fusion operators (weighted binomial comultiplication (WBC) and uncertainty mean scan statistics (UMSS)) based on MTENN  to quantify the fused uncertainty of a sub-sequence for early event detection.. 
We validate the performance of our approach with state-of-the-art techniques on real-world audio and video datasets. Theoretic analysis and empirical studies demonstrate the effectiveness and efficiency of the proposed framework in both detection delay and accuracy.

\chapter{Conclusion and Future Work}
\label{chapter:conclusion}
    \section{Conclusion of Completed Work}
In this dissertation, the proposed research aims on the design a general multi-source uncertainty framework to quantify the inherent uncertainties of deep neural networks.
We are focused on three major types of direction, including uncertainty-aware semi-supervised learning on graph data (Chapter~\ref{chapter:2}), uncertainty-aware robust semi-supervised learning (Chapter~\ref{chapter:3}), and uncertainty-aware early event detection with multi-labels (Chapter~\ref{chapter:4}).

In Chapter~\ref{chapter:2}, we study the uncertainty decomposition problem for graph neural networks. We first provide a theoretical analysis of
the relationships between different types of uncertainties. Then, we proposed a multi-source uncertainty framework of GNNs for semi-supervised
node classification. Our proposed framework provides an effective way of predicting node classification and out-of-distribution detection considering multiple types of uncertainty. We leveraged various types of uncertainty estimates from both deep learning and evidence/belief theory domains. Through our extensive experiments, we found that dissonance-based detection yielded the best performance on misclassification detection while vacuity-based detection performed the best for OOD detection, compared to other competitive counterparts. In particular, it was noticeable that applying GKDE and the Teacher network further enhanced the accuracy in node classification and uncertainty estimates.

We further study uncertainty in robust semi-supervised learning setting in Chapter~\ref{chapter:3}. In this setting, traditional semi-supervised learning (SSL) performance can degrade substantially and is sometimes even worse than simple supervised learning approaches due to the OODs involved in the unlabeled pool. To solve this problem, we first study the impact of OOD data on SSL algorithms and demonstrate empirically on synthetic data that the SSL algorithms' performance depends on how close the OOD instances are to the decision boundary (and the ID data instances), which can also be measured via vacuity uncertainty introduced in Chapter~\ref{chapter:2}. Based on this observation, we proposed a novel unified uncertainty-aware robust SSL framework that treats the weights directly as hyper-parameters in conjunction with weighted batch normalization, which is designed to improve the robustness of BN against OODs. In addition, we proposed two efficient bi-level algorithms for our proposed robust SSL approach, including
meta-approximation and implicit-differentiation based, that have different tradeoffs on computational efficiency and accuracy. We also conduct a theoretical analysis of the impact of faraway OODs in the BN step and discuss the connection between our approach (high-order approximation based on implicit differentiation) and low-order approximation approaches.

Finally, in Chapter~\ref{chapter:4}, we consider a time series setting for early event detection. In this setting, a temporal event with multiple labels occurs sequentially along the timeline. The goal of this work is to accurately detect all classes at the ongoing stage of an event within the least amount of time. The problem is formulated as an online multi-label time series classification problem.
To this end, technically, we propose a novel framework, Multi-Label Temporal Evidential Neural Network, and two novel uncertainty estimation heads (weighted binomial comultiplication (WBC) and uncertainty mena scan statistics(UMSS)) to quantify the fused uncertainty of a sub-sequence for early event detection. We empirically show that the proposed approach outperforms state-of-the-art techniques on real-world datasets.

In light of these discoveries, the following are some interesting future avenues to investigate.

\section{Future Work}
\subsection{Quantification of multidimensional uncertainty}
Different forms of the distribution have respective limitations in quantifying uncertainty. For example, a Dirichlet distribution can estimate vacuity and
dissonance, but not vagueness (e.g., non-distinctive class labels, such as ‘A or B’). Vagueness can emerge only in a hyper-Dirichlet distribution; but both Dirichlet and hyper-Dirichlet distributions are unimodels and cannot estimate data
source dependent uncertainty, i.e., the training data fused from multiple sources. Even subjective logic does not provide ways of measuring the uncertainty types based on multi-model or heterogeneous distributions over a simplex. Therefore, future researchers can explore different forms of distributions to measure different types of uncertainties, such as Dirichlet, hyper-Dirichlet, a mixture of Dirichlet, Logistic Normal distributions, or implicit generative models that are free of parametric forms about the distribution.

\subsection{Interpretation of Multidimensional Uncertainty}
Based on our in-depth literature review on different uncertainty types, we observed that some uncertainty types were studied in different domains as DL and belief
theory with using distinctive terminologies but referring to the same uncertainty types. For example, the distributional uncertainty in deep learning and vacuity in subjective logic are designed to measure uncertainty caused by a lack of information and knowledge.
Based on my prior work~\cite{zhao2020uncertainty}, we found that vacuity is more effective than distributional uncertainty for node-level OOD detection in graph data. This finding explains the difference between distributional uncertainty and vacuity, even if the intent of formulating them was from the same origin. Hence, in this plan, I will further delve into how and why these two similar types of uncertainties perform differently through empirical experiments and theoretical proof. Here is the research question for future investigation: what type of uncertainty is more critical than other types to maximize decision effectiveness under what problem contexts (e.g., images, graph data) and what task types (e.g., classification prediction or OOD)?


\appendix 
\chapter*{Proofs of the Proposed Theorems}
\label{appendix}
    \section{Proof of Theorem~\ref{theorem_un}}
\label{app:A1}

\textbf{Interpretation}. {\bf Theorem 1.1 (a)} implies that increases in both uncertainty types may not happen 
at the same time.  A higher vacuity leads to a lower dissonance, and vice versa (a higher dissonance leads to a lower vacuity).  This indicates that a high dissonance only occurs only when a large amount of evidence is available and the vacuity is low.  {\bf Theorem 1.1 (b)} shows relationships between vacuity and epistemic uncertainty in which vacuity is an upper bound of 
epistemic uncertainty.  Although some existing approaches~\cite{josang2016subjective, sensoy2018evidential}  treat epistemic uncertainty the same as vacuity, it is not necessarily true except for an extreme case \feng{where} a sufficiently large amount of evidence available, making vacuity  \feng{close to} zero. {\bf Theorem 1.2 (a) and (b)} explain how entropy differs from vacuity and/or dissonance.  We observe \feng{that} entropy is 1 when either vacuity or dissonance is 0.  This implies that entropy cannot distinguish different types of uncertainty \feng{due to different root causes}.  For example, \feng{a} high entropy is observed when an example is either an OOD or misclassified example.  Similarly, \feng{a} high aleatoric uncertainty \feng{value} and 
\feng{a} low epistemic uncertainty \feng{value} are observed under both cases.  However, vacuity and dissonance can capture \feng{different} causes of uncertainty \feng{due to lack of information and knowledge and to conflicting evidence, respectively.}  For example, an OOD objects typically show \feng{a} high vacuity \feng{value and a} low dissonance \feng{value} while \feng{a conflicting} \feng{prediction} exhibit\feng{s} \feng{a} low vacuity \feng{and a} high dissonance.

\begin{proof}
    1. (a) 
    Let the opinion $\omega = [b_1, \ldots, b_K, u_v]$, where $K$ is the number of classes, $b_i$ is the belief for class $i$, $u_v$ is the uncertainty mass (vacuity), and $\sum_{i=1}^K b_i + u_v =1$. Dissonance has an upper bound with 
\begin{eqnarray}
u_{diss} &=& \sum_{i=1}^K \Big(\frac{b_i\sum_{j=1, j\neq i}^K b_j \text{Bal}(b_i, b_j)}{\sum_{j=1, j\neq i}^K b_j} \Big) \\
&\le & \sum_{i=1}^K \Big(\frac{b_i\sum_{j=1, j\neq i}^K b_j }{\sum_{j=1, j\neq i}^K b_j} \Big), \quad (\text{since } 0\le \text{Bal}(b_i, b_j) \le 1 ) \nonumber \\
&=& \sum_{i=1}^K b_i, \nonumber
\end{eqnarray}
where $\text{Bal}(b_i, b_j)$ is the relative mass balance, then we have
\begin{eqnarray}
u_{v} + u_{diss} \le \sum_{i=1}^K b_i + u_v = 1.
\end{eqnarray}

1. (b) For the multinomial random variable $y$, we have
\begin{eqnarray}
y \sim \text{Cal}(\p), \quad \p\sim \text{Dir}({\bm \alpha}),
\end{eqnarray}
where $\text{Cal(\p)}$ is the categorical distribution and $ \text{Dir}({\bm \alpha})$ is Dirichlet distribution. Then we have
\begin{eqnarray}
\text{Prob}(y|{\bm \alpha}) = \int \text{Prob}(y|\p) \text{Prob}(\p|{\bm \alpha}) d\p,
\end{eqnarray}
\feng{and} the epistemic \feng{uncertainty is} estimated by mutual information,
\begin{eqnarray}
    \mathcal{I}[y, \p|{\bm \alpha}]
    =  \mathcal{H}\Big[\mathbb{E}_{\text{Prob}(\p|{\bm \alpha})}[P(y|\p)]\Big]- \mathbb{E}_{\text{Prob}(\p|{\bm \alpha})}\Big[\mathcal{H}[P(y|\p)] \Big].
\end{eqnarray}
Now we consider another measure of ensemble diversity\feng{:} \textit{Expected Pairwise KL-Divergence} between each model in the ensemble. Here the expected pairwise KL-Divergence between two independent distribution\feng{s, including} $P(y|\p_1)$ \feng{and} $P(y|\p_2)$, where $\p_1$ and $\p_2$ are two independent samples from $\text{Prob}(\p|{\bm \alpha})$, \feng{can be} computed,
\begin{eqnarray}
    \mathcal{K}[y, \p|{\bm \alpha}] &=& \mathbb{E}_{\text{Prob}(\p_1|{\bm \alpha}\text{Prob}(\p_2|{\bm \alpha})}\Big[KL[P(y| \p_1)\| P(y| \p_2)]   \Big]  \\
    &=& -\sum_{i=1}^K \mathbb{E}_{\text{Prob}(\p_1|{\bm \alpha})}[P(y|\p_1)] \mathbb{E}_{\text{Prob}(\p_2|{\bm \alpha})}[\ln P(y|\p_2)]  - \mathbb{E}_{\text{Prob}(\p|{\bm \alpha})}\Big[\mathcal{H} [P(y|\p)] \Big] \nonumber \\
    &\ge& \mathcal{I}[y, \p|{\bm \alpha}], \nonumber
\end{eqnarray}
where $\mathcal{I}[y, \p_1|{\bm \alpha}] = \mathcal{I}[y, \p_2|{\bm \alpha}]$. \feng{W}e consider Dirichlet ensemble, the \textit{Expected Pairwise KL Divergence},
\begin{eqnarray}
    \mathcal{K}[y, \p|{\bm \alpha}] &=& -\sum_{i=1}^K \frac{\alpha_i}{S} \Big(\psi(\alpha_i ) - \psi(S)  \Big)  - \sum_{i=1}^K -\frac{\alpha_i}{S}\Big(\psi(\alpha_i + 1) - \psi(S + 1)  \Big) \nonumber \\
    &=& \frac{K-1}{S},
\end{eqnarray}
where $S = \sum_{i=1}^K \alpha_i$ and $\psi(\cdot)$ is the \textit{digamma Function}, which is the derivative of the natural logarithm of the gamma function.
Now \feng{we obtain} the relation\feng{s} between vacuity and epistemic,
\begin{eqnarray}
    \underbrace{\frac{K}{S}}_{\text{Vacuity}} >\mathcal{K}[y, \p|{\bm \alpha}] = \frac{K-1}{S}
    \ge  \underbrace{\mathcal{I}[y, \p|{\bm \alpha}]}_{\text{Epistemic}}.
\end{eqnarray}

2. (a) For \feng{an} out-of-distribution sample, $\alpha=[1, \ldots, 1]$,
\feng{the} vacuity \feng{can be calculated as}
\begin{eqnarray}
u_v &=& \frac{K}{\sum_{i=1}^K \alpha_i} = \frac{K}{K} = 1, 
\end{eqnarray}
and the belief mass $b_i = (\alpha_i - 1)/\sum_{i=1}^K \alpha_i= 0$, we estimate dissonance,
\begin{eqnarray}
u_{diss} &=& \sum_{i=1}^K \Big(\frac{b_i\sum_{j=1, j\neq i}^K b_j \text{Bal}(b_i, b_j)}{\sum_{j=1, j\neq i}^K b_j} \Big) = 0.
\end{eqnarray}
Given the expected probability $\hat{p} = [1/K, \ldots, 1/K]^\top$, the entropy is calculated based on $\log_K$,
\begin{eqnarray}
u_{en} = \mathcal{H}[\hat{p}] =-\sum_{i=1}^K \hat{p}_i \log_K \hat{p}_i = -\sum_{i=1}^K \frac{1}{K} \log_K \frac{1}{K} = \log_K \frac{1}{K}^{-1} =\log_K K = 1\feng{,}
\end{eqnarray}
where $\mathcal{H}(\cdot)$ is the entropy.  Based on Dirichlet distribution, the aleatoric uncertainty refers to the expected entropy,
\begin{eqnarray}
u_{alea} &=& \mathbb{E}_{p\sim \text{Dir}(\alpha)}[\mathcal{H}[p]] \\
&=& -\sum_{i=1}^K \frac{\Gamma (S)}{\prod_{i=1}^K\Gamma(\alpha_i)} \int_{S_K} p_i\log_K p_i \prod _{i=1}^K p_i^{\alpha_i-1} d {\bm p} \nonumber \\
&=& - \frac{1}{\ln K}\sum_{i=1}^K \frac{\Gamma (S)}{\prod_{i=1}^K\Gamma(\alpha_i)} \int_{S_K} p_i\ln p_i \prod _{i=1}^K p_i^{\alpha_i-1} d {\bm p} \nonumber \\
&=& -\frac{1}{\ln K}\sum_{i=1}^K \frac{\alpha_i}{S} \frac{\Gamma(S+1)}{\Gamma(\alpha_i +1)\prod_{i'=1, \neq i}^K \Gamma(\alpha_{i'})} \int_{S_K} p_i^{\alpha_i}\ln p_i \prod _{i'=1, \neq i}^K p_{i'}^{\alpha_{i'}-1} d {\bm p} \nonumber \\
&=& \frac{1}{\ln K}\sum_{i=1}^K \frac{\alpha_i}{S} \big(\psi(S+1) - \psi(\alpha_i +1) \big) \nonumber \\
&=& \frac{1}{\ln K}\sum_{i=1}^K \frac{1}{K}(\psi(K+1)-\psi(2)) \nonumber \\
&=& \frac{1}{\ln K} (\psi(K+1)-\psi(2)) \nonumber \\
&=&\frac{1}{\ln K} ( \psi(2) +\sum_{k=2}^K \frac{1}{k}-\psi(2)) \nonumber \\
&=&\frac{1}{\ln K} \sum_{k=2}^K \frac{1}{k} <\frac{1}{\ln K} \ln K = 1, \nonumber
\label{Eq_alea}
\end{eqnarray}
where $S= \sum_{i=1}^K \alpha_i$\feng{,} ${\bm p} = [p_1, \ldots, p_K]^\top$, \feng{and} $K\ge 2$ is the number of category. The epistemic uncertainty \feng{can be calculated via} the mutual information,
\begin{eqnarray}
u_{epis} &=& \mathcal{H}[\mathbb{E}_{p\sim \text{Dir}(\alpha)}[p]] - \mathbb{E}_{p\sim \text{Dir}(\alpha)}[\mathcal{H}[p]] \\
&=& \mathcal{H}[\hat{p}] - u_{alea} \nonumber \\ 
&=& 1 - \frac{1}{\ln K} \sum_{k=2}^K \frac{1}{k} < 1. \nonumber
\end{eqnarray}
To compare aleatoric \feng{uncertainty} \feng{with} epistemic uncertainty, we first prove that aleatoric uncertainty (Eq.~\eqref{Eq_alea}) is monotonically increasing and converging to 1 \feng{as $K$ increases}. \feng{B}ased on \textit{Lemma~\ref{lemma1}}\feng{,} we have
\begin{eqnarray}
&&\Big(\ln(K+1)-\ln K \Big)\sum_{k=2}^K \frac{1}{k} < \frac{\ln K}{K+1} \nonumber \\
&&\Rightarrow \ln(K+1)\sum_{k=2}^K \frac{1}{k}  < \ln K \Big(\sum_{k=2}^K \frac{1}{k}  + \frac{1}{K+1} \Big) = \ln K \sum_{k=2}^{K+1} \frac{1}{k} \nonumber \\
&& \Rightarrow \frac{1}{\ln K} \sum_{k=2}^K \frac{1}{k}  < \frac{1}{\ln (K+1)} \sum_{k=2}^{K+1} \frac{1}{k} \label{use_lemma1}\feng{.}
\end{eqnarray}
\feng{B}ased on Eq.~\eqref{use_lemma1} and Eq.~\eqref{Eq_alea}, we prove that aleatoric uncertainty is monotonically increasing with respect to $K$. So the minimum aleatoric \feng{can be shown to} be $\frac{1}{\ln 2} \frac{1}{2}$\feng{,} when $K=2$.

\feng{Similarly,} for epistemic uncertainty, which is monotonically decreasing as $K$ increases based on \textit{Lemma~\ref{lemma1}}, the maximum epistemic \feng{can be shown to} be $1- \frac{1}{\ln 2} \frac{1}{2}$ when $K=2$. Then we have,
\begin{eqnarray}
u_{alea} \ge  \frac{1}{\ln 2} \frac{1}{2} > 1 - \frac{1}{2\ln 2} \ge u_{epis} 
\end{eqnarray}

Therefore, we prove \feng{that} $1= u_v = u_{en}>  u_{alea} > u_{epis} > u_{diss} = 0 $.
    
2. (b) For a conflict\feng{ing} prediction, i.e., $\alpha=[\alpha_1, \ldots, \alpha_K]$\feng{,}  with $\alpha_1 = \alpha_2 =\cdots = \alpha_K =C$, and $S=\sum_{i=1}^K \alpha_i = CK$, the expected  probability $\hat{p}=[1/K, \ldots, 1/K]^\top$, the belief mass $b_i=(\alpha_i-1)/S$, \feng{and the} vacuity \feng{can be calculated} as
\begin{eqnarray}
u_v &=& \frac{K}{S} \xrightarrow[]{S\rightarrow \infty} 0, 
\end{eqnarray}
and the dissonance \feng{can be calculated} as
\begin{eqnarray}
u_{diss} &=& \sum_{i=1}^K \Big(\frac{b_i\sum_{j=1, j\neq i}^K b_j \text{Bal}(b_i, b_j)}{\sum_{j=1, j\neq i}^K b_j} \Big) = \sum_{i=1}^K b_i \\
&=& \sum_{i=1}^K\left(\frac{a_i-1}{\sum_{i=1}^Ka_i}\right) \nonumber \\
&=&\frac{\sum_{i=1}^K a_i-k}{\sum_{i=1}^K a_i} \nonumber\\
&=&1-\frac{K}{S}\xrightarrow[]{S\rightarrow \infty} 1. \nonumber
\end{eqnarray}
Given the expected probability $\hat{p} = [1/K, \ldots, 1/K]^\top$, the entropy \feng{can be calculated based on Dirichlet distribution},
\begin{eqnarray}
u_{en} &=& \mathcal{H}[\hat{p}] = -\sum_{i=1}^K \hat{p}_i \log_K \hat{p}_i = 1\feng{,}
\end{eqnarray}
and \feng{the} aleatoric uncertainty is estimated as the expected entropy,
\begin{eqnarray}
u_{alea} &=& \mathbb{E}_{p\sim \text{Dir}(\alpha)}[\mathcal{H}[p]] \\
&=& -\sum_{i=1}^K \frac{\Gamma (S)}{\prod_{i=1}^K\Gamma(\alpha_i)} \int_{S_K} p_i\log_K p_i \prod _{i=1}^K p_i^{\alpha_i-1} d {\bm p} \nonumber \\
&=& - \frac{1}{\ln K}\sum_{i=1}^K \frac{\Gamma (S)}{\prod_{i=1}^K\Gamma(\alpha_i)} \int_{S_K} p_i\ln p_i \prod _{i=1}^K p_i^{\alpha_i-1} d {\bm p} \nonumber \\
&=& -\frac{1}{\ln K}\sum_{i=1}^K \frac{\alpha_i}{S} \frac{\Gamma(S+1)}{\Gamma(\alpha_i +1)\prod_{i'=1, \neq i}^K \Gamma(\alpha_{i'})} \int_{S_K} p_i^{\alpha_i}\ln p_i \prod _{i'=1, \neq i}^K p_{i'}^{\alpha_{i'}-1} d {\bm p} \nonumber \\
&=& \frac{1}{\ln K}\sum_{i=1}^K \frac{\alpha_i}{S} \big(\psi(S+1) - \psi(\alpha_i +1) \big) \nonumber \\
&=& \frac{1}{\ln K}\sum_{i=1}^K \frac{1}{K}(\psi(S+1)-\psi(C + 1)) \nonumber \\
&=& \frac{1}{\ln K} (\psi(S+1)-\psi(C + 1)) \nonumber \\
&=&\frac{1}{\ln K} ( \psi(C + 1) +\sum_{k=C + 1}^S \frac{1}{k}-\psi(C + 1)) \nonumber \\
&=& \frac{1}{\ln K} \sum_{k=C+1}^S \frac{1}{k} \xrightarrow[]{S\rightarrow \infty} 1. \nonumber 
\end{eqnarray}
The epistemic \feng{uncertainty can be calculated via} mutual information,
\begin{eqnarray}
u_{epis} &=& \mathcal{H}[\mathbb{E}_{p\sim \text{Dir}(\alpha)}[p]] - \mathbb{E}_{p\sim \text{Dir}(\alpha)}[\mathcal{H}[p]] \\
&=& \mathcal{H}[\hat{p}] - u_{alea} \nonumber \\ 
&=& 1 - \frac{1}{\ln K} \sum_{k=C+1}^S \frac{1}{k} \xrightarrow[]{S\rightarrow \infty} 0. \nonumber 
\end{eqnarray}
Now we compare aleatoric uncertainty \feng{with} vacuity, 
\begin{eqnarray}
u_{alea} &=& \frac{1}{\ln K} \sum_{k=C+1}^S \frac{1}{k} \\
&=& \frac{1}{\ln K} \sum_{k=C+1}^{CK} \frac{1}{k} \nonumber \\
&=& \frac{\ln(CK+1)-\ln(C+1)}{\ln K} \nonumber \\ 
&=& \frac{\ln(K-\frac{K-1}{C+1})}{\ln K} \nonumber \\ 
&>& \frac{\ln(K-\frac{K-1}{2})}{\ln K} \nonumber \\ 
&=& \frac{\ln(4/K+4/K + 1/2)}{\ln K} \nonumber \\ 
&\ge& \frac{\ln[3(4/K+4/K + 1/2)^{\frac{1}{3}}]}{\ln K} \nonumber \\ 
&=& \frac{\ln 3 + \frac{1}{3} \ln (\frac{K^2}{32})}{\ln K} \nonumber \\ &=& \frac{\ln 3 + \frac{2}{3} \ln K - \frac{1}{3}\ln 32}{\ln K} > \frac{2}{3}. \nonumber
\label{eq_alea2}
\end{eqnarray}
\feng{B}ased on Eq.~\eqref{eq_alea2}, when $C > \frac{3}{2}$, we have
\begin{eqnarray}
u_{alea} > \frac{2}{3} > \frac{1}{C} = u_v
\end{eqnarray}
We have already prove\feng{d} that $u_v>u_{epis}$, when $u_{en}=1$, we have $u_{alea}>u_{diss}$
Therefore, we prove \feng{that} $u_{en}>  u_{alea} > u_{diss} >  u_{v} > u_{epis}$ with $u_{en} = 1, u_{diss}\rightarrow 1, u_{alea}\rightarrow 1, u_{v}\rightarrow 0, u_{epis}\rightarrow 0$
\end{proof}

\begin{lemma}
For all integer $N\ge 2$, we have $\sum_{n=2}^N\frac{1}{n} < \frac{\ln N}{(N+1)\ln (\frac{N+1}{N})}$. \label{lemma1}
\end{lemma}

\begin{proof}
We will prove by induction that, for all integer $N\ge 2$,
\begin{eqnarray}
\sum_{n=2}^N\frac{1}{n} < \frac{\ln N}{(N+1)\ln (\frac{N+1}{N})} \label{statement1}.
\end{eqnarray}
\textit{Base case}: \feng{W}hen $N=2$\feng{,} we have $ \frac{1}{2} < \frac{\ln 2}{3 \ln \frac{3}{2}}$ \feng{and} Eq.~\eqref{statement1} is true for $N=2$.

\textit{Induction step}: \feng{L}et \feng{the} integer $K\ge 2$ is given and suppose Eq.~\eqref{statement1} is true for $N=K$, then
\begin{eqnarray}
\sum_{k=2}^{K+1}\frac{1}{k}= \frac{1}{K + 1} + \sum_{k=2}^{K}\frac{1}{k} 
< \frac{1}{K + 1} + \frac{\ln K}{(K+1)\ln (\frac{K+1}{K})} 
= \frac{\ln(K + 1)}{(K + 1)\ln (\frac{K + 1}{K})}. 
\label{step1}
\end{eqnarray}

Denote that $g(x) = (x+1)\ln (\frac{x+1}{x})$ with $x> 2$\feng{.} \feng{W}e get its derivative, $g'(x) = \ln (1+ \frac{1}{x})- \frac{1}{x}< 0$\feng{,} such that $g(x)$ is monotonically decreasing, which results in $g(K)> g(K + 1)$\feng{.} \feng{B}ased on Eq.~\eqref{step1} we have,
\begin{eqnarray}
\sum_{k=2}^{K+1}\frac{1}{k}< \frac{\ln (K +1)}{g(K)} < \frac{\ln (K + 1)}{g(K + 1)} = \frac{\ln(K + 1)}{(K + 2)\ln (\frac{K + 2}{K+1})}. 
\end{eqnarray}
Thus, Eq.~\eqref{statement1} holds for $N=K+1$, and the proof of the induction step is complete.

\textit{Conclusion}: By the principle of induction, Eq.~\eqref{statement1} is true for all integer $N\ge 2$.
\end{proof}

\section{Derivations for Joint Probability and KL Divergence}\label{app:A2}

\subsection{Joint Probability}
We infer the joint probability (Eq.~\eqref{Baye_model}) by:
\begin{eqnarray}
\small
&&p(\y |A, \rr; \mathcal{G}) = \int \int \text{Prob}(\y | \p) \text{Prob}(\p | A, \rr; \bm{\theta} ) \text{Prob}(\bm{\theta} | \mathcal{G}) d \p d\bm{\theta} \nonumber\\
&\approx& \int \int \text{Prob}(\y | \p) \text{Prob}(\p | A, \rr; \bm{\theta} ) q(\bm{\theta}) d \p d\theta  \nonumber\\
&\approx& \frac{1}{M}\sum_{m=1}^M \int \text{Prob}(\y | \p) \text{Prob}(\p | A, \rr;\bm{\theta}^{(m)} ) d \p, \quad  \bm{\theta}^{(m)} \sim q( \bm{\theta}) \nonumber \\
&\approx& \frac{1}{M}\sum_{m=1}^M \int \sum_{i=1}^N \text{Prob}(\y_i | \p_i) \text{Prob}(\p_i | A, \rr;\bm{\theta}^{(m)} ) d \p_i, \quad  \bm{\theta}^{(m)} \sim q( \bm{\theta}) \nonumber \\
&\approx& \frac{1}{M}\sum_{m=1}^M  \sum_{i=1}^N \int \text{Prob}(\y_i | \p_i) \text{Prob}(\p_i | A, \rr;\bm{\theta}^{(m)} ) d \p_i, \quad \bm{\theta}^{(m)} \sim q( \bm{\theta}) \nonumber \\
&\approx& \frac{1}{M}\sum_{m=1}^M  \prod_{i=1}^N \int \text{Prob}(\y_i | \p_i)  \text{Dir}(\p_i|\bm{\alpha}_i^{(m)}) d \p_i,\quad  \bm{\alpha}^{(m)} = f(A, \rr, \bm{\theta}^{(m)}),q\quad \bm{\theta}^{(m)} \sim q( \bm{\theta}), \nonumber 
\end{eqnarray}
where the posterior over class label $p$ will be given by the mean of the Dirichlet:
\begin{eqnarray}
\text{Prob}(y_i = p | \bm{\theta}^{(m)}) = \int \text{Prob}(y_i =p | \p_i) \text{Prob}(\p_i | A, \rr;\bm{\theta}^{(m)} ) d \p_i = \frac{\alpha_{ip}^{(m)}}{\sum_{k=1}^K \alpha_{ik}^{(m)}}. \nonumber
\end{eqnarray}
The probabilistic form for a specific node $i$ by using marginal probability,
\begin{eqnarray}
\text{Prob}(\y_i | A, \rr; \mathcal{G}) &=& \sum_{y\setminus y_i} \text{Prob}(\y | A, \rr; \mathcal{G})  \nonumber \\
&=& \sum_{y\setminus y_i} \int \int \prod_{j=1}^N\text{Prob}(\y_j | \p_j) \text{Prob}(\p_j | A, \rr; \bm{\theta} ) \text{Prob}(\bm{\theta} | \mathcal{G}) d \p d\bm{\theta} \nonumber \\
&\approx& \sum_{y\setminus y_i} \int \int \prod_{j=1}^N\text{Prob}(\y_j | \p_j) \text{Prob}(\p_j | A, \rr; \bm{\theta} ) q(\bm{\theta})d \p d\bm{\theta} \nonumber \\
&\approx&\sum_{m=1}^M \sum_{y\setminus y_i} \int \prod_{j=1}^N\text{Prob}(\y_j | \p_j) \text{Prob}(\p_j | A, \rr; \bm{\theta}^{(m)} )  d \p,\quad \bm{\theta}^{(m)} \sim q( \bm{\theta})  \nonumber \\
&\approx& \sum_{m=1}^M \Big[\sum_{y\setminus y_i} \int \prod_{j=1}^N\text{Prob}(\y_j | \p_j) \text{Prob}(\p_j | A, \rr; \bm{\theta}^{(m)} )  d \p_j\Big], \quad \bm{\theta}^{(m)} \sim q( \bm{\theta})  \nonumber \\
&\approx& \sum_{m=1}^M  \Big[\sum_{y\setminus y_i}  \prod_{j=1, j\neq i}^N\text{Prob}(\y_j | A, \rr_j; \bm{\theta}^{(m)})   \Big] \text{Prob}(\y_i | A, \rr; \bm{\theta}^{(m)} ),\quad \bm{\theta}^{(m)} \sim q( \bm{\theta})  \nonumber \\
&\approx& \sum_{m=1}^M \int \text{Prob}(\y_i | \p_i) \text{Prob}(\p_i |A, \rr; \bm{\theta}^{(m)} ) d \p_i , \quad \bm{\theta}^{(m)} \sim q( \bm{\theta}).  \nonumber
\end{eqnarray}
To be specific, the probability of label $p$ is,
\begin{eqnarray}
\text{Prob}(y_i=p | A, \rr; \mathcal{G}) \approx \frac{1}{M} \sum_{m=1}^M \frac{\alpha_{ip}^{(m)}}{\sum_{k=1}^K \alpha_{ik}^{(m)}}, \quad  \bm{\alpha}^{(m)} = f(A, \rr, \bm{\theta}^{(m)}),\quad \bm{\theta}^{(m)} \sim q( \bm{\theta}). \nonumber
\end{eqnarray}
\subsection{KL-Divergence }
KL-divergence between $\text{Prob}({\bf y} | {\bf r}; \bm{\gamma},  \mathcal{G})$ and  $\text{Prob}({\bf y} | \hat{\p})$ is given by 
\begin{eqnarray}
\text{KL}[\text{Prob}(\y | A, \rr;\mathcal{G})|| \text{Prob}({\bf y} | \hat{\p}))]
&=& \mathbb{E}_{\text{Prob}(\y | A, \rr;\mathcal{G})}\Big[\log \frac{\text{Prob}(\y | A, \rr;\mathcal{G})} {\text{Prob}({\bf y} | \hat{\p})} \Big] \nonumber \\
&\approx&\mathbb{E}_{\text{Prob}(\y | A, \rr;\mathcal{G})} \Big[\log \frac{\prod_{i=1}^N  \text{Prob}(\y_i | A, \rr; \mathcal{G})}{\prod_{i=1}^N  \text{Prob}({\bf y} | \hat{\p})} \Big] \nonumber \\
&\approx&  \mathbb{E}_{\text{Prob}(\y | A, \rr;\mathcal{G})} \Big[\sum_{i=1}^N \log\frac{\text{Prob}(\y_i | A, \rr; \mathcal{G})}{\text{Prob}({\bf y} | \hat{\p})} \Big] \nonumber\\
&\approx& \sum_{i=1}^N \mathbb{E}_{\text{Prob}(\y | A, \rr;\mathcal{G})} \Big[ \log\frac{\text{Prob}(\y_i | A, \rr; \mathcal{G})}{\text{Prob}({\bf y} | \hat{\p})} \Big] \nonumber\\
&\approx& \sum_{i=1}^N \sum_{j=1}^K \text{Prob}(y_i=j | A, \rr; \mathcal{G}) \Big( \log\frac{\text{Prob}(y_i=j | A, \rr; \mathcal{G})}{\text{Prob}(y_i=j| \hat{\p})} \Big) \nonumber
\end{eqnarray}

The KL divergence between two Dirichlet distributions $\text{Dir}(\alpha)$ and $\text{Dir}(\hat{\alpha})$ can be obtained in closed form as,
\begin{eqnarray}
&&\text{KL}[\text{Dir}(\alpha)\| \text{Dir}(\hat{\alpha})] \nonumber \\ &&= \ln \Gamma(S) - \ln \Gamma(\hat{S}) + \sum_{c=1}^K \big(\ln \Gamma(\hat{\alpha}_c) - \ln \Gamma(\alpha_c) \big)   + \sum_{c=1}^K (\alpha_c - \hat{\alpha}_c)(\psi(\alpha_c) - \psi(S)),   \nonumber
\end{eqnarray}
where $S = \sum_{c=1}^K \alpha_c$ and $\hat{S} = \sum_{c=1}^K \hat{\alpha}_c$.

\section{Proof of Theorem~\ref{theorem_rssl_conv}}\label{app:A3}
For our robust SSL, the training loss $\mathcal{L}_T$ can be decomposition into supervised loss $\mathcal{L}_{L}$ and unsupervised loss $\mathcal{L}_{U}$, i.e., $\mathcal{L}_{T}(\theta) = \mathcal{L}_{L}(\theta) + \w \mathcal{L}_{U}(\theta)$, where $\mathcal{L}_{L}(\theta) = \sum\nolimits_{(\x_i, y_i)\in \mathcal{D}} l(f(\x_i, \theta), y_i) $ and $\mathcal{L}_{U}(\theta) = \sum\nolimits_{x_j\in \mathcal{U}} r(f(\x_j, \theta))$.
\begin{lemma}
Suppose the validation loss function $\mathcal{L}_V$ is Lipschitz smooth with constant L, and the unsupervised loss function $\mathcal{L}_U$ have $\gamma$-bounded gradients. Then the gradient of the validation loss function with respect to $\w$ is Lipschitz continuous.
\label{lemma2}
\end{lemma}

\begin{proof}
For the meta approximation method, \feng{the gradient of the validation loss function with respect to $\w$} can be written as:
\begin{eqnarray}
    \nabla_\w \partial \mathcal{L}_V(\theta^*(\w)) &=& \frac{\partial \mathcal{L}_V(\theta^*(\w))}{\partial \theta^*} \cdot \frac{\theta^*(\w)}{\partial \w} \nonumber \\
    &=& \frac{\partial \mathcal{L}_V(\theta^*(\w))}{\partial \theta^*} \cdot \frac{\theta-\alpha \frac{\partial \mathcal{L}_T(\theta)}{\partial \theta}}{\partial \w} \nonumber \\
     &=& \frac{\partial \mathcal{L}_V(\theta^*(\w))}{\partial \theta^*} \cdot \frac{\theta-\alpha \frac{\partial (\mathcal{L}_L(\theta)+\w \mathcal{L}_U(\theta))}{\partial \theta}}{\partial \w} \nonumber \\
     &=& -\alpha \frac{\partial \mathcal{L}_V(\theta^*(\w))}{\partial \theta^*} \cdot \frac{\partial\mathcal{L}_U(\theta)}{\partial \theta}\feng{.}
     \label{w_grad}
\end{eqnarray}
Taking gradient\feng{s} \feng{on} both sides of Eq~\eqref{w_grad} ]\feng{with respect to {\bf w}}, we have
\begin{eqnarray}
     \|\nabla^2_\w \partial \mathcal{L}_V(\theta^*(\w)) \|
     &=& -\alpha \Big\|\frac{\partial}{\partial \w} \Big( \frac{\partial \mathcal{L}_V(\theta^*(\w))}{\partial \theta^*} \cdot \frac{\partial\mathcal{L}_U(\theta)}{\partial \theta} \Big)\Big\| \nonumber \\
     &=& -\alpha \Big\|\frac{\partial}{\partial \theta^*} \Big( \frac{\partial \mathcal{L}_V(\theta^*(\w))}{\partial \w}\Big) \cdot \frac{\partial\mathcal{L}_U(\theta)}{\partial \theta}\Big\|  \nonumber \\
      &=& -\alpha \Big\|\frac{\partial}{\partial \theta} \Big( -\alpha \frac{\partial \mathcal{L}_V(\theta^*(\w))}{\partial \theta^*} \cdot \frac{\partial\mathcal{L}_U(\theta)}{\partial \theta}\Big) \cdot \frac{\partial\mathcal{L}_U(\theta)}{\partial \theta}\Big\|  \nonumber \\
       &=& \alpha^2 \Big\| \frac{\partial^2 \mathcal{L}_V(\theta^*(\w))}{\partial \theta^* \partial \theta^*} \frac{\partial\mathcal{L}_U(\theta)}{\partial \theta} \cdot \frac{\mathcal{L}_U(\theta)}{\partial \theta}\Big\|  \nonumber \\
       &\le & \alpha^2 \gamma^2 L \feng{,}
\end{eqnarray}
since $\|\frac{\partial^2 \mathcal{L}_V(\theta^*(\w))}{\partial \theta^* \partial \theta^*}\|\le L, \|\frac{\mathcal{L}_U(\theta)}{\partial \theta}\| \le \gamma$. \feng{Let} $\tilde{L} = \alpha^2 \gamma^2 L $\feng{.} \feng{B}ased on \feng{the} Lagrange mean value theorem, we have,
\begin{eqnarray}
    \|\nabla \mathcal{L}_V(\theta^*(\w_i)) - \nabla \mathcal{L}_V(\theta^*(\w_j)) \| \le \tilde{L} \|\w_i-\w_j\|, \quad \text{for all } \w_i, \w_j\feng{,}
\end{eqnarray}
where $\nabla \mathcal{L}_V(\theta^*(\w_i)) = \nabla \mathcal{L}_V(\theta^*(\w))|_{\w_i}$
\end{proof}
Based on Lemma~\ref{lemma2}, now we start to proof of theorem~\ref{theorem_rssl_conv}.

\begin{proof}
First, according to the updating rule, we have:
\begin{eqnarray}
    \mathcal{L}_V(\theta_{t+1}) - \mathcal{L}_V(\theta_{t}) &=& 
    \mathcal{L}_V(\theta_{t}-\alpha \nabla_\theta \mathcal{L}_T(\theta_t, \w_t)) - \mathcal{L}_V(\theta_{t-1}-\alpha \nabla_\theta \mathcal{L}_T(\theta_{t-1}, \w_{t-1}))   \nonumber \\
    &=& \underbrace{ \mathcal{L}_V(\theta_{t}-\alpha \nabla_\theta \mathcal{L}_T(\theta_t, \w_t)) - \mathcal{L}_V(\theta_{t-1}-\alpha \nabla_\theta \mathcal{L}_T(\theta_{t}, \w_{t}))}_{(a)} + \nonumber \\
    && \underbrace{\mathcal{L}_V(\theta_{t-1}-\alpha \nabla_\theta \mathcal{L}_T(\theta_t, \w_t)) - \mathcal{L}_V(\theta_{t-1}-\alpha \nabla_\theta \mathcal{L}_T(\theta_{t-1}, \w_{t-1}))}_{(b)}
\end{eqnarray}
and for term (a), we have
\begin{eqnarray}
    &&\mathcal{L}_V(\theta_{t}-\alpha \nabla_\theta \mathcal{L}_T(\theta_t, \w_t)) - \mathcal{L}_V(\theta_{t-1}-\alpha \nabla_\theta \mathcal{L}_T(\theta_{t}, \w_{t})) \nonumber \\
    &\le &  (\theta_t- \theta_{t-1})\cdot \nabla_\theta \mathcal{L}_V(\theta_{t-1}-\alpha \nabla_\theta \mathcal{L}_T(\theta_{t}, \w_{t})) +\frac{L}{2}\|\theta_t-\theta_{t-1} \|^2_2 \quad (\text{Lipschitzs smooth})  \nonumber \\
    &\le & \alpha\gamma^2+\frac{L}{2}\alpha^{2} \gamma^2 \nonumber \\
    &=&\alpha \gamma^2(\frac{\alpha L}{2} + 1) \feng{.}
\end{eqnarray}
For term (b), we can adapt a Lipschitz-continuous function to make $\mathcal{L}_V$ smooth w.r.t. $\w$. Then we have,
\begin{eqnarray}
    &&\mathcal{L}_V(\theta_{t-1}-\alpha \nabla_\theta \mathcal{L}_T(\theta_t, \w_t)) - \mathcal{L}_V(\theta_{t-1}-\alpha \nabla_\theta \mathcal{L}_T(\theta_{t-1}, \w_{t-1})) \nonumber \\
    &=&\mathcal{L}_V(\theta^*(\w_t)) - \mathcal{L}_V(\theta^*(\w_{t-1})) \nonumber \\
    &\le & (\w_t-\w_{t-1})\cdot \nabla_\w \mathcal{L}_V(\theta_t)+ \frac{\tilde{L}}{2}\|\w_t-\w_{t-1} \|^2_2 \qquad \blue{\text{(From Lemma~\ref{lemma2})}} \nonumber \\
    &=& -\beta \nabla_\w  \mathcal{L}_V(\theta_t)\cdot \nabla_\w \mathcal{L}_V(\theta_t) + \frac{\zxj{\tilde{L}}}{2}\|-\beta \nabla_\w  \mathcal{L}_V(\theta_t)\|^2_2 \nonumber \\
    &=& (\frac{\zxj{\tilde{L}}}{2}\beta^2-\beta) \| \nabla_\w  \mathcal{L}_V(\theta_t)\|^2_2.
\end{eqnarray}
Then we have,
\begin{eqnarray}
    \mathcal{L}_V(\theta_{t+1}) - \mathcal{L}_V(\theta_{t}) \le \alpha \gamma^2(\frac{\zxj{\alpha} L}{2} + 1) + (\frac{\zxj{\tilde{L}}}{2}\beta^2-\beta) \| \nabla_\w  \mathcal{L}_V(\theta_t)\|^2_2
\end{eqnarray}
Summing up the above inequalities and rearranging the terms, we can obtain
\begin{eqnarray}
    \sum_{t=1}^T (\beta -\frac{\zxj{\tilde{L}}}{2}\beta^2) \| \nabla_\w  \mathcal{L}_V(\theta_t)\|^2_2 &\le& 
    \mathcal{L}_V(\theta_1) - \mathcal{L}_V(\theta_{T+1}) + \alpha \gamma^2(\frac{\zxj{\alpha} LT}{2}-T) \nonumber \\
    &\le& \mathcal{L}_V(\theta_1) + \alpha \gamma^2(\frac{\zxj{\alpha} LT}{2} + T)
\end{eqnarray}
Furthermore, we can deduce that,
\begin{eqnarray}
    \min_t \mathbb{E}\big[ \| \nabla_\w  \mathcal{L}_V(\theta_t)\|^2_2 \big] &\le&  \frac{ \sum_{t=1}^T (\beta -\frac{\zxj{\tilde{L}}}{2}\beta^2) \| \nabla_\w  \mathcal{L}_V(\theta_t)\|^2_2 }{ \sum_{t=1}^T (\beta -\frac{L}{2}\beta^2) } \nonumber \\
    &\le&  \frac{ \sum_{t=1}^T (\beta -\frac{\zxj{\tilde{L}}}{2}\beta^2) \| \nabla_\w  \mathcal{L}_V(\theta_t)\|^2_2 }{ \sum_{t=1}^T (\beta -\frac{L}{2}\beta^2) } \nonumber \\
     &\le&  \frac{1}{ \sum_{t=1}^T (2\beta -\zxj{\tilde{L}}\beta^2) }\Big[  2\mathcal{L}_V(\theta_1) + \alpha \gamma^2(\zxj{\alpha} LT+2T) \Big] \nonumber \\
      &\le&  \frac{1}{ \sum_{t=1}^T \beta }\Big[  2\mathcal{L}_V(\theta_1) + \alpha \gamma^2( LT+2T) \Big] \nonumber \\
      &=& \frac{ 2\mathcal{L}_V(\theta_1)}{T}\frac{1}{\beta} + \frac{\alpha \gamma^2( L+ 2)}{\beta} \nonumber \\ 
      &=& \frac{ 2\mathcal{L}_V(\theta_1)}{T}\max\{L, \frac{\sqrt{T}}{C}\} + \min\{1, \frac{k}{T}\} \max\{L, \frac{\sqrt{T}}{C}\}\gamma^2 ( L+2)\nonumber \\
      &=& \frac{ 2\mathcal{L}_V(\theta_1)}{C\sqrt{T}} + \frac{k\gamma^2 ( L+2)}{C\sqrt{T}} = \mathcal{O}(\frac{1}{\sqrt{T}})
\end{eqnarray}
The third inequality holds for $ \sum_{t=1}^T \beta \le \sum_{t=1}^T (2\beta -\zxj{\tilde{L}}\beta^2)$. Therefore, we can conclude that our algorithm can always achieve $\min_{0\le t \le T} \mathbb{E}[\|\nabla_\w \mathcal{L}_V(\theta_t)  \|^2_2] \le \mathcal{O}(\frac{1}{\sqrt{T}}) $ in T steps.
\end{proof}

\section{Proof of Proposition~\ref{IFT_connection}}\label{app:A4}
\begin{proof}
The bi-level optimization problem for weight hyperparameters $\w$ using a model characterized by parameters $\theta$ is as follows:
\begin{align}\label{bilevel}
    \w^{*} = \underset{\w}{\operatorname{argmin\hspace{0.7mm}}}\mathcal{L}_{V}(\w, \theta^{*}(\w)) \text{\hspace{1mm}where} \\
    \theta^{*}(\w) = \underset{\theta}{\operatorname{argmin\hspace{0.7mm}}} \mathcal{L}_{T}(\w, \theta)
\end{align}

\begin{case} \textbf{IFT Approach:}
In order to optimize $\w$ using gradient descent, we need to calculate the weight gradient $\frac{\partial \mathcal{L}_{V}(\theta^*(\w), \w)}{ \partial \w}$. Using chain rule in Eq.~\eqref{bilevel}, we have:
\begin{eqnarray}
    \frac{\partial \mathcal{L}_{V} (\theta^*(\w), \w)}{ \partial \w}
    =\underbrace{\frac{\partial \mathcal{L}_{V} }{ \partial \w}}_{(a)} + \underbrace{\frac{\partial \mathcal{L}_{V} }{ \partial \theta^*(\w)}}_{(b)} \times \underbrace{\frac{\partial \theta^*(\w)}{\partial \w}}_{(c)} \label{ID_gradient2}
\end{eqnarray}
where (a) is the direct weight gradient and (b) is the direct parameter gradient, which is easy to compute. The tricky part is the term (c) (best-response Jacobian). 

In the IFT approach, we approximate (c) using the Implicit function theorem,
\begin{eqnarray}
    \frac{\partial \theta^* (\w)}{\w} = - \underbrace{\Big[ \frac{\partial \mathcal{L}_{T}}{\partial \theta\partial \theta^T} \Big]^{-1}}_{(d)} \times  \underbrace{\frac{\partial \mathcal{L}_{T}}{\partial \w\partial \theta^T}}_{(e)} \label{IFT2}
\end{eqnarray}

However, computing Eq.~\eqref{IFT2} is challenging when using deep nets because it requires inverting a high dimensional Hessian (term (d)), which often requires $\mathcal{O}(m^3)$ operations. Therefore, the IFT approach ~\cite{lorraine2020optimizing} uses the Neumann series approximation to effectively compute the Hessian inverse term(d), which is as follows,

\begin{align}\label{inverseapprox}
    \Big[ \frac{\partial \mathcal{L}_{T}}{\partial \theta\partial \theta^T} \Big]^{-1} =  \lim_{P \rightarrow{}\infty} \sum_{p=0}^P \Big[I- \frac{\partial \mathcal{L}_{T}}{\partial \theta\partial \theta^T} \Big]^p
\end{align}
where $I$ is the identity matrix.

Assuming $P=0$ in Eq.~\ref{inverseapprox}, we have $\Big[ \frac{\partial \mathcal{L}_{T}}{\partial \theta\partial \theta^T} \Big]^{-1} =  \mathbf{I}$ and substituting it in Eq.~\eqref{IFT2}, we have:
\begin{eqnarray}
    \frac{\partial \theta^* (\w)}{\w} = - \frac{\partial \mathcal{L}_{T}}{\partial \w\partial \theta^T}
\end{eqnarray}

Substituting the above equation in Eq.~\eqref{ID_gradient2}, we have:
\begin{eqnarray}
    \frac{\partial \mathcal{L}_{V} (\theta^*(\w), \w)}{ \partial \w}
    =\frac{\partial \mathcal{L}_{V} }{ \partial \w} - \frac{\partial \mathcal{L}_{V}(\theta^*(\w)) }{ \partial \theta^*(\w)} \times \frac{\partial^2 \mathcal{L}_{T}}{\partial \w\partial \theta^T}
\end{eqnarray}

Since we are using unweighted validation loss $\mathcal{L}_V$, there is no dependence of validation loss on weights directly, i.e., $\frac{\partial \mathcal{L}_{V}}{\partial \w} = 0$. Hence, the weight gradient is as follows:
\begin{eqnarray}
    \frac{\partial \mathcal{L}_{V} (\theta^*(\w), \w)}{ \partial \w} = - \frac{\partial \mathcal{L}_{V}(\theta^*(\w)) }{ \partial^2 \theta^*(\w)} \times \frac{\partial \mathcal{L}_{T}}{\partial \w\partial \theta^T}
\end{eqnarray}

Since we are using one-step gradient approximation, we have $\theta^*(\w) = \theta - \alpha \frac{\partial \mathcal{L}_T(\w, \theta)}{\partial \theta}$ where $\alpha$ is the model parameters learning rate.

The weight update step is as follows:
\begin{eqnarray}
    \w^* = \w + \beta \frac{\partial \mathcal{L}_{V}(\theta^*(\w)) }{ \partial \theta^*(\w)} \times \frac{\partial^2 \mathcal{L}_{T}}{\partial \w\partial \theta^T}
\end{eqnarray}
where $\beta$ is the weight learning rate.
\end{case}

\begin{case} \textbf{Meta-approximation Approach} In Meta-approximation approach, we have $\theta^*(\w) = \theta - \alpha \frac{\partial \mathcal{L}_T(\w, \theta)}{\partial \theta}$ where $\alpha$ is the model parameters learning rate.

Using the value of $\theta^{*}$, the gradient of validation loss with weight hyperparameters is as follows:

\begin{align}
    \frac{\partial \mathcal{L}_{V} (\theta^*(\w))}{ \partial \w} &= \frac{\partial \mathcal{L}_{V}(\theta - \alpha \frac{\partial \mathcal{L}_T(\w, \theta)}{\partial \theta})}{\partial \w} \\
    &= - \frac{\partial \mathcal{L}_{V}(\theta^*(\w))}{\partial \theta^*(\w)} \times \alpha \frac{\partial^2 \mathcal{L}_T(\w, \theta)}{\partial \theta \partial \w^T}
\end{align}

Assuming $\alpha = 1$, we have:
\begin{align}
    \frac{\partial \mathcal{L}_{V} (\theta^*(\w))}{ \partial \w} = - \frac{\partial \mathcal{L}_{V}(\theta^*(\w))}{\partial \theta^*(\w)} \times \frac{\partial^2 \mathcal{L}_T(\w, \theta)}{\partial \theta \partial \w^T}
\end{align}

Hence, the weight update step is as follows:
\begin{eqnarray}
    \w^* = \w + \beta \frac{\partial \mathcal{L}_{V}(\theta^*(\w)) }{ \partial \theta^*(\w)} \times \frac{\partial^2 \mathcal{L}_{T}}{\partial \w\partial \theta^T}
\end{eqnarray}
where $\beta$ is the weight learning rate
\end{case}
As shown in the cases above, we have a similar weight update. Hence, it inherently means that using the meta-approximation of $J=1$ is equivalent to using an identity matrix as the Hessian inverse of training loss with $\theta$ (i.e., $P=0$ for implicit differentiation approach).
\end{proof}

\section{Proof of Proposition~\ref{proposition_WBN}}\label{app:A5}
\begin{proof}
(1). The mini-batch mean of $\mathcal{I}$,
\begin{eqnarray}
    \mu_{\mathcal{I}} = \frac{1}{m}\sum_{i=1}^m \x_i = \mu_{I}\feng{.}
\end{eqnarray}
The mixed mini-batch mean of $\mathcal{IO}$,
\begin{eqnarray}
     \mu_{\mathcal{IO}} &=& \frac{1}{2m}(\sum_{i=1}^m \x_i +\sum_{i=1}^m \hat{\x}_i) \nonumber \\
     &=& \frac{1}{2}\mu_{\mathcal{I}} + \frac{1}{2}\mu_{\mathcal{O}}\feng{.}
\end{eqnarray}
\feng{T}hen we have,
\begin{eqnarray}
     \|\mu_{\mathcal{IO}}- \mu_{\mathcal{I}}\|_2 &=& \|\frac{1}{2}\mu_{\mathcal{I}} + \frac{1}{2}\mu_{\mathcal{O}} -\mu_{\mathcal{I}}\|_2 \nonumber\\
     &=& \| \frac{1}{2}\mu_{\mathcal{O}} -\mu_{\mathcal{I}}\|_2 >\frac{L}{2} \gg 0 \feng{.}
\end{eqnarray}
The mini-batch variance of $\mathcal{I}$,
\begin{eqnarray}
     \sigma^2_{\mathcal{I}} = \frac{1}{m}\sum_{i=1}^m (\x_i - \mu_{\mathcal{I}})^2 \feng{.}
\end{eqnarray}
The traditional batch normalizing transform based on mini-batch $\mathcal{I}$ for $\x_i$
\begin{eqnarray}
    BN_{\mathcal{I}}(\x_i) = \gamma \frac{\x_i-\mu_\mathcal{I}}{\sqrt{\sigma^2_{\mathcal{I}}+\epsilon}} +\beta
    \label{BN_I}\feng{.}
\end{eqnarray}
The mini-batch variance of $\mathcal{IO}$,
\begin{eqnarray}
     \sigma^2_{\mathcal{IO}}
     &=& \frac{1}{2m}(\sum_{i=1}^m \x_i^2+ \sum_{i=1}^m \hat{\x}_i^2) -\mu_{\mathcal{IO}}^2 \nonumber \\
     &=& \frac{1}{2}\sigma^2_{\mathcal{I}} + \frac{1}{2}\sigma^2_{\mathcal{O}}+ \frac{1}{4}(\mu_{\mathcal{O}}-\mu_{\mathcal{I}})^2 \nonumber \\
     &\approx& \frac{1}{4}(\mu_{\mathcal{O}}-\mu_{\mathcal{I}})^2
     \label{var_IO}\feng{,}
\end{eqnarray}
and the traditional batch normalizing transform based on mini-batch $\mathcal{IO}$ for $\x_i$,
\begin{eqnarray}
     BN_{\mathcal{IO}}(\x_i) &=& \gamma \frac{\x_i -\mu_{\mathcal{B}}(IO)}{\sqrt{\sigma^2_{\mathcal{B}}(IO)+\epsilon}} +\beta \nonumber \\
     &=& \gamma\Big( \frac{\x_i -\mu_O}{2\sqrt{\sigma^2_{\mathcal{B}}(IO)+\epsilon}} + \frac{\x_i -\mu_I}{2\sqrt{\sigma^2_{\mathcal{B}}(IO)+\epsilon}}\Big) + \beta \nonumber \\
     &\approx& \gamma\Big( \frac{\x_i -\mu_O}{\|\mu_O-\mu_I \|_2} + \frac{\x_i -\mu_I}{\|\mu_O-\mu_I \|_2}\Big) + \beta \nonumber \\
     &\approx& \gamma \frac{\x_i -\mu_O}{\|\mu_O-\mu_I \|_2} + \beta
      \label{BN_IO} \feng{.}
\end{eqnarray}
The approximation\feng{s} in Eq.~\eqref{var_IO} and Eq.~\eqref{BN_IO} hold when $\sigma^2_{\mathcal{I}}$ \feng{and} $\sigma^2_{\mathcal{O}}$ have \feng{the} same magnitude level\feng{s} as $\mu_I$, i.e., $\|\mu_{\mathcal{O}} - \sigma^2_{\mathcal{I}}\|_2>L$ and $\|\mu_{\mathcal{O}} - \sigma^2_{\mathcal{O}}\|_2 > L$. Compar\feng{ing} Eq.~\eqref{BN_I} \feng{with} Eq.~\eqref{BN_IO}, we prove that $BN_{\mathcal{I}}(\x_i) \neq BN_{\mathcal{IO}}(\x_i)$. 
 
(2) The weighted mini-batch mean of $\mathcal{IO}$,
\begin{eqnarray}
     \mu^\w_{\mathcal{IO}} &=& \frac{\sum_{i=1}^m w^i_{\mathcal{I}} \x_i +\sum_{i=1}^m w^i_{\mathcal{O}} \hat{\x}_i}{\sum_{i=1}^m w^i_{\mathcal{I}}+ \sum_{i=1}^m w^i_{\mathcal{O}}}  \nonumber \\
     &=& \frac{\sum_{i=1}^m 1\cdot \x_i +\sum_{i=1}^m 0\cdot \hat{\x}_i}{\sum_{i=1}^m 1+ \sum_{i=1}^m 0}  \nonumber \\
     &=& \mu_{I} \feng{,} \nonumber
\end{eqnarray}
\feng{witch} prove\feng{s} that $\mu_{\mathcal{I}} = \mu^{\w}_{\mathcal{IO}}$. The weighted mini-batch variance of $\mathcal{IO}$,
\begin{eqnarray}
     {\sigma^{\w}_{\mathcal{IO}}}^2 &=& \frac{\sum_{i=1}^m w^i_{\mathcal{I}} (\x_i-\mu^{\w}_{\mathcal{IO}})^2}{\sum_{i=1}^m w^i_{\mathcal{I}}+ \sum_{i=1}^m w^i_{\mathcal{O}}}  
     + \frac{\sum_{i=1}^m w^i_{\mathcal{O}} (\hat{\x}_i--\mu^{\w}_{\mathcal{IO}})^2}{\sum_{i=1}^m w^i_{\mathcal{I}}+ \sum_{i=1}^m w^i_{\mathcal{O}}}  \nonumber \\
     &=& \frac{\sum_{i=1}^m 1\cdot (\x_i-\mu_{\mathcal{IO}})^2 +{\bf 0}}{\sum_{i=1}^m 1+ \sum_{i=1}^m 0}  \nonumber \\
     &=& \frac{\sum_{i=1}^m 1\cdot (\x_i-\mu_{\mathcal{IO}})^2 }{m}  \nonumber \\
     &=& \sigma^2_{\mathcal{I}}\feng{.}  \nonumber
\end{eqnarray}
The weighted batch normalizing transform based on mini-batch $\mathcal{IO}$ for $\x_i$,
\begin{eqnarray}
    WBN_{\mathcal{IO}}(\x_i, \w) = \gamma \frac{\x_i-\mu^{\w}_\mathcal{IO}}{\sqrt{{\sigma^{\w}_{\mathcal{IO}}}^2+\epsilon}} +\beta =  \gamma \frac{\x_i-\mu_\mathcal{I}}{\sqrt{\sigma^2_{\mathcal{I}}+\epsilon}} +\beta \feng{,}
\end{eqnarray}
which proves that $BN_{\mathcal{I}}(\x_i) = WBN_{\mathcal{IO}}(\x_i, \w)$.
\end{proof}

\begin{thesisbib}  
    \bibliography{reference}
\end{thesisbib}  

\begin{biosketch}
    Xujiang Zhao is currently a last-year PhD student in Computer Science at The University of Texas at Dallas. He is working under the supervision of Prof. Feng Xujiang. Before joining the Pattern Discovery and Machine Learning Laboratory, he received an MS in computer science from the University of Science and Technology of China (USTC) and a bachelor's degree at Chongqing University in China. His research interests include machine learning and data mining, particularly in uncertainty quantification, reinforcement learning, and semi-supervised learning. Xujiang published his work in top-tier machine learning and data mining conferences, including NeurIPS, AAAI, ICDM, EMNLP, etc. So far, Xujiang has 14 accepted publications, and 7 of them are first-author papers. In addition, he has done two internships at Alibaba Damo Academy in the summer of 2019 and NEC Laboratories America in the summer of 2021.

Besides, Xujiang is passionate about professional services. He was invited to serve on the program committee at the ICML, NeurIPS, ICLR, KDD, AAAI, WSDM, and SDM for the most recent two consecutive years. 

In the last semester of his doctoral program, Xujiang was actively looking for full-time jobs. Before graduation, he received 8 onsite interviews and 5 of them extended employment offers, including \textit{NEC Laboratories America, Inc.}, \textit{Amazon}, \textit{Black Sesame Technologies}, \textit{etc.} Xujiang eventually decided to choose \textit{NEC Laboratories America, Inc. }, which is the US-based Center for NEC Corporation's global network of corporate research laboratories. Xujiang was hired as a researcher and will continue his machine learning and data mining research.
\end{biosketch}

\begin{vita}  
  \begin{center}
    {\LARGE\bfseries Xujiang Zhao} \\[5pt]
    Jun 06, 2022
\end{center}


{\large\bfseries Contact Information:\par}
\medskip
\noindent\vtop{\hsize=.49\hsize
Department of Computer Science\par
The University of Texas at Dallas\par
800 W.~Campbell Rd.\par
Richardson, TX 75080-3021, U.S.A.\par}
\hfil\vtop{\hsize=.49\hsize
Email: \texttt{xujiang.zhao@utdallas.edu}\par}\par

\bigskip

{\large\bfseries Educational History:\par}
\begin{itemize}
  \item PhD, Computer Science, The University of Texas at Dallas, 2022
  \item MS, Computer Science, University of Science and Technology of China, 2017
  \item BS, Civil Engineering, Chongqing University, 2014
\end{itemize}

{\large\bfseries Working Experience:\par}
\begin{itemize}
  \item Researcher, NEC Laboratories America, Inc., August 2022 -- Present    
  \item Research Intern, NEC Laboratories America, Inc., May 2021 -- August 2021
  \item Research Intern,  Alibaba Damo Academy, May 2019 -- August 2019
  \item Research Assistant, The University of Texas at Dallas, September 2019 -- August 2022
  \item Research Assistant, The Research Foundation for SUNY, January 2018 -- July 2019
\end{itemize}


{\large\bfseries Awards and Scholarships:\par}
\begin{itemize}
   \item Conference on Neural Information Processing Systems (\textit{NeurIPS}) -- Student Travel Award, 2020
  \item IEEE International Conference on Data Mining (\textit{ICDM}) -- Student Travel Award, 2019
  \item Outstanding Graduate Award of Chongqing University, 2014
  \item First-class College Scholarship, 2014
  \item National Scholarship, 2013
  \item First Prize in The National Drawing Skills and Advanced Technology, 2012
\end{itemize}

{\large\bfseries Professional Services:\par}
\begin{itemize}
    \item Program Committee Member for the Conference on Neural Information Processing Systems (\textit{NeurIPS}), 2022
    \item Program Committee Member for the International Conference on Machine Learning (\textit{ICML}), 2022
    \item Program Committee Member for the ACM Conference on Knowledge Discovery and Data Mining (\textit{ACM SIGKDD}), 2022
    \item Program Committee Member for the SIAM International Conference on Data Mining (\textit{SDM}), 2022
    \item Program Committee Member for the ACM International Conference on Web Search and Data Mining (\textit{WSDM}), 2022
    \item Program Committee Member for the International Conference on Learning Representations (\textit{ICLR}), 2022
    \item Program Committee Member for the AAAI Conference on Artificial Intelligence (\textit{AAAI}), 2022
    \item Program Committee Member for the Conference on Neural Information Processing Systems (\textit{NeurIPS}), 2021
    \item Program Committee Member for the ACM Conference on Knowledge Discovery and Data Mining (\textit{ACM SIGKDD}), 2021
    \item Program Committee Member for the ACM Conference on Knowledge Discovery and Data Mining (\textit{ACM SIGKDD}), 2020
\end{itemize}

{\large\bfseries Presentations:\par}
\begin{itemize}
    \item IEEE International Conference on Acoustics, Speech and Signal Processing (Oral Presentation), Virtual, April 2022
    \item Neural Information Processing Systems (Spotlight Presentation), Virtual, December 2020
    \item IEEE International Conference on Big Data (Oral Presentation), Seattle WA, November 2018
    \item Institute of Information Engineering, Chinese Academy of Sciences (Keynote)
, Beijing China, November 2018
    \item IEEE International Conference on Data Mining (Oral Presentation), Singapore, November 2018
    \item IEEE Military Communications Conference (Oral Presentation), Los Angeles CA , October 2018
\end{itemize}

{\large\bfseries Publications:\par}
\begin{enumerate}
    \item \textbf{Xujiang Zhao}, Xuchao Zhang, Wei Cheng, Wenchao Yu, Yuncong Chen, Haifeng Chen, Feng Chen. “SEED: Sound Event Early Detection via Evidential Uncertainty”. IEEE International Conference on Acoustics, Speech and Signal Processing (ICASSP 2022).
\item  Haoliang Wang, Chen Zhao,\textbf{ Xujiang Zhao}, Feng Chen. “Layer Adaptive Deep Neural Networks for Out-of-distribution Detection”. Pacific-Asia Conference on Knowledge Discovery and Data Mining (PAKDD 2022).
\item  Krishnateja Killamsetty, \textbf{Xujiang Zhao}, Feng Chen, Rishabh Iyer. “RETRIEVE: Corset Selection for Efficient and Robust Semi-Supervised Learning”. Advances in neural information processing systems  (NeurIPS 2021).
\item  Liyan Xu, Xuchao Zhang, \textbf{Xujiang Zhao}, Haifeng Chen, Feng Chen, Jinho D. Choi. “Boosting Cross-Lingual Transfer via Self-Learning with Uncertainty Estimation”. 2021 Conference on Empirical Methods in Natural Language Processing (EMNLP 2021), Short Paper.
\item  Zhuoyi Wang, Chen Zhao, Yuqiao Chen, Hemeng Tao, Yu Lin, \textbf{Xujiang Zhao}, Yigong Wang and Latifur Khan. “CLEAR: Contrastive-Prototype Learning with Drift Estimation for Resource Constrained Stream Mining.” In Proceeding of TheWebConf 2021 (WWW 2021).
\item  Yibo Hu, Yuzhe Ou, \textbf{Xujiang Zhao}, Feng Chen. “Multidimensional Uncertainty-Aware Evidential Neural Networks.” In Proceeding of the Thirty-fifth AAAI Conference on Artificial Intelligence (AAAI 2021).
\item  \textbf{Xujiang Zhao}, Krishnateja Killamsetty, Rishabh Iyer, Feng Chen. “Robust Semi-Supervised Learning with Out of Distribution Data.” arXiv preprint arXiv:2010.03658, 2020
\item  \textbf{Xujiang Zhao}, Feng Chen, Shu Hu, Jin-Hee Cho. “Uncertainty Aware Semi-Supervised Learning on Graph Data.” Advances in neural information processing systems  (NeurIPS 2020, Spotlight; Acceptance rate: 4\%). 
\item  Weishi Shi, \textbf{Xujiang Zhao}, Qi Yu, Feng Chen. “Multifaceted Uncertainty Estimation for Label-Efficient Deep Learning.” Advances in neural information processing systems (NeurIPS 2020). 
\item  Adil Alim, \textbf{Xujiang Zhao}, Jin-Hee Cho, Feng Chen. “Uncertainty-Aware Opinion Inference Under Adversarial Attacks.” In 2019 IEEE International Conference on Big Data (Big Data 2019), pp. 6-15. IEEE, 2019.
\item  \textbf{Xujiang Zhao}, Yuzhe. Ou, Lance. Kaplan, Feng. Chen, and Jin-Hee. Cho. “Quantifying Classification Uncertainty using Regularized Evidential Neural Networks.” accepted to AAAI 2019 Fall Symposium Series, Artificial Intelligence in Government and Public Sector.
\item  \textbf{Xujiang Zhao}, Shu Hu, Jin-Hee Cho, and Feng Chen. “Uncertainty-based Decision Making using Deep Reinforcement Learning.” In 2019 22nd International Conference on Information Fusion (FUSION 2019), pp. 1-8. IEEE, 2019.
\item  \textbf{Xujiang Zhao}, Feng Chen, and Jin-Hee Cho. “Deep Learning for Predicting Dynamic Uncertain Opinions in Network Data.” In 2018 IEEE International Conference on Big Data 
\item  \textbf{Xujiang Zhao}, Feng Chen, and Jin-Hee Cho. "Deep Learning based Scalable Inference of Uncertain Opinions." In 2018 IEEE International Conference on Data Mining (ICDM 2018), pp. 807-816. IEEE, 2018. (Full paper; Acceptance rate: 8.86\%)
\item  \textbf{Xujiang Zhao}, Feng Chen, and Jin-Hee Cho. “Uncertainty-Based Opinion Inference on Network Data Using Graph Convolutional Neural Networks.” In MILCOM 2018-2018 IEEE Military Communications Conference (MILCOM 2018), pp. 731-736. IEEE, 2018.
\end{enumerate}

\end{vita}  

\end{document}